\documentclass{article}

\usepackage[left=2.5cm,right=2.5cm,top=2.5cm,bottom=2.5cm]{geometry}

\usepackage{palatino}

\usepackage{graphicx} 
\usepackage{subfig}
\usepackage{algorithm}
\usepackage{algorithmic}
\usepackage{amssymb}
\usepackage{url}            
\usepackage{booktabs}       
\usepackage{amsfonts}       
\usepackage{nicefrac}       
\usepackage{microtype}      
\usepackage{cleveref}
\usepackage{xfrac}
\usepackage[table]{xcolor}

\usepackage{enumitem}

\usepackage{amsmath}

\usepackage{subfig}

\usepackage{geometry}

\usepackage{wrapfig}

\usepackage{amsmath, amssymb, amsthm, amsfonts}
\usepackage{graphicx}
\usepackage{caption}
\usepackage{color}
\usepackage{wrapfig, framed}
\usepackage{mathrsfs}

\newtheorem{lemma}{Lemma}
\newtheorem{theorem}{Theorem}
\newtheorem{corollary}{Corollary}

\newcommand{\Exp}{\mathbb{E}}

\newcommand{\Var}{\operatorname{Var}}
\newcommand{\Norm}{\mathcal{N}}
\newcommand{\Real}{\mathbb{R}}

\newcommand{\Cov}{\operatorname{Cov}}

\newcommand{\htheta}{\widehat{\theta}}
\newcommand{\btheta}{\bar{\theta}}

\newcommand{\sign}{\operatorname{sign}}

\newcommand{\argmin}{\mathop{\mathrm{argmin}}}
\newcommand{\argmax}{\mathop{\mathrm{argmax}}}

\usepackage{url}
\usepackage{xspace}
\usepackage{comment}
\usepackage{color}
\usepackage{afterpage}

\usepackage{bm}

\usepackage{cleveref}

\newcommand{\bDelta}{{\bar{\Delta}}}

\global\long\def\Pr{\mathbb{P}}

\begin{document}
%
\title{Statistical inference using SGD}

\author{
Tianyang Li\textsuperscript{1} \\
\texttt{lty@cs.utexas.edu}
    \and 
    Liu Liu\textsuperscript{1} \\
    \texttt{liuliu@utexas.edu}
		\and
    Anastasios Kyrillidis\textsuperscript{2} \\
    \texttt{anastasios.kyrillidis@ibm.com}
    \and
    Constantine Caramanis\textsuperscript{1} \\
    \texttt{constantine@utexas.edu}   
    \and
    {}\textsuperscript{1} The University of Texas at Austin \\
    {}\textsuperscript{2} IBM T.J. Watson Research Center, Yorktown Heights
}

\date{}

\maketitle

\begin{abstract}
We present a novel method for frequentist statistical inference in $M$-estimation problems, based on stochastic gradient descent (SGD) \emph{with a fixed step size}: we demonstrate that the average of such SGD sequences can be used for statistical inference, after proper scaling.
An intuitive analysis using the Ornstein-Uhlenbeck process suggests that such averages are asymptotically normal.
To show the merits of our scheme, we apply it to both synthetic and real data sets, and demonstrate that its accuracy is comparable to classical statistical methods, while requiring potentially far less computation.
\end{abstract}


\section{Introduction}

In $M$-estimation, the minimization of empirical risk functions (RFs) provides point estimates of the model parameters. 
Statistical inference then seeks to assess the quality of these estimates; \emph{e.g.}, by obtaining confidence intervals or solving hypothesis testing problems.  
Within this context, a classical result in statistics states that the asymptotic distribution of the empirical RF's minimizer is normal, centered around the population RF's minimizer \cite{van2000asymptotic}. 
Thus, given the mean and covariance of this normal distribution, we can infer a range of values, along with probabilities, that allows us to quantify the probability that this interval includes the true minimizer. 

The Bootstrap \cite{efron1982jackknife,efron1994introduction} is a classical tool for obtaining estimates of the mean and covariance of this distribution. 
The Bootstrap operates by generating samples from this distribution (usually, by re-sampling with or without replacement from the entire data set) and 
repeating the estimation procedure over these different re-samplings. 
As parameter dimensionality and data size grow, the Bootstrap becomes increasingly --even prohibitively-- expensive. 

In this context, we follow a different path:  we show that inference can also be accomplished by directly using stochastic gradient descent (SGD), both for point estimates and inference, {\em with a fixed step size over the data set}. 
It is well-established that fixed step-size SGD is by and large the dominant method used for large scale data analysis. 
We prove, and also demonstrate empirically, that \emph{the average of SGD sequences, obtained by minimizing RFs, can also be used for statistical inference.} 
Unlike the Bootstrap, our approach does not require creating many large-size subsamples from the data, neither re-running SGD from scratch for each of these subsamples.
Our method only uses first order information from gradient computations, and does not require any second order information. 
Both of these are important for large scale problems, where re-sampling many times, or computing Hessians, may be computationally prohibitive. 

\paragraph{Outline and main contributions:}
This paper studies and analyzes a simple, \emph{fixed step size}\footnote{{\em Fixed step size} means we use the same step size every iteration, but the step size is smaller with more total number of iterations. In contrast, {\em constant step size} means the step size is constant no matter how many iterations taken.}, SGD-based algorithm for inference in $M$-estimation problems. 
Our algorithm produces samples, whose covariance converges to the covariance of the $M$-estimate, without relying on bootstrap-based schemes, and also avoiding direct and costly computation of second order information. 
Much work has been done on the asymptotic normality of SGD, as well as on the Stochastic Gradient Langevin Dynamics (and variants) in the Bayesian setting. 
As we discuss in detail in \Cref{sec:related}, this is the first work to provide finite sample inference results, using fixed step size, and without imposing overly restrictive assumptions on the convergence of fixed step size SGD.

The remainder of the paper is organized as follows. 
In the next section, we define the inference problem for $M$-estimation, and recall basic results of asymptotic normality and how these are used. 
\Cref{sec:sgd} is the main body of the paper: we provide the algorithm for creating bootstrap-like samples, and also provide the main theorem of this work.  
As the details are involved, we provide an intuitive analysis of our algorithm and explanation of our main results, using an asymptotic Ornstein-Uhlenbeck process approximation for the SGD process \cite{kushner1981asymptotic, pflug1986stochastic, Benveniste:1990:AAS:95267, kushner2003stochastic, mandt2016variational},
and we postpone the full proof until the appendix. 
We specialize our main theorem to the case of linear regression (see supplementary material), and also that of logistic regression. 
For logistic regression in particular, we require a somewhat different approach, as the logistic regression objective is not strongly convex. 
In \Cref{sec:related}, we present related work and elaborate how this work differs from existing research in the literature. 
Finally, in  the experimental section, we provide parts of our numerical experiments that illustrate the behavior of our algorithm, and corroborate our theoretical findings. 
We do this using synthetic data for linear and logistic regression, and also by considering the Higgs detection \cite{baldi2014searching} and the LIBSVM Splice data sets. 
A considerably expanded set of empirical results is deferred to the appendix. 

Supporting our theoretical results, our empirical findings suggest that the SGD inference procedure produces results similar to bootstrap while using far fewer operations. 
Thereby, we  produce a more efficient inference procedure applicable in large scale settings, where other approaches fail.


\section{Statistical inference for $M$-estimators}
\label{sec:inference}

Consider the problem of estimating a set of parameters $\theta^{\star} \in \mathbb{R}^p$ using $n$ samples $\left\{X_i \right\}_{i = 1}^n$, drawn from some distribution $P$ on the sample space $\mathcal X$.
In frequentist inference, we are interested in estimating the minimizer $\theta^{\star}$ of the population risk:
\begin{small}
\begin{align}
\theta^{\star} = \argmin_{\theta \in \mathbb{R}^p} \Exp_P[f(\theta; X)] =  \argmin_{\theta \in \mathbb{R}^p} \int_x{f(\theta; x)} \,\mathrm{d}P(x),
\end{align}
\end{small}
where we assume that $f(\cdot ; x): \mathbb{R}^p \rightarrow \mathbb{R}$ is real-valued and convex; further, we will use $\Exp \equiv \Exp_P$, unless otherwise stated.
In practice, the distribution $P$ is unknown.
We thus estimate $\theta^{\star}$ by solving an empirical risk minimization (ERM) problem, where we use the estimate $\htheta$:
\begin{align}
\htheta = \argmin_{\theta \in \mathbb{R}^p} \frac{1}{n} \sum_{i=1}^n f(\theta; X_i). \label{eq:sgd_2}
\end{align}

Statistical inference consists of techniques for obtaining information beyond point estimates $\htheta$, such as confidence intervals. 
These can be performed if there is an asymptotic limiting distribution associated with $\htheta$ \cite{wasserman2013all}. 
Indeed, under standard and well-understood regularity conditions, the solution to $M$-estimation problems satisfies asymptotic normality. 
That is, the distribution $\sqrt{n} (\htheta - \theta^{\star})$ converges weakly to a normal distribution:
\begin{align}
\sqrt{n} (\htheta - \theta^{\star}) \longrightarrow \Norm(0, {H^{\star}}^{-1} G^{\star} {H^{\star}}^{-1} ), \label{eq:m-est:normal}
\end{align}
where  $$H^{\star} = \Exp[\nabla^2 f({\theta^{\star}}; X) ],$$ and $$G^{\star}= \Exp[\nabla f({\theta^{\star}}; X) \cdot \nabla f({\theta^{\star}}; X)^\top ];$$ see also Theorem 5.21 in \cite{van2000asymptotic}.
We can therefore use this result, as long as we have a good estimate of the covariance matrix: ${H^{\star}}^{-1} G^{\star} {H^{\star}}^{-1}$.
The central goal of this paper is obtaining accurate estimates for ${H^{\star}}^{-1} G^{\star} {H^{\star}}^{-1}$.

A naive way to estimate ${H^{\star}}^{-1} G^{\star} {H^{\star}}^{-1}$ is through the empirical estimator $\widehat{H}^{-1} \widehat{G} \widehat{H}^{-1}$ where:
\begin{align}
\widehat{H} &= \frac{1}{n} \sum_{i=1}^n \nabla^2 f({\htheta}; X_i) \quad \quad \text{and} \quad \quad \nonumber\\
\widehat{G} &= \frac{1}{n} \sum_{i=1}^n \nabla f({\htheta}; X_i) \nabla f({\htheta}; X_i)^\top.
\end{align}
Beyond calculating\footnote{In the case of maximum likelihood estimation, we have $H^{\star} = G^{\star}$---which is called Fisher information. Thus, the covariance of interest is ${H^{\star}}^{-1} = {G^{\star}}^{-1}$. This can be estimated either using $\widehat{H}$ or $\widehat{G}$. } $\widehat{H}$ and $\widehat{G}$, this computation requires an inversion of $\widehat{H}$ and matrix-matrix multiplications in order to compute $\widehat{H}^{-1} \widehat{G} \widehat{H}^{-1}$---a key computational bottleneck in high dimensions. 
Instead, our method uses SGD to directly estimate $\widehat{H}^{-1} \widehat{G} \widehat{H}^{-1}$. 


\section{Statistical inference using SGD}\label{sec:sgd}

Consider the optimization problem in \eqref{eq:sgd_2}.
For instance, in maximum likelihood estimation (MLE), $f(\theta; X_i)$ is a negative log-likelihood function.
For simplicity of notation, we use $f_i(\theta)$ and $f(\theta)$ for $f(\theta; X_i)$ and $\frac{1}{n} \sum_{i = 1}^n f(\theta; X_i)$, respectively, for the rest of the paper.

The SGD algorithm with a fixed step size $\eta$, is given by the iteration
\begin{align}
\theta_{t+1} = \theta_t - \eta  g_s (\theta_t) \label{eq:sgd_def},
\end{align}
where $g_s(\cdot)$ is an unbiased estimator of the gradient, \emph{i.e.},  $\Exp[g_s (\theta) \mid \theta ] = \nabla f(\theta)$, where the expectation is w.r.t. the stochasticity in the $g_s(\cdot)$ calculation.
A classical example of an unbiased estimator of the gradient is $g_s(\cdot) \equiv \nabla f_i(\cdot)$, where $i$ is a uniformly random index over the samples $X_i$.

\begin{figure}
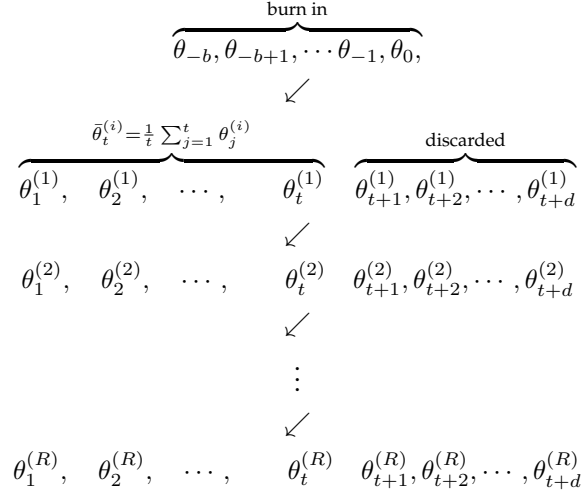

\centering
\scalebox{1}{
\parbox{\textwidth}{
\begin{gather*}
 \overbrace{\theta_{-b}, \theta_{-b+1},\cdots \theta_{-1}, \theta_{0},}^{\text{burn in}}  \nonumber \\
 \swarrow \\
 \overbrace{
\theta_1^{(1)}, \quad \theta_2^{(1)}, \quad  \cdots, \quad \quad \theta_t^{(1)} }^{ \btheta_t^{(i)} =  \frac{1}{t} \sum_{j=1}^t \theta_j^{(i)} }   \quad \overbrace{
\theta_{t+1}^{(1)},  \theta_{t+2}^{(1)}, \cdots , \theta_{t+d}^{(1)}
}^{\text{discarded}}
 \nonumber \\
 \swarrow \\
 \theta_1^{(2)}, \quad \theta_2^{(2)}, \quad \cdots,\quad \quad \theta_t^{(2)}   \quad \theta_{t+1}^{(2)},  \theta_{t+2}^{(2)}, \cdots , \theta_{t+d}^{(2)}
   \nonumber \\
  \swarrow  \\
 \vdots  \nonumber \\
 \swarrow \\
 \theta_1^{(R)}, \quad \theta_2^{(R)}, \quad \cdots,\quad \quad  \theta_t^{(R)}  \quad \theta_{t+1}^{(R)},  \theta_{t+2}^{(R)}, \cdots , \theta_{t+d}^{(R)}
 \nonumber
\end{gather*}
\caption{Our SGD inference procedure}
\label{fig:sgd-inf}
}
}
\end{figure}

\emph{Our inference procedure uses the average of $t$ consecutive SGD iterations}.
In particular, the algorithm proceeds as follows:
Given a sequence of SGD iterates, we use the first SGD iterates $\theta_{-b}, \theta_{-b+1}, \dots, \theta_0$ as a burn in period; we discard these iterates.
Next, for each ``segment'' of $t+d$ iterates,
we use the first $t$ iterates to compute $\btheta_t^{(i)} =  \frac{1}{t} \sum_{j=1}^t \theta_j^{(i)}$
and discard the last $d$ iterates, where $i$ indicates the $i$-th segment.
This procedure is illustrated in Figure \ref{fig:sgd-inf}.
As the final empirical minimum $\widehat{\theta}$, we use in practice $\htheta \approx \frac{1}{R} \sum_{i=1}^R \btheta_t^{(i)}$ \cite{Bubeck:2015:COA:2858997.2858998}.

Some practical aspects of our scheme are discussed below.

\emph{Step size $\eta$ selection and length $t$:}
\Cref{thm:sgd-strongly-convex-lipschitz-stat-inf} below is consistent only for SGD with fixed step size that depends on the number of samples taken. 
Our experiments, however, demonstrate that choosing a constant (large) $\eta$ gives equally accurate results with significantly reduced running time. 
We conjecture that a better understanding of $t$'s and $\eta$'s influence requires stronger bounds for SGD with constant step size.
Heuristically, calibration methods for parameter tuning in subsampling methods (\cite{politis2012subsampling}, Ch. 9) could be used for hyper-parameter tuning in our SGD procedure. We leave the problem of finding maximal (provable) learning rates for future work.

\emph{Discarded length $d$:}
Based on the analysis of mean estimation in the appendix, if we discard $d$ SGD iterates in every segment, the correlation between consecutive
$\theta^{(i)}$ and $\theta^{(i+1)}$ is of the order of $C_1 e^{-C_2 \eta d}$,
where $C_1$ and $C_2$ are data dependent constants.
This can be used as a rule of thumb to reduce correlation between samples from our SGD inference procedure.

\emph{Burn-in period $b$:}
The purpose of the burn-in period $b$, is to ensure that samples are generated when SGD iterates are sufficiently close to the optimum.
This can be determined using heuristics for SGD convergence diagnostics.
Another approach is to use other methods (\emph{e.g.}, SVRG \cite{johnson2013accelerating}) to find the optimum,
and use a relatively small $b$ for SGD to reach stationarity, similar to Markov Chain Monte Carlo burn-in.

\emph{Statistical inference using $\btheta_t^{(i)}$ and $\htheta$:}
Similar to ensemble learning \cite{opitz1999popular}, we use $i=1, 2, \dots,R$ estimators for statistical inference:
\begin{align}
\label{eq:stat-inf-formula}
\theta^{(i)}=\htheta + \sqrt{\frac{K_s \cdot t}{n}} \left(\btheta_t^{(i)} - \htheta\right).
\end{align}
Here, $K_s$ is a scaling factor that depends on how the stochastic gradient $g_s$ is computed.
We show examples of $K_s$ for mini batch SGD in linear regression and logistic regression in the corresponding sections.
Similar to other resampling methods such as bootstrap and subsampling, 
we use quantiles or variance of 
$\theta^{(1)}, \theta^{(2)}, \dots , \theta^{(R)}$ for statistical inference. 

\subsection{Theoretical guarantees}

Next, we provide the main theorem of our paper. 
Essentially, this provides conditions under which our algorithm is guaranteed to succeed, and hence has inference capabilities.


\begin{theorem}
\label{thm:sgd-strongly-convex-lipschitz-stat-inf}
For a differentiable convex function $f(\theta) = \frac{1}{n} \sum_{i=1}^n f_i(\theta)$, with gradient $\nabla f(\theta)$, let $\htheta \in \mathbb{R}^p$ be its minimizer, according to \eqref{eq:sgd_2}, and denote its Hessian at $\htheta$ by $H := \nabla^2 f(\htheta) = \frac{1}{n} \cdot \sum_{i = 1}^n \nabla^2 f_i(\htheta)$.
Assume that $\forall \theta \in \mathbb{R}^p$, $f$ satisfies:
\begin{itemize}[leftmargin=1cm]
\item [($F_1$)] ~Weak strong convexity: $ (\theta - \htheta)^\top \nabla f(\theta) \geq \alpha \| \theta - \htheta \|_2^2 $, for constant $\alpha > 0$,
\item [($F_2$)] ~Lipschitz gradient continuity: $\| \nabla f(\theta) \|_2 \leq L \| \theta - \htheta \|_2 $, for constant $L > 0$,
\item [($F_3$)] ~Bounded Taylor remainder: $\|\nabla f(\theta) - H(\theta - \htheta)\|_2 \leq E \|\theta - \htheta\|_2^2$, for constant $E > 0$,
\item [($F_4$)] ~Bounded Hessian spectrum at $\htheta$: $0 < \lambda_L \leq \lambda_i(H) \leq \lambda_U < \infty$, $\forall i$.
\end{itemize}
Furthermore, let $g_s(\theta)$ be a stochastic gradient of $f$, satisfying: \vspace{-0.1cm}
\begin{itemize}[leftmargin=1cm]
\item [($G_1$)] ~$\Exp\left[ g_s(\theta) \mid \theta \right] = \nabla f(\theta)$,
\item [($G_2$)] ~$\Exp\left[ \|g_s(\theta)\|_2^2 \mid \theta \right] \leq A \|\theta - \htheta\|_2^2 + B$,
\item [($G_3$)] ~$\Exp\left[ \|g_s(\theta)\|_2^4 \mid \theta \right] \leq C \|\theta - \htheta\|_2^4 + D $,
\item [($G_4$)] ~$\left\| \Exp\left[ g_s(\theta) g_s(\theta)^\top \mid \theta \right] - G \right\|_2
	\leq A_1 \|\theta - \htheta\|_2  + A_2 \|\theta - \htheta\|_2^2 + A_3 \|\theta - \htheta\|_2^3
		+ A_4 \|\theta - \htheta\|_2^4 $,
\end{itemize}
where $G = \Exp[g_s(\htheta) g_s(\htheta)^\top \mid \htheta]$ and, for positive, data dependent constants $A, B, C, D, A_i$, for $i = 1, \dots, 4$.

Assume that $\|\theta_1 - \htheta\|_2^2 = O(\eta)$; then for sufficiently small step size $\eta>0$, the average SGD sequence, $\btheta_t$, satisfies:
\begin{align}
\left\| t \Exp[ (\btheta_t - \htheta)( \btheta_t - \htheta)^\top] - H^{-1} G H^{-1} \right\|_2  \nonumber
\\
\lesssim
	\sqrt{\eta} + \sqrt{\tfrac{1}{t \eta} + t \eta^2}.
\end{align}
\end{theorem}

We provide the full proof in the appendix, and also we give precise (data-dependent) formulas for the above constants. For ease of exposition, we leave them as constants in the expressions above.
Further, in the next section, we relate a continuous approximation of SGD to Ornstein-Uhlenbeck process \cite{robbins1951stochastic} to give an intuitive explanation of our results.

{\em Discussion}. For linear regression, assumptions ($F_1$), ($F_2$), ($F_3$), and ($F_4$) are satisfied when the empirical risk function is not degenerate. In mini batch SGD using sampling with replacement, assumptions ($G_1$), ($G_2$), ($G_3$), and ($G_4$) are satisfied.
Linear regression's result is presented in Corollary \ref{cor:sgd-stat-inf-linear-regression} in the appendix.

For logistic regression, assumption ($F_1$) is not satisfied because the empirical risk function in this case is strictly but not strongly convex.
Thus, we cannot apply Theorem \ref{thm:sgd-strongly-convex-lipschitz-stat-inf} directly.
Instead, we consider the use of SGD on the \emph{square of the empirical risk function plus a constant; see eq. \eqref{eq:alt-logistic} below}.
When the empirical risk function is not degenerate,
\eqref{eq:alt-logistic} satisfies assumptions ($F_1$), ($F_2$), ($F_3$), and ($F_4$).
We cannot directly use vanilla SGD to minimize \eqref{eq:alt-logistic},
instead we describe a modified SGD procedure  for minimizing \eqref{eq:alt-logistic} in Section
\ref{subsec:logistic-stat-inf},
which satisfies assumptions ($G_1$), ($G_2$), ($G_3$), and ($G_4$).
We believe that this result is of interest by its own.
We present the result specialized for logistic regression in Corollary \ref{cor:sgd-stat-inf-logistic-regression}.

Note that \Cref{thm:sgd-strongly-convex-lipschitz-stat-inf} proves consistency for SGD with fixed step size,
requiring $\eta \to 0$ when $t \to \infty$.
However, we empirically observe in our experiments that a sufficiently large \emph{constant} $\eta$ gives better results.
We conjecture that the average of consecutive iterates in SGD with \emph{larger constant step size} converges to the optimum and we consider it for future work.


\subsection{Intuitive interpretation via the Ornstein-Uhlenbeck process approximation}
\label{subsec:thm1-informal}

Here, we describe a continuous approximation of the discrete SGD process and relate it to the Ornstein-Uhlenbeck process \cite{robbins1951stochastic}, to give an intuitive explanation of our results. 
In particular, under regularity conditions, the stochastic process $\Delta_t = \theta_t -\htheta$
asymptotically converges to an Ornstein-Uhlenbeck process $\Delta(t)$,
\cite{kushner1981asymptotic, pflug1986stochastic, Benveniste:1990:AAS:95267, kushner2003stochastic, mandt2016variational} that satisfies:
\begin{align}
\label{eq:sde-approx}
\,\mathrm{d} \Delta(T) = -  H \Delta(T)\, \mathrm{d}T + \sqrt{\eta} G^{\frac{1}{2}} \, \mathrm{d} B(T),
\end{align}
where $B(T)$ is a standard Brownian motion.
Given \eqref{eq:sde-approx}, $ \sqrt{t}(\btheta_t -\htheta)$ can be approximated as

\begin{equation}
\begin{aligned}
\label{eq:sgd-sde-int}
 \sqrt{t} (\btheta_t -\htheta) &= \tfrac{1}{\sqrt{t}} \sum_{i=1}^t (\theta_i - \htheta) \\
 &= \tfrac{1}{\eta \sqrt{t}} \sum_{i=1}^t (\theta_i - \htheta) \eta  \approx  \tfrac{1}{\eta \sqrt{t}} \int_0^{t \eta} \Delta(T)  \, \mathrm{d} T,
\end{aligned}
\end{equation}

where we use the approximation that $\eta \approx \, \mathrm{d} T$. 
By rearranging terms in \eqref{eq:sde-approx} and multiplying both sides by $H^{-1}$, we can rewrite the stochastic differential equation \eqref{eq:sde-approx} as
$\Delta(T) \, \mathrm{d}T  = - H^{-1} \,\mathrm{d} \Delta(T)  + \sqrt{\eta} H^{-1} G^{\frac{1}{2}} \,\mathrm{d} B(T)$.
Thus, we have
\begin{align}
\label{eq:delta-int}
& \int_0^{t \eta} \Delta(T)  \, \mathrm{d} T = \nonumber \\
& - H^{-1} (\Delta(t \eta)-\Delta(0)) + \sqrt{\eta} H^{-1} G^{\frac{1}{2}} B(t \eta).
\end{align}
After plugging \eqref{eq:delta-int} into \eqref{eq:sgd-sde-int} we have
\begin{align*}
&\sqrt{t} \left(\btheta_t -\htheta \right) \approx  \\
&- \tfrac{1}{\eta \sqrt{t}} H^{-1} \left(\Delta(t \eta)-\Delta(0) \right) + \tfrac{1}{\sqrt{t \eta}}  H^{-1} G^{\frac{1}{2}} B(t \eta).
\end{align*}
When $\Delta(0)=0$, the variance $\Var\big[- \sfrac{1}{\eta \sqrt{t}} \cdot H^{-1} (\Delta(t \eta)-\Delta(0))\big] = O\left(\sfrac{1}{t \eta}\right) $.
Since $\sfrac{1}{\sqrt{t \eta}} \cdot H^{-1} G^{\frac{1}{2}} B(t \eta) \sim \Norm(0,  H^{-1} G H^{-1} )$,
when $\eta \to 0$   and  $\eta t \to \infty$, we conclude that
\begin{align*}
\sqrt{t} (\btheta_t -\htheta) \sim \Norm(0,  H^{-1} G H^{-1} ).
\end{align*} 

\subsection{Logistic regression} \label{subsec:logistic-stat-inf}

We next apply our method to logistic regression.
We have $n$ samples $(X_1, y_1), (X_2, y_2), \dots (X_n, y_n)$ where $X_i \in \Real^p$ consists of features and $y_i \in \{+1, -1\}$ is the label.
We estimate $\theta$ of a linear classifier $\sign(\theta^T X) $ by: 
\begin{align*}
\htheta = \argmin_{\theta \in \mathbb{R}^p} \frac{1}{n} \sum_{i=1}^n \log\left(1+\exp(-y_i \theta^\top  X_i)\right).
\end{align*}

We cannot apply Theorem \ref{thm:sgd-strongly-convex-lipschitz-stat-inf}
directly because the empirical logistic risk is not strongly convex; it does not satisfy assumption ($F_1$).
Instead, we consider the convex function
\begin{small}
\begin{align}
\label{eq:alt-logistic}
f(\theta) = \frac{1}{2} \left(c+  \frac{1}{n} \sum_{i=1}^n \log\left(1+\exp(-y_i \theta^\top  X_i)\right)\right)^2, \nonumber
\\ \text{where $c > 0$~ (e.g., $c=1$).}
\end{align}
\end{small}
The gradient of $f(\theta)$ is a product of two terms
\begin{small}
\begin{align*}
\nabla f(\theta) = \underbrace{\left(c+  \frac{1}{n} \sum_{i=1}^n \log\left(1+\exp(-y_i \theta^\top X_i)\right)\right)}_{\Psi}
	\times \\ 
\underbrace{\nabla \left( \frac{1}{n} \sum_{i=1}^n \log\left(1+\exp(-y_i \theta^\top X_i)\right)\right)}_{\Upsilon}.
\end{align*}
\end{small}
Therefore, we can compute 
$g_s = \Psi_s \Upsilon_s$, \emph{using two independent random variables satisfying}  $\Exp[ \Psi_s \mid \theta ] =\Psi $ and $\Exp[ \Upsilon_s \mid \theta] = \Upsilon$.
For $\Upsilon_s$, we have
$
\Upsilon_s = \frac{1}{S_\Upsilon} \sum_{i \in I_t^\Upsilon} \nabla \log(1+\exp(-y_i \theta^\top  X_i))
$,
where $I_t^\Upsilon$ are $S_\Upsilon$ indices sampled from $[n]$ uniformly at random with replacement.
For $\Psi_s$, we have
$
\Psi_s = c + \frac{1}{S_\Psi} \sum_{i \in I_t^\Psi} \log(1+\exp(-y_i \theta^\top  X_i))
$,
where $I_t^\Psi$ are $S_\Psi$ indices uniformly sampled from $[n]$ with or without replacement.
Given the above, we have $\nabla f(\theta)^\top  (\theta - \htheta) \geq \alpha \| \theta - \htheta \|_2^2$ for some constant $\alpha$ by the generalized self-concordance of logistic regression \cite{bach2010self, bach2014adaptivity}, and therefore the assumptions are now satisfied.

For convenience, we write $k(\theta) = \frac{1}{n} \sum_{i=1}^n k_i(\theta)$ where $ k_i(\theta) = \log(1+\exp(-y_i \theta^\top  X_i))$.
Thus $f(\theta) = (k(\theta)+c)^2$, $\Exp[\Psi_s \mid \theta]  = k(\theta) + c$, and $\Exp[\Upsilon_s \mid \theta] = \nabla k(\theta)$.

\begin{corollary}
\label{cor:sgd-stat-inf-logistic-regression}
Assume $\|\theta_1 - \htheta\|_2^2 = O(\eta)$; also $S_\Psi = O(1)$, $S_\Upsilon=O(1)$ are bounded. Then, we have
\begin{small}
\begin{align*}
\Big\| t \Exp \left[ (\btheta_t - \htheta)( \btheta_t - \htheta)^\top \right] &- H^{-1} G H^{-1} \Big\|_2 \lesssim \sqrt{\eta} + \sqrt{\tfrac{1}{t \eta} + t \eta^2},
\end{align*}
\end{small}
where $H = \nabla^2 f(\htheta) = (c + k(\htheta)) \nabla^2 k(\htheta)$.
Here, $G= \frac{1}{S_\Upsilon} K_G(\htheta) \frac{1}{n} \sum_{i=1}^n \nabla k_i(\htheta) k_i(\htheta)^\top $
with $K_G(\theta) = \Exp[\Psi(\theta)^2]$ depending on how indexes are sampled to compute $\Psi_s$: 
\begin{itemize}
\item with replacement:  \small{$K_G(\theta) = \frac{1}{S_\Psi} ( \frac{1}{n} \sum_{i=1}^n (c+ k_i(\theta))^2 )
	+ \frac{S_\Psi-1}{S_\Psi} (c+k(\theta) )^2$} \vspace{-0.2cm},
\item no replacement: \small{$K_G(\theta) = \frac{1-\frac{S_\Psi-1}{n-1}}{S_\Psi}  ( \frac{1}{n} \sum_{i=1}^n (c+ k_i(\theta))^2 )  + \frac{S_\Psi-1}{S_\Psi} \frac{n}{n-1}  (c+k(\theta) )^2 $.}
\end{itemize}
Quantities other than $t$ and $\eta$ are data dependent constants.
\end{corollary}

As with the results above, in the appendix we give data-dependent expressions for the constants.
Simulations suggest that the term $t \eta^2$ in our bound is an artifact of our analysis.
Because in logistic regression  the estimate's covariance is
$ \frac{\left(\nabla^2 k(\htheta)\right)^{-1}}{n}  \left(\frac{\sum_{i=1}^n \nabla k_i(\htheta) \nabla k_i(\htheta)^\top}{n}  \right) \left(\nabla^2 k(\htheta)\right)^{-1}$,
we set the scaling factor $K_s=\frac{(c+k(\htheta))^2}{K_G(\htheta)}$ in \eqref{eq:stat-inf-formula} for statistical inference.
Note that $K_s \approx 1$ for sufficiently large $S_\Psi$.  


\section{Related work}{\label{sec:related}}

\emph{Bayesian inference:}
First and second order iterative optimization algorithms --including SGD, gradient descent, and variants-- naturally define a Markov chain. 
Based on this principle, 
most related to this work is the case of stochastic gradient Langevin dynamics (SGLD) for Bayesian inference -- namely, for sampling from the posterior distributions -- using a variant of SGD \cite{welling2011bayesian, bubeck2015finite, mandt2016variational, mandt2017stochastic}. 
We note that, here as well, the vast majority of the results rely on using a decreasing step size. Very recently, \cite{mandt2017stochastic} uses a heuristic approximation for Bayesian inference, and provides results for fixed step size.

Our problem is different in important ways from the Bayesian inference problem. 
In such 
parameter estimation problems, the covariance of the estimator only depends on the gradient of the likelihood function. 
This is not the case, however, in general frequentist $M$-estimation problems (\emph{e.g.}, linear regression).  
In these cases, the covariance of the estimator depends both on the gradient and Hessian of the empirical risk function. 
For this reason, without second order information, SGLD methods are poorly suited for general $M$-estimation problems in frequentist inference. 
In contrast, our method exploits properties of averaged SGD, and computes the estimator's covariance without second order information. 

\emph{Connection with Bootstrap methods:}
The classical approach for statistical inference is to use the bootstrap \cite{efron1994introduction, shao2012jackknife}. 
Bootstrap samples are generated by 
replicating the entire data set by resampling, and then solving the optimization problem 
on each generated set of the data.  
We identify our algorithm and its analysis as an alternative to bootstrap methods. Our analysis is also specific to SGD, and thus sheds light on the statistical properties of this very widely used algorithm.

\emph{Connection with stochastic approximation methods:}
It has been long observed in stochastic approximation that under certain conditions, SGD displays asymptotic normality for both the setting of {\em decreasing step size}, e.g., \cite{ljung2012stochastic,polyak1992acceleration}, and more recently, \cite{toulis2014asymptotic,chen2016statistical}; and also for {\em fixed step size}, e.g., \cite{Benveniste:1990:AAS:95267}, Chapter 4. All of these results, however, provide their guarantees with the requirement that the stochastic approximation iterate converges to the optimum. For decreasing step size, this is not an overly burdensome assumption, since with mild assumptions it can be shown directly. As far as we know, however, it is not clear if this holds in the fixed step size regime. To side-step this issue, \cite{Benveniste:1990:AAS:95267} provides results only when the (constant) step-size approaches 0 (see Section 4.4 and 4.6, and in particular Theorem 7 in \cite{Benveniste:1990:AAS:95267}).  Similarly, while \cite{kushner2003stochastic} has asymptotic results on the average of consecutive stochastic approximation iterates with constant step size, it assumes convergence of iterates (assumption A1.7 in Ch. 10) -- an assumption we are unable to justify in even simple settings.

Beyond the critical difference in the assumptions, the majority of the ``classical'' subject matter seeks to prove asymptotic results about different flavors of SGD, but does not properly consider its use for inference. Key exceptions are the recent  work in \cite{toulis2014asymptotic} and \cite{chen2016statistical}, which follow up on \cite{polyak1992acceleration}. Both of these rely on decreasing step size, for reasons mentioned above. The work in \cite{chen2016statistical} uses SGD with decreasing step size for estimating an $M$-estimate's covariance. Work in \cite{toulis2014asymptotic} studies implicit SGD with decreasing step size and proves results similar to \cite{polyak1992acceleration}, however it does not use SGD to compute confidence intervals.

Overall, to the best of our knowledge, there are no prior results establishing asymptotic normality for SGD with fixed step size for general M-estimation problems (that do not rely on overly restrictive assumptions, as discussed). 


\begin{figure*}[!t]
\centering
\subfloat[][Normal.]{
\includegraphics[width=.32\textwidth]{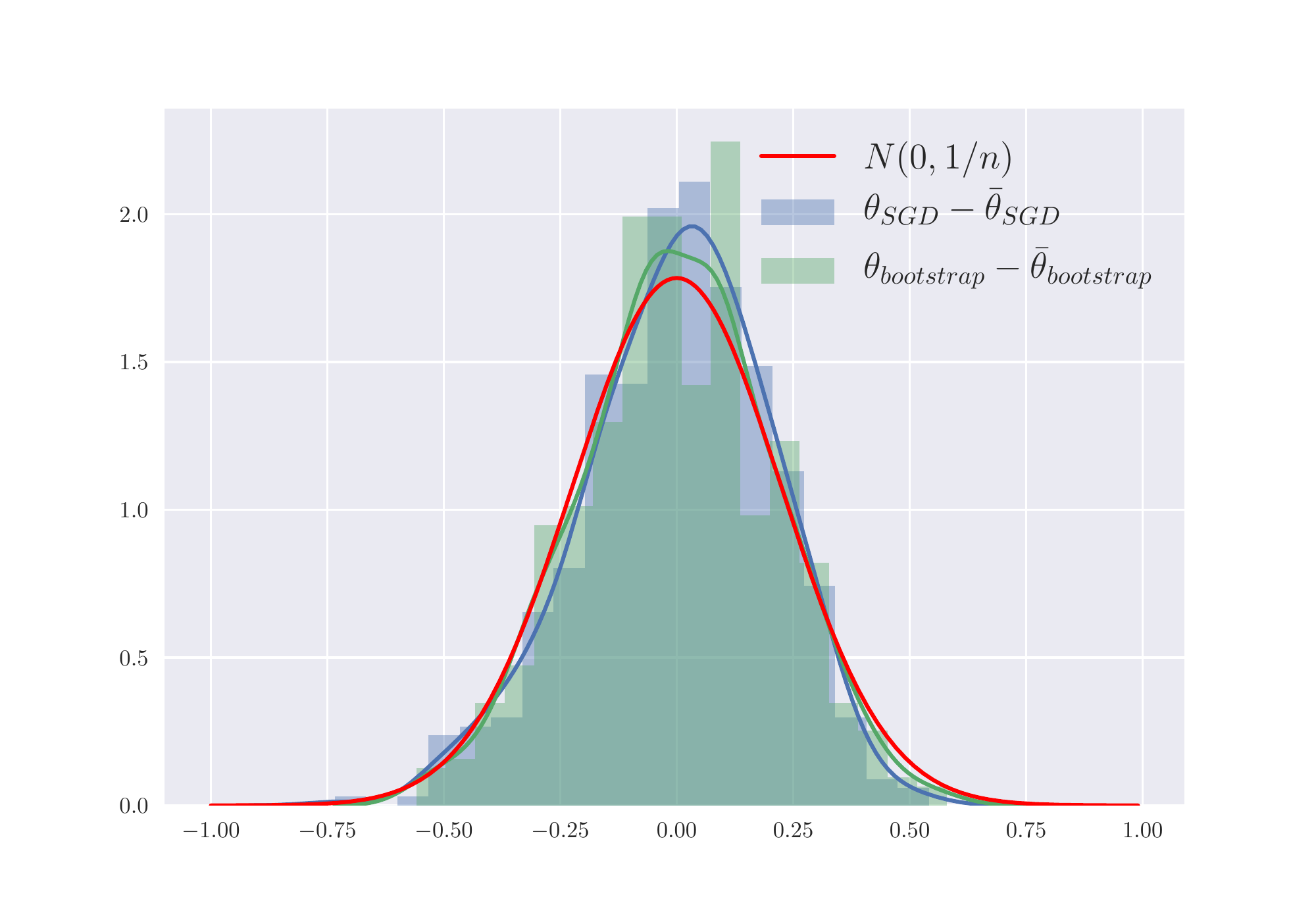}
\label{fig:exp:sim:1d-normal-nips}
}
\subfloat[][Exponential.]{
\includegraphics[width=.32\textwidth]{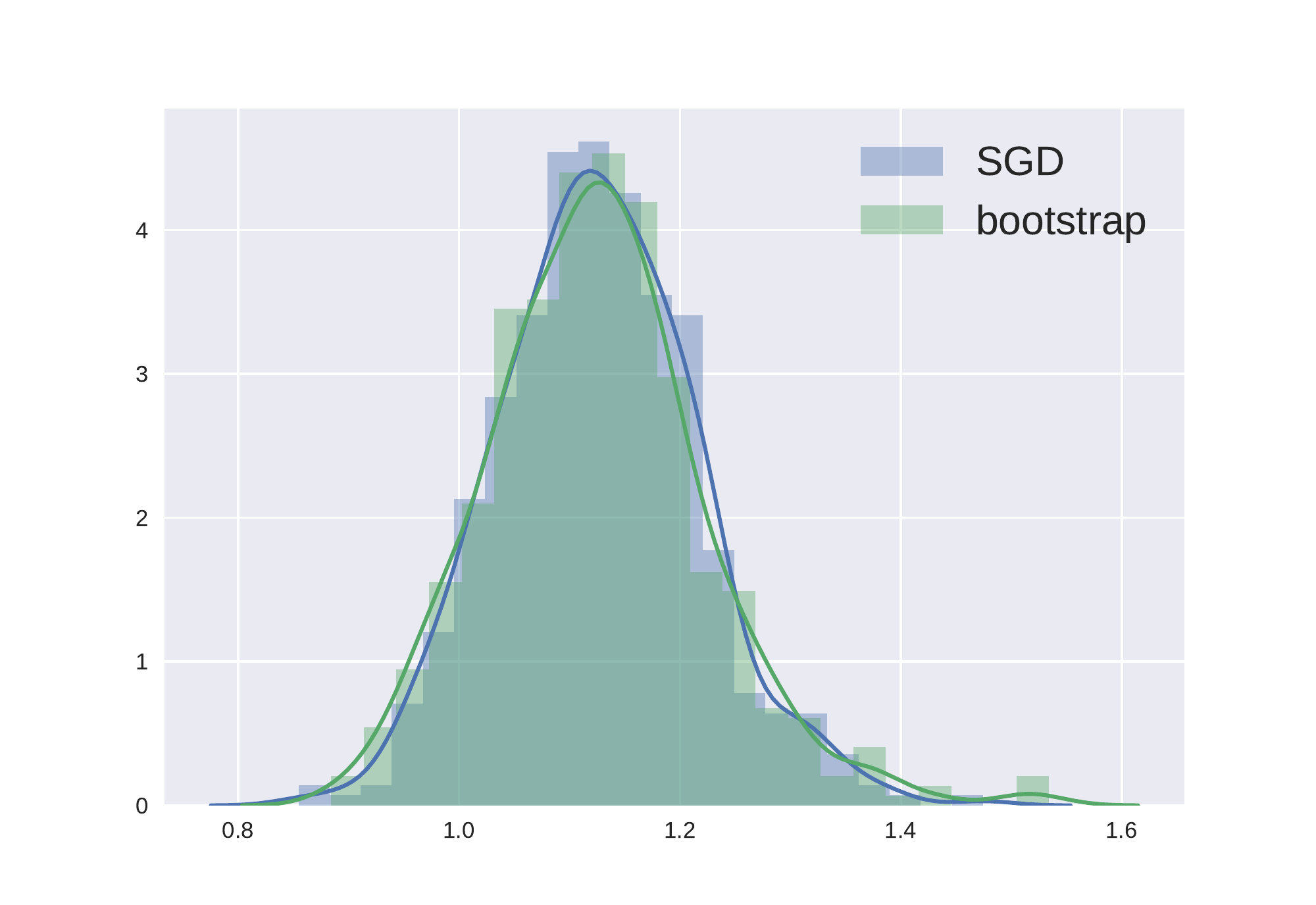}
\label{fig:exp:sim:1d-exponential-nips}
}
\subfloat[][Poisson.]{
\includegraphics[width=.32\textwidth]{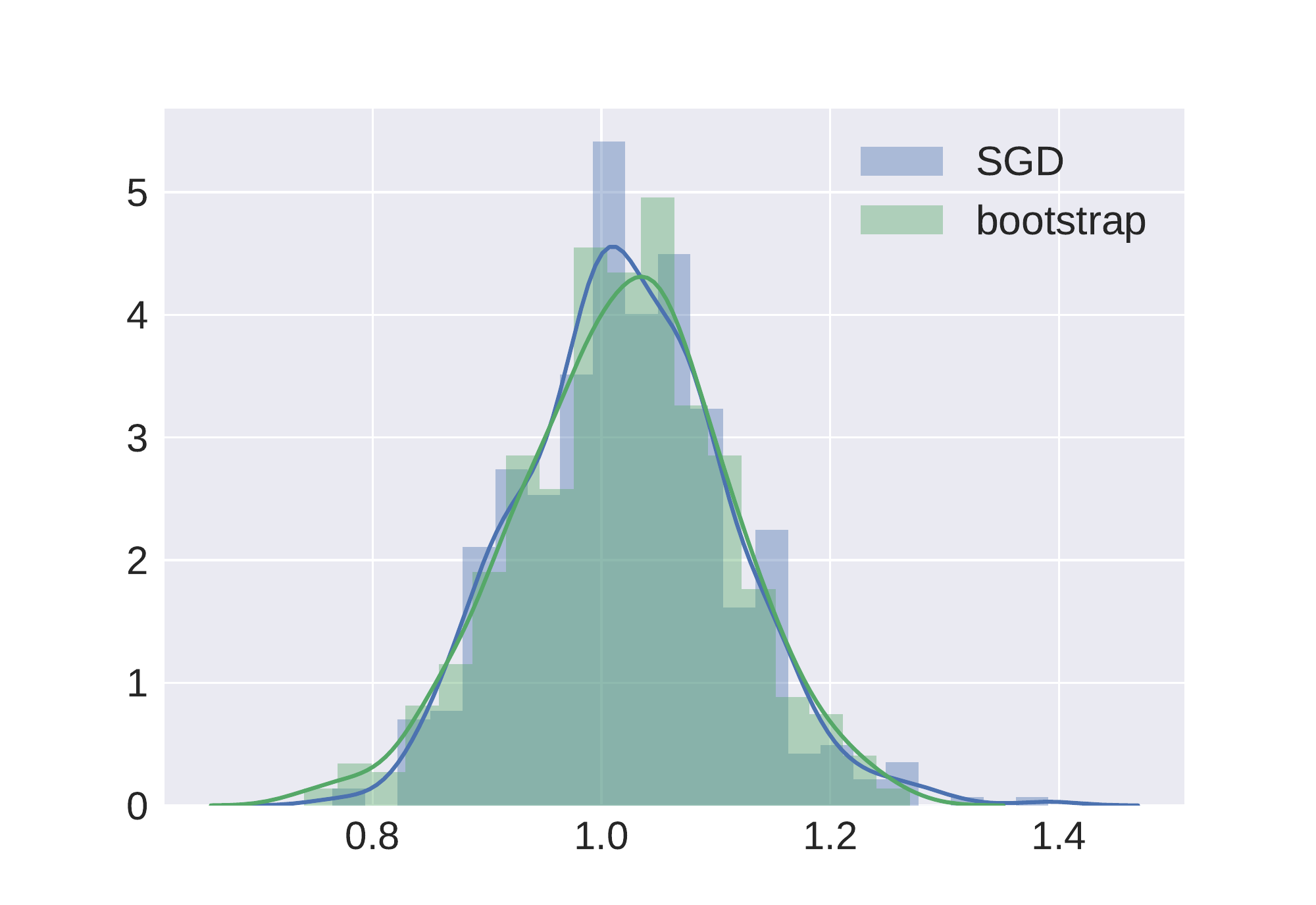}
\label{fig:exp:sim:1d-poisson-nips}
}
\caption{Estimation in univariate models.}
\label{fig:exp:sim:univariate} 
\end{figure*}

\begin{table*}[!h]
\centering
\rowcolors{2}{white}{black!05!white}
\resizebox{\textwidth}{!}{%
\subfloat[][Bootstrap (0.941, 4.14), normal approximation (0.928, 3.87)]{
 \begin{tabular}{c c c c c  c c}
 \toprule
 $\eta$ & & $t=100$ & & $t=500$ & & $t=2500$ \\
 \cmidrule{1-1} \cmidrule{3-3} \cmidrule{5-5} \cmidrule{7-7}
 $0.1$ & & (0.957, 4.41) & & (0.955, 4.51) & & (0.960, 4.53) \\
 $0.02$ & & (0.869, 3.30) & & (0.923, 3.77) & & (0.918, 3.87) \\
 $0.004$ & & (0.634, 2.01) & & (0.862, 3.20) & & (0.916, 3.70) \\
 \bottomrule
 \end{tabular}
 \label{tab:exp:coverage:linear:1-nips}
}

\subfloat[][Bootstrap (0.938, 4.47), normal approximation (0.925, 4.18)]{
 \begin{tabular}{c c c c  c  c c}
 \toprule
 $\eta$ & & $t=100$ & & $t=500$ & & $t=2500$ \\
 \cmidrule{1-1} \cmidrule{3-3} \cmidrule{5-5} \cmidrule{7-7}
 $0.1$ & & (0.949, 4.74) & & (0.962, 4.91) & & (0.963, 4.94) \\
 $0.02$ & & (0.845, 3.37) & & (0.916, 4.01) & & (0.927, 4.17) \\
 $0.004$ & & (0.616, 2.00) & &  (0.832, 3.30) & & (0.897, 3.93) \\
 \bottomrule
 \end{tabular}
\label{tab:exp:coverage:linear:2-nips}
}
}
\caption{Linear regression. \textit{Left}: Experiment 1, \textit{Right}: Experiment 2.}
\label{tab:exp:coverage:linear} 
\end{table*}

\begin{table*}[!h]
\centering
\rowcolors{2}{white}{black!05!white}
\resizebox{\textwidth}{!}{%
\subfloat[][Bootstrap (0.932, 0.253), normal approximation (0.957, 0.264)]{
 \begin{tabular}{c c c c c c c}
 \toprule
 $\eta$ & & $t=100$ & & $t=500$ & & $t=2500$ \\
  \cmidrule{1-1} \cmidrule{3-3} \cmidrule{5-5} \cmidrule{7-7}
 $0.1$ & & (0.872, 0.204) & & (0.937, 0.249) & & (0.939, 0.258) \\
 $0.02$ & &  (0.610, 0.112) & & (0.871, 0.196) & & (0.926, 0.237) \\
 $0.004$ & & (0.312, 0.051) & & (0.596, 0.111) &  & (0.86, 0.194) \\
 \bottomrule
 \end{tabular}
 \label{tab:exp:coverage:logistic:1-nips}
}
~
\subfloat[][Bootstrap (0.932, 0.245), normal approximation (0.954, 0.256)]{
 \begin{tabular}{c c c c c c c}
 \toprule
 $\eta$ & & $t=100$ & & $t=500$ & & $t=2500$ \\
 \cmidrule{1-1} \cmidrule{3-3} \cmidrule{5-5} \cmidrule{7-7}
 $0.1$ & & (0.859, 0.206) & & (0.931, 0.255) & & (0.947, 0.266) \\
 $0.02$ & & (0.600, 0.112) & & (0.847, 0.197) & & (0.931, 0.244) \\
 $0.004$ & & (0.302, 0.051) & & (0.583, 0.111) & & (0.851, 0.195) \\
 \bottomrule
 \end{tabular}
\label{tab:exp:coverage:logistic:2-nips}
}
}
\caption{Logistic regression. \textit{Left}: Experiment 1, \textit{Right}: Experiment 2.}
\label{tab:exp:coverage:logistic} 
\end{table*}

\section{Experiments}
\label{sec:experiments}

\subsection{Synthetic data}

The coverage probability is defined as
$
\small{\frac{1}{p} \sum_{i=1}^p \Pr[\theta^\star_i \in \hat{C}_i]}
$
where $\small{\theta^\star = \argmin_\theta \Exp[f(\theta, X)] \in \Real^p}$, and
$\hat{C}_i$ is the estimated confidence interval for the $i^{\text{th}}$ coordinate.
The average confidence interval width is defined as
$
\frac{1}{p} \sum_{i=1}^p (\hat{C}_i^u - \hat{C}_i^l)
$
where $[\hat{C}_i^l , \hat{C}_i^u]$ is the estimated confidence interval for the $i^{\text{th}}$ coordinate.
In our experiments, coverage probability and average confidence interval width are estimated through simulation.
We use the  empirical quantile  of our SGD inference procedure and bootstrap to compute the 95\% confidence intervals for each coordinate of the parameter.
%
For results given as a pair $(\alpha, \beta)$, it usually indicates (coverage probability, confidence interval length).

\subsubsection{Univariate models}

In \Cref{fig:exp:sim:univariate}, we compare our SGD inference procedure with $(i)$ Bootstrap and $(ii)$ normal approximation with inverse Fisher information in univariate models.
We observe that our method and Bootstrap have similar statistical properties.
\Cref{fig:exp:sim:univariate-qq} in the appendix shows
Q-Q plots of samples from our SGD inference procedure.

\medskip
{\em Normal distribution mean estimation:} \Cref{fig:exp:sim:1d-normal-nips} compares 500 samples from SGD inference procedure and Bootstrap
versus the distribution $\Norm(0, 1/n)$,
using $n=20$ i.i.d. samples from $\Norm(0, 1)$.
We used mini batch SGD described in Sec. \ref{sec:mean-est-exact}.
For the parameters, we used $\eta=0.8$, $t=5$, $d=10$, $b=20$, and mini batch size of 2.
Our SGD inference procedure gives (0.916 , 0.806),
Bootstrap gives (0.926 , 0.841),
and normal approximation gives (0.922 , 0.851).

\medskip
{\em Exponential distribution parameter estimation:} \Cref{fig:exp:sim:1d-exponential-nips} compares 500 samples from inference procedure and Bootstrap,
using $n=100$ samples from an exponential distribution with PDF
$\lambda e^{-\lambda x}$ where $\lambda=1$.
We used SGD for MLE with mini batch sampled with replacement.
For the parameters, we used $\eta=0.1$, $t=100$, $d=5$, $b=100$, and mini batch size of 5.
Our SGD inference procedure gives (0.922, 0.364),
Bootstrap gives (0.942 , 0.392),
and normal approximation gives (0.922, 0.393).

\medskip
{\em Poisson distribution parameter estimation:} \Cref{fig:exp:sim:1d-poisson-nips} compares 500 samples from inference procedure and Bootstrap,
using $n=100$ samples from a Poisson distribution with PDF
$\lambda^x e^{-\lambda x}$ where $\lambda=1$.
We used SGD for MLE with mini batch sampled with replacement.
For the parameters, we used $\eta=0.1$, $t=100$, $d=5$, $b=100$, and mini batch size of 5.
Our SGD inference procedure gives (0.942 , 0.364),
Bootstrap gives (0.946 , 0.386),
and normal approximation gives (0.960 , 0.393).

\subsubsection{Multivariate models}


In these experiments, we set $d = 100$, used mini-batch size of $4$, and used 200 SGD samples. In all cases, we compared with Bootstrap using 200 replicates. We computed the coverage probabilities using 500 simulations. Also, we denote $1_p = \begin{bmatrix} 1 & 1 & \dots & 1 \end{bmatrix}^\top \in \Real^p$.
Additional simulations comparing covariance matrix computed with different methods are given in Sec. \ref{subsubsec:nip2017:appendix:experiments:mult}.

\textbf{Linear regression:}
{\em Experiment 1:}
Results for the case where $X \sim \Norm(0, I) \in \Real^{10}$,
$Y = {w^*}^T X + \epsilon$,
$w^* = 1_p / \sqrt{p}$, and $\epsilon \sim \Norm(0, \sigma^2=10^2)$ with $n=100$ samples is given in
 \Cref{tab:exp:coverage:linear:1-nips}.
Bootstrap gives (0.941, 4.14),
and confidence intervals computed using the error covariance and normal approximation gives (0.928, 3.87).
%
{\em Experiment 2:}
Results for the case where $X \sim \Norm(0, \Sigma) \in \Real^{10}$, $\Sigma_{ij} = 0.3^{|i-j|}$,
$Y = {w^*}^T X + \epsilon$,
$w^* = 1_p / \sqrt{p}$,
and $\epsilon \sim \Norm(0, \sigma^2=10^2)$ with $n=100$ samples is given in
 \Cref{tab:exp:coverage:linear:2-nips}.
Bootstrap gives (0.938, 4.47), and confidence intervals computed using the error covariance and normal approximation gives (0.925, 4.18).

\textbf{Logistic regression:}
Here we show results for logistic regression trained using vanilla SGD with mini batch sampled with replacement.
Results for modified SGD (Sec. \ref{subsec:logistic-stat-inf}) are given in Sec. \ref{subsubsec:nip2017:appendix:experiments:mult}.
{\em Experiment 1:}
Results for the case where $\Pr[Y=+1] = \Pr[Y=-1]=1/2$,
$X \mid Y \sim \Norm(0.01 Y 1_p / \sqrt{p}, I) \in \Real^{10}$
with $n=1000$ samples is given in
 \Cref{tab:exp:coverage:logistic:1-nips}.
Bootstrap gives (0.932, 0.245), and confidence intervals computed using inverse Fisher matrix as the error covariance and normal approximation gives (0.954, 0.256).
%
{\em Experiment 2:}
Results for the case where $\Pr[Y=+1] = \Pr[Y=-1]=1/2$,
$X \mid Y \sim \Norm(0.01 Y 1_p / \sqrt{p}, \Sigma) \in \Real^{10}$,
$\Sigma_{ij}=0.2^{|i-j|}$ with $n=1000$ samples is given in
 \Cref{tab:exp:coverage:logistic:2-nips}.
Bootstrap gives (0.932, 0.253), and confidence intervals computed using inverse Fisher matrix as the error covariance and normal approximation gives (0.957, 0.264).

\subsection{Real data}
Here, we compare covariance matrices computed using our SGD inference procedure,
bootstrap, and inverse Fisher information matrix
on the LIBSVM Splice data set,
and we observe that they have similar statistical properties.

\subsubsection{Splice data set}

The Splice data set~\footnote{\url{https://www.csie.ntu.edu.tw/~cjlin/libsvmtools/datasets/binary.html}}
 contains 60 distinct features with 1000 data samples.
This is a classification problem between two classes of splice junctions in a DNA sequence.
We use a logistic regression model trained using vanilla SGD.

In Figure \ref{fig:exp:real:libsvm:splice:main},
we  compare  the covariance matrix computed using our SGD inference procedure and bootstrap
$n=1000$ samples.
We used 10000 samples from both bootstrap and our SGD inference procedure with
$t=500$, $d=100$, $\eta=0.2$, and mini batch size of $6$.

\begin{figure}[h!]
\centering
\subfloat[Bootstrap]{
\includegraphics[width=0.48\textwidth]{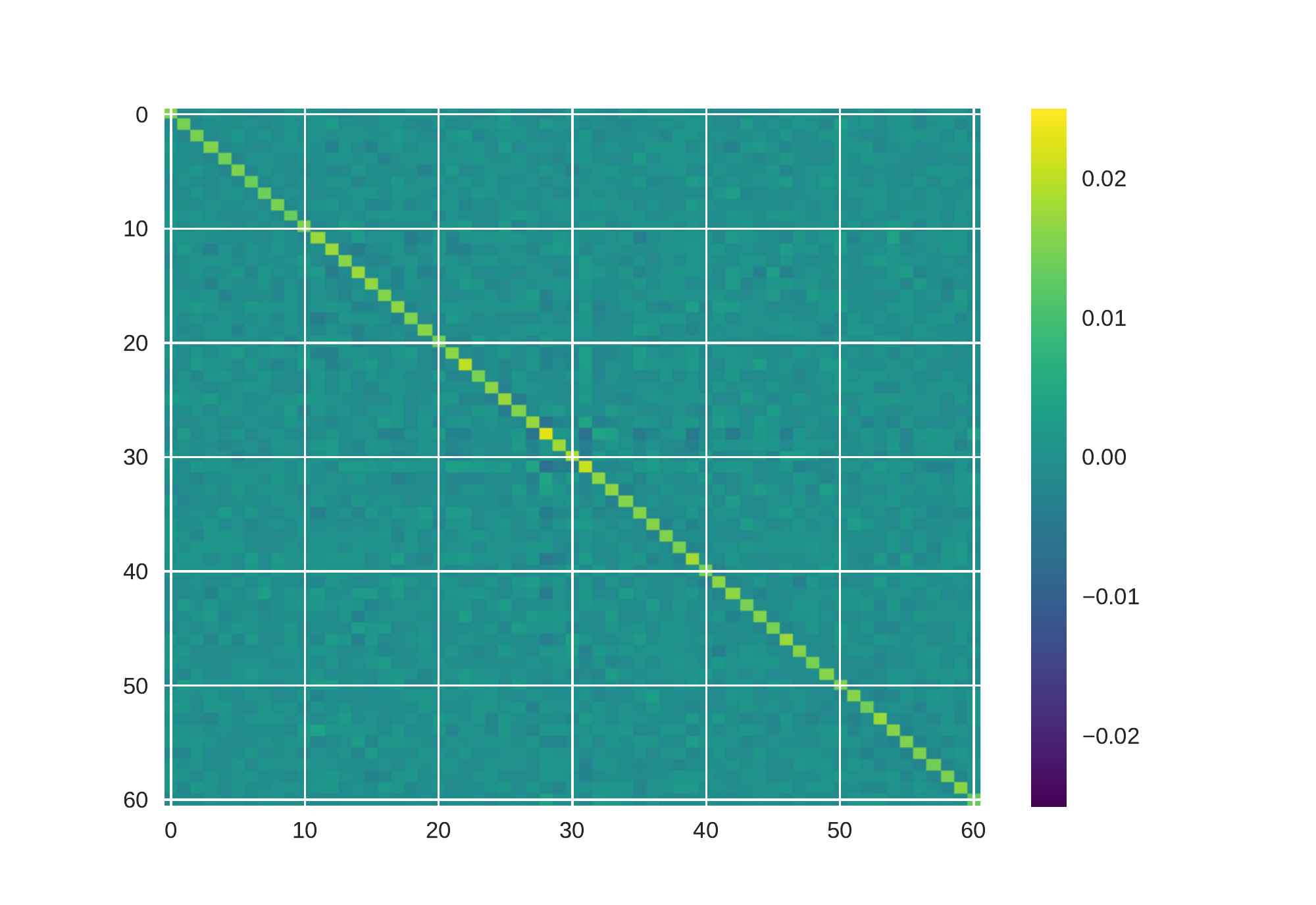}
}
\subfloat[SGD inference covariance]{
\includegraphics[width=0.48\textwidth]{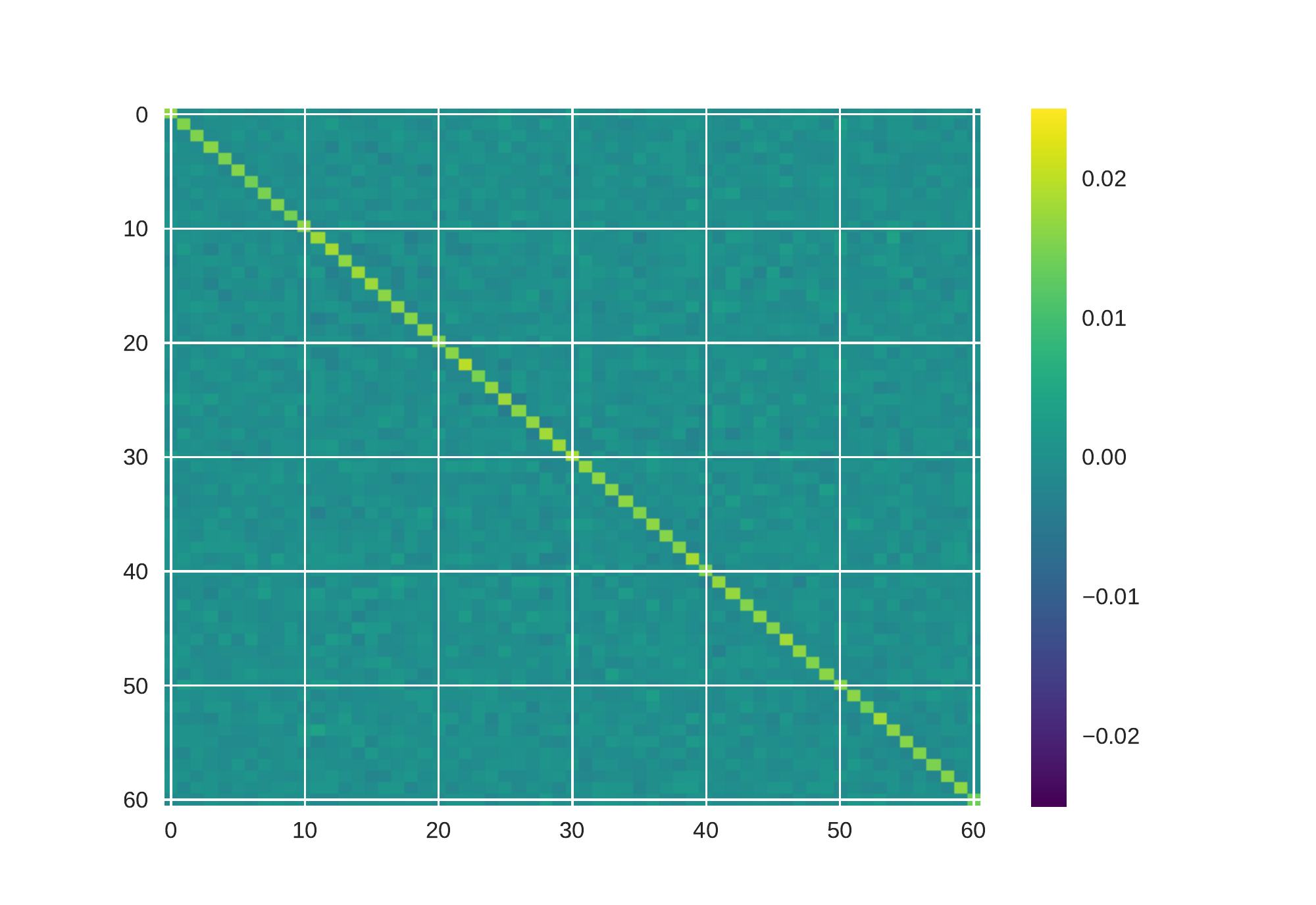}
}
\caption{Splice data set}
\label{fig:exp:real:libsvm:splice:main}
\end{figure}

\begin{figure}[h!]
\centering
\subfloat[
Original ``0'': 
logit -46.3, \newline
CI (-64.2, -27.9)
]{
\includegraphics[width=0.48\textwidth]{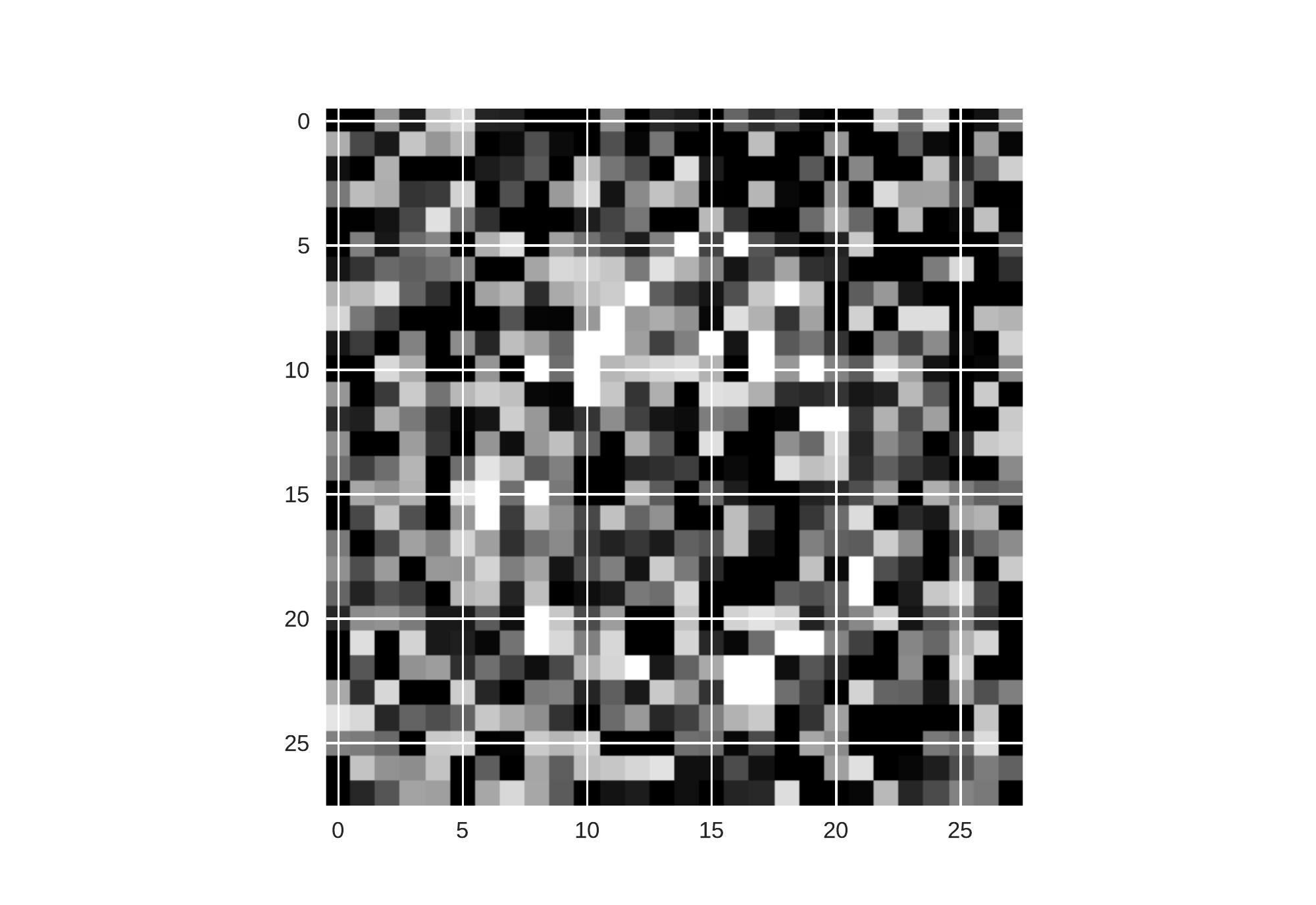}
}
\subfloat[
Adversarial ``0'': 
logit 16.5, \newline
CI (-10.9, 30.5)
]{
\includegraphics[width=0.48\textwidth]{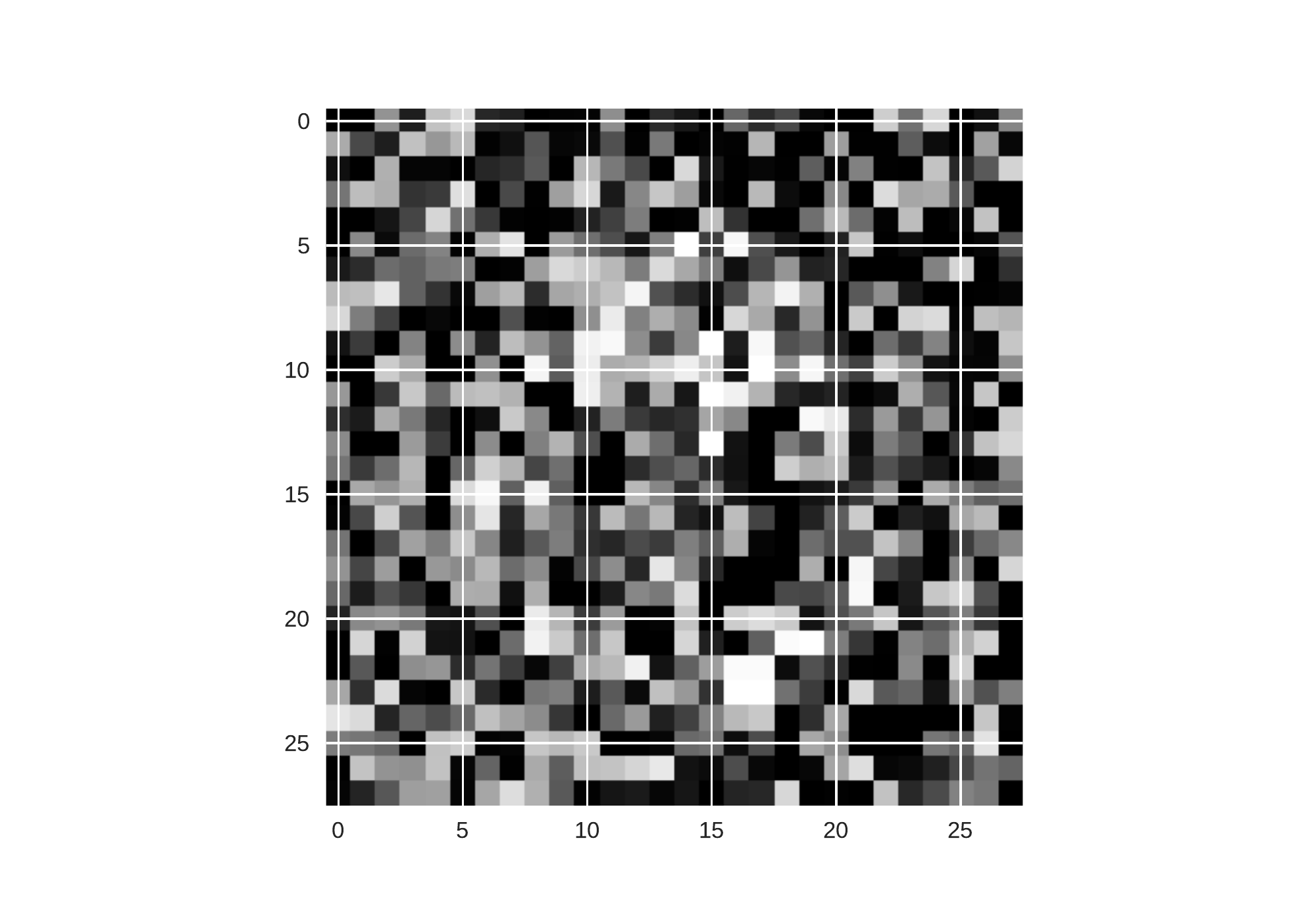}
}
\caption{MNIST}
\label{fig:exp:real:mnist}
\end{figure}

\subsubsection{MNIST}

Here, we train a binary logistic regression classifier to classify 0/1 using a noisy MNIST data set, 
and demonstrate that adversarial examples 
produced by gradient attack \cite{goodfellow2014explaining} 
(perturbing an image in the direction of loss function's gradient with respect to data) 
can be detected using prediction  intervals. 
We flatten each $28 \times 28$ image into a 784 dimensional vector, 
and  train a linear classifier using pixel values as features. 
To add noise to each image, where each original pixel is either 0 or 1, 
we randomly changed 70\% pixels to random numbers uniformly on $[0, 0.9]$. 
Next we train the classifier on the noisy MNIST data set, 
and generate adversarial examples using this noisy MNIST data set. 
\Cref{fig:exp:real:mnist} shows each image's logit value 
($\log \frac{\Pr[1 \mid \text{image}]}{\Pr[0 \mid \text{image}]}$) 
and its 95\% confidence interval (CI) computed using quantiles from our SGD inference procedure.

\subsection{Discussion}

In our experiments, we observed that using a larger step size $\eta$ produces accurate results with significantly accelerated convergence time. This might imply that the $\eta$ term in Theorem \ref{thm:sgd-strongly-convex-lipschitz-stat-inf}'s bound is an artifact of our analysis.  Indeed, although \Cref{thm:sgd-strongly-convex-lipschitz-stat-inf} only applies to SGD with fixed step size,
where $\eta t \to \infty$ and $\eta^2 t \to 0$ imply that the step size should be smaller when the number of consecutive iterates used for the average is larger, our experiments suggest that we can use a (data dependent) constant step size $\eta$ and only require $\eta t \to \infty$.

In the experiments, our SGD inference procedure uses
$
(t+d) \cdot S \cdot p
$
operations to produce a sample, and Newton method uses  $n \cdot (\text{matrix inversion complexity } = \Omega(p^2)) \cdot (\text{number of Newton iterations}~t)$  operations to produce a sample.  The experiments therefore suggest that our SGD inference procedure produces results similar to Bootstrap while using far fewer operations.

\section{Acknowledgments}
This work was partially supported by NSF Grants 1609279, 1704778 and 1764037, and also by the USDoT through the Data-Supported Transportation Operations and Planning (D-STOP) Tier 1 University Transportation Center. 
A.K. is supported by the IBM Goldstine fellowship. 
We thank Xi Chen, Philipp Kr\"ahenb\"uhl, Matthijs Snel, and Tom Spangenberg for insightful discussions. 

{
\fontsize{9.0pt}{10.0pt} \selectfont
\bibliography{ref.bib}
\bibliographystyle{alpha}
}

\newpage

\appendix

\onecolumn



\section{Exact analysis of mean estimation}
\label{sec:mean-est-exact}

In this section, we give an exact analysis of our method in the least squares, mean estimation problem. 
For $n$ i.i.d. samples $X_1, X_2, \dots , X_n $, the mean is estimated by solving the following optimization problem 
\begin{align*}
\hat{\theta} = \argmax_{\theta \in \mathbb{R}^p} \frac{1}{n} \sum_{i=1}^n \tfrac{1}{2} \|X_i - \theta\|_2^2 = \frac{1}{n} \sum_{i=1}^n X_i.
\end{align*}
In the case of mini-batch SGD, we sample $S=O(1)$ indexes uniformly randomly with replacement from $[n]$; denote that index set as $I_t$.
For convenience, we write  $Y_t = \frac{1}{S} \sum_{i \in I_t} X_i$, 
Then, in the $t^{\text{th}}$ mini batch SGD step, the update step is 
\begin{align}
\theta_{t+1} = \theta_t - \eta ( \theta_t -Y_t ) = (1 - \eta) \theta_t + \eta Y_t,  \label{eq::mean-est-sgd-iter}
 \end{align}
which is the same as the exponential moving average. 
And we have 
\begin{align}
\sqrt{t} \htheta_t = - \frac{1}{\eta \sqrt{t}} (\theta_{t+1} - \theta_1) + \frac{1}{\sqrt{t}} \sum_{i=1}^n Y_i . 
\end{align}
Assume that $\|\theta_1 -\htheta\|_2^2 = O(\eta)$, then from Chebyshev's inequality $ - \frac{1}{\eta \sqrt{t}} (\theta_{t+1} - \theta_1) \to 0 $  almost surely when $t \eta \to \infty$. 
By the central limit theorem, $\frac{1}{\sqrt{t}} \sum_{i=1}^n Y_i$ converges weakly to 
$\Norm(\htheta, \frac{1}{S} \hat{\Sigma})$ 
with $\hat{\Sigma} = \frac{1}{n} \sum_{i=1}^n (X_i - \htheta )( X_i - \htheta )^\top $. 
From \eqref{eq::mean-est-sgd-iter}, we have $\| \Cov(\theta_a, \theta_b) \|_2 = O(\eta (1-\eta)^{|a-b|})$ uniformly for all $a, b$, where the constant is data dependent. 
Thus, for our SGD inference procedure, 
we have $\| \Cov(\theta^{(i)}, \theta^{(j)}) \|_2 = O(\eta (1-\eta)^{d + t |i-j|})$. 
Our SGD inference procedure does not generate samples that are independent  conditioned on the data, 
whereas   replicates are independent conditioned on the data in bootstrap, 
but this suggests that our SGD inference procedure can produce 
``almost independent'' samples if we discard sufficient number of SGD iterates in each segment. 

When estimating a mean using our SGD inference procedure where each mini batch is $S$ elements sampled with replacement, 
we set $K_s = S$ in \eqref{eq:stat-inf-formula}. 


\section{Linear Regression}
\label{sec:linear-stat-inf}

In linear regression, the empirical risk function satisfies:
\begin{align*}
f(\theta) = \frac{1}{n} \sum_{i=1}^n \tfrac{1}{2}(\theta^\top  x_i - y_i)^2,
\end{align*} 
where $y_i$ denotes the observations of the linear model and $x_i$ are the regressors.
To find an estimate to $\theta^\star$, one can use SGD with stochastic gradient give by:
\begin{align*}
g_s[\theta_t] = \frac{1}{S} \sum_{i \in I_t} \nabla f_i(\theta_t),
\end{align*}
where $I_t$ are $S$ indices uniformly sampled from $[n]$ with replacement.

Next, we state a special case of Theorem \ref{thm:sgd-strongly-convex-lipschitz-stat-inf}.
Because the Taylor remainder  $\nabla f(\theta) - H(\theta - \htheta) = 0$,
linear regression has a stronger result than general $M$-estimation problems.
\begin{corollary}
\label{cor:sgd-stat-inf-linear-regression}
Assume that $\|\theta_1 - \htheta\|_2^2 = O(\eta)$, we have
\begin{align*}
\left\| t \Exp[ (\btheta_t - \htheta)( \btheta_t - \htheta)^\top ] - H^{-1} G H^{-1} \right\|_2
\lesssim \sqrt{\eta}+ \frac{1}{\sqrt{t \eta} },
\end{align*}
where $H = \frac{1}{n} \sum_{i=1}^n x_i x_i^\top $ and
$G = \frac{1}{S} \frac{1}{n} \sum_{i=1}^n (x_i^\top \htheta - y_i)^2 x_i x_i^\top  $.

We assume that $S = O(1)$  is bounded,
and quantities other than $t$ and $\eta$ are data dependent constants.
\end{corollary}

As with our main theorem, in the appendix we provide explicit data-dependent expressions for the constants in the result.
Because in linear regression  the estimate's covariance is
$ \frac{1}{n}  (\frac{1}{n} \sum_{i=1}^n x_i x_i^\top)^{-1})
(\frac{1}{n} (x_i^\top \htheta - y_i)(x_i^\top \htheta - y_i)^\top)
(\frac{1}{n} \sum_{i=1}^n x_i x_i^\top)^{-1} )$,
we set the scaling factor $K_s=S$ in \eqref{eq:stat-inf-formula} for statistical inference.



\section{Proofs}

\subsection{Proof of Theorem \ref{thm:sgd-strongly-convex-lipschitz-stat-inf}}

We first assume that $\theta_1 = \htheta$. 
For ease of notation, we denote
\begin{align}
\Delta_t = \theta_t - \htheta,
\end{align}
and, without loss of generality, we assume that $\htheta=0$.
The stochastic gradient descent recursion satisfies:
\begin{align*}
\theta_{t+1} &= \theta_t - \eta \cdot g_s(\theta_t) \\
				  &= \theta_t - \eta \cdot \left( g_s(\theta_t) - \nabla f(\theta_t) + \nabla f(\theta_t) \right) \\
				  &= \theta_t - \eta \cdot \nabla f(\theta_t) - \eta \cdot e_t,
\end{align*}
where $e_t =g_s(\theta_t) - \nabla f(\theta_t)$.
Note that $e_1, e_2, \dots $ is a martingale difference sequence.
We use
\begin{align}
g_i = \nabla f_i(\htheta) \quad \quad \text{and,} \quad \quad
H_i = \nabla^2 f_i(\htheta)
\end{align} to denote the gradient component at index $i$, and the Hessian component at index $i$, at optimum $\htheta$, respectively.
Note that $\sum g_i = 0$ and $\tfrac{1}{n} \sum H_i = H$.

For each $f_i$, its Taylor expansion around $\htheta$ is
\begin{align}
f_i(\theta) = f_i(\htheta) + g_i^\top (\theta - \htheta) +
	\frac{1}{2}  (\theta - \htheta)^\top H_i (\theta - \htheta) +
	R_i(\theta, \htheta),
\end{align}
where $R_i(\theta, \htheta)$ is the remainder term.
For convenience, we write
$R = \tfrac{1}{n} \sum R_i$.

For the proof, we require the following lemmata.
The following lemma states that $ \Exp[ \| \Delta_t\|_2^2] = O(\eta) $ as $t \to \infty$ and $\eta \to 0$.

\begin{lemma}{\label{lem:tas_00}}
For data dependent, positive constants $\alpha, A, B$ according to assumptions $(F_1)$ and $(G_2)$ in Theorem 1, and given assumption $(G_1)$, we have
\begin{align}
\Exp\left[ \left\| \Delta_{t}\right\|_2^2 \right] \leq (1 - 2 \alpha \eta + A \eta^2)^{t-1} \|\Delta_1\|_2^2
	+ \frac{B \eta}{2 \alpha - A \eta},
\end{align} under the assumption $\eta < \frac{2\alpha}{A}$.
\end{lemma}
\begin{proof}
As already stated, we assume without loss of generality that $\htheta = 0$.
This further implies that: $g_s(\theta_t) = g_s(\theta_t - \htheta) = g_s(\Delta_t)$, and
\begin{align*}
\Delta_{t + 1} = \Delta_t - \eta \cdot g_s(\Delta_t).
\end{align*}
Given the above and assuming expectation $\Exp[\cdot]$ w.r.t. the selection of a sample from $\{X_i\}_{i = 1}^n$, we have:
\begin{align}
\Exp\left[ \| \Delta_{t+1} \|_2^2 \mid \Delta_t \right] &= \Exp\left[ \|\Delta_t - \eta g_s(\Delta_t) \|_2^2 \mid \Delta_t \right]  \nonumber \\
																		&= \Exp\left[ \|\Delta_t\|_2^2 \mid \Delta_t \right] + \eta^2 \cdot  \Exp\left[\|g_s(\Delta_t)\|_2^2 \mid \Delta_t \right] - 2 \eta \cdot \Exp \left[ g_s(\Delta_t)^\top \Delta_t \mid \Delta_t \right] \nonumber \\
																		&= \|\Delta_t\|_2^2 + \eta^2 \cdot  \Exp\left[\|g_s(\Delta_t)\|_2^2 \mid \Delta_t \right] - 2 \eta \cdot \nabla f(\Delta_t)^\top \Delta_t \nonumber \\
																		&\stackrel{(i)}{\leq} \|\Delta_t\|_2^2 + \eta^2 \cdot \left(A \cdot \|\Delta_t\|_2^2 + B\right) - 2 \eta \cdot \alpha \|\Delta_t\|_2^2 \nonumber \\
																		&= (1- 2 \alpha \eta + A \eta^2 ) \| \Delta_t \|_2^2 + \eta^2  B.
\end{align}
where $(i)$ is due to assumptions $(F_1)$ and $(G_2)$ of Theorem 1.
Taking expectations for every step $t = 1, \dots $ over the whole history, we obtain the recursion:
\begin{align*}
\Exp\left[ \| \Delta_{t+1} \|_2^2 \right] &\leq (1- 2 \alpha \eta + A \eta^2 )^{t-1} \| \Delta_1 \|_2^2 + \eta^2 B \cdot \sum_{i = 0}^{t-1} (1 - 2\alpha \eta + A \eta^2)^i \\
													 &= (1- 2 \alpha \eta + A \eta^2 )^{t-1} \| \Delta_1 \|_2^2 + \eta^2 B \cdot \tfrac{1 - (1 - 2\alpha \eta + A \eta^2)^t}{2\alpha \eta - A \eta^2} \\
													 &\leq (1- 2 \alpha \eta + A \eta^2 )^{t-1} \| \Delta_1 \|_2^2 + \tfrac{\eta B}{2\alpha - A \eta}.						
\end{align*}
\end{proof}

The following lemma states that $ \Exp\left[ \| \Delta_t\|_2^4 \right] = O(\eta^2) $ as $t \to \infty$ and $\eta \to 0$.

\begin{lemma}{\label{lem:tas_01}}
For data dependent, positive constants $\alpha, A, B, C, D$ according to assumptions $(F_1)$, $(G_1)$, $(G_2)$ in Theorem 1, we have:
\begin{align}
\Exp[ \| \Delta_t\|_2^4] \leq & (1-4\alpha\eta + A ( 6 \eta^2 + 2 \eta^3  ) + B(3 \eta  + \eta^2) + C(  2 \eta^3+ \eta^4) )^{t-1} \|\Delta_1\|_2^4  \nonumber \\
	& \quad \quad \quad ~~~+ \frac{B  (   3 \eta^2 +  \eta^3  ) + D (  2 \eta^2 + \eta^3) }{4\alpha - A (   6 \eta + 2 \eta^2  ) -B(3 + \eta)  - C(  2 \eta^2+ \eta^3)}.
\end{align}
\end{lemma}
\begin{proof}
Given $\Delta_t$, we have the following sets of (in)equalities:
\begin{align}
& \Exp \left[ \| \Delta_{t+1} \|_2^4 \mid \Delta_t \right] \nonumber \\
= & \Exp \left[ \|\Delta_t - \eta g_s(\Delta_t) \|_2^4 \mid \Delta_t \right]  \nonumber \\
= &  \Exp\left[ (\|\Delta_t\|_2^2 - 2 \eta \cdot g_s(\Delta_t)^\top \Delta_t + \eta^2 \|g_s(\Delta_t)\|_2^2 )^2 \mid \Delta_t \right]  \nonumber \\
= & \Exp \big[
\|\Delta_t\|_2^4 + 4 \eta^2 (g_s(\Delta_t)^\top \Delta_t)^2 + \eta^4 \|g_s(\Delta_t)\|_2^4 - 4 \eta \cdot g_s(\Delta_t)^\top \Delta_t \|\Delta_t\|_2^2 \nonumber
	\\ &\quad \quad \quad \quad + 2 \eta^2 \cdot \|g_s(\Delta_t)\|_2^2 \|\Delta_t\|_2^2 - 4 \eta^3 \cdot g_s(\Delta_t)^\top \Delta_t \|g_s(\Delta_t)\|_2^2
\mid \Delta_t \big] \nonumber \\
\stackrel{(i)}{\leq} & \Exp\big[
\|\Delta_t\|_2^4 + 4 \eta^2 \cdot \|g_s(\Delta_t)\|_2^2 \cdot \|\Delta_t\|_2^2  +  \eta^4 \|g_s(\Delta_t)\|_2^4
	- 4 \eta \cdot g_s(\Delta_t)^\top \Delta_t \|\Delta_t\|_2^2 \nonumber
	\\ &\quad \quad \quad \quad + 2 \eta^2 \cdot \|g_s(\Delta_t)\|_2^2 \cdot  \|\Delta_t\|_2^2  + 2 \eta^3 \cdot (\|g_s(\Delta_t)\|_2^2  + \|\Delta_t\|_2^2 ) \cdot \|g_s(\Delta_t)\|_2^2
\mid \Delta_t \big] \nonumber \\
\stackrel{(ii)}{\leq} & \Exp \left[
\|\Delta_t\|_2^4 + (  2 \eta^3+ \eta^4) \|g_s(\Delta_t)\|_2^4
	+ (   6 \eta^2 + 2 \eta^3  ) \|g_s(\Delta_t)\|_2^2 \|\Delta_t\|_2^2
\mid \Delta_t \right] - 4 \alpha \eta \|\Delta_t\|_2^4   \nonumber \\
\stackrel{(iii)}{\leq}& (1-4\alpha\eta) \|\Delta_t\|_2^4 + (   6 \eta^2 + 2 \eta^3  ) (A \|\Delta_t\|_2^2  + B) \|\Delta_t\|_2^2
	+ (  2 \eta^3+ \eta^4) (C \|\Delta_t\|_2^4  + D) \nonumber \\
=& (1-4\alpha\eta + A (   6 \eta^2 + 2 \eta^3  ) + C(  2 \eta^3+ \eta^4) ) \|\Delta_t\|_2^4  +
	B (   6 \eta^2 + 2 \eta^3  ) \|\Delta_t\|_2^2 + D (  2 \eta^3+ \eta^4)   \nonumber \\
\stackrel{(iv)}{\leq}& (1-4\alpha\eta + A (   6 \eta^2 + 2 \eta^3  ) + C(  2 \eta^3+ \eta^4) ) \cdot  \|\Delta_t\|_2^4  +
	B (   3 \eta +  \eta^2  ) (\eta^2 + \|\Delta_t\|_2^4) + D (  2 \eta^3+ \eta^4)   \nonumber \\
=& ( 1-4\alpha\eta + A (   6 \eta^2 + 2 \eta^3  ) + B (   3 \eta +  \eta^2  ) + C(  2 \eta^3+ \eta^4)  )  \cdot \|\Delta_t\|_2^4
	+ B \eta^2 (   3 \eta +  \eta^2  ) + D (  2 \eta^3+ \eta^4),
\end{align}
where $(i)$ is due to $(g_s(\Delta_t)^\top \Delta_t)^2 \leq \|g_s(\Delta_t)\|_2^2 \cdot \|\Delta_t\|_2^2$
and $-2 g_s(\Delta_t)^\top \Delta_t \leq \|g_s(\Delta_t)\|_2^2  + \|\Delta_t\|_2^2$, $(ii)$ is due to assumptions $(G_1)$ and $(F_1)$ in Theorem 1,
$(iii)$ is due to assumptions $(G_2)$ and $(G_3)$ in Theorem 1,
and $(iv)$ is due to $2 \eta  \|\Delta_t\|_2^2 \leq \eta^2  + \|\Delta_t\|_2^4$.
Similar to the proof of the previous lemma, applying the above rule recursively and w.r.t. the whole history of estimates, we obtain:
\begin{align*}
\Exp \left[ \| \Delta_{t+1} \|_2^4 \right] &\leq (1-4\alpha\eta + A (   6 \eta^2 + 2 \eta^3  ) + B(3 \eta  + \eta^2) + C(  2 \eta^3+ \eta^4) )^{t-1} \|\Delta_1\|_2^4  \nonumber \\
	&\quad + \left(B \eta^2 (   3 \eta +  \eta^2  ) + D (  2 \eta^3 + \eta^4)\right)  \cdot \sum_{i = 0}^{t-1} \left(1-4\alpha\eta + A (   6 \eta^2 + 2 \eta^3  ) + B(3 \eta  + \eta^2) + C(  2 \eta^3+ \eta^4) \right)^i \\
	&\leq (1-4\alpha\eta + A (   6 \eta^2 + 2 \eta^3  ) + B(3 \eta  + \eta^2) + C(  2 \eta^3+ \eta^4) )^{t-1} \|\Delta_1\|_2^4  \nonumber \\
	&\quad + \frac{B \eta^2 (   3 \eta +  \eta^2  ) + D (  2 \eta^3 + \eta^4)}{4\alpha\eta - A (   6 \eta^2 + 2 \eta^3  ) - B(3 \eta  + \eta^2) - C(  2 \eta^3+ \eta^4)},
\end{align*}
which is the target inequality, after simple transformations.
\end{proof}

We know that:
\begin{align*}
\Delta_{t} = \Delta_{t-1} - \eta g_s(\Delta_{t-1})
\end{align*}
Using the Taylor expansion formula around the point $\Delta_{t-1}$ and using the assumption that $\widehat{\theta} = 0$, we have:
\begin{align*}
f(\Delta_{t-1}) = f(\htheta) + \nabla f(\htheta)^\top \Delta_{t-1} + \frac{1}{2} \Delta_{t-1}^\top H \Delta_{t-1} + R(\Delta_{t-1})
\end{align*}
Taking further the gradient w.r.t. $\Delta_{t-1}$ in the above expression, we have:
\begin{align*}
\nabla f(\Delta_{t-1}) = H \Delta_{t-1} + \nabla R(\Delta_{t-1})
\end{align*}
Using the identity $g_s(\Delta_{t-1}) = \nabla f(\Delta_{t-1}) + e_{t-1}$, our SGD recursion can be re-written as:
\begin{align}
\Delta_t& = \left(I - \eta H\right) \Delta_{t-1} - \eta \left(\nabla R(\Delta_{t-1}) + e_{t-1} \right) = \left(I-\eta H\right)^{t-1} \Delta_1 -  \eta \sum_{i =1}^{t-1} \left(I-\eta H\right)^{t-1-i} \left(  e_i + \nabla R(\Delta_i)\right). \label{eq:tas_01}
\end{align}

For $t\geq 2$ and since: $\btheta = \btheta - \htheta = \bDelta_t = \tfrac{1}{t} \sum_{i = 1}^t (\theta_i - \htheta) = \frac{1}{t} \sum_{i = 1}^t \Delta_i$, we get:
\begin{align}
t (\btheta - \htheta) = \sum_{i=1}^t \Delta_i &= \sum_{i = 1}^t  \left(I-\eta H\right)^{i-1} \Delta_1 - \eta \sum_{j=1}^{t} \sum_{i=1}^j (I-\eta H)^{j-1-i}  ( e_i + \nabla R(\Delta_i) ) \nonumber \\
			      &\stackrel{(i)}{=} \left(I - (I - \eta H)^{t}\right) \tfrac{ H^{-1} }{\eta}  \Delta_1 - \eta \sum_{j=1}^{t} \sum_{i=1}^j (I-\eta H)^{j-1-i}  ( e_i + \nabla R(\Delta_i) ). \label{eq:tas_00}
\end{align}
where $(i)$ holds due to the assumption that the eigenvalues of $I - \eta H$ satisfy $|\lambda_i(I -\eta H)| < 1$, and thus, the geometric series of matrices:
$\sum_{k = 0}^{n-1} T^k = (I - T)^{-1}(I - T^n)$,
is utilized above. In our case, $T = (I - \eta H)$.

For the latter term in \eqref{eq:tas_00}, using a variant of Abel's sum formula, we have:
\begin{align}
\eta \sum_{j=1}^{t} \sum_{i=1}^j (I-\eta H)^{j-1-i}  ( e_i + \nabla R(\Delta_i) ) &= \eta \sum_{j=0}^{t-1} \sum_{i=1}^j (I-\eta H)^{j-i}  ( e_i + \nabla R(\Delta_i) ) \\
&=  \eta \sum_{i=1}^{t-1} \left( \sum_{j=0}^{t-i-1}  (I-\eta H)^{j} \right) (  e_i + \nabla R(\Delta_i))   \nonumber \\
&= \sum_{i=1}^{t-1} \left(I - (I-\eta H)^{t-i}\right) H^{-1} (  e_i + \nabla R(\Delta_i))  \nonumber \\
&=  H^{-1} \sum_{i=1}^{t-1} e_i + H^{-1}\sum_{i=1}^{t-1} \nabla R(\Delta_i) - H^{-1} \sum_{i=1}^{t-1} (I-\eta H)^{t-i} (e_i + \nabla R(\Delta_i) )  \nonumber \\
\stackrel{(i)}{=}& H^{-1} \sum_{i=1}^{t-1} e_i + H^{-1}\sum_{i=1}^{t-1} \nabla R(\Delta_i) + \tfrac{H^{-1}}{\eta} (I-\eta H) (\Delta_t - (I-\eta H)^{t-1} \Delta_1 ),
\end{align}
where $(i)$ follows from the fact ${ \sum_{i=1}^{t-1} (I-\eta H)^{t-i} (e_i + \nabla R(\Delta_i) ) =  (I-\eta H) \frac{1}{\eta} (\Delta_t - (I-\eta H)^{t-1} \Delta_1 ) }$, based on the expression \eqref{eq:tas_01}.

The above combined lead to:
\begin{align}
\sqrt{t} \bDelta_t  = \underbrace{\tfrac{1}{\sqrt{t}} (I - (I - \eta H)^{t}) \tfrac{ H^{-1} }{\eta}  \Delta_1}_{\varphi_1} \underbrace{-\tfrac{1}{\sqrt{t}}  H^{-1} \sum_{i=1}^{t-1} e_i}_{\varphi_2} \underbrace{-\tfrac{1}{\sqrt{t}} H^{-1}\sum_{i=1}^{t-1} \nabla R(\Delta_i)}_{\varphi_3} \underbrace{-\tfrac{1}{\sqrt{t}} \tfrac{H^{-1}}{\eta} (I-\eta H) (\Delta_t - (I-\eta H)^{t-1} \Delta_1 )}_{\varphi_4}. \label{eq:rn:00}
\end{align}

For the main result of the theorem, we are interested in the following quantity:
\begin{align*}
\left\| t \Exp[ (\btheta_t - \htheta)( \btheta_t - \htheta)^\top] - H^{-1} G H^{-1} \right\|_2 = \left\| t \Exp[ \bDelta_t \bDelta_t^\top] - H^{-1} G H^{-1} \right\|_2
\end{align*}

Using the $\varphi_i$ notation, we have $\Exp[t\bDelta_t\bDelta_t] = \Exp[(\varphi_1+\varphi_2+\varphi_3+\varphi_4)(\varphi_1+\varphi_2+\varphi_3+\varphi_4)^\top]$.
Thus, we need to bound:
\begin{align}
&\left\| t \Exp[ (\btheta_t - \htheta)( \btheta_t - \htheta)^\top] - H^{-1} G H^{-1} \right\|_2 = \left\|  \Exp[(\varphi_1+\varphi_2+\varphi_3+\varphi_4)(\varphi_1+\varphi_2+\varphi_3+\varphi_4)^\top] - H^{-1} G H^{-1} \right\|_2 \nonumber  \\
&= \left\|\Exp[\varphi_2 \varphi_2^\top] -H^{-1} G H^{-1}
	+ \Exp[\varphi_2(\varphi_1+\varphi_4+\varphi_3)^\top +(\varphi_1+\varphi_4+\varphi_3)\varphi_2^\top
	+(\varphi_1+\varphi_4+\varphi_3)(\varphi_1+\varphi_4+\varphi_3)^\top] \right\|_2 \nonumber \\
&\leq \left\|\Exp\left[\varphi_2 \varphi_2^\top\right] -H^{-1} G H^{-1} \right\|_2 
	+ \left\|\Exp[\varphi_2(\varphi_1+\varphi_4+\varphi_3)^\top]\right\|_2 +\left\|\Exp[(\varphi_1+\varphi_4+\varphi_3)\varphi_2^\top]\right\|_2
	\nonumber \\
	&\quad \quad \quad \quad \quad \quad \quad \quad \quad \quad \quad \quad ~~~~+\left\|\Exp[(\varphi_1+\varphi_4+\varphi_3)(\varphi_1+\varphi_4+\varphi_3)^\top] \right\|_2 \nonumber \\
&\stackrel{(i)}{\lesssim} \|\Exp[\varphi_2 \varphi_2^\top] -H^{-1} G H^{-1} \|_2
	+ \sqrt{\Exp[\|\varphi_2\|_2^2] (\Exp[\|\varphi_1\|_2^2] + \Exp[\|\varphi_4\|_2^2] + \Exp[\|\varphi_3\|_2^2])}
	+ \Exp[\|\varphi_1\|_2^2] + \Exp[\|\varphi_4\|_2^2] + \Exp[\|\varphi_3\|_2^2] \label{eq:tas_02}
\end{align}
where $(i)$ is due to the successive use of the AM-GM rule:
\begin{align}
\| \Exp[ a b^\top ] \|_2 \leq \sqrt{ \Exp[\|a\|_2^2] \Exp[\|b\|_2^2] } \leq \frac{1}{2} \Exp[\|a\|_2^2] + \Exp[\|b\|_2^2].
\end{align}
for two $p$-dimensional random vectors $a$ and $b$.
Indeed, for any fixed unit vector $u$ we have
$\| \Exp[ a b^\top ] u  \|_2 = \| \Exp[a(b^\top u)] \|_2 \leq  \Exp[\|a\|_2 |b^\top u|]  \leq \Exp[\|a\|_2 \|b\|_2] \leq \sqrt{ \Exp[\|a\|_2^2] \Exp[\|b\|_2^2] }$.
We used the fact  $\|\Exp[x]\|_2 \leq \Exp[\|x\|_2]$  because $\|x\|_2$ is convex.
Here also, the $\lesssim$ hides any constants appearing from applying successively the above rule.

Therefore, to proceed bounding the quantity of interest, we need to bound the terms $\Exp[\|\varphi_i\|_2^2]$.
In the statement of the theorem we have $\Delta_1 = 0$---however similar bounds will hold if $\|\Delta_1\|_2^2=O(\eta)$;
thus, for each of the above $\varphi_i$ terms we have the following.

\begin{align}
\varphi_1 := \tfrac{1}{\sqrt{t}} (I - (I - \eta H)^{t}) \frac{ H^{-1} }{\eta}  \Delta_1 =0,  \quad \quad \quad \quad  \text{(due to $\Delta_1 = 0$)}
\end{align}

\begin{align}
\Exp[\|\varphi_4\|_2^2] &:= \Exp \left[ \left\|  -\tfrac{1}{\sqrt{t}} \tfrac{H^{-1}}{\eta} (I-\eta H) (\Delta_t - (I-\eta H)^{t-1} \Delta_1 ) \right\|_2^2  \right] \nonumber \\
&\leq \Exp \left[ \|H^{-1}\|_2^2 \cdot \|I - \eta H\|_2^2 \cdot \tfrac{1}{\eta^2 t} \|  \Delta_t \|_2^2  \right] \stackrel{(i)}{\leq} \frac{1 - \eta \lambda_U}{\lambda_L } \cdot \tfrac{1}{\eta^2 t} \cdot \Exp[ \|\Delta_t\|_2^2] \nonumber \\
&\stackrel{(ii)}{\leq} \frac{1 - \eta \lambda_U}{\lambda_L } \frac{1}{\eta^2 t} \left((1 - 2\alpha \eta + A \eta^2)^{t-1} \|\Delta_1\|_2^2 + \frac{B \eta}{2\alpha - A \eta}\right) \nonumber \\
&= \frac{1 - \eta \lambda_U}{\lambda_L } \frac{B}{t \eta(2\alpha - A \eta)} \nonumber \\
&= O\left(\frac{1}{t \eta}\right) 
\end{align}
where $(i)$ is due to Assumption $(F_4)$, $(ii)$ is due to Lemma \ref{lem:tas_00}, and we used in several places the fact that $\Delta_1 = 0$.

\begin{align}
\Exp[\|\varphi_3\|_2^2] &:= \Exp\left[ \left\| -\tfrac{1}{\sqrt{t}} H^{-1} \sum_{i=1}^{t-1} \nabla R(\Delta_i) \right\|_2^2 \right]  \leq \Exp \left[ \tfrac{1}{t} \cdot \|H^{-1}\|_2^2 \cdot \left\|\sum_{i=1}^{t-1} \nabla R(\Delta_i) \right\|_2^2 \right]  \stackrel{(i)}{\leq} \Exp \left[ \tfrac{1}{\lambda_L} \tfrac{1}{t} \left(\sum_{i=1}^{t-1} \|\nabla R(\Delta_i)\|_2\right)^2 \right] \nonumber \\
&\stackrel{(ii)}{\leq} \Exp \left[ \tfrac{E^2}{\lambda_L \cdot t} \left(\sum_{i=1}^{t-1} \|\Delta_i\|_2^2\right)^2 \right] \stackrel{(iii)}{\leq} \tfrac{E^2}{\lambda_L \cdot t} (t-1) \cdot \Exp \left[ \sum_{i=1}^{t-1} \|\Delta_i\|_2^4 \right] \nonumber \\
&\leq  \tfrac{E^2}{\lambda_L \cdot } (t-1) \sum_{i=1}^{t-1} \left((1-4\alpha\eta + A (   6 \eta^2 + 2 \eta^3  ) + C(  2 \eta^3+ \eta^4) )^{t-1} \|\Delta_1\|_2^4
	 + \frac{B  (   3 \eta^2 +  \eta^3  ) + D (  2 \eta^2 + \eta^3) }{4\alpha - A (   6 \eta + 2 \eta^2  ) - C(  2 \eta^2+ \eta^3)} \right)  \nonumber \\
&\stackrel{(iv)}{=} \tfrac{E^2}{\lambda_L}  \tfrac{(t-1)^2}{t} \frac{B  (   3 \eta^2 +  \eta^3  ) + D (  2 \eta^2 + \eta^3) }{4\alpha - A (   6 \eta + 2 \eta^2  ) - C(  2 \eta^2+ \eta^3)} \stackrel{(v)}{=} O(t \eta^2).
\end{align}
where $(i)$ is due to Assumption $(F_4)$ and due to $|\sum_{i} \chi_i|^2 \leq \sum_{i} |\chi_i|^2$, $(ii)$ is due to Assumption $(F_3)$ on bounded remainder, 
$(iii)$ is due to the inequality $\left(\sum_{i = 1}^n \chi_i^2\right)^2 \leq n \cdot \sum_{i = 1}^n \chi_i^2$, $(iv)$ is due to $\Delta_1 = 0$, $(v)$ is due to $\eta$ being an small constant compared to $\alpha$ and thus $\frac{B  (   3 \eta^2 +  \eta^3  ) + D (  2 \eta^2 + \eta^3) }{4\alpha - A (   6 \eta + 2 \eta^2  ) - C(  2 \eta^2+ \eta^3)} = \frac{O(\eta^2)}{O(1)}$.

\begin{align}
\Exp[\|\varphi_2\|_2^2] &:= \Exp\left[ \left\| - \frac{1}{\sqrt{t}}  H^{-1} \sum_{i=1}^{t-1} e_i \right\|_2^2 \right] \stackrel{(i)}{=} \frac{1}{t} \sum_{i=1}^{t-1} \Exp[ \|H^{-1} e_i\|_2^2 ]  \stackrel{(ii)}{\leq} \tfrac{\lambda_U}{t}  \sum_{i=1}^{t-1} \Exp[ \| e_i\|_2^2 ] \nonumber \\
&= \tfrac{\lambda_U}{t}  \sum_{i=1}^{t-1} \Exp[ \| g_s(\Delta_i) - \nabla f(\Delta_i) \|_2^2 ] \leq \tfrac{2 \lambda_U}{t} \left(\sum_{i=1}^{t-1} \Exp[ \| g_s(\Delta_i) \|_2^2 ] + \sum_{i=1}^{t-1} \Exp[ \| \nabla f(\Delta_i) \|_2^2 ] \right) \nonumber \\
&\stackrel{(iii)}{\leq} \tfrac{2\lambda_U}{t} \left((t-1)B+ (A+L^2) \sum_{i=1}^{t-1} \Exp[\|\Delta_i\|_2^2] \right) \nonumber \\ 
&\stackrel{(iv)}{\leq} \tfrac{2\lambda_U}{t} \left((t-1)B+ (A+L^2) \sum_{i=1}^{t-1} \left((1 - 2\alpha \eta + A \eta^2)^{t-1} \|\Delta_1\|_2^2
	+ \frac{B \eta}{2\alpha - A \eta}\right) \right) \nonumber \\
&= \tfrac{2\lambda_U (t-1)}{t} \left(B +  (A+L^2) \frac{B \eta}{2\alpha - A \eta} \right) =O(1),
\end{align}
where $(i)$ is due to $ {\Exp[(H^{-1} e_i)^\top  H^{-1} e_j ] = 0 }$ for ${ i \neq j }$, $(ii)$ is due to Assumption $(F_4)$, $(iii)$ is due to Assumptions $(F_2)$ and $(G_2)$, $(iv)$ is due to Lemma \ref{lem:tas_00}.

Finaly, for the term $\Exp[\varphi_2 \varphi_2^\top]$, we have
\begin{align}
\Exp[\varphi_2 \varphi_2^\top] = \Exp\left[ \left(- \tfrac{1}{\sqrt{t}}  H^{-1} \sum_{i=1}^{t-1} e_i \right) \left(- \tfrac{1}{\sqrt{t}}  H^{-1} \sum_{i=1}^{t-1} e_i\right)^\top \right] = \tfrac{1}{t} H^{-1}  \left(\sum_{i=1}^{t-1} \Exp[e_i e_i^\top]\right)  H^{-1}.
\end{align}
and thus:
\begin{align*}
\left\|\Exp[\varphi_2 \varphi_2^\top] - H^{-1} G H^{-1}\right\|_2 &= \left\|\tfrac{1}{t} H^{-1}  \left(\sum_{i=1}^{t-1} \Exp[e_i e_i^\top]\right)  H^{-1} - H^{-1} G H^{-1}\right\|_2 \\
&= \left\| \tfrac{1}{t} H^{-1}  \left(\sum_{i=1}^{t-1} \Exp[e_i e_i^\top] - G + G\right)  H^{-1} - H^{-1} G H^{-1} \right\|_2 \\
&= \left\| \tfrac{1}{t} H^{-1}  \left(\sum_{i=1}^{t-1} \Exp[e_i e_i^\top] - G\right)  H^{-1} - \tfrac{t-1}{t} \cdot H^{-1} G H^{-1} \right\|_2 \\
&\leq \tfrac{1}{t} H^{-1}  \left(\sum_{i=1}^{t-1} \left\| \Exp[e_i e_i^\top] - G \right\|_2 \right)  H^{-1} + \tfrac{t-1}{t} \left\|  H^{-1} G H^{-1} \right\|_2 \\
\end{align*}

For each term $\left\| \Exp[e_i e_i^\top] - G \right\|_2, \forall i$, we have
\begin{small}
\begin{align}
\| \Exp[e_i e_i^\top] - G  \|_2 &= \left\| \Exp[g_s(\Delta_i) g_s(\Delta_i)^\top]  - \Exp[\nabla f(\Delta_i)\nabla f(\Delta_i)^\top] -G \right\|_2 \nonumber \\
&= \left\| \Exp[(g_s(\Delta_i) - \nabla f(\Delta_i)) (g_s(\Delta_i) - \nabla f(\Delta_i))^\top] -G \right\|_2 \nonumber \\
&= \left\| \Exp[g_s(\Delta_i)g_s(\Delta_i)^\top] - \Exp[g_s(\Delta_i)\nabla f(\Delta_i)^\top] - \Exp[\nabla f(\Delta_i) g_s(\Delta_i)^\top] + \Exp[\nabla f(\Delta_i) \nabla f(\Delta_i)^\top] -G \right\|_2 \nonumber \\
&\stackrel{(i)}{=} \left\| \Exp[g_s(\Delta_i)g_s(\Delta_i)^\top] - 2\Exp[\nabla f(\Delta_i)\nabla f(\Delta_i)^\top]+ \Exp[\nabla f(\Delta_i) \nabla f(\Delta_i)^\top] -G \right\|_2 \nonumber \\
&\stackrel{(ii)}{\leq} \Exp[ \|\nabla f(\Delta_i)\|_2^2 ] + \Exp\left[  A_1 \| \Delta_i\|_2 + A_2 \| \Delta_i\|_2^2 + A_3 \| \Delta_i\|_2^3 + A_4  \| \Delta_i\|_2^4 \right] \nonumber \\
&\stackrel{(iii)}{\leq} L^2 \Exp\left[\| \Delta_i\|_2^2\right] + A_1 \sqrt{\Exp\left[\| \Delta_i\|_2^2\right]} + A_2 \Exp\left[\| \Delta_i\|_2^2\right] + \frac{A_3}{2} \Exp\left[\| \Delta_i\|_2^2 + \| \Delta_i\|_2^4\right] + A_4 \Exp\left[\| \Delta_i\|_2^4\right] \nonumber \\
&=  A_1 \sqrt{\Exp[\| \Delta_i\|_2^2]} + \left(L^2 + A_2 + \tfrac{A_3}{2}\right) \Exp[\| \Delta_i\|_2^2] + \left(\tfrac{A_3}{2} + A_4\right) \Exp[\| \Delta_i\|_2^4] \nonumber \\
&\stackrel{(iv)}{\leq} A_1 \sqrt{ (1 - 2\alpha \eta + A \eta^2)^{t-1} \|\Delta_1\|_2^2
	+ \frac{B \eta}{2\alpha - A \eta} }
	+\left(L^2 + A_2 + \frac{A_3}{2}\right) \left( (1 - 2\alpha \eta + A \eta^2 )^{t-1} \|\Delta_1\|_2^2
	+ \frac{B \eta}{2\alpha - A \eta} \right)  \nonumber \\
	&\quad + \left(\frac{A_3}{2} + A_4 \right) \left((1-4\alpha\eta + A (   6 \eta^2 + 2 \eta^3  ) + C(  2 \eta^3+ \eta^4) )^{t-1} \|\Delta_1\|_2^4
	 + \frac{B  (   3 \eta^2 +  \eta^3  ) + D (  2 \eta^2 + \eta^3) }{4\alpha - A (   6 \eta + 2 \eta^2  ) - C(  2 \eta^2+ \eta^3)} \right) \nonumber \\
&=  A_1 \sqrt{  \frac{B \eta}{2\alpha - A \eta} }
	+ \left(L^2 + A_2 + \frac{A_3}{2}\right) \frac{B \eta}{2\alpha - A \eta}
	+ \left(\frac{A_3}{2} + A_4 \right)  \frac{B  (   3 \eta^2 +  \eta^3  ) + D (  2 \eta^2 + \eta^3) }{4\alpha - A (   6 \eta + 2 \eta^2  ) - C(  2 \eta^2+ \eta^3)}.
\end{align}
\end{small}
where $(i)$ is due to Assumption $(G_1)$, $(ii)$ is due to Cauchy-Schwartz inequality and Assumption $(G_4)$, $(iii)$ is due to Assumption $(F_2)$, $(iv)$ is due to Lemmas \ref{lem:tas_00}-\ref{lem:tas_01}.

Then, we have:
\begin{small}
\begin{align}
&\left\|\Exp[\varphi_2 \varphi_2^\top] - H^{-1} G H^{-1}\right\|_2 \nonumber \\ 
&\stackrel{(i)}{\leq} \tfrac{t-1}{t} \left\|  H^{-1} G H^{-1} \right\|_2 +  \tfrac{t-1}{\lambda_L^2 \cdot t} \left( A_1 \sqrt{  \frac{B \eta}{2\alpha - A \eta} }
	+ \left(L^2 + A_2 + \tfrac{A_3}{2}\right) \frac{B \eta}{2\alpha - A \eta}
	+ \left(\tfrac{A_3}{2} + A_4\right)  \frac{B  (   3 \eta^2 +  \eta^3  ) + D (  2 \eta^2 + \eta^3) }{4\alpha - A (   6 \eta + 2 \eta^2  ) - C(  2 \eta^2+ \eta^3)} \right) \nonumber \\
&= O(\sqrt{\eta}).
\end{align}
\end{small}
where $(i)$ is due to Assumption, and $(ii)$ is after removing constants and observing that the dominant term in the second part is $O(\sqrt{\eta})$.
Combining all the above in \eqref{eq:tas_02}, we obtain:
\begin{align*}
\left\| t \Exp[ (\btheta_t - \htheta)( \btheta_t - \htheta)^\top] - H^{-1} G H^{-1} \right\|_2 \lesssim \sqrt{\eta} + \sqrt{\frac{1}{t \eta} + t \eta^2}.
\end{align*}

\hfill $\blacksquare$

\subsection{Proof of Corollary \ref{cor:sgd-stat-inf-linear-regression}}

\begin{proof}[Proof of Corollary \ref{cor:sgd-stat-inf-linear-regression}]

Here we use the same notations as the proof of Theorem \ref{thm:sgd-strongly-convex-lipschitz-stat-inf}.
Because linear regression satisfies $\nabla f(\theta) - H(\theta - \htheta) = 0$,
we do not have to consider the Taylor remainder term in our analysis.
And we do not need 4-th order bound for SGD.
Due to the fact that the quadratic function is strongly convex, we have $\Delta^\top \nabla f(\Delta+\htheta) \geq \lambda_L \| \Delta \|_2^2 $.

By sampling with replacement, we have
\begin{align}
\Exp[\|g_s(\theta_t)\|_2^2 \mid \theta_t] &= \|\nabla f(\theta_t)\|_2^2 + \Exp[\|e_t\|_2^2  \mid \theta_t] \nonumber \\
&= \|\nabla f(\theta_t)\|_2^2 + \tfrac{1}{S}\left(\tfrac{1}{n} \sum  \| \nabla f_i(\theta_t)\|_2^2 - \|\nabla f(\theta_t)\|_2^2\right) \nonumber \\
&\leq L^2 (1 - \tfrac{1}{S})\| \Delta_t \|_2^2 + \tfrac{1}{S} \tfrac{1}{n} \sum \|x_i(x_i^\top \theta_t - y_i)\|_2^2 \nonumber \\
&= L^2 (1 - \tfrac{1}{S})\| \Delta_t \|_2^2 + \tfrac{1}{S} \tfrac{1}{n} \sum \|x_ix_i^\top \Delta_t + x_i x_i^\top \htheta- y_i x_i\|_2^2 \nonumber \\
&\leq  L^2 (1 - \tfrac{1}{S})\| \Delta_t \|_2^2 +
	2 \tfrac{1}{S} \tfrac{1}{n} \sum ( \|x_ix_i^\top \Delta_t\|_2^2 + \|x_i x_i^\top \htheta- y_i x_i\|_2^2)
	\nonumber \\
&\leq \left(L^2 (1 - \tfrac{1}{S}) + 2 \tfrac{1}{S} \tfrac{1}{n} \sum \|x_i\|_2^4 \right) \| \Delta_t \|_2^2
	+ 2 \tfrac{1}{S} \tfrac{1}{n} \sum \|x_i x_i^\top \htheta- y_i x_i\|_2^2.
\end{align}

We also have
\begin{align}
\left\| \Exp[g_s(\theta) g_s(\theta)^\top \mid \theta] -  G \right\|_2 &= \left\|  \tfrac{1}{S} \tfrac{1}{n} \sum \nabla f_i(\theta) f_i(\theta)^\top - \nabla f(\theta) \nabla f(\theta)^\top -  G  \right\|_2 \nonumber \\
&\leq  \|\nabla f(\theta)\|_2^2 + \tfrac{1}{S} \left\| \tfrac{1}{n} \sum \nabla f_i(\theta) f_i(\theta)^\top - G \right\|_2 \nonumber  \\
&\leq  \|\nabla f(\theta)\|_2^2 + \tfrac{1}{S} \left\| \tfrac{1}{n} \sum  (g_i + H_i \Delta) (g_i + H_i \Delta)^\top - G \right\|_2 \nonumber \\
&\leq  \|\nabla f(\theta)\|_2^2 + \tfrac{1}{S} \left\| \tfrac{1}{n} \sum H_i \Delta g_i^\top + g_i \Delta^\top H_i + H_i \Delta \Delta^\top H_i \right\|_2 \nonumber \\
&\leq  \|\nabla f(\theta)\|_2^2 + \tfrac{1}{S} \left( \tfrac{2}{n} \|H_i\|_2 \|g_i\|_2\right) \| \Delta \|_2 + \tfrac{1}{S} \left( \tfrac{1}{n} \sum \|H_i\|_2^2\right) \| \Delta \|_2^2 \nonumber \\
&\leq \tfrac{1}{S}\left( \tfrac{2}{n} \|H_i\|_2 \|g_i\|_2\right) \| \Delta \|_2+  \left(L^2 + \tfrac{1}{S} \tfrac{1}{n} \sum \|H_i\|_2^2 \right)\| \Delta \|_2^2,
\end{align}
where $g_i = x_i (x_i^\top \htheta - y_i)$ and $H_i = x_i x_i^\top$.

Following Theorem \ref{thm:sgd-strongly-convex-lipschitz-stat-inf}'s proof, we have
\begin{align}
\left\| t \Exp[ (\btheta_t - \htheta)( \btheta_t - \htheta)^\top] - H^{-1} G H^{-1} \right\|_2  \lesssim \sqrt{\eta}+ \frac{1}{\sqrt{t \eta} }.
\end{align}
\end{proof}

\subsection{Proof of Corollary \ref{cor:sgd-stat-inf-logistic-regression}}
\begin{proof}[Proof of Corollary \ref{cor:sgd-stat-inf-logistic-regression}]

Here we use the same notations as the proof of Theorem \ref{thm:sgd-strongly-convex-lipschitz-stat-inf}.
Because $\nabla^2 f(\theta) = \nabla k(\theta) \nabla k(\theta)^\top + (k(\theta) +c ) \nabla^2 k(\theta)$,
$f(\theta) $ is convex.
The following lemma shows that $\nabla f(\theta) = (k(\theta) +c) \nabla k(\theta) $ is Lipschitz.
\begin{lemma}
\begin{align}
\| \nabla f(\theta) \|_2 \leq L \| \Delta \|_2
\end{align}
for some data dependent constant $L$.
\end{lemma}
\begin{proof}

First, because
\begin{align}
\nabla k(\theta) = \tfrac{1}{n} \sum  - \frac{-y_i x_i}{1 + \exp(y_i \theta^\top x_i)},
\end{align}
we have
\begin{align}
\| \nabla k(\theta) \|_2 \leq \tfrac{1}{n} \sum \|x_i\|_2.
\end{align}

Also, we have
\begin{align}
\| \nabla^2 k(\theta) \|_2 = \left\| \tfrac{1}{n} \sum \frac{\exp(y_i \theta^\top x_i)}{(1 + \exp(y_i \theta^\top x_i))^2} x_i  x_i^\top \right\|_2 \leq \tfrac{1}{4n} \sum \|x_i\|_2^2,
\end{align}
which implies
\begin{align}
\| \nabla k(\theta)  \|_2 \leq \tfrac{1}{4n} \sum \|x_i\|_2^2 \|\Delta\|_2.
\end{align}

Further:
\begin{align}
k(\theta) &= \tfrac{1}{n} \sum \log( 1 + \exp(-y_i \Delta^\top x_i  -y_i \htheta^\top x_i))  \nonumber \\
&\leq \tfrac{1}{n} \sum \log(1 + \exp(\|x_i\|_2 \| \Delta \|_2 -y_i \htheta^\top x_i)) \nonumber \\
&\stackrel{(i)}{\leq} \tfrac{1}{n} \sum ( \log(1 + \exp( -y_i \htheta^\top x_i)) + \|x_i\|_2 \| \Delta \|_2 )
\end{align}
where step (i) follows from ${ \log(1 + \exp(a + b))  \leq \log(1 + e^b) + |a| }$.
Thus, we have
\begin{align}
\| \nabla f(\theta) \|_2 &= \left\|  (k(\theta) +c) \nabla k(\theta) \right\|_2 \leq k( \theta ) \| \nabla k(\theta) \|_2  + c \|  \nabla k(\theta) \|_2 \nonumber \\
&\leq  \left( c+\tfrac{1}{n} \sum  \log(1 + \exp( -y_i \htheta^\top x_i)) \right) \| \nabla k(\theta) \|_2 + \left(\tfrac{1}{n} \sum \|x_i\|_2\right)^2 \| \Delta \|_2,
\end{align}
and we can conclude that $\| \nabla f(\theta) \|_2 \leq L \| \Delta \|_2$ for some data dependent constant $L$.
\end{proof}

Next, we show that $f(\theta)$ has a bounded Taylor remainder.
\begin{lemma}
\begin{align}
\|\nabla f(\theta) - H(\theta - \htheta)\|_2 \leq E \|\theta - \htheta\|_2^2,
\end{align}
for some data dependent constant $E$.
\end{lemma}
\begin{proof}

Because $\nabla f(\theta) = (k(\theta) + c) \nabla k(\theta)$, we know that $\|\nabla f(\theta) \|_2 = O(\|\Delta\|_2)$ when $\|\Delta\|_2=\Omega(1)$ where the constants are data dependent.
Because $f(\theta)$ is infinitely differentiable, by the Taylor expansion we know that
$\|\nabla f(\theta) - H(\theta - \htheta)\|_2 =O( \|\theta - \htheta\|_2^2) $ when
$\|\Delta\|_2=O(1)$ where the constants are data dependent. 
Combining the above,
we can conclude $\|\nabla f(\theta) - H(\theta - \htheta)\|_2 \leq E \|\theta - \htheta\|_2^2 $ for some data dependent constant $E$.

\end{proof}

In the following lemma, we will show that $\nabla f(\theta)^\top (\theta - \htheta) \geq \alpha \| \theta - \htheta \|_2^2$ for some data dependent constant $\alpha$.
\begin{lemma}
\begin{align}
\nabla f(\theta)^\top (\theta - \htheta) \geq \alpha \| \theta - \htheta \|_2^2,
\end{align}
for some data dependent constant $\alpha$.
\end{lemma}

\begin{proof}
We know that
\begin{align}
\nabla f(\theta)^\top \Delta = (k (\theta) +c) \nabla k(\theta)^\top \Delta.
\end{align}

First, notice that locally (when $\|\Delta\|_2 = O(\frac{\lambda_L}{E})$) we have
\begin{align}
\nabla k(\theta)^\top \Delta \gtrsim \Delta^\top H \Delta \gtrsim \lambda_L \|\Delta\|_2^2,
\end{align}
because of the optimality condition.
This lower bounds $\nabla f(\theta)^\top (\theta - \htheta)$ when $\|\Delta\|_2 = O(\frac{\lambda_L}{E})$.
Next we will lower bound it when  $\|\Delta\|_2 = \Omega(\frac{\lambda_L}{E})$.

Consider the function for $t \in [0, \infty)$, we have
\begin{align}
g(t) & = \nabla f(\htheta + u t)^\top u t \nonumber \\
&= (k(\htheta + u t) + c) \nabla k(\htheta + u t)^\top u t \nonumber \\
&= k(\htheta + u t) \nabla k(\htheta + u t)^\top u t + c \nabla k(\htheta + u t)^\top u t,
\end{align}
where $u = \frac{\Delta}{\|\Delta\|_2}$.
Because  $k(\theta)$ is convex, $\nabla k(\htheta + u t)^\top u$ is an increasing function in $t$,
thus we have  $\nabla k(\htheta + u t)^\top u = \Omega(\frac{\lambda_L^2	}{E})$ when $t=\Omega(\frac{\lambda_L}{E})$.
And we can deduce $\nabla k(\htheta + u t)^\top u t = \Omega(\frac{\lambda_L^2	}{E} t)$ when  $t=\Omega(\frac{\lambda_L}{E})$.

Similarly, because  $k(\theta)$ is convex, $k(\htheta + u t) $ is an increasing function in $t$.
Its derivative   $\nabla k(\htheta + u t)^\top u = \Omega(\frac{\lambda_L^2	}{E})$ when $t=\Omega(\frac{\lambda_L}{E})$.
So we have $k(\htheta + u t) = \Omega(\frac{\lambda_L^2	}{E}t)$ when $t=\Omega(\frac{\lambda_L}{E})$.

Thus, we have
\begin{align}
k(\htheta + u t) \nabla k(\htheta + u t)^\top u t = \Omega\left(\frac{\lambda_L^4	}{E^2} t^2\right),
\end{align}
when $t=\Omega(\frac{E}{\lambda_L})$.

And we can conclude that $\nabla f(\theta)^\top (\theta - \htheta) \geq \alpha \| \theta - \htheta \|_2^2$  for some data dependent constant $\alpha=\Omega(\min\{\lambda_L , \frac{\lambda_L^4	}{E^2} \})$.
\end{proof}

Next, we will prove properties about $g_s = \Psi_s \Upsilon_s$.
\item \begin{align}
\Exp[\|\Upsilon\|_2^2 \mid \theta] = \frac{1}{S_\Upsilon} \left(  \tfrac{1}{n} \sum \|\nabla k_i(\theta)\|_2^2 -\|\nabla k(\theta)\|_2^2\right) + \|\nabla k(\theta)\|_2^2
\lesssim \tfrac{1}{n} \|x_i\|_2^2
\end{align}

\begin{align}
\Exp[\Psi_s^2] &\stackrel{(i)}{\leq} \tfrac{1}{n} \sum (c+k_i(\theta))^2 \nonumber \\
&= \tfrac{1}{n} \sum \left( c+\log( 1 + \exp(-y_i \htheta^\top x_i - y_i \Delta x_i) ) \right)^2 \nonumber \\
&\stackrel{(ii)}{\lesssim} \tfrac{1}{n} \sum \|x_i\|^2 \|\Delta\|_2^2 + \tfrac{1}{n}\sum(c+\log(1+\exp(-y_i \htheta^\top x_i )))^2,
\end{align}
where $(i)$ follows from ${ \Exp\left[\left(\frac{\sum_{j=1}^S X_j}{S}\right)^2\right] \leq \Exp\left[\frac{\sum_{j=1}^S X_j^2}{S}\right] }$ and $(ii)$ follows from ${ \log(1 + \exp(a + b))  \leq \log(1 + e^b) + |a| }$.

Thus, we have
\begin{align}
\Exp[ \| g_s \|_2^2(\theta)  \mid \theta] = \Exp[\Psi^2 \mid \theta ] \cdot \Exp[\|\Upsilon\|_2^2 \mid \theta] \lesssim A \|\Delta\|_2^2 + B
\end{align}
for some data dependent constants $A$ and $B$.

For the fourth-moment quantities, we have:
\begin{align}
\Exp[\|\Upsilon\|_2^4 \mid \theta] &= \Exp \left[ \left\| \frac{1}{S_\Upsilon} \sum_{i \in I_t^\Upsilon} \nabla \log(1+\exp(-y_i \theta^\top x_i)) \right\|_2^4 \right] \nonumber  \\
&\leq \Exp \left[ \left( \frac{1}{S_\Upsilon} \sum_{i \in I_t^\Upsilon} \| \nabla \log(1+\exp(-y_i \theta^\top x_i)) \|_2 \right)^4 \right] \nonumber  \\
&\leq \Exp\left[ \left( \frac{1}{S_\Upsilon}  \sum_{i \in I_t^\Upsilon} \|x_i\|_2 \right)^4 \right] \leq \tfrac{1}{n} \sum \|x_i\|_2^4.
\end{align}

\begin{align}
\Exp[\Psi_s^4] &\stackrel{(i)}{\leq} \tfrac{1}{n} \sum (c+k_i(\theta))^4 = \tfrac{1}{n} \sum \left( c+\log( 1 + \exp(-y_i \htheta^\top x_i - y_i \Delta x_i) ) \right)^4 \nonumber \\
\stackrel{(ii)}{\lesssim} & \tfrac{1}{n} \sum \|x_i\|^4 \|\Delta\|_2^4 + \tfrac{1}{n}\sum \left(c+\log(1+\exp(-y_i \htheta^\top x_i ))\right)^4,
\end{align}
`where $(i)$ follows from ${ \Exp\left[\left(\frac{\sum_{j=1}^S X_j}{S}\right)^4\right] \leq \Exp\left[\frac{\sum_{j=1}^S X_j^4}{S}\right] }$ and (ii) follows from ${ \log(1 + \exp(a + b))  \leq \log(1 + e^b) + |a| }$.

Combining the above, we get:
\begin{align}
\Exp[ \| g_s \|_2^4(\theta)  \mid \theta] = \Exp\left[\Psi^4 \mid \theta \right] \cdot \Exp\left[\|\Upsilon\|_2^4 \mid \theta\right] \lesssim C \|\Delta\|_2^4 + D,
\end{align}
for some data dependent constants $C$ and $D$.

Finally, we need a bound for the quantity $\| \Exp[\nabla g_s(\theta) \nabla g_s(\theta)^\top] -G  \|_2$. 
We observe:
\begin{align}
\left\| \Exp[\nabla g_s(\theta) \nabla g_s(\theta)^\top] -G  \right\|_2 &\leq \left\|  \tfrac{K_G(\theta)}{n} \sum \nabla k_i(\theta) \nabla k_i(\theta)^\top  - \tfrac{K_G(\htheta)}{n} \sum \nabla k_i(\htheta) \nabla k_i(\htheta)^\top  \right\|_2 \nonumber \\
&\leq  \Big\|  \tfrac{K_G(\theta)}{n} \sum \nabla k_i(\theta) \nabla k_i(\theta)^\top   -  \tfrac{K_G(\theta)}{n} \sum \nabla k_i(\htheta) \nabla k_i(\htheta)^\top  
	  \nonumber \\ 
	  &\quad \quad \quad \quad \quad + \tfrac{K_G(\theta)}{n} \sum \nabla k_i(\htheta) \nabla k_i(\htheta)^\top -  \tfrac{K_G(\htheta)}{n} \sum \nabla k_i(\htheta) \nabla k_i(\htheta)^\top \Big\|_2 \nonumber \\
&\leq \tfrac{K_G(\theta)}{n} \left\|\sum \left( \nabla k_i(\theta) \nabla k_i(\theta)^\top - \nabla k_i(\htheta) \nabla k_i(\htheta)^\top  \right) \right\|_2 \nonumber \\ 
	  &\quad \quad \quad \quad \quad + |K_G(\theta) - K_G(\htheta)| \cdot \left\| \tfrac{1}{n} \sum \nabla k_i(\htheta) \nabla k_i(\htheta)^\top \right\|_2.
 \end{align}

Because
\begin{align}
K_G(\theta) &= O(1 + \|\Delta\|_2 + \|\Delta\|_2^2), \\
 \tfrac{1}{n} \left\|\sum ( \nabla k_i(\theta) \nabla k_i(\theta)^\top - \nabla k_i(\htheta) \nabla k_i(\htheta)^\top  ) \right\|_2  &=  O( \|\Delta\|_2 + \|\Delta\|_2^2),  \\
|K_G(\theta) - K_G(\htheta)|  &= O( \|\Delta\|_2 + \|\Delta\|_2^2),
 \end{align}
we may conclude that
\begin{align}
\| \Exp[g_s(\theta) g_s(\theta)^\top \mid \theta] - G \|_2
	\leq A_1 \|\theta - \htheta\|_2  + A_2 \|\theta - \htheta\|_2^2 + A_3 \|\theta - \htheta\|_2^3
		+ A_4 \|\theta - \htheta\|_2^4,
\end{align}
for some data dependent constants $A_1$, $A_2$, $A_3$, and $A_4$.

Combining above results and using Theorem \ref{thm:sgd-strongly-convex-lipschitz-stat-inf},
we have
\begin{align}
\left\| t \Exp[ (\btheta_t - \htheta)( \btheta_t - \htheta)^\top] - H^{-1} G H^{-1} \right\|_2 \lesssim \sqrt{\eta} + \sqrt{\frac{1}{t \eta} + t \eta^2}.
\end{align}
\end{proof}


\section{Experiments}
\label{sec:appendix:experiments}

Here we present additional experiments on our SGD inference procedure.

\subsection{Synthetic data}

\begin{figure}[h]
\centering
\subfloat[][Normal.]{
\includegraphics[width=.3\textwidth]{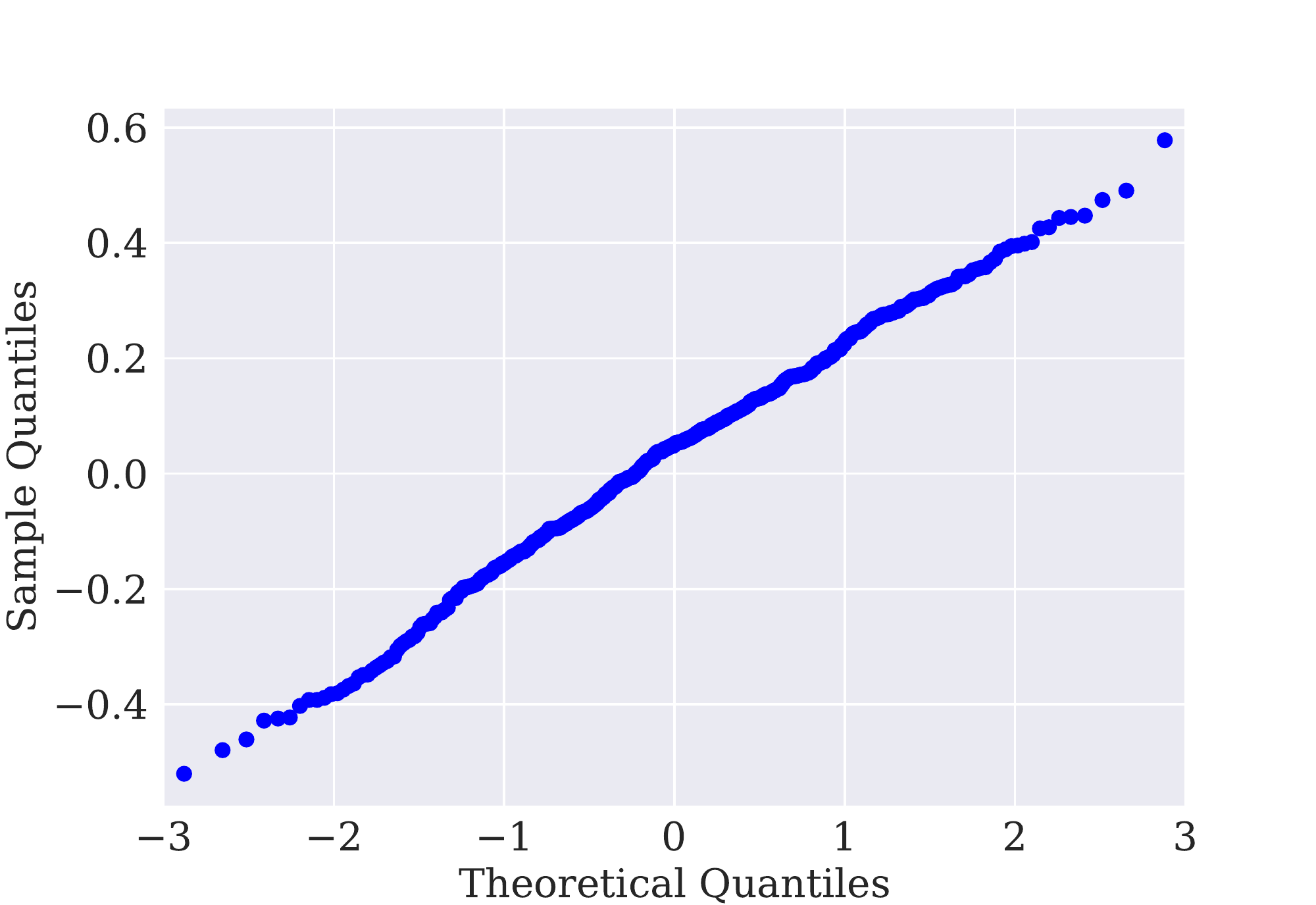}
}
\subfloat[][Exponential.]{
\includegraphics[width=.3\textwidth]{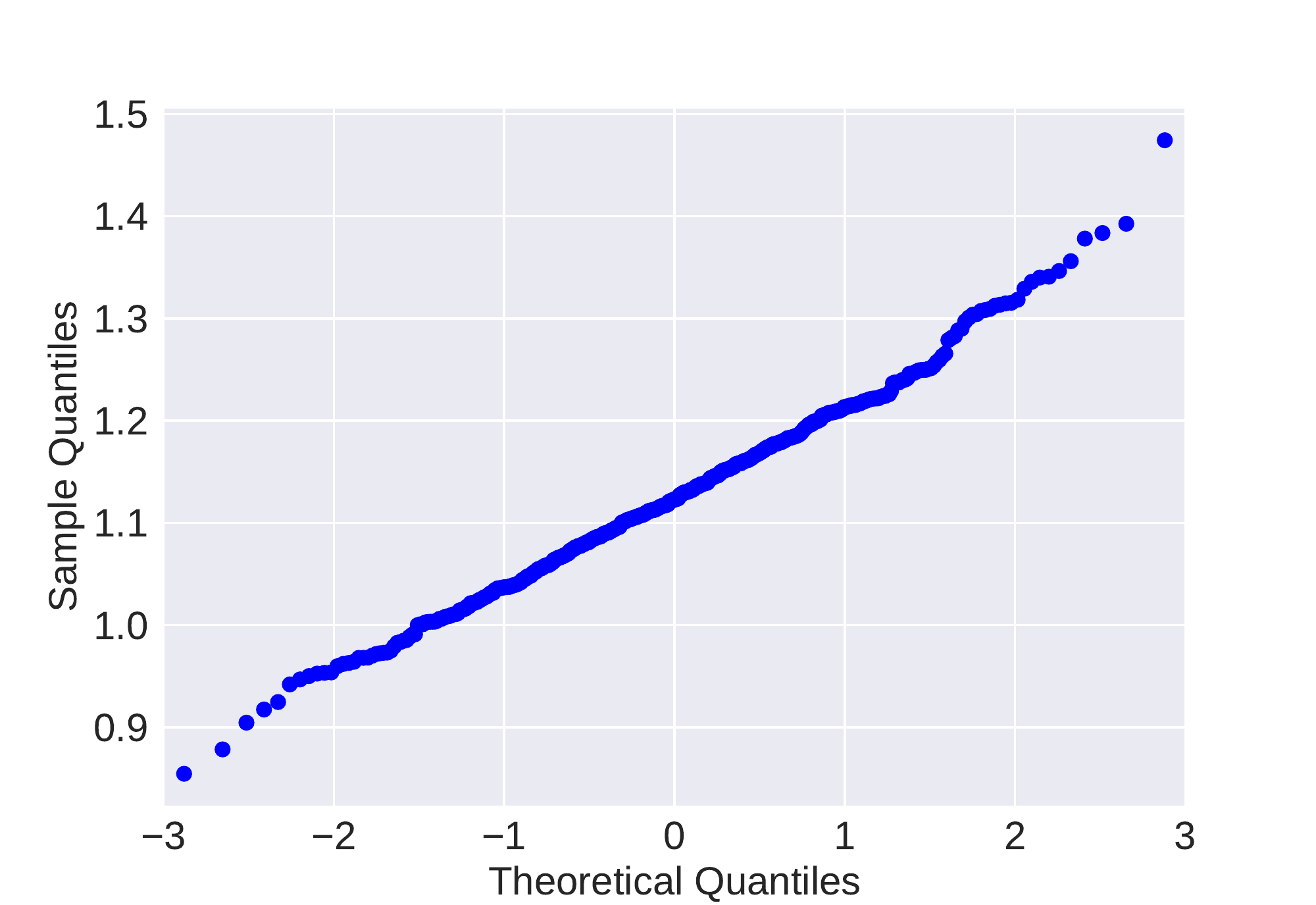}
}
\subfloat[][Poisson.]{
\includegraphics[width=.3\textwidth]{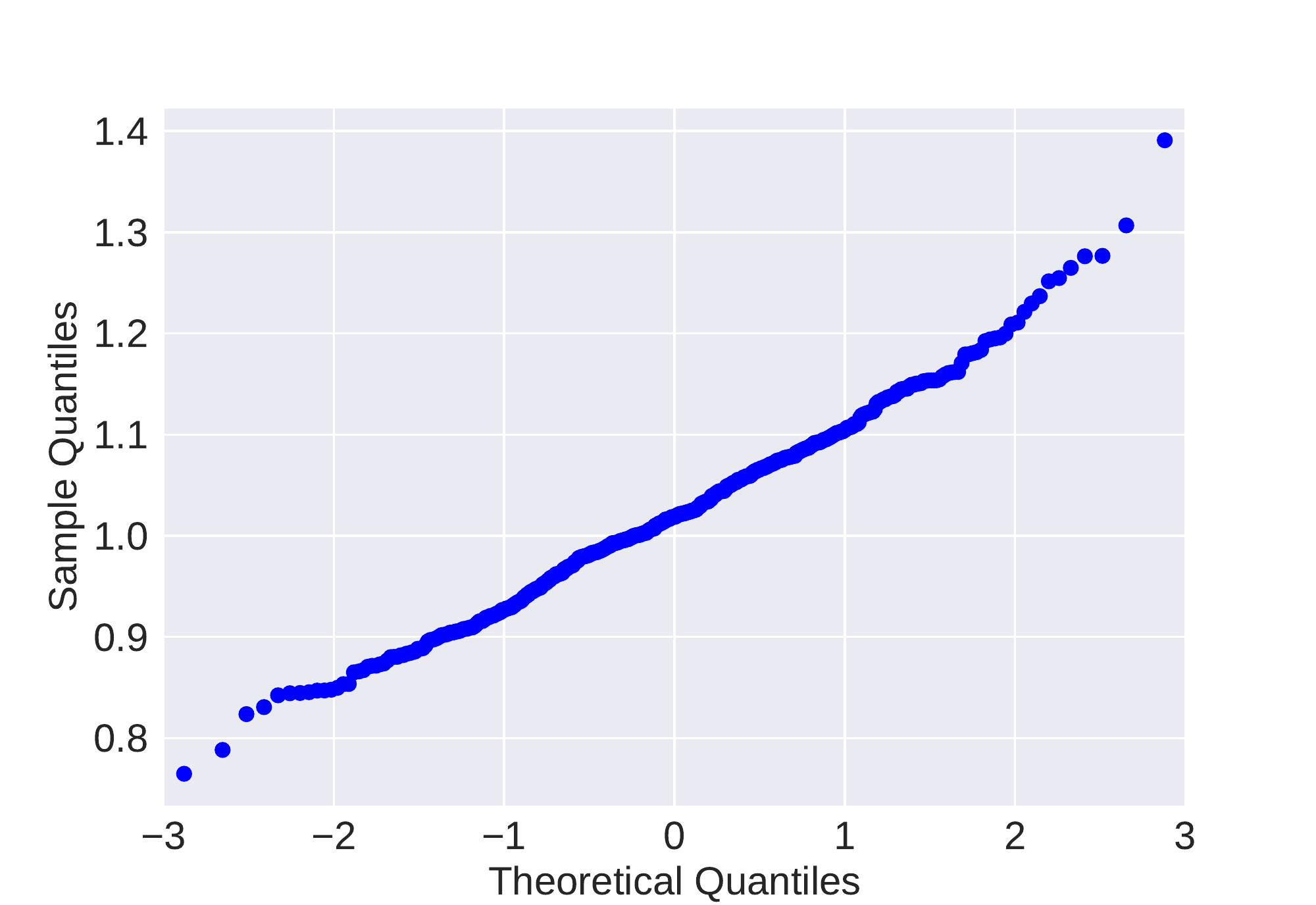}
}
\caption{Estimation in univariate models: Q-Q plots for samples shown in \Cref{fig:exp:sim:univariate}}
\label{fig:exp:sim:univariate-qq}
\end{figure}

\Cref{fig:exp:sim:univariate-qq} shows
Q-Q plots for samples shown in \Cref{fig:exp:sim:univariate}.

\subsubsection{Multivariate models}
\label{subsubsec:nip2017:appendix:experiments:mult}

Here we show Q-Q plots per coordinate for samples from our SGD inference procedure.

Q-Q plots per coordinate for samples in linear regression experiment 1 is shown in
\Cref{fig:nips2017:appendix:linear1:qq-coord}.
Q-Q plots per coordinate for samples in linear regression experiment 2 is shown in
\Cref{fig:nips2017:appendix:linear2:qq-coord}.

\begin{figure}[htb]
\centering
\includegraphics[width=0.19\textwidth]{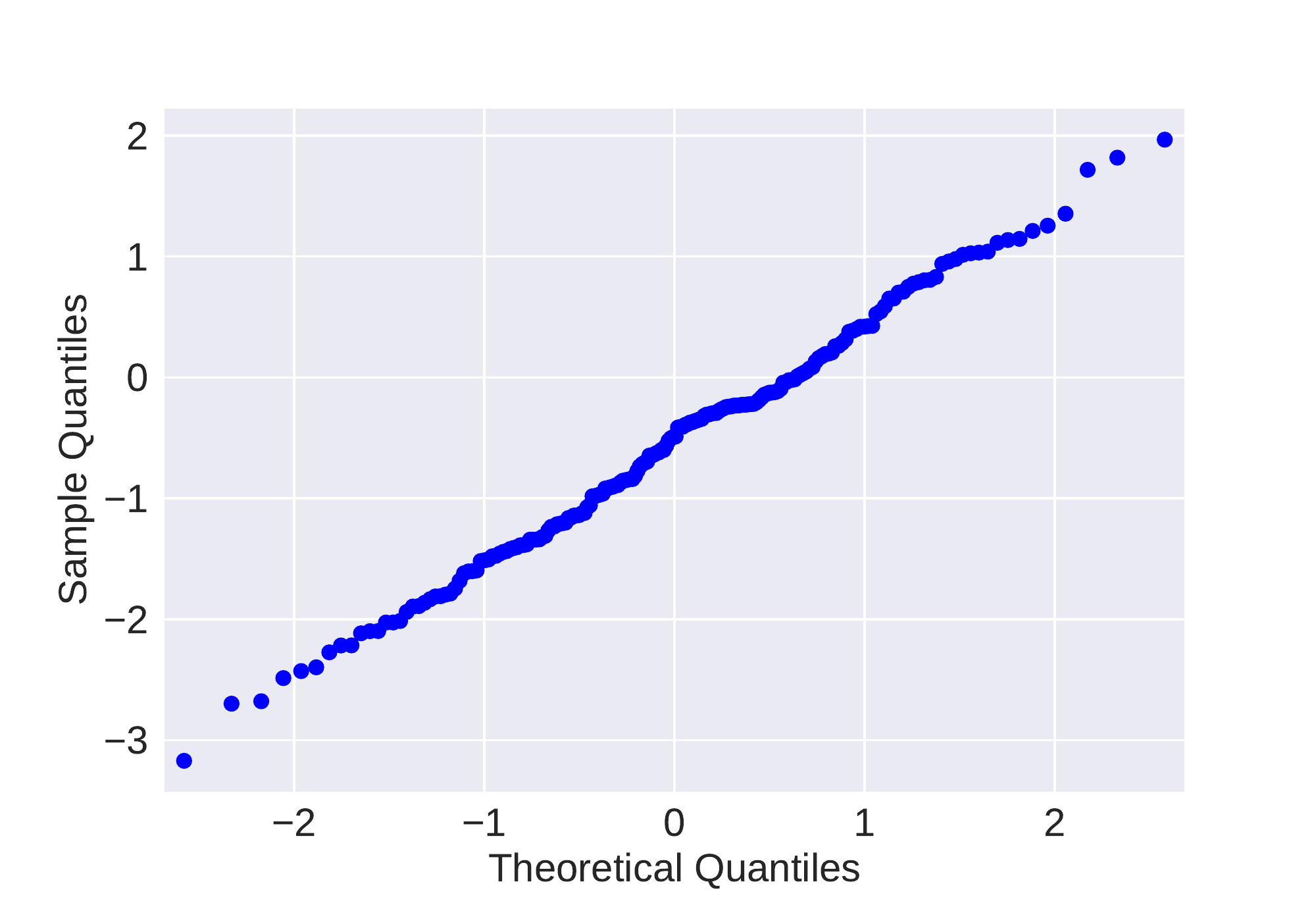}
\includegraphics[width=0.19\textwidth]{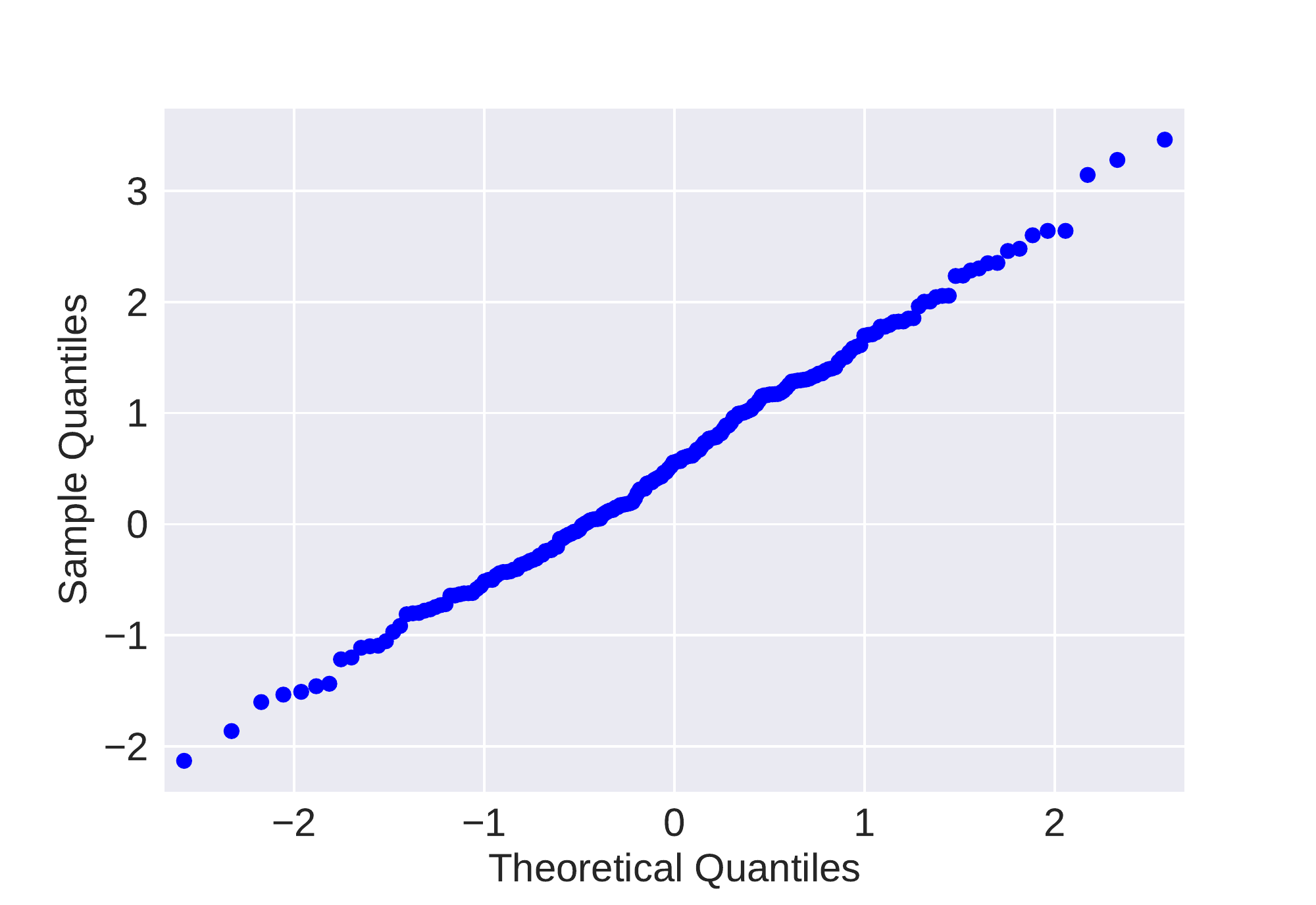}
\includegraphics[width=0.19\textwidth]{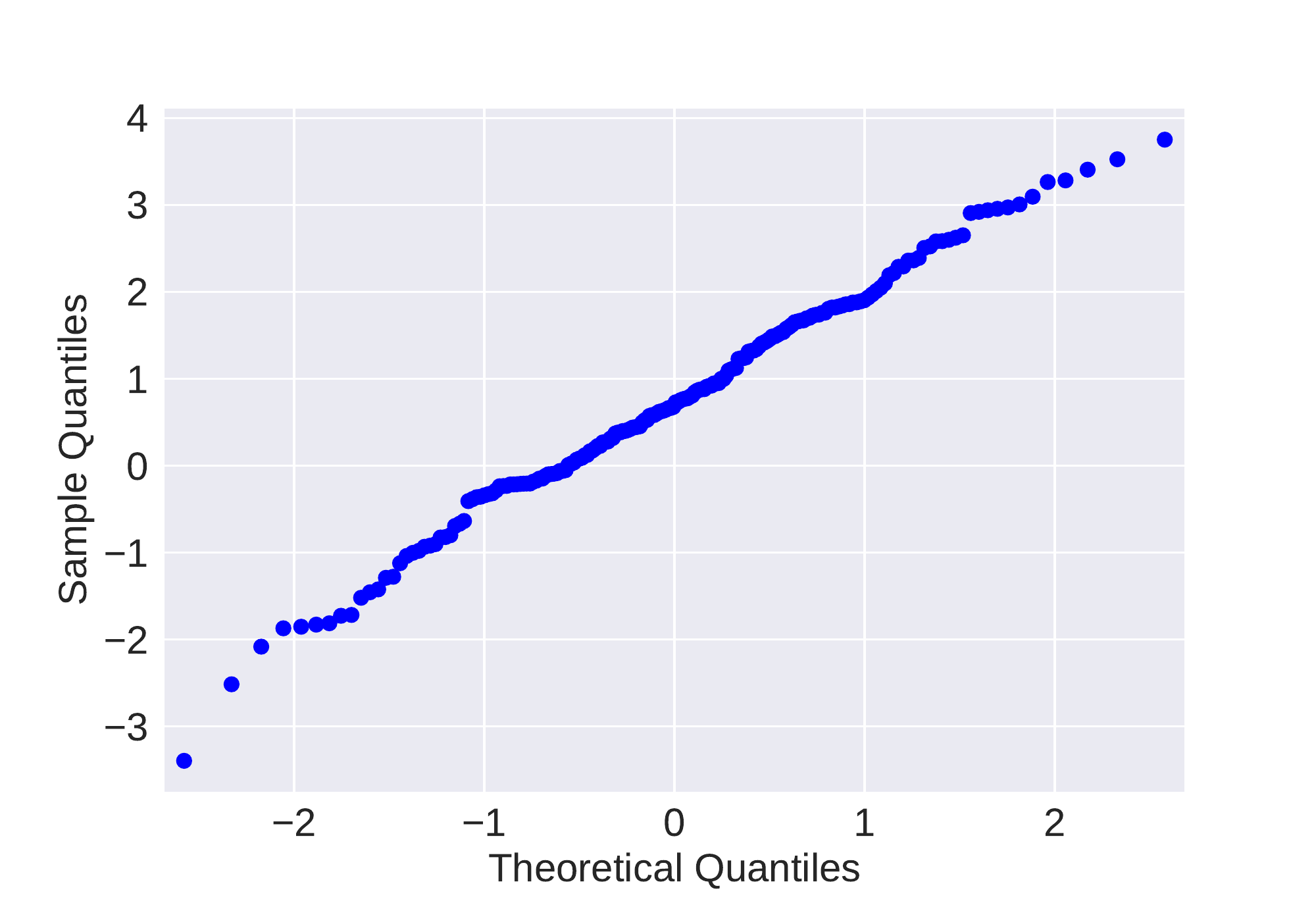}
\includegraphics[width=0.19\textwidth]{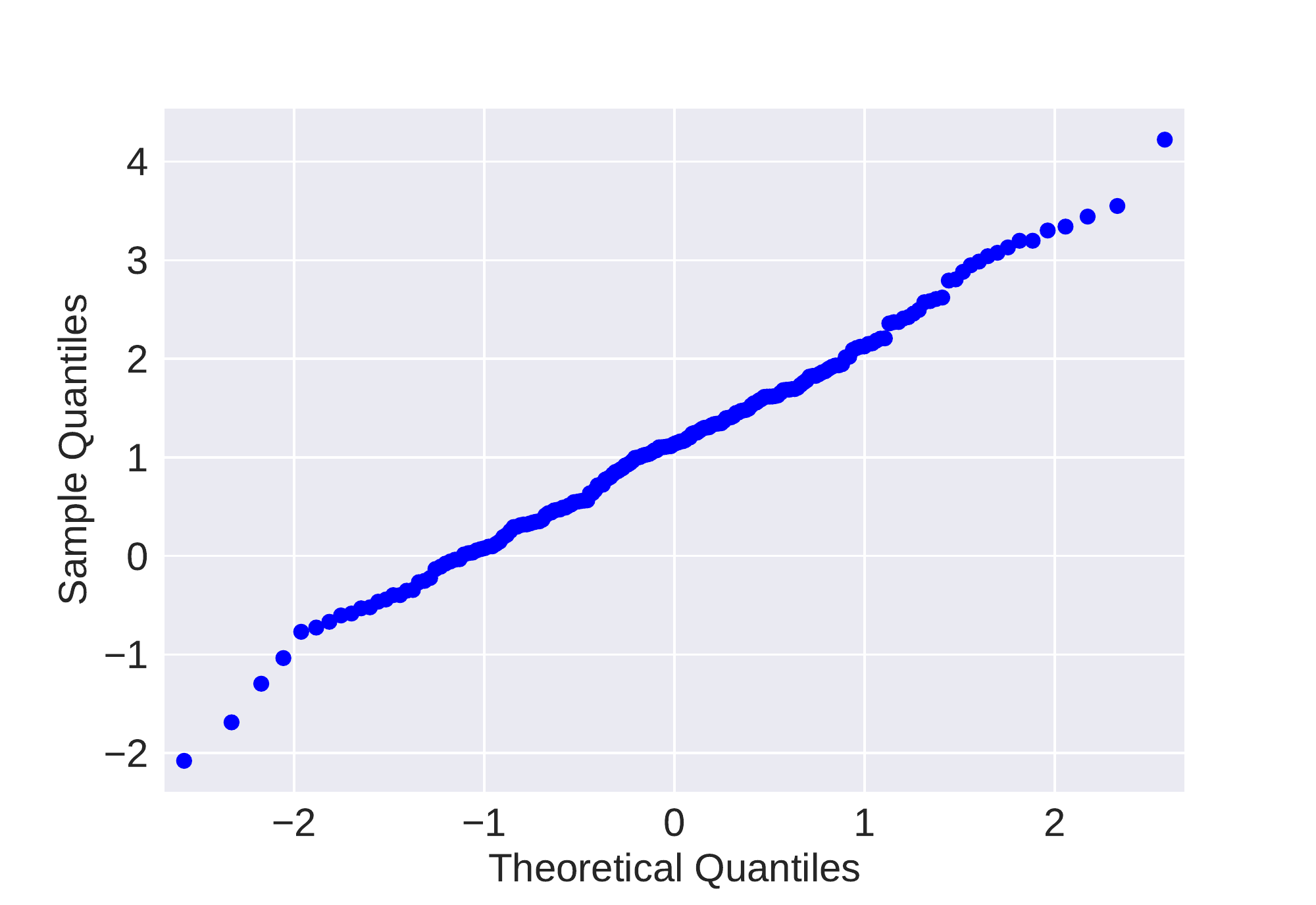}
\includegraphics[width=0.19\textwidth]{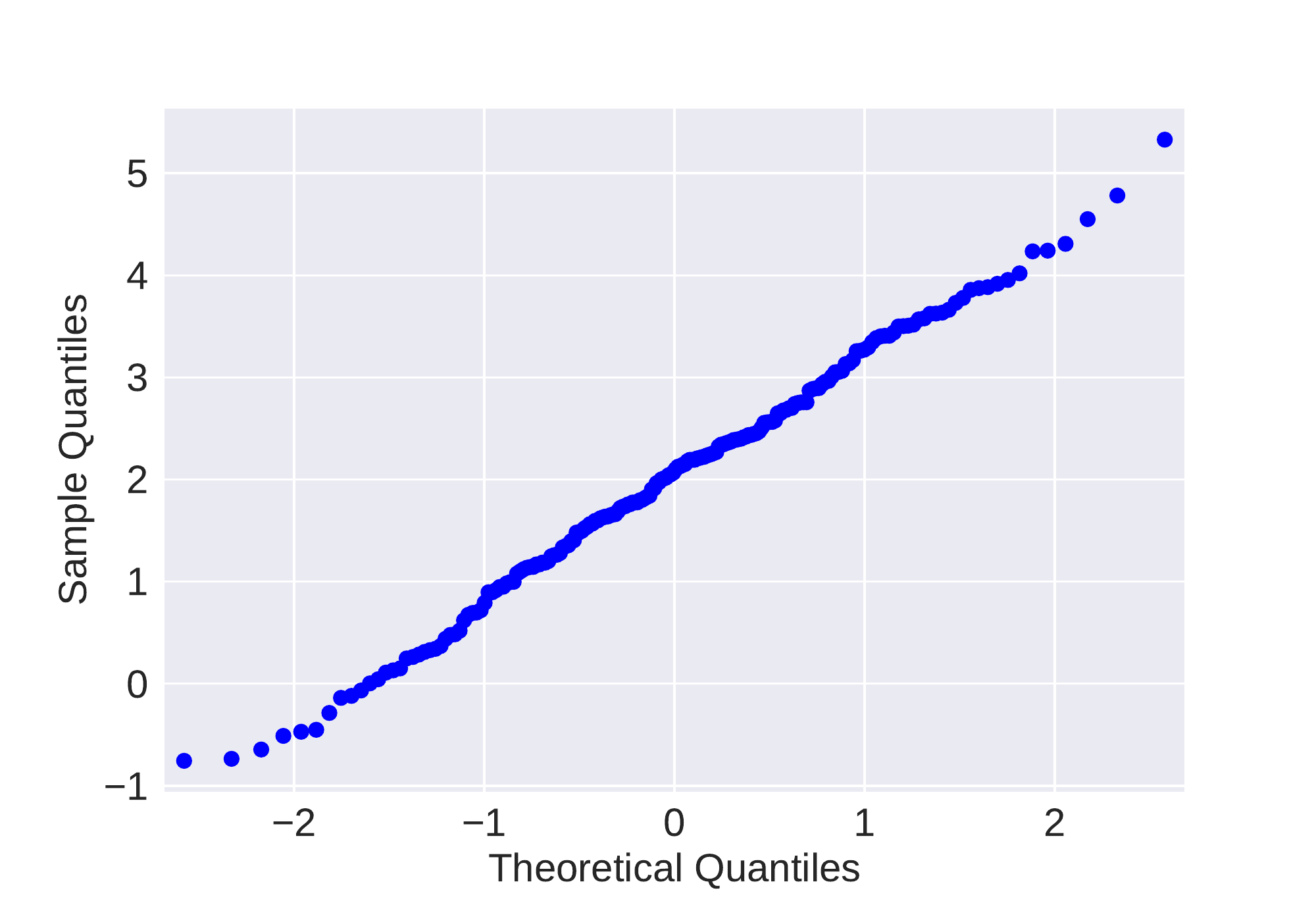}
\\
\includegraphics[width=0.19\textwidth]{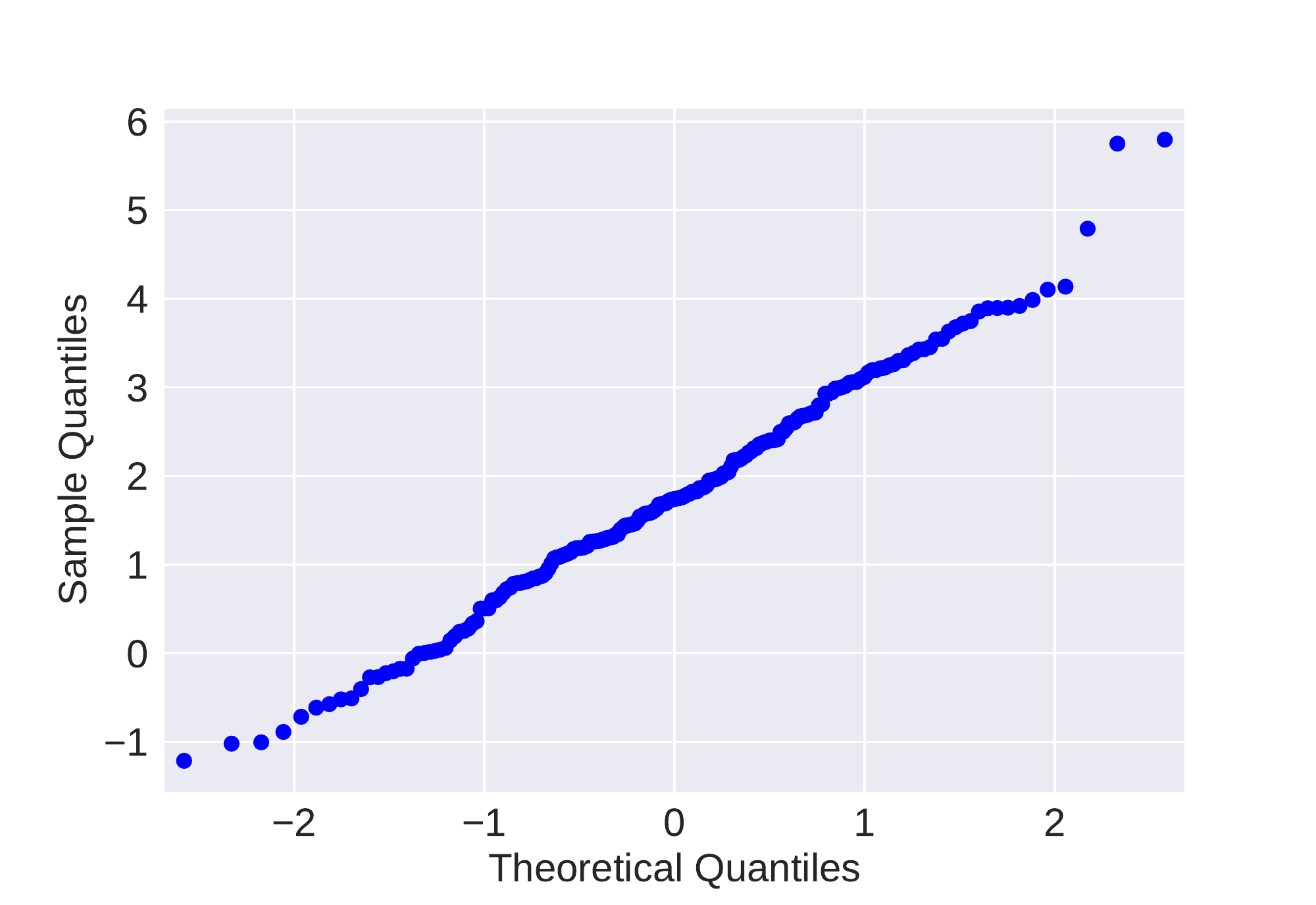}
\includegraphics[width=0.19\textwidth]{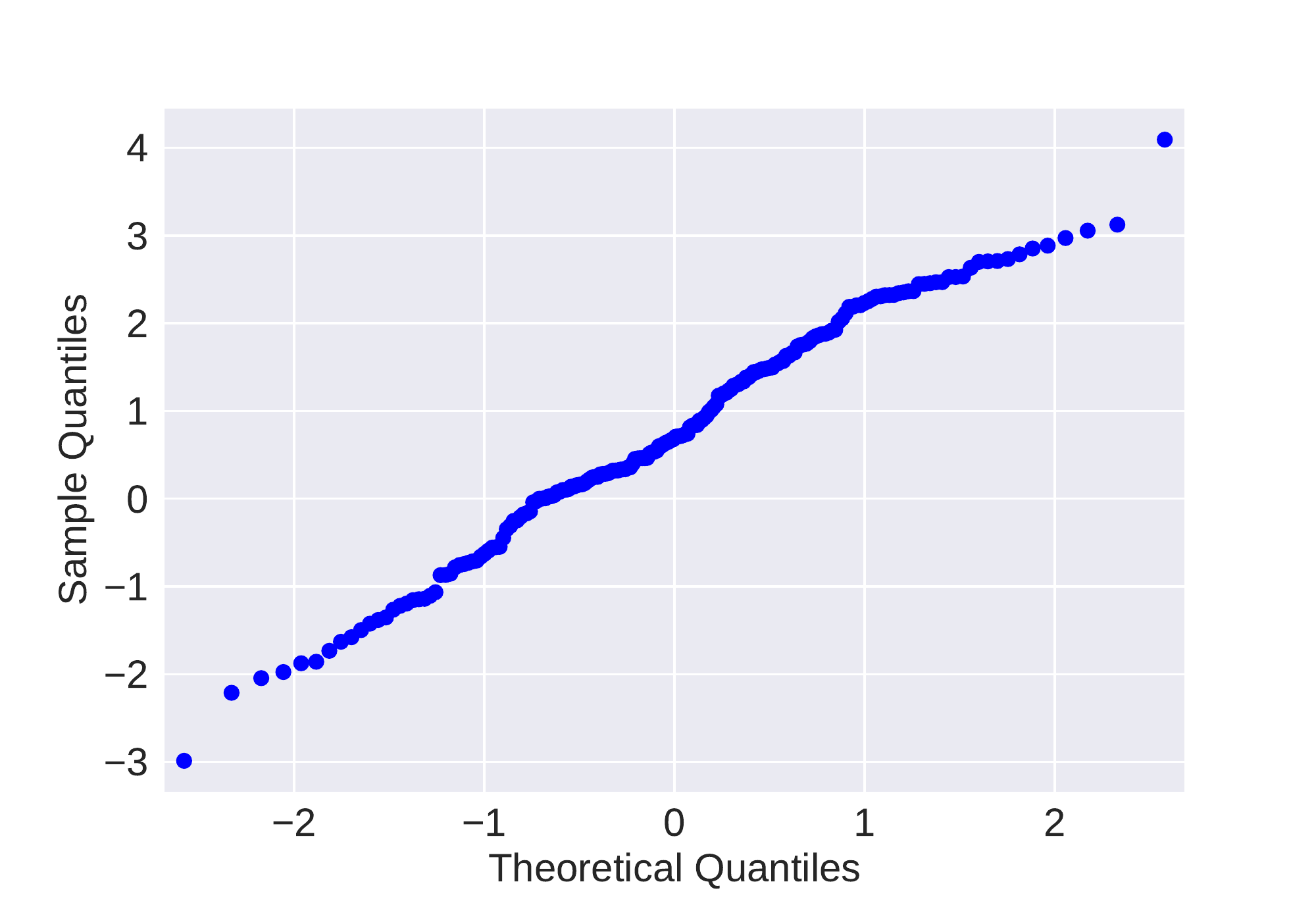}
\includegraphics[width=0.19\textwidth]{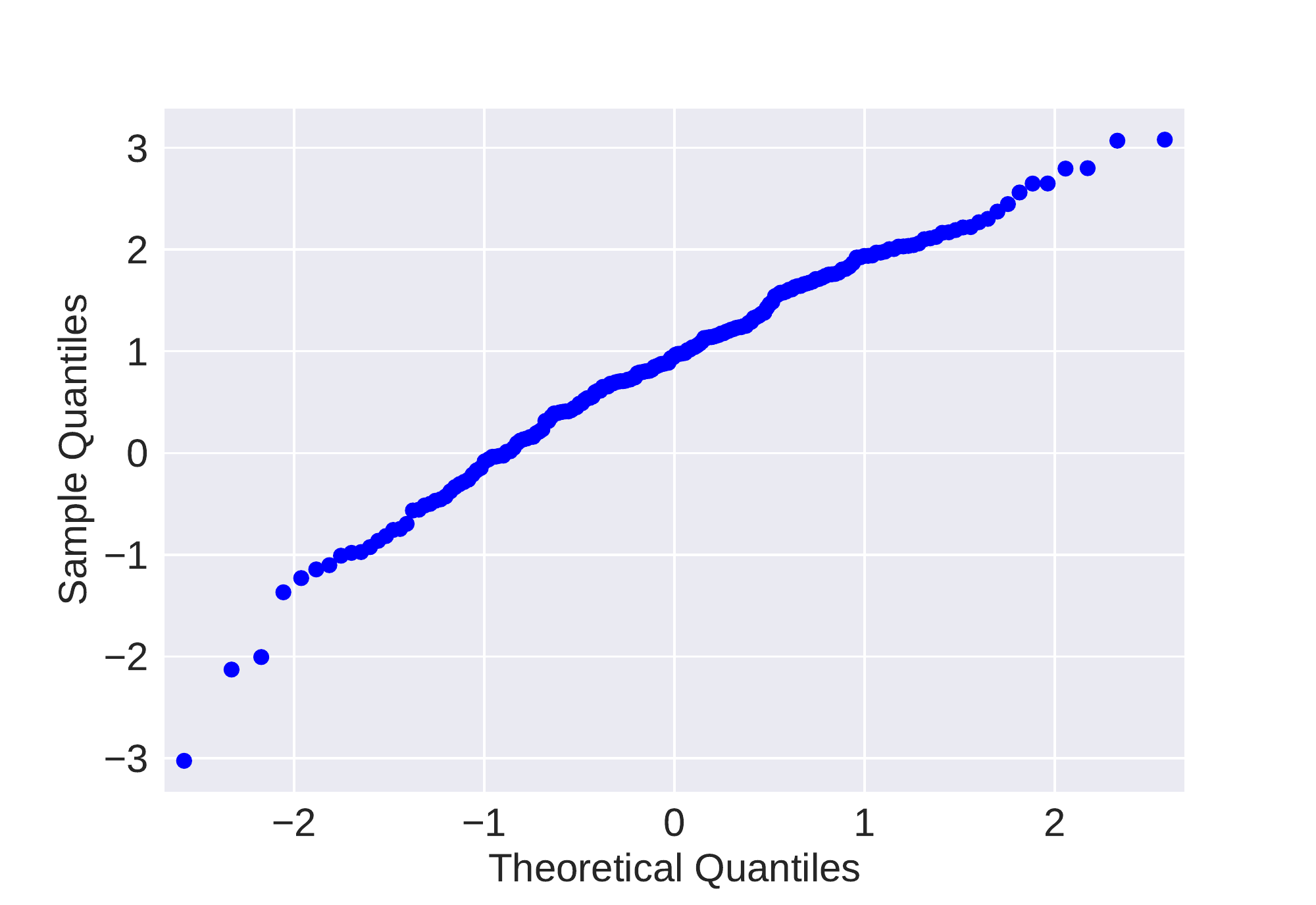}
\includegraphics[width=0.19\textwidth]{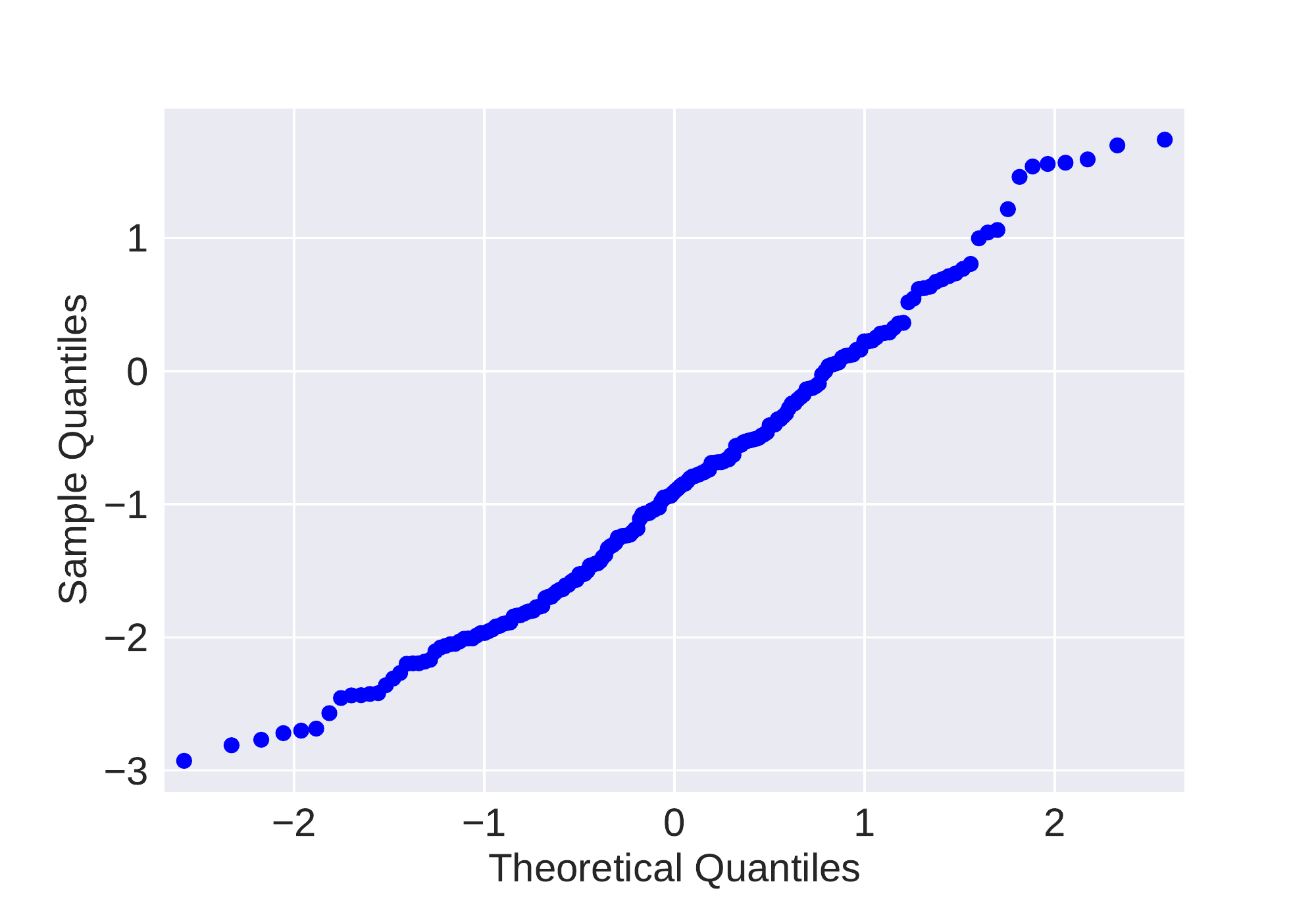}
\includegraphics[width=0.19\textwidth]{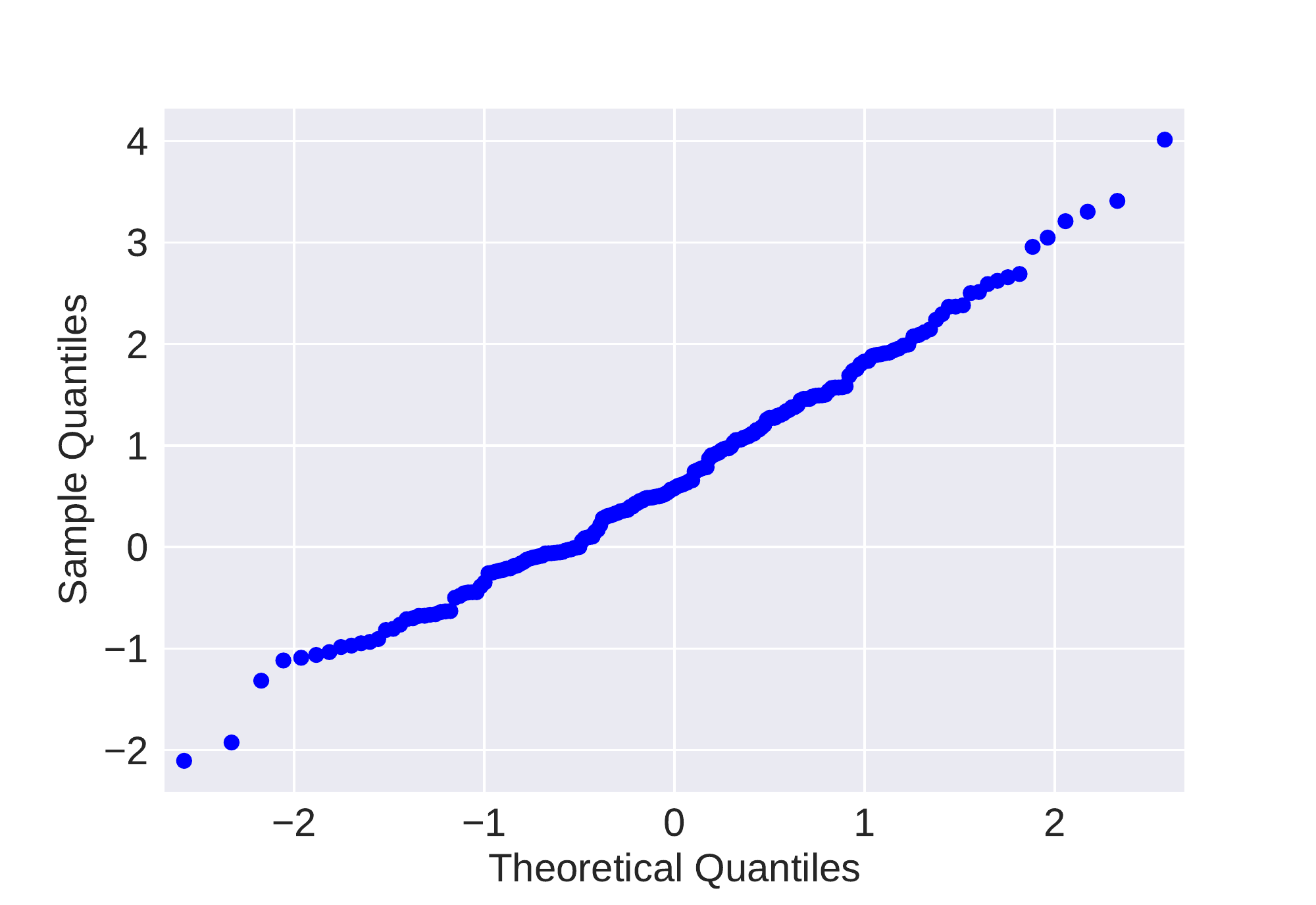}
\caption{Linear regression experiment 1: Q-Q plots per coordinate}
\label{fig:nips2017:appendix:linear1:qq-coord}
\end{figure}

\begin{figure}[htb]
\centering
\includegraphics[width=0.19\textwidth]{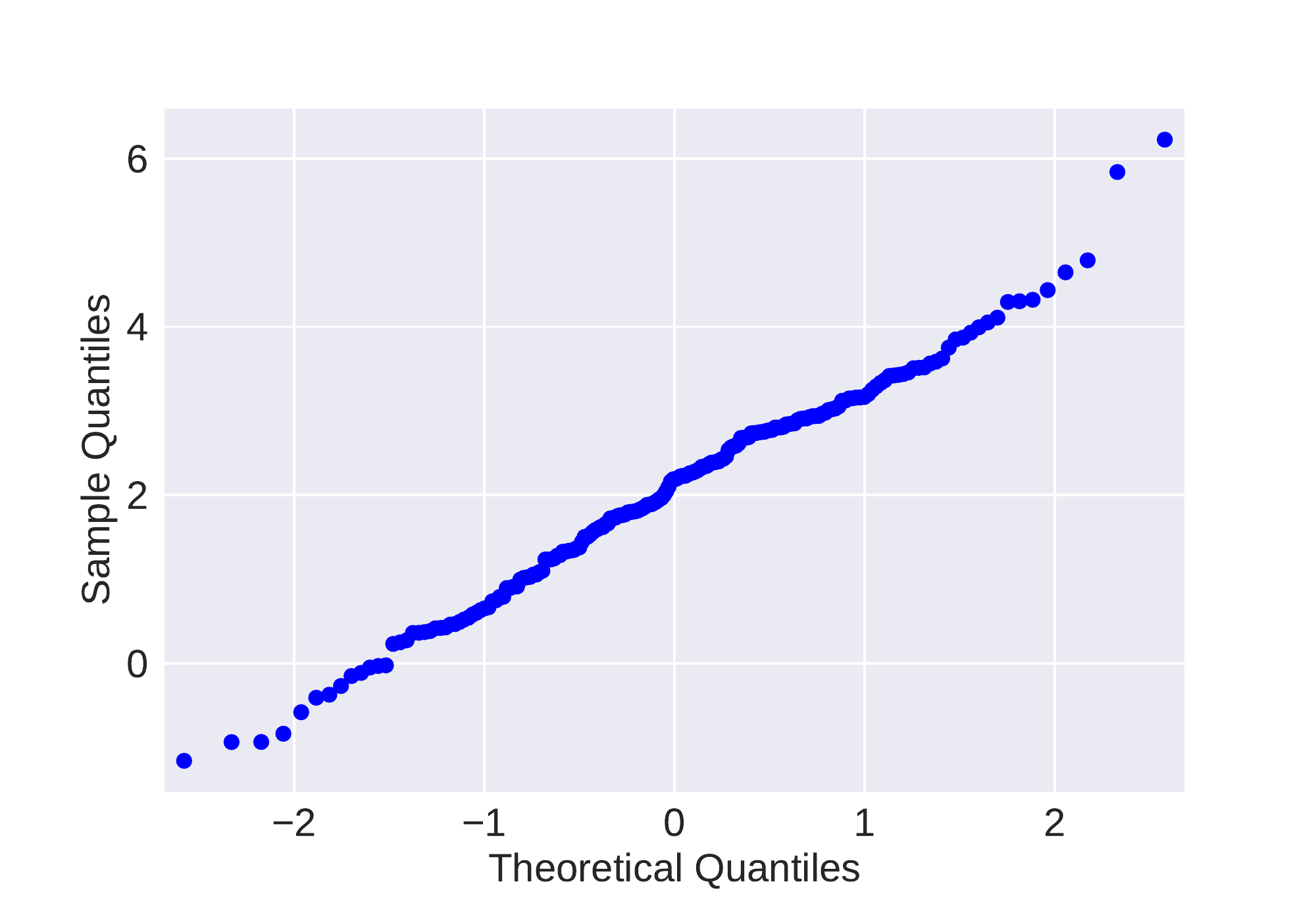}
\includegraphics[width=0.19\textwidth]{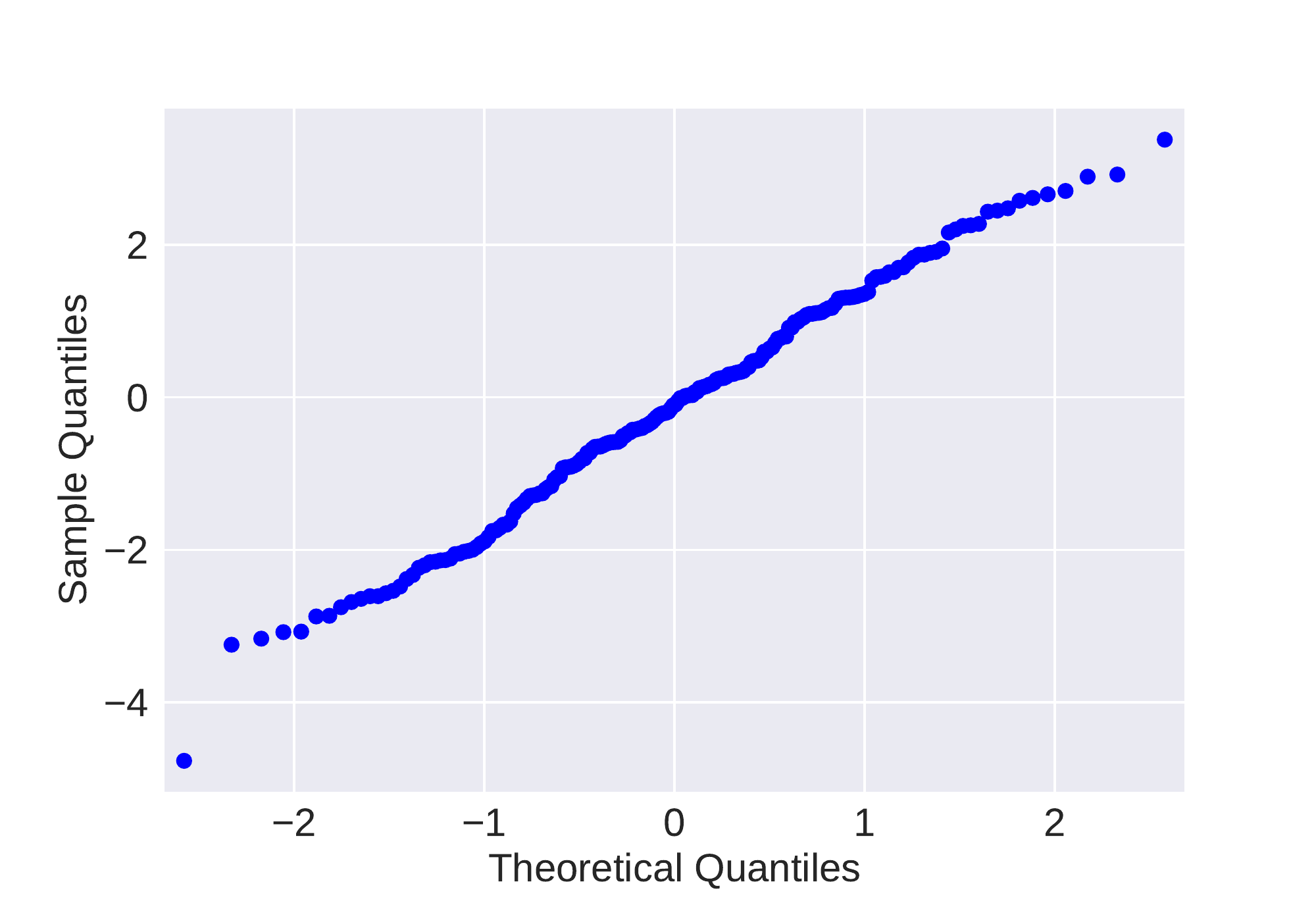}
\includegraphics[width=0.19\textwidth]{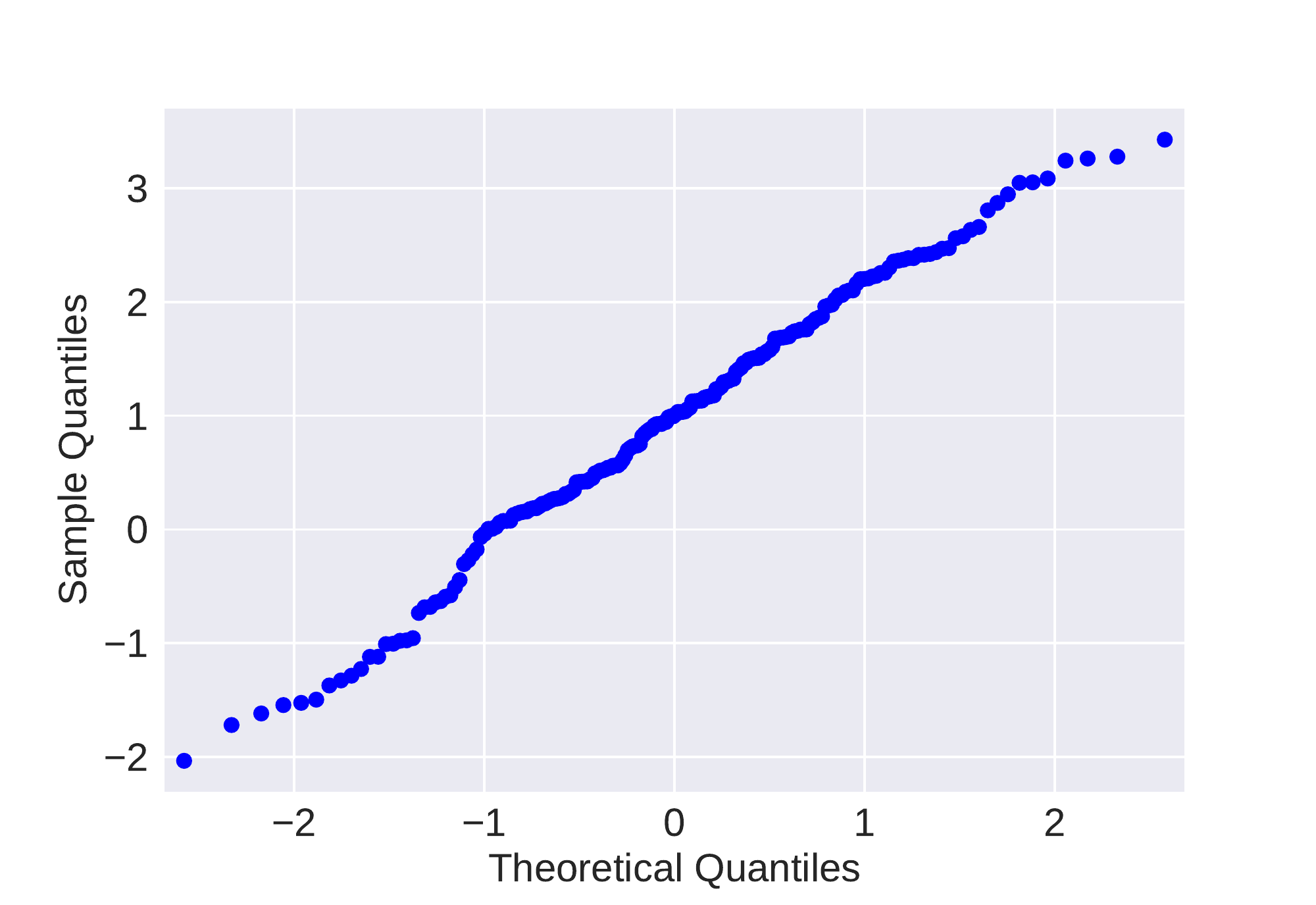}
\includegraphics[width=0.19\textwidth]{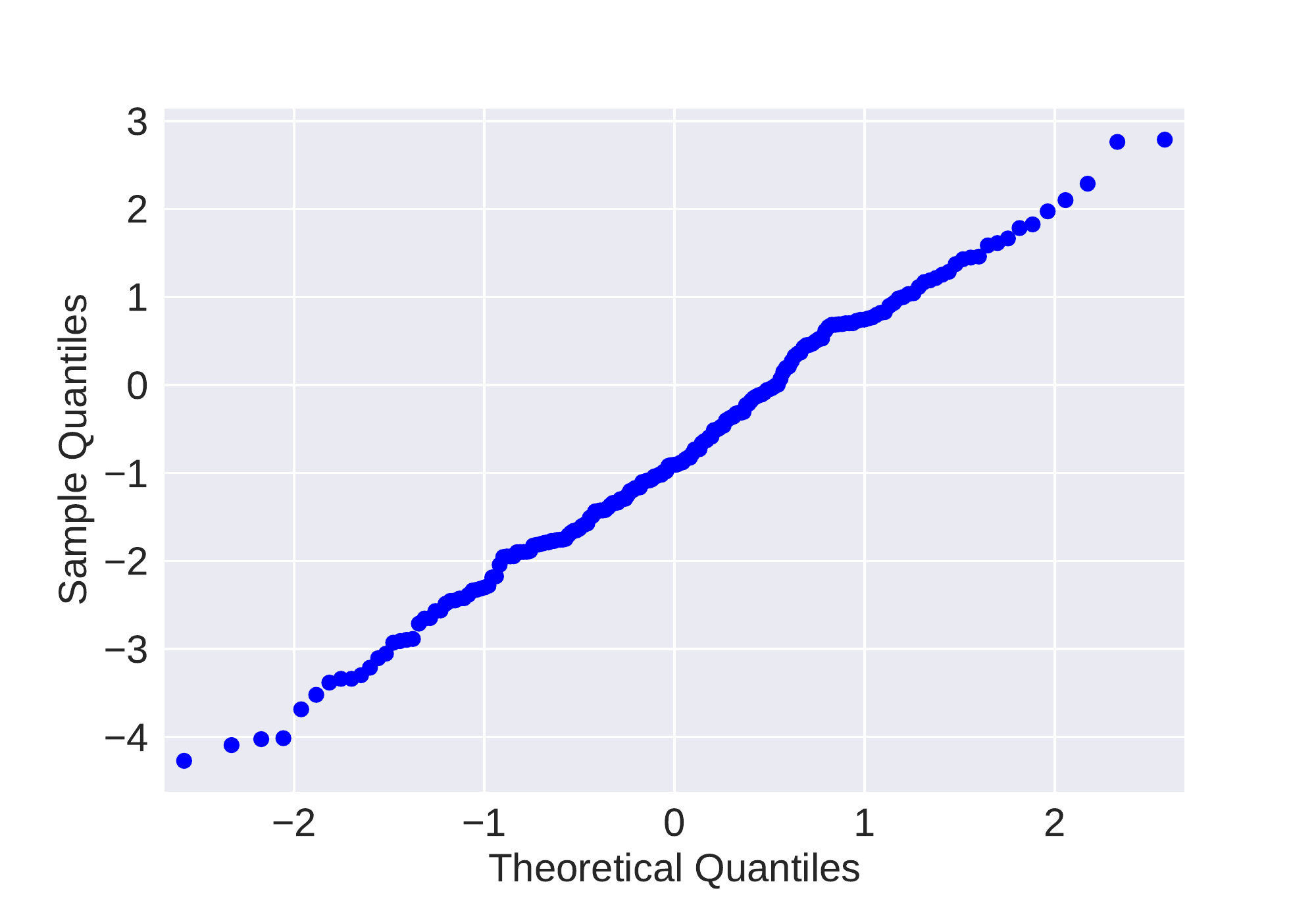}
\includegraphics[width=0.19\textwidth]{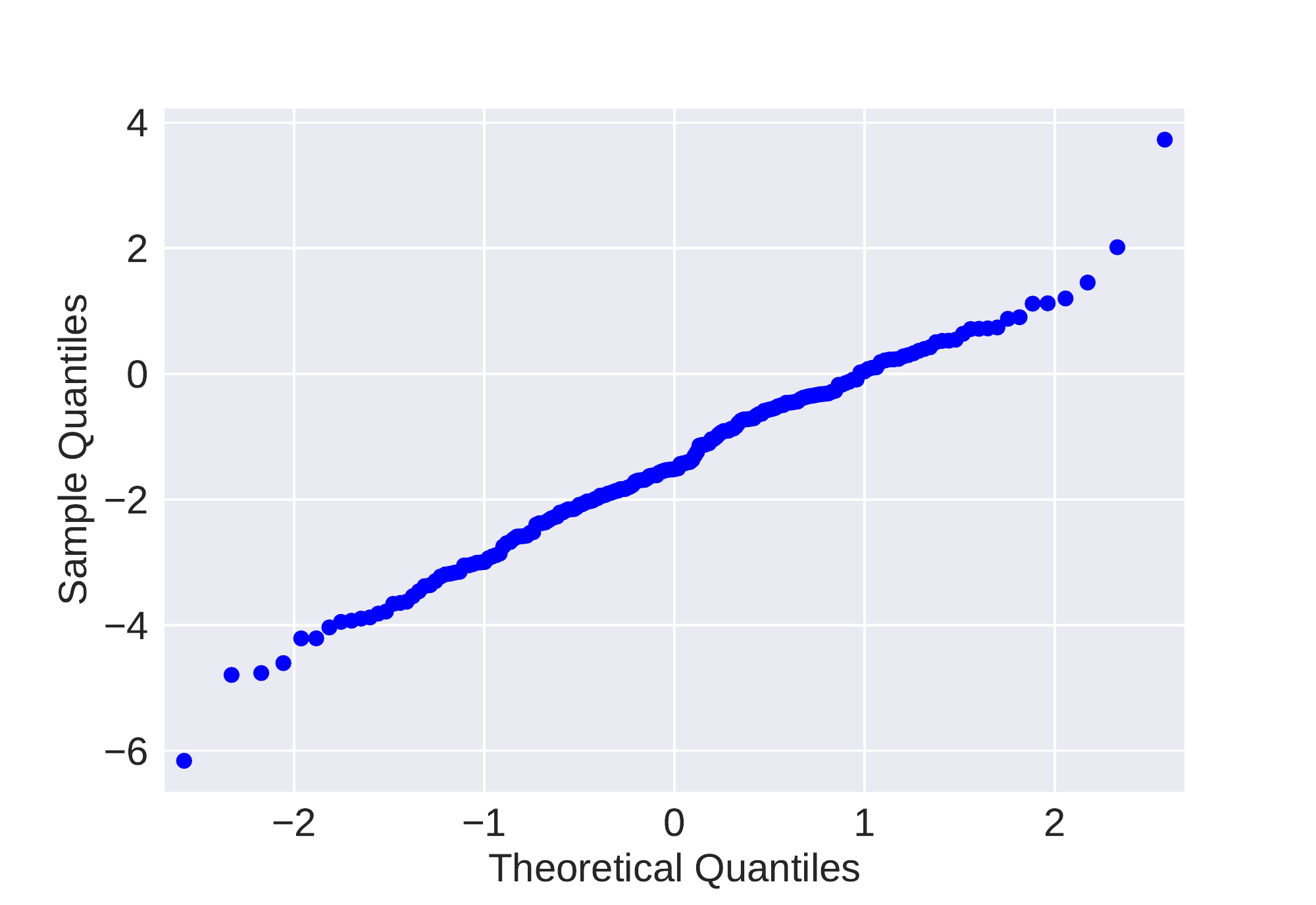}
\\
\includegraphics[width=0.19\textwidth]{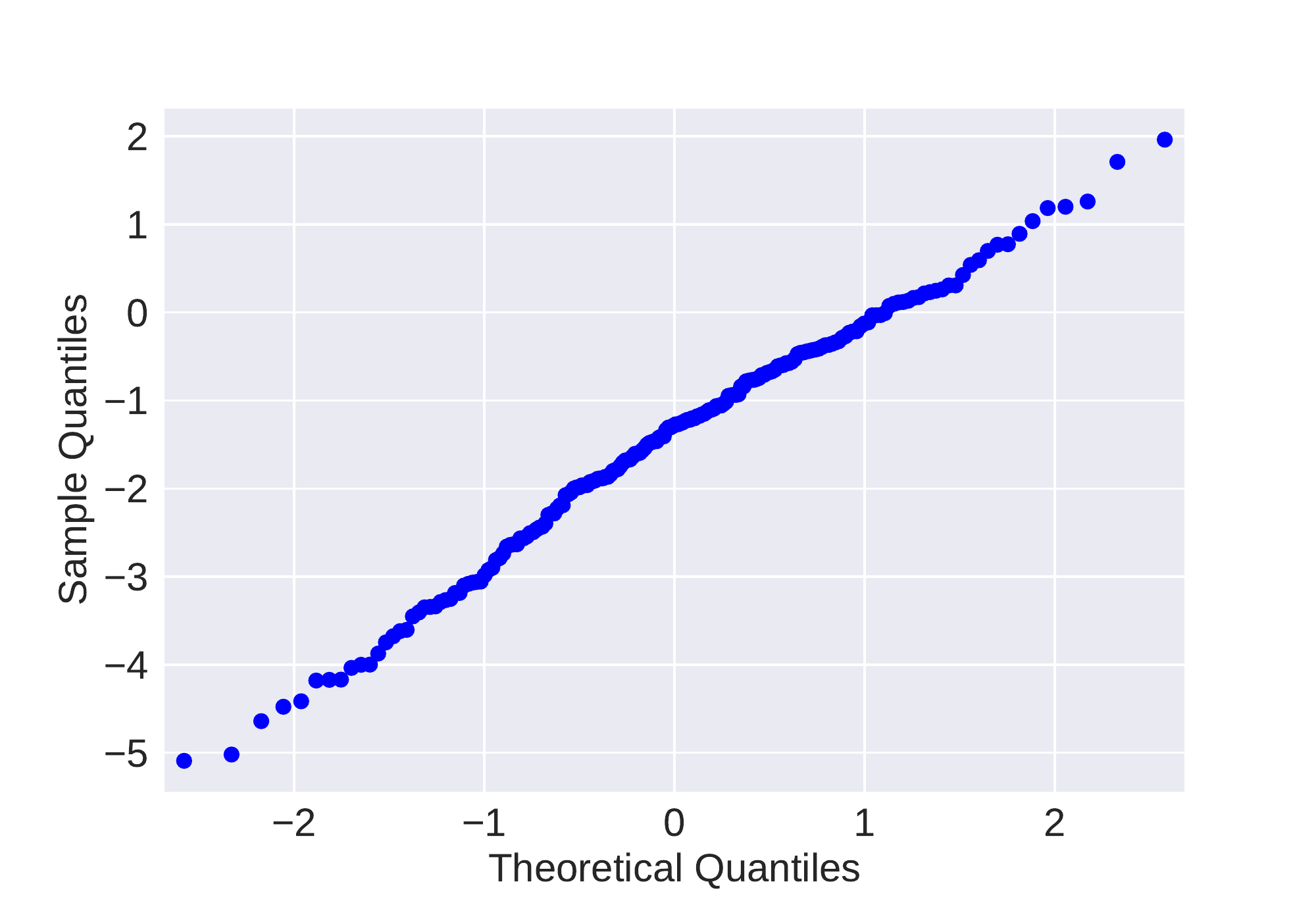}
\includegraphics[width=0.19\textwidth]{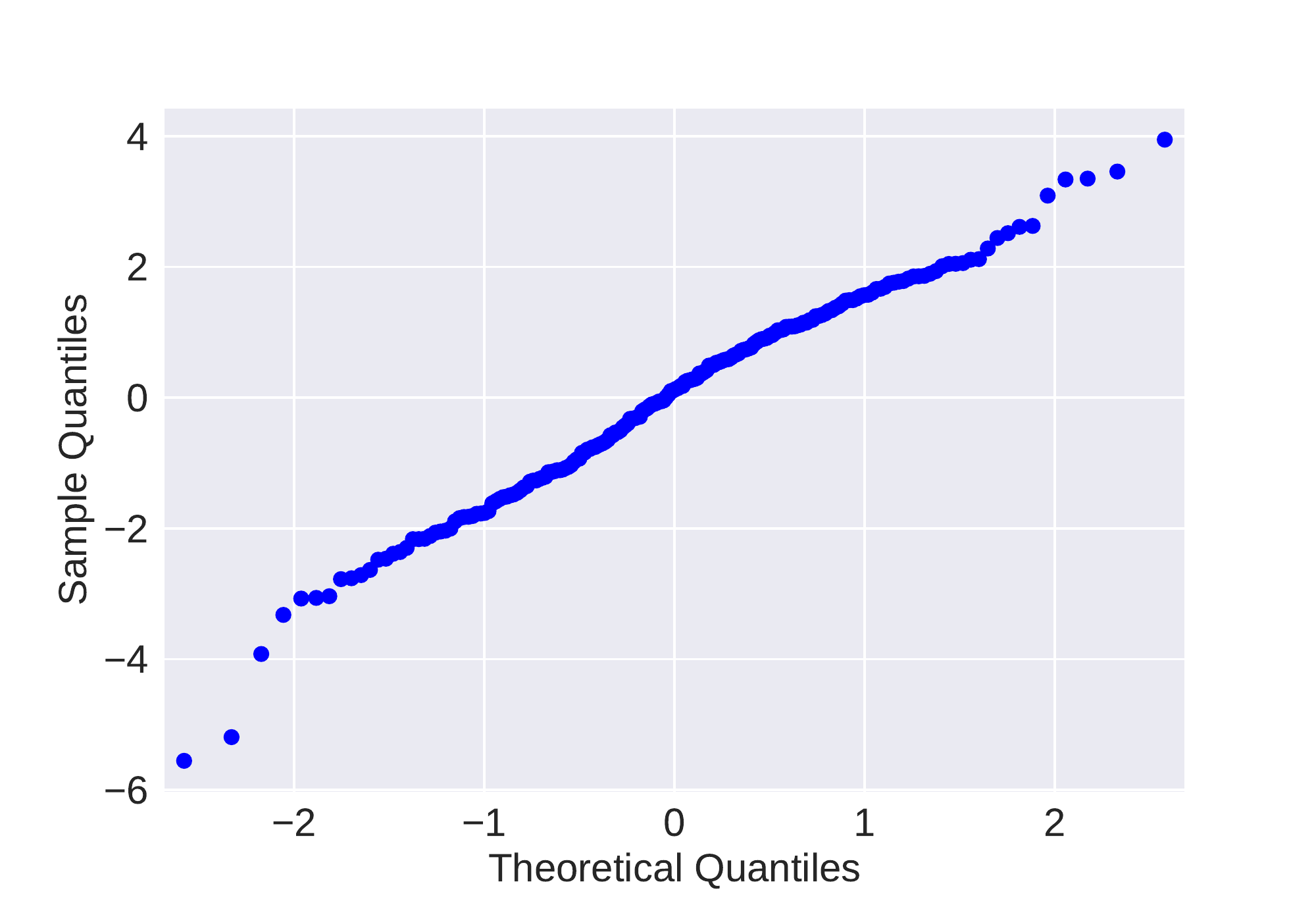}
\includegraphics[width=0.19\textwidth]{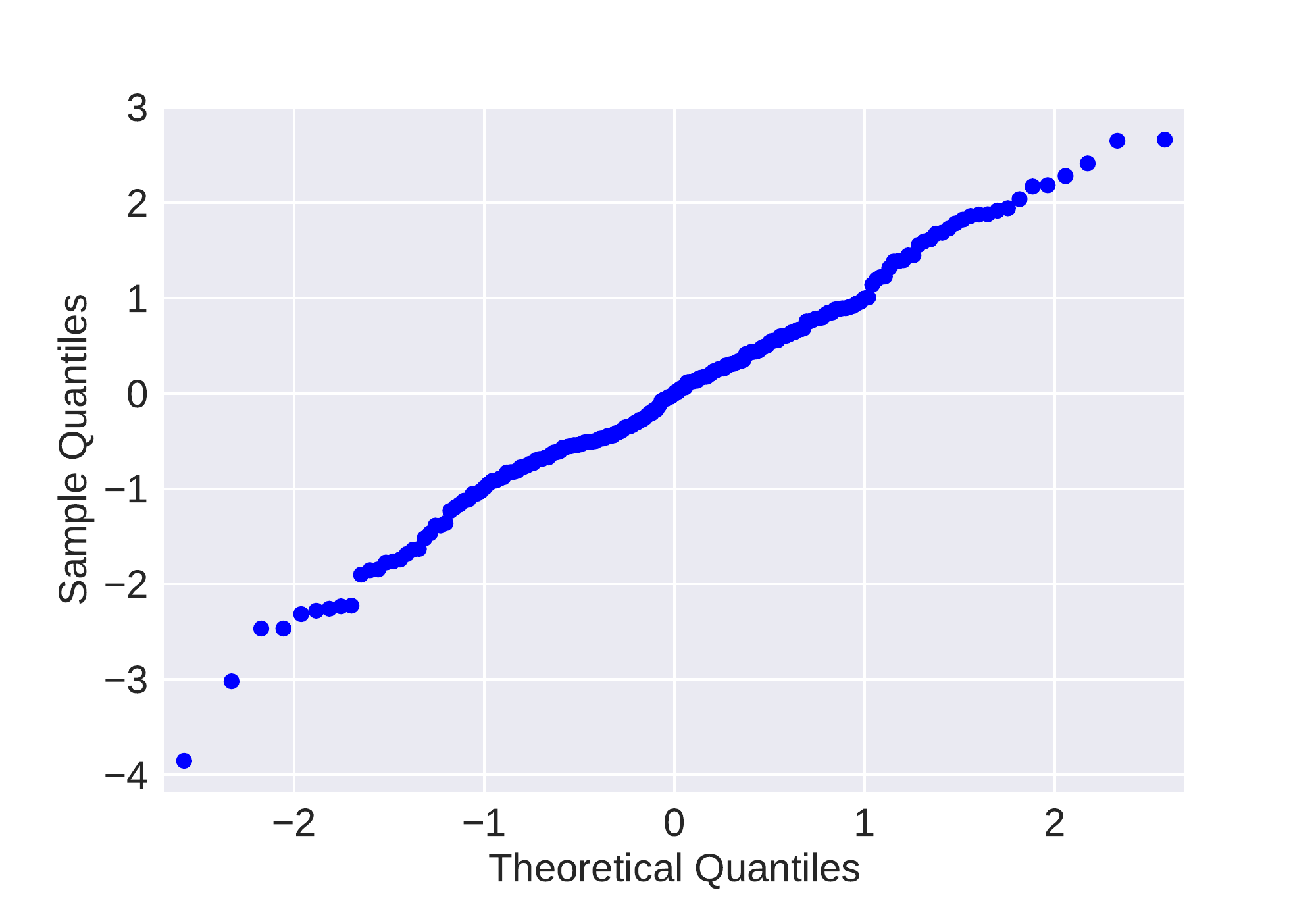}
\includegraphics[width=0.19\textwidth]{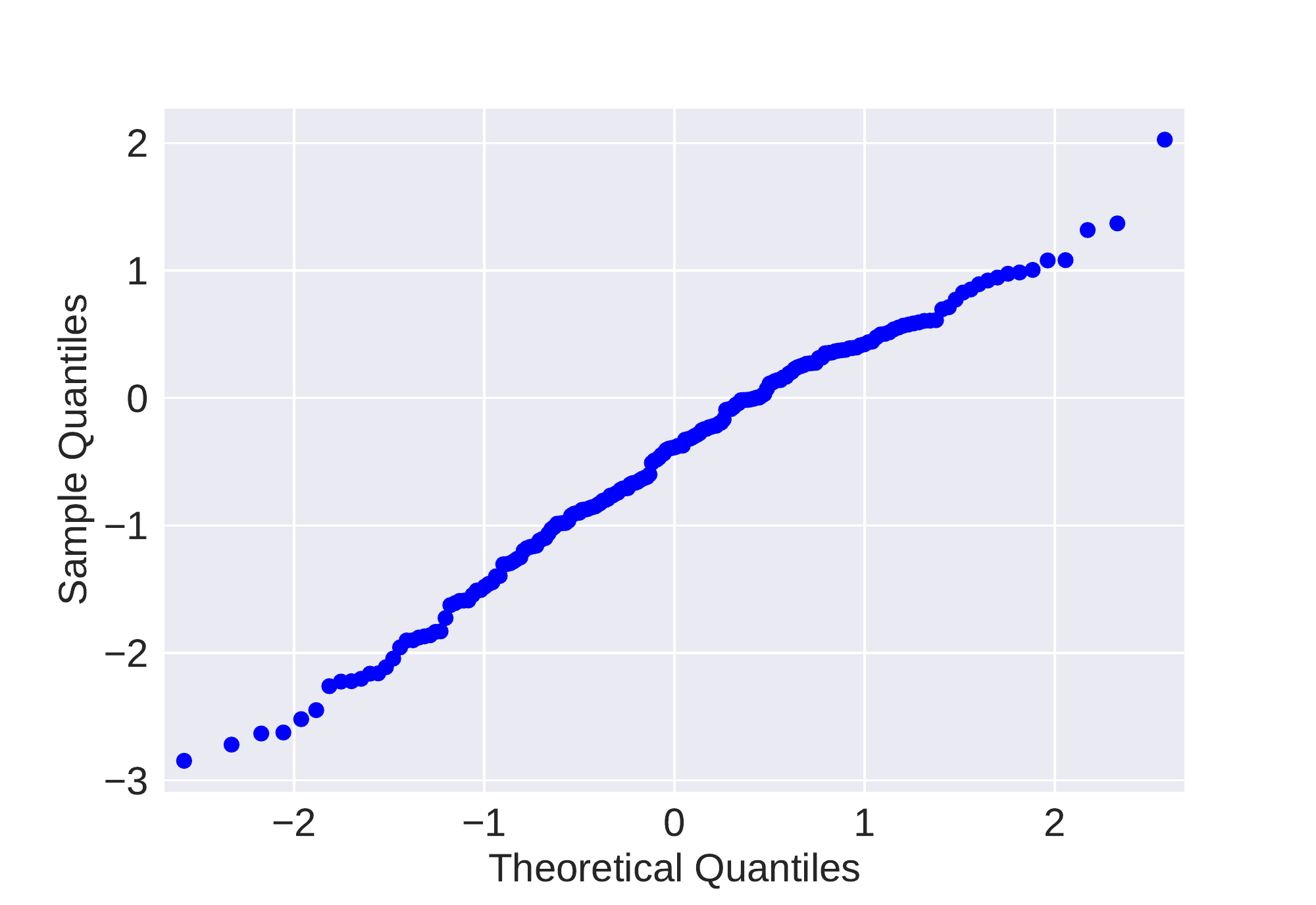}
\includegraphics[width=0.19\textwidth]{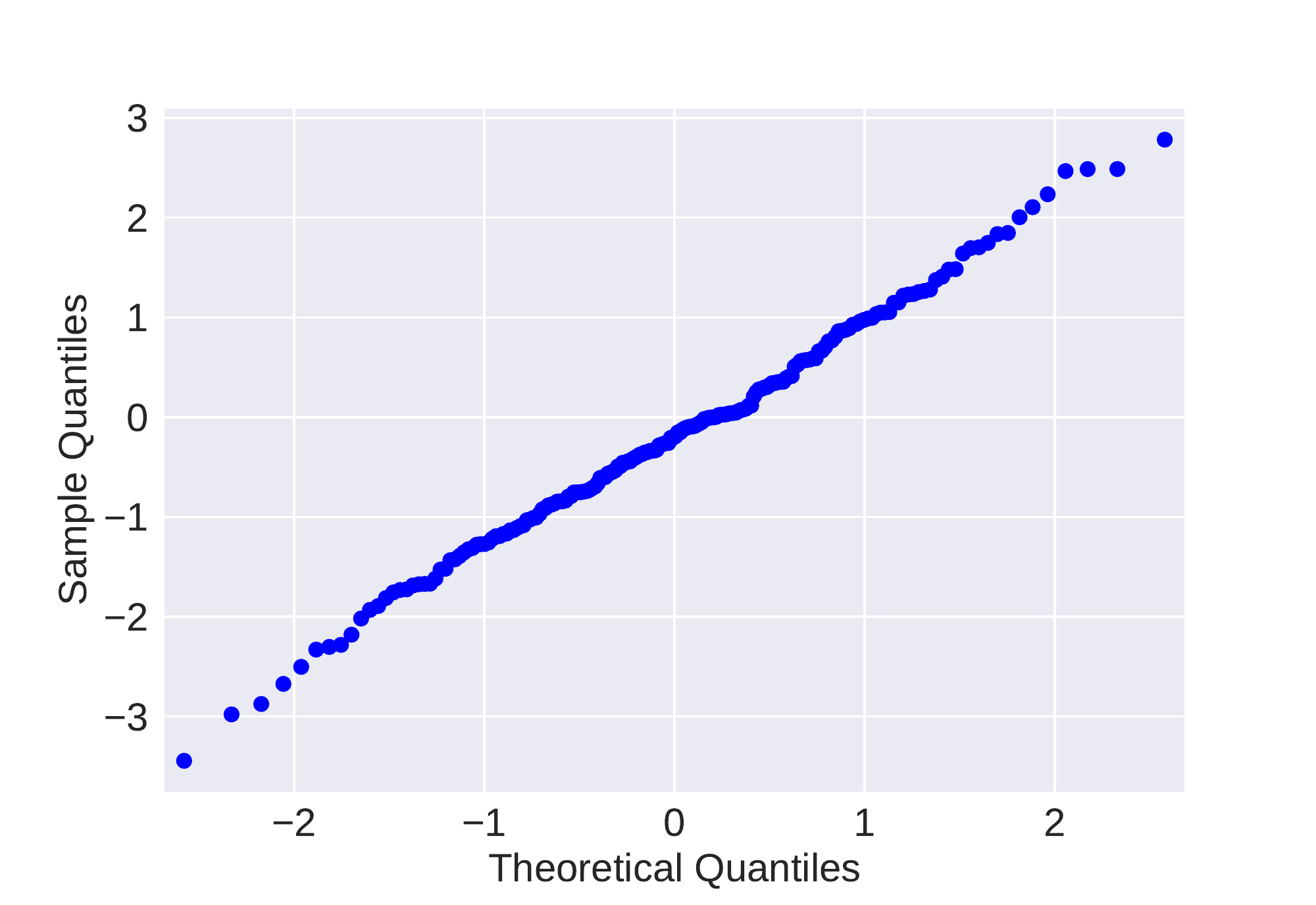}
\caption{Linear regression experiment 2: Q-Q plots per coordinate}
\label{fig:nips2017:appendix:linear2:qq-coord}
\end{figure}

Q-Q plots per coordinate for samples in logistic regression experiment 1 is shown in
\Cref{fig:nips2017:appendix:logistic1:qq-coord}.
Q-Q plots per coordinate for samples in logistic regression experiment 2 is shown in
\Cref{fig:nips2017:appendix:logistic2:qq-coord}.

\begin{figure}[t]
\centering
\includegraphics[width=0.19\textwidth]{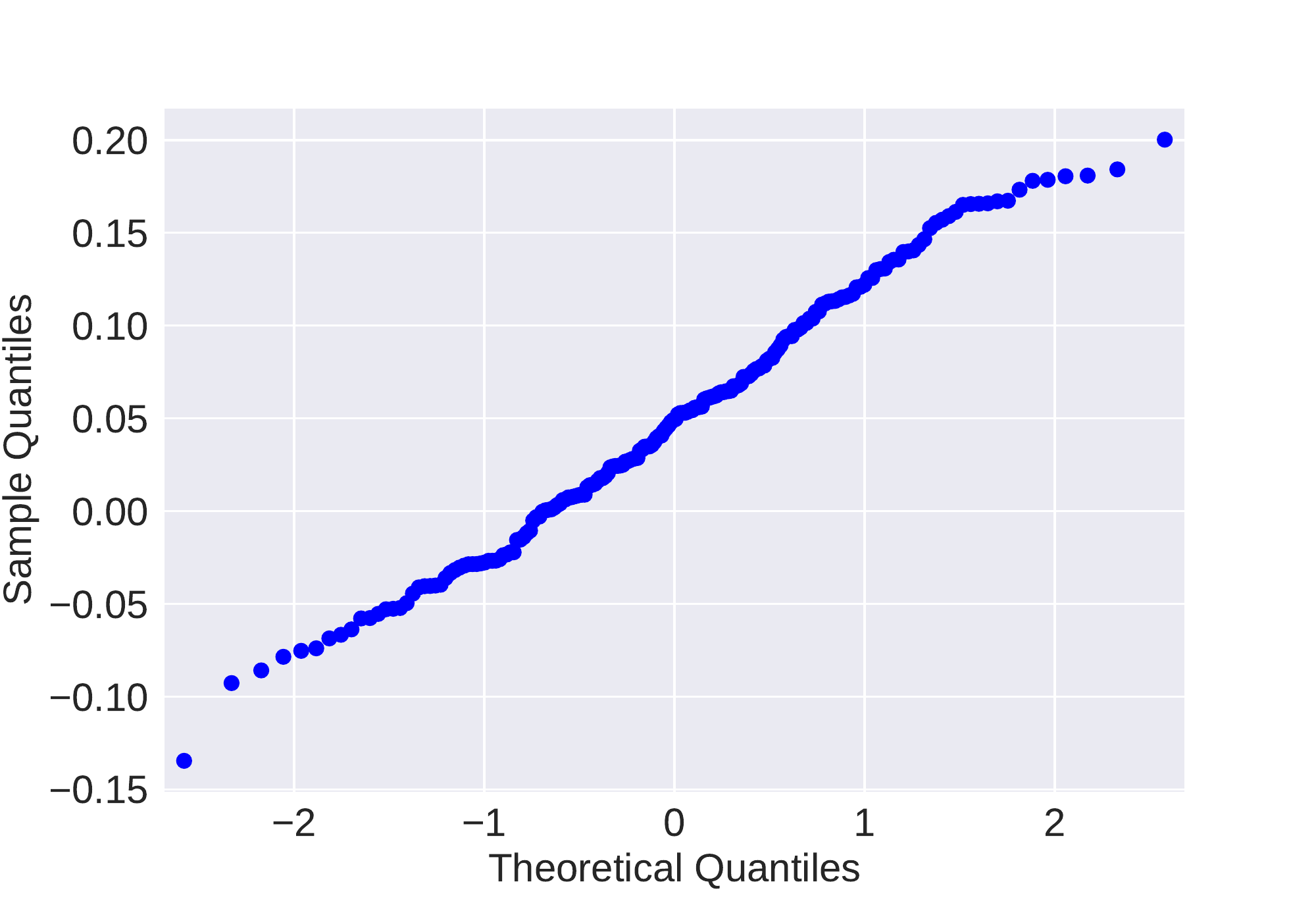}
\includegraphics[width=0.19\textwidth]{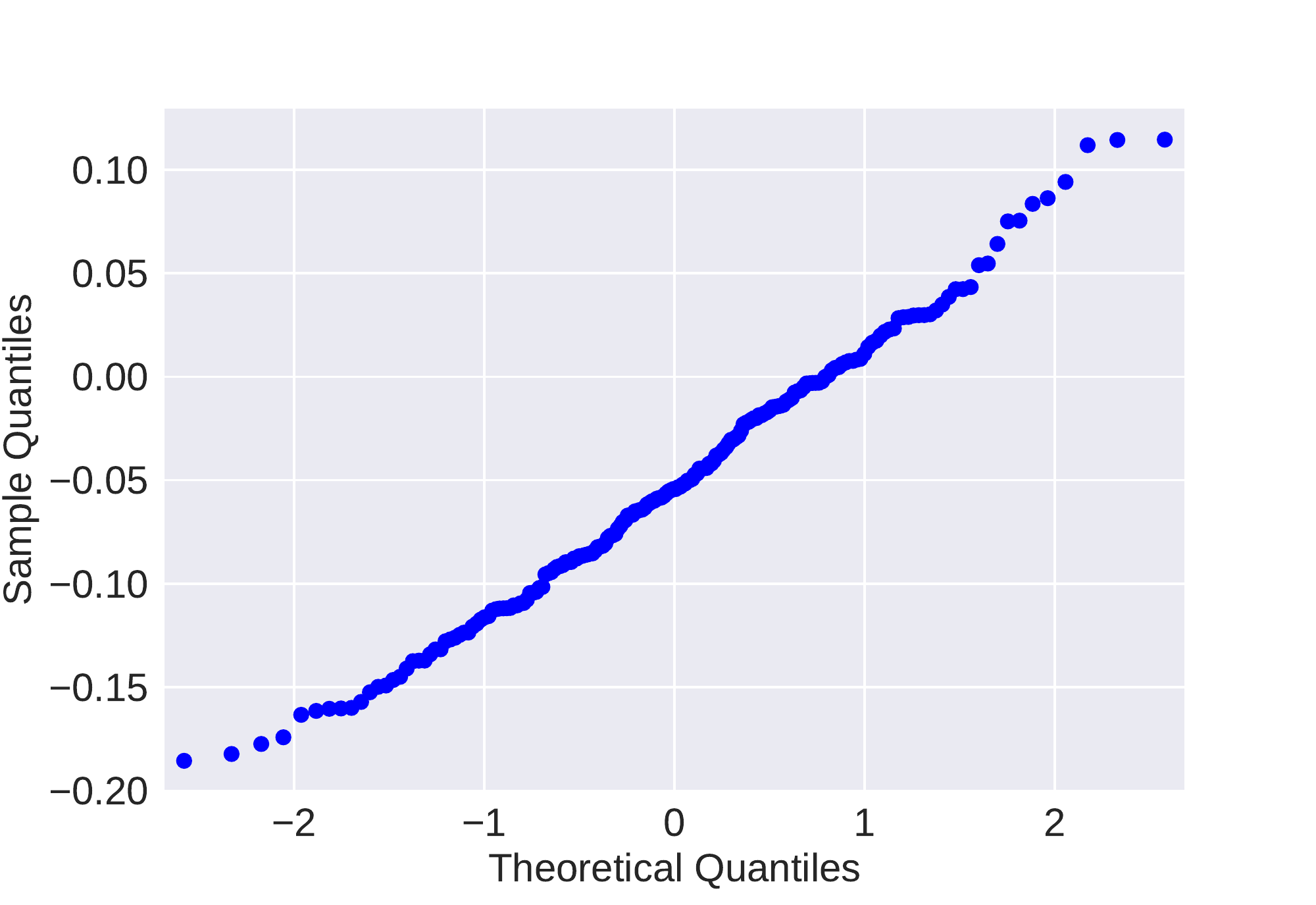}
\includegraphics[width=0.19\textwidth]{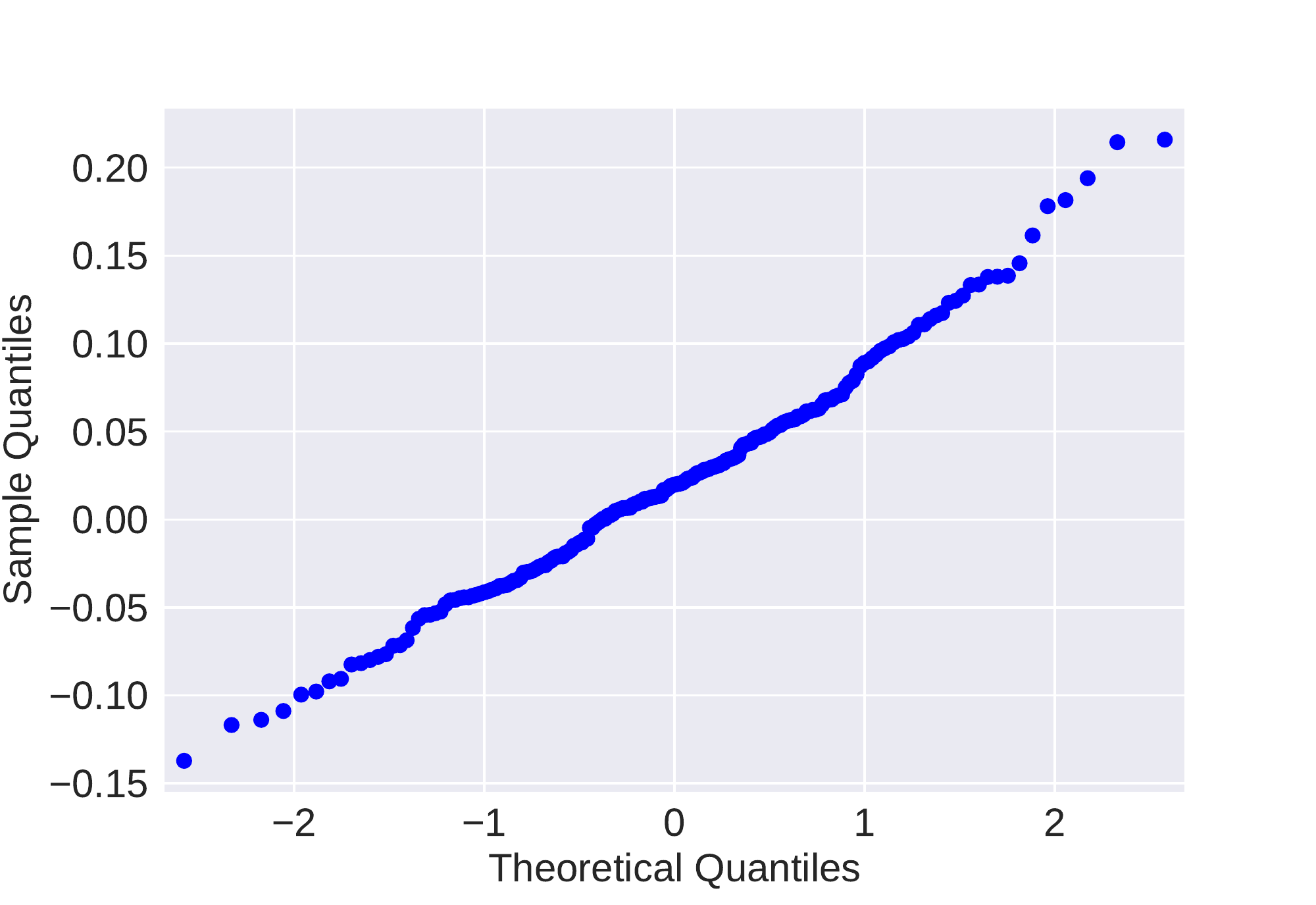}
\includegraphics[width=0.19\textwidth]{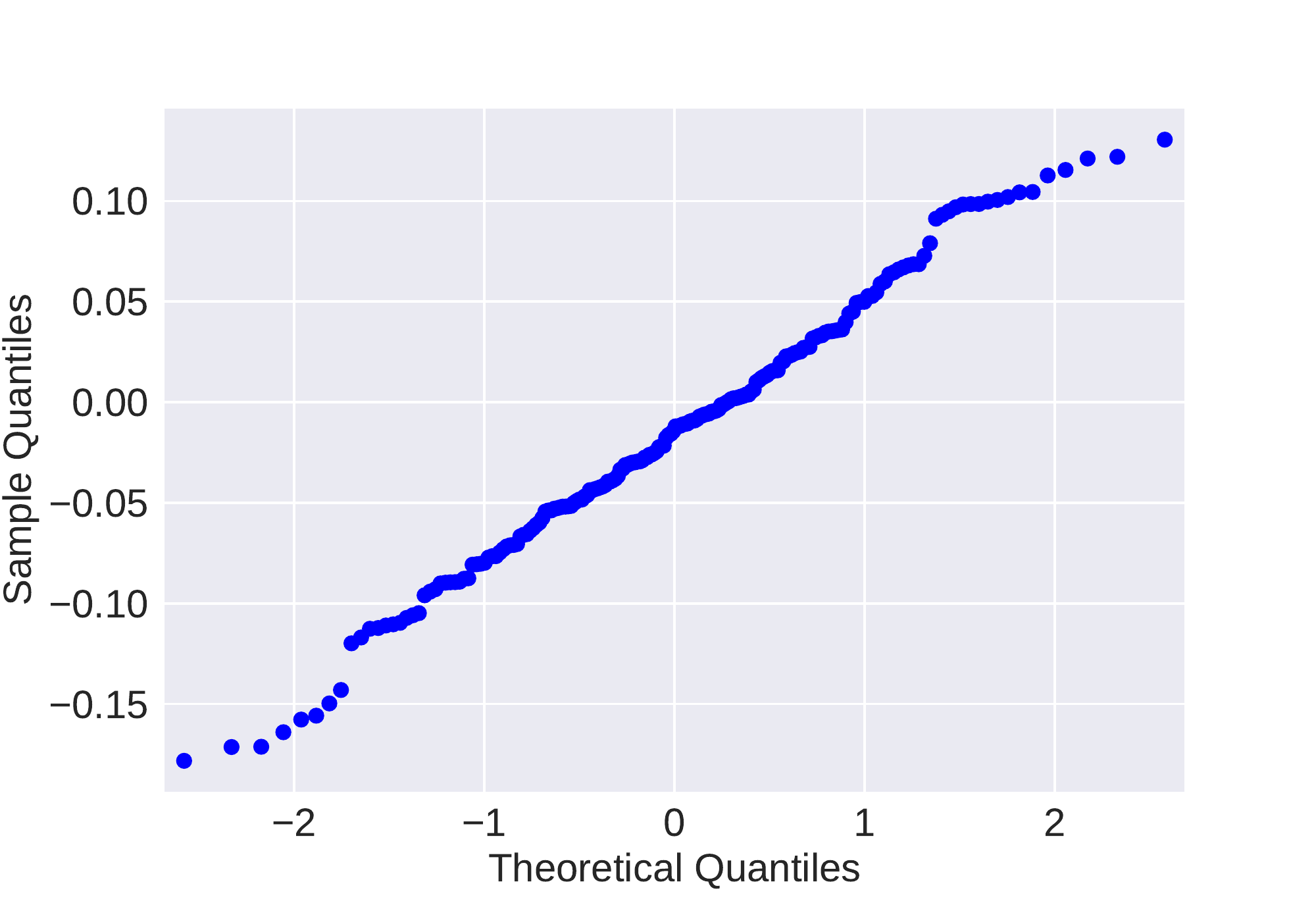}
\includegraphics[width=0.19\textwidth]{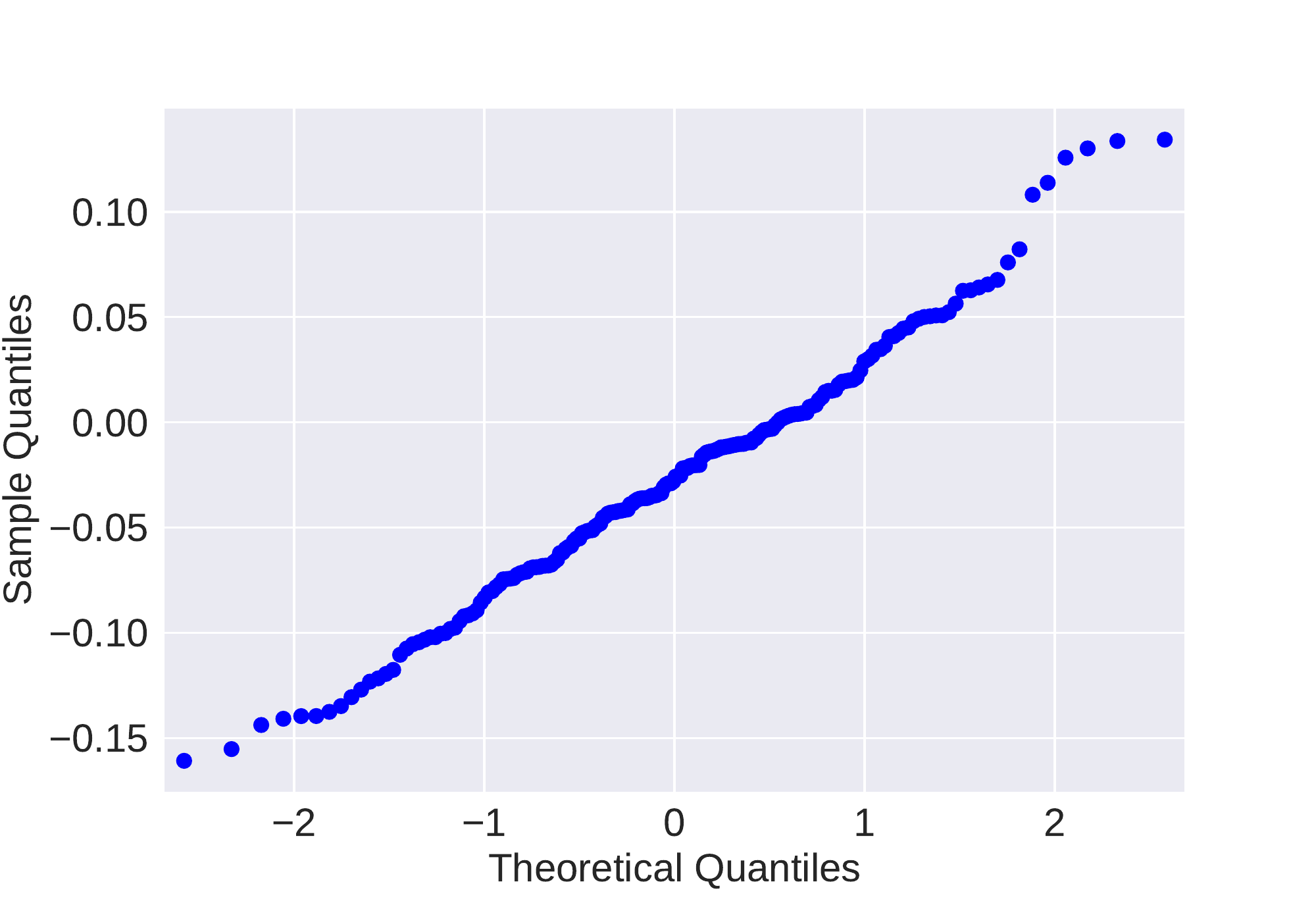}
\\
\includegraphics[width=0.19\textwidth]{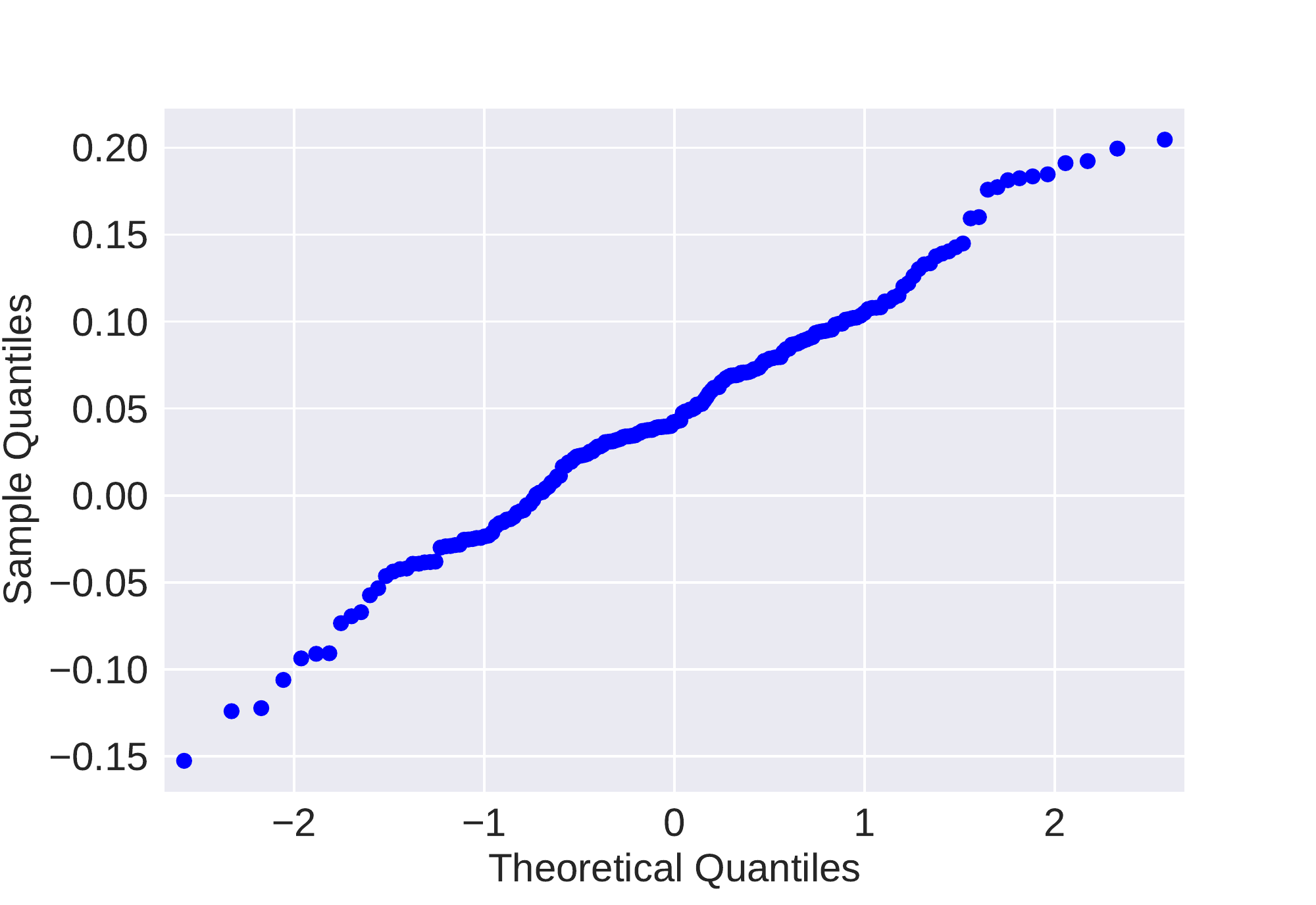}
\includegraphics[width=0.19\textwidth]{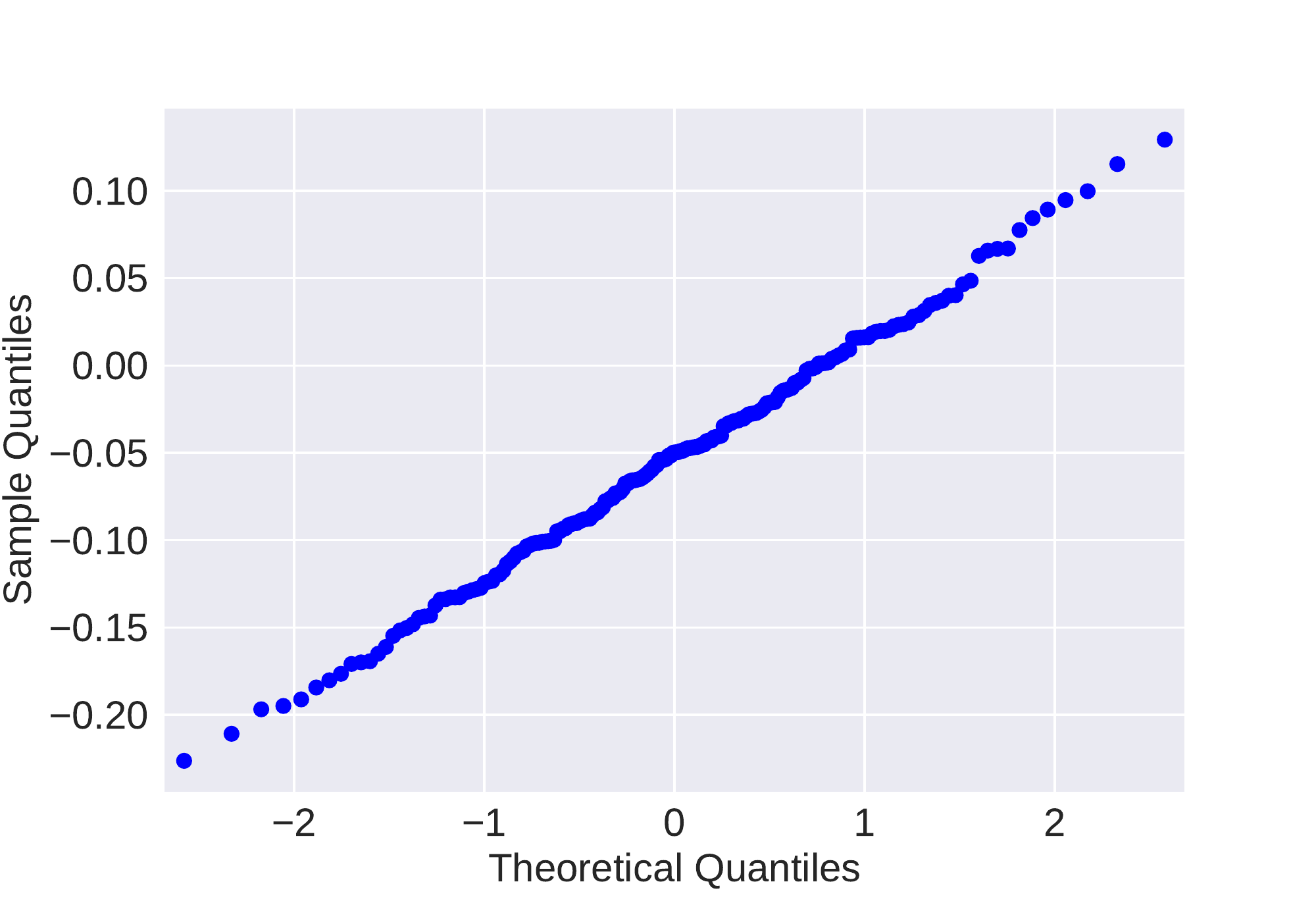}
\includegraphics[width=0.19\textwidth]{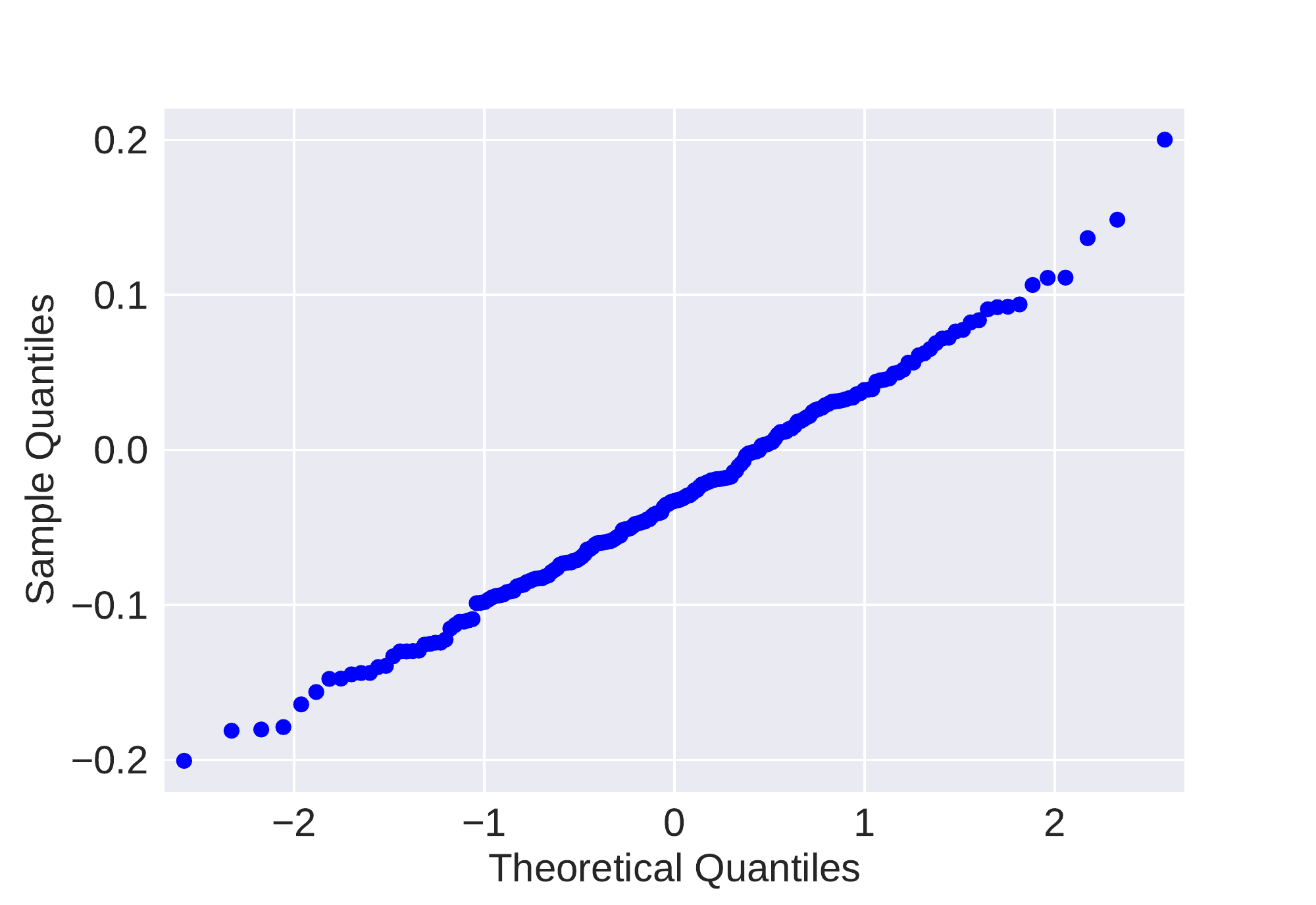}
\includegraphics[width=0.19\textwidth]{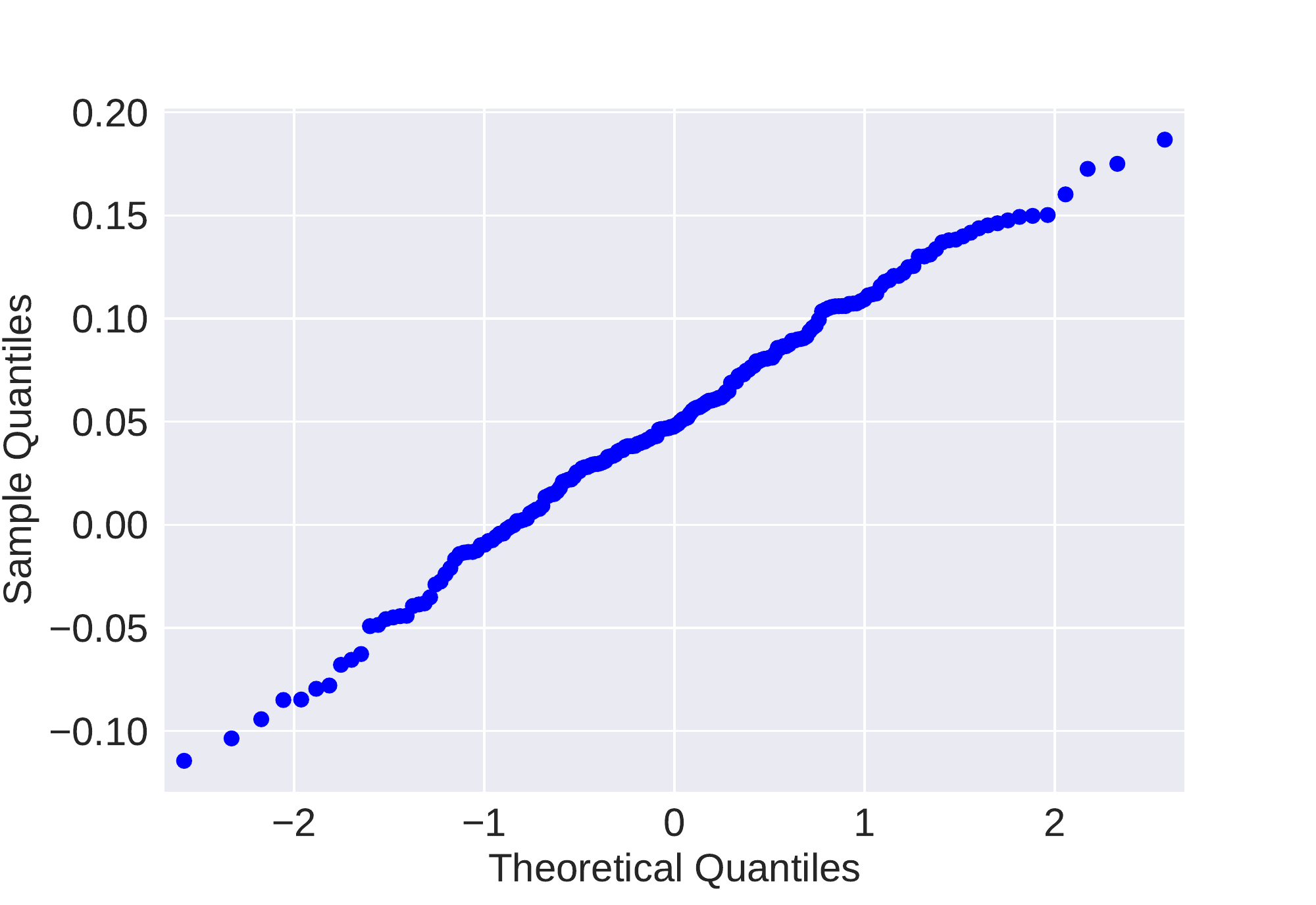}
\includegraphics[width=0.19\textwidth]{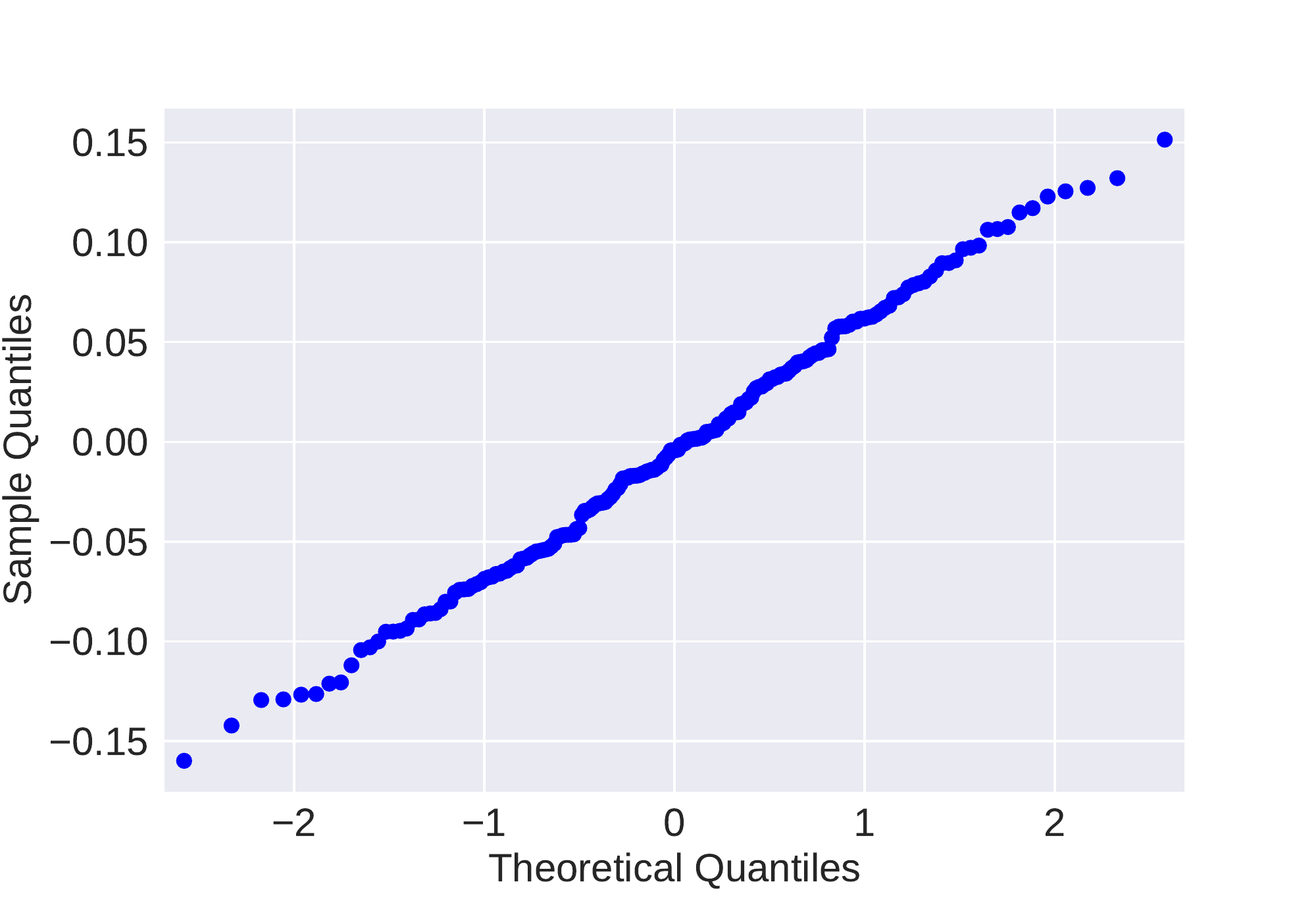}
\caption{Logistic regression experiment 1: Q-Q plots per coordinate}
\label{fig:nips2017:appendix:logistic1:qq-coord}
\end{figure}

\begin{figure}[t]
\centering
\includegraphics[width=0.19\textwidth]{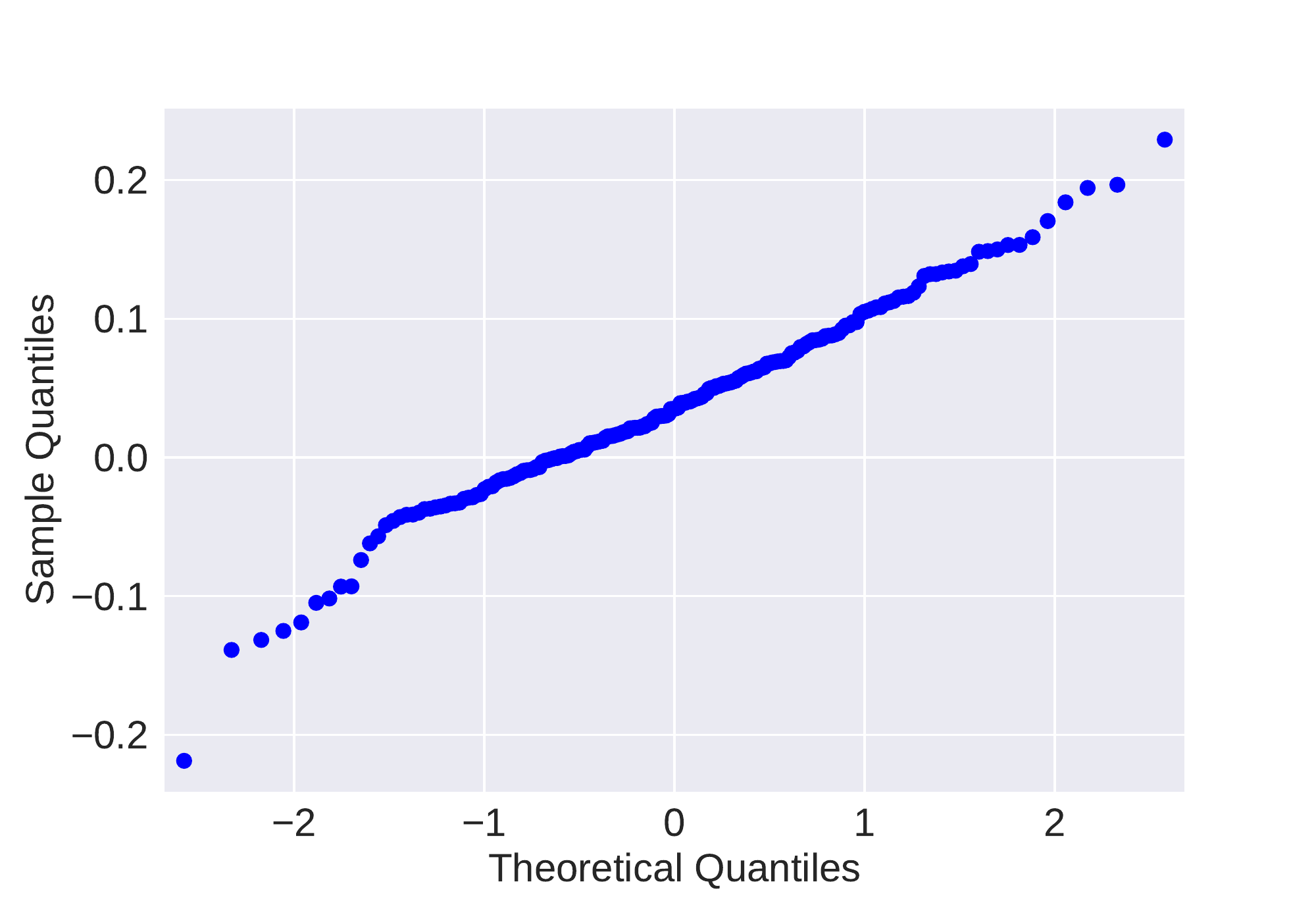}
\includegraphics[width=0.19\textwidth]{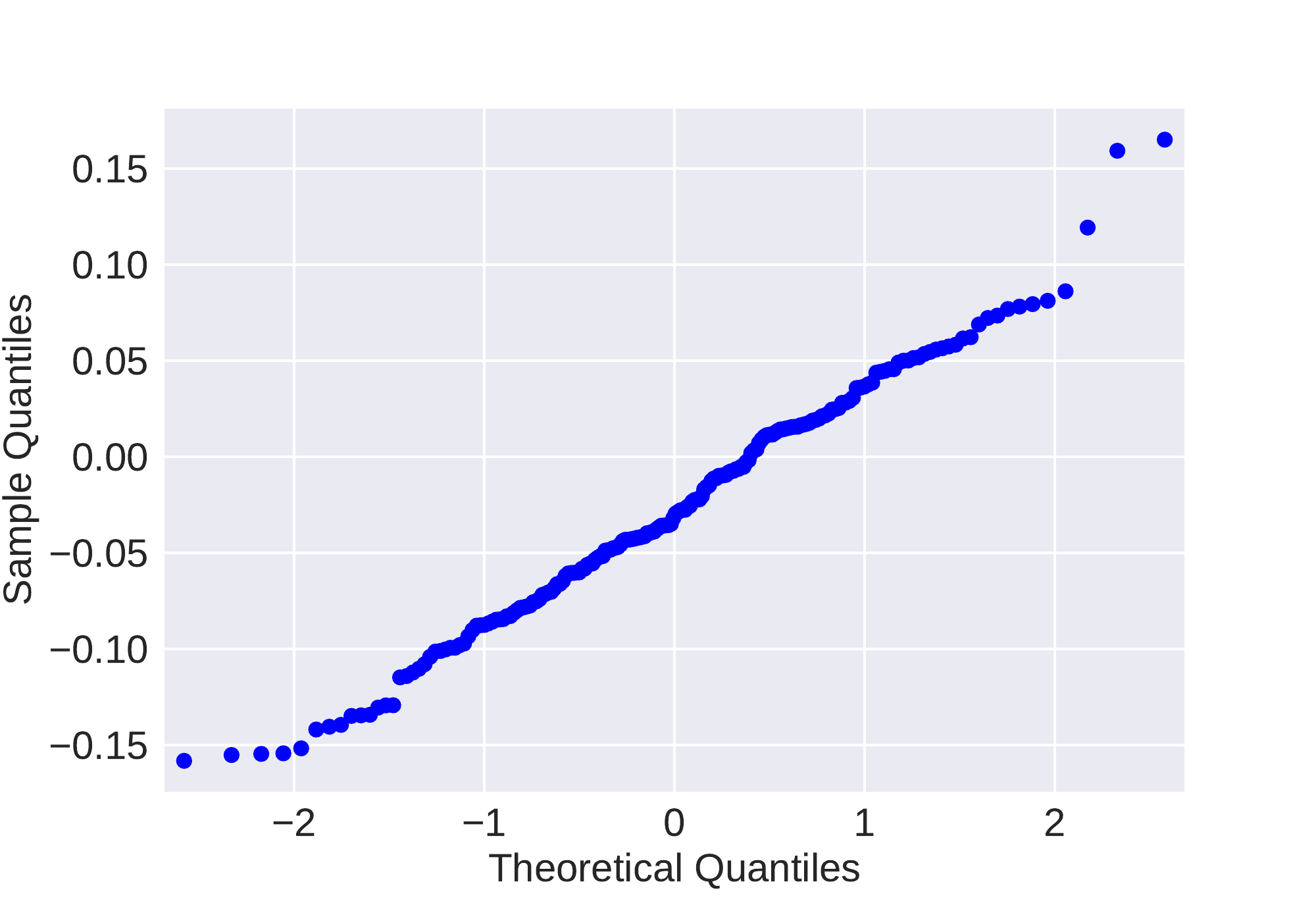}
\includegraphics[width=0.19\textwidth]{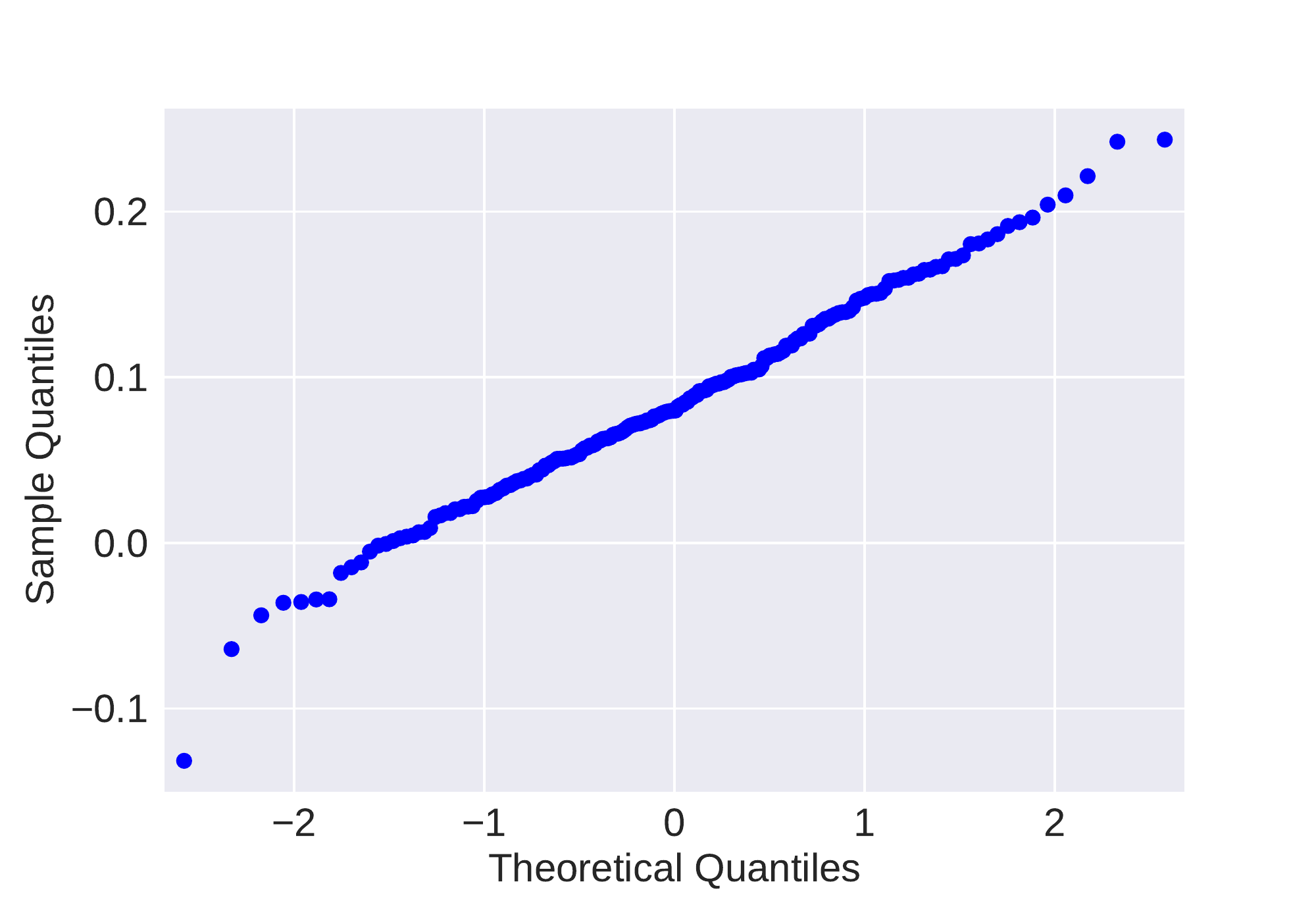}
\includegraphics[width=0.19\textwidth]{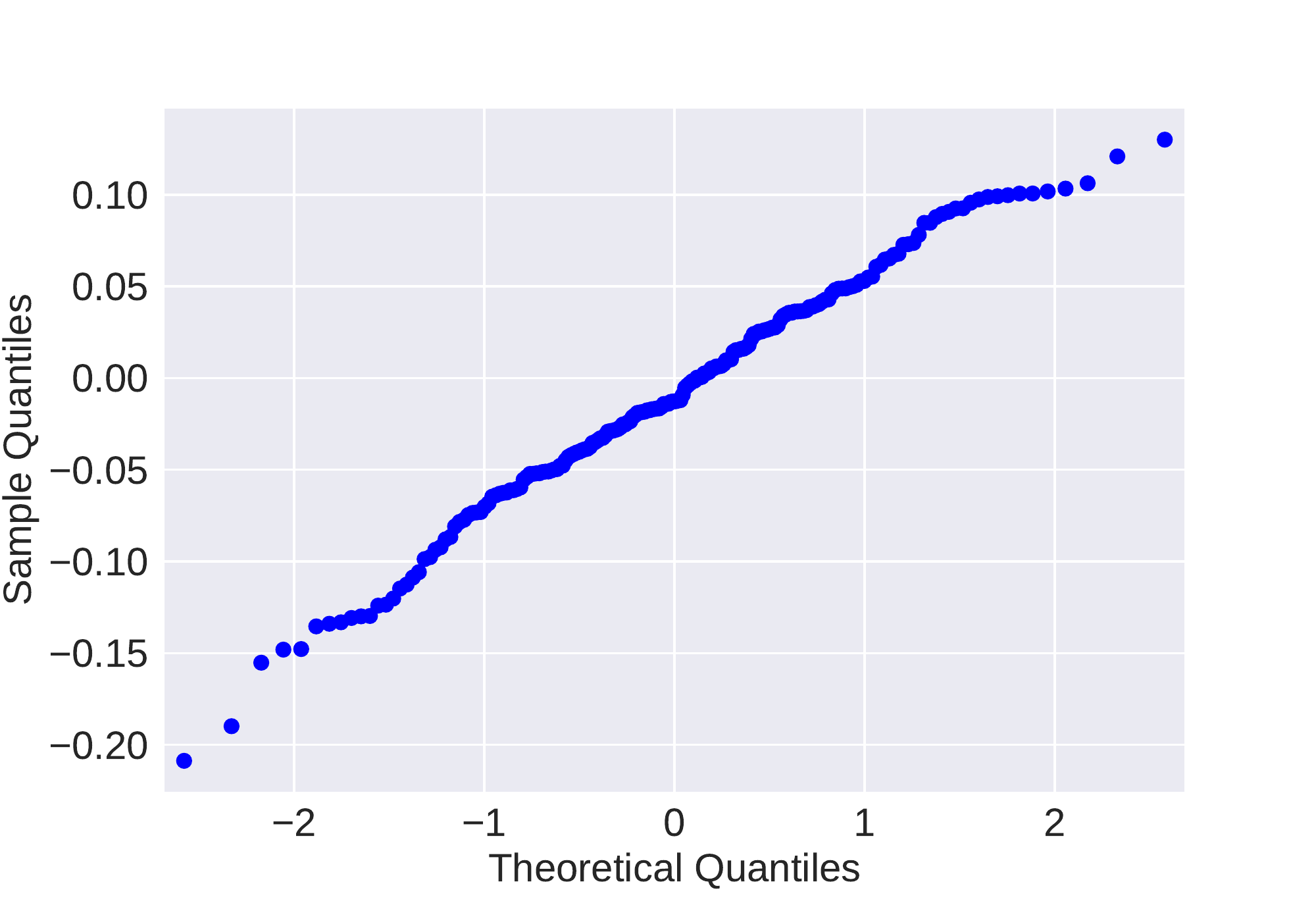}
\includegraphics[width=0.19\textwidth]{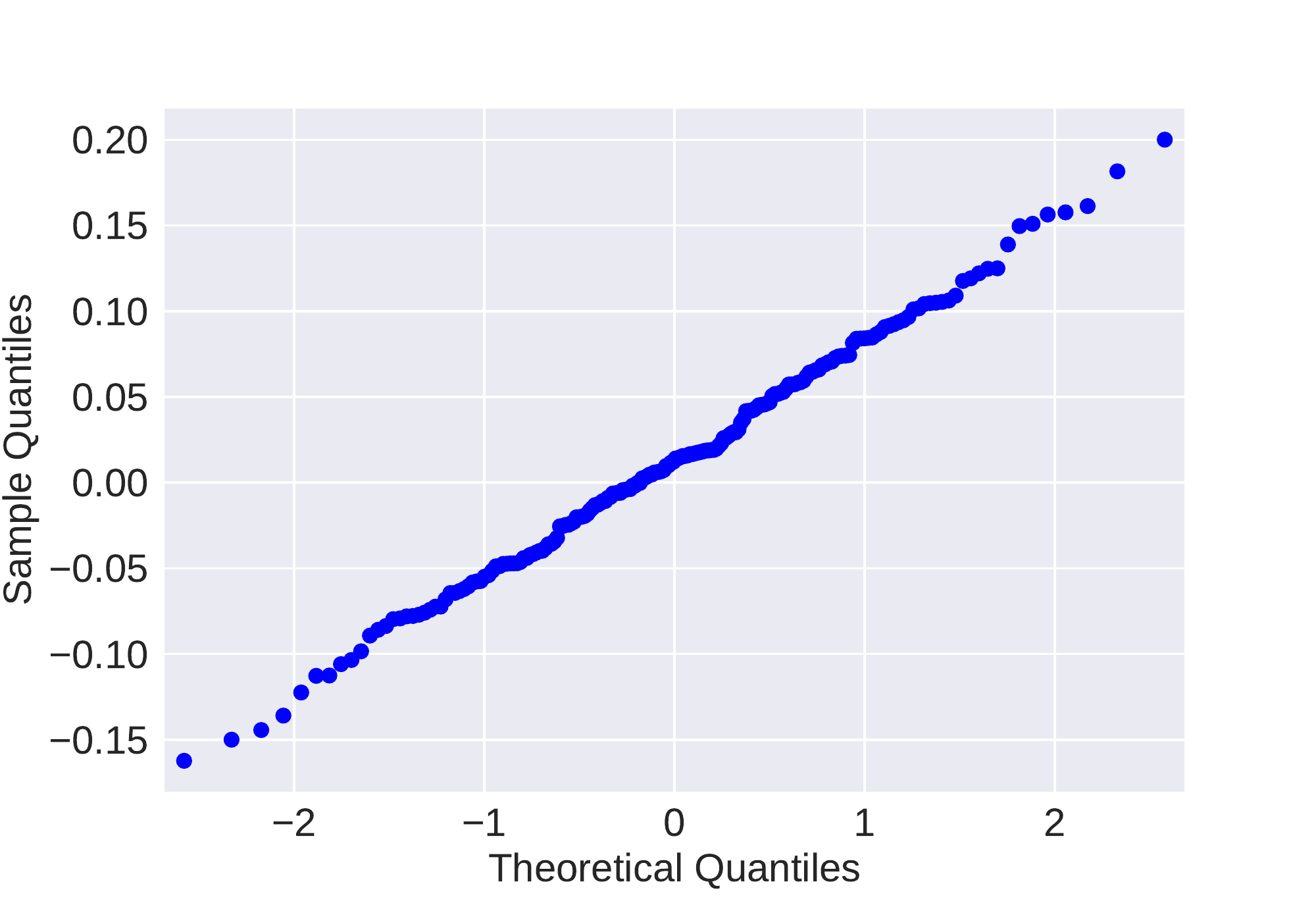}
\\
\includegraphics[width=0.19\textwidth]{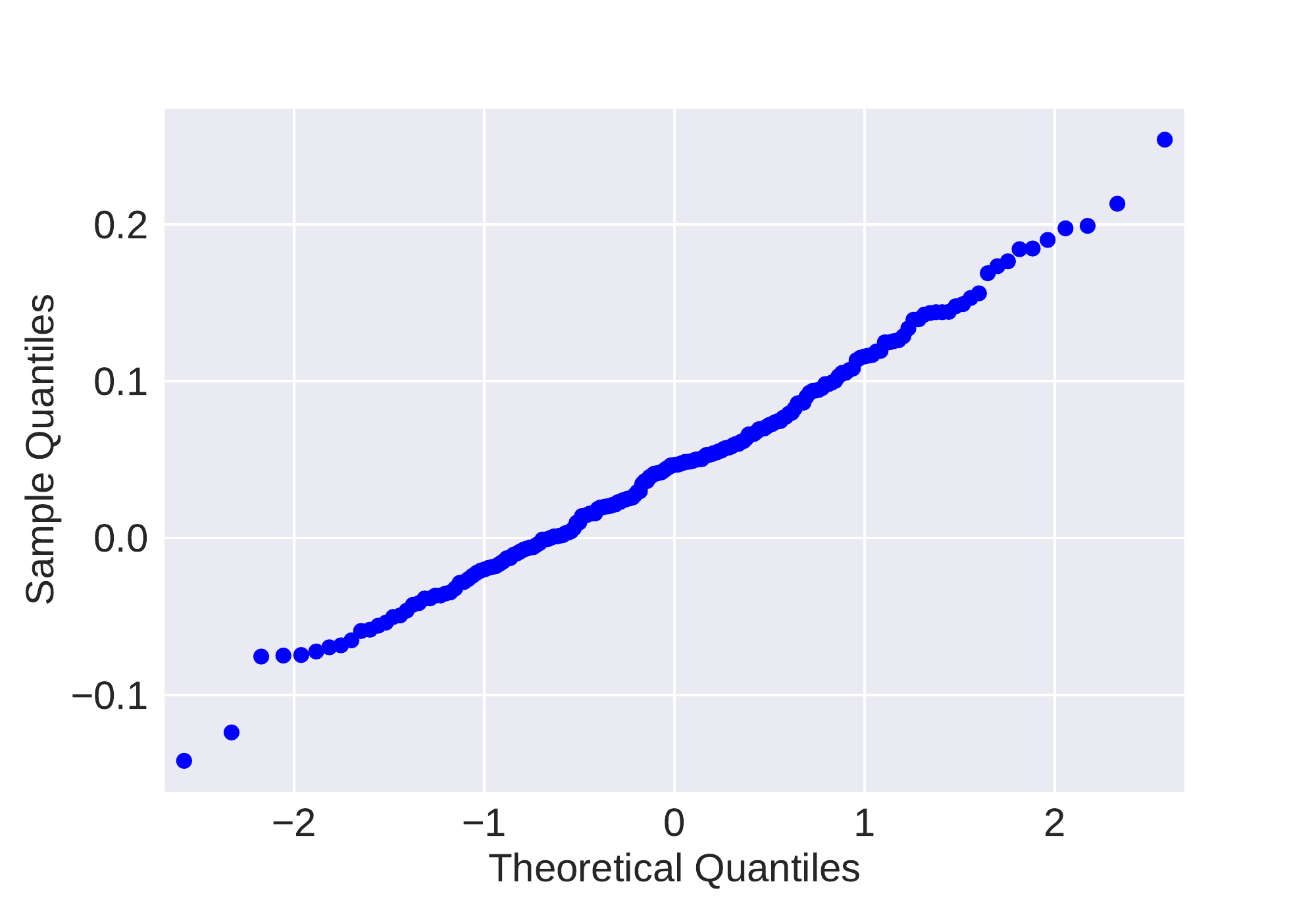}
\includegraphics[width=0.19\textwidth]{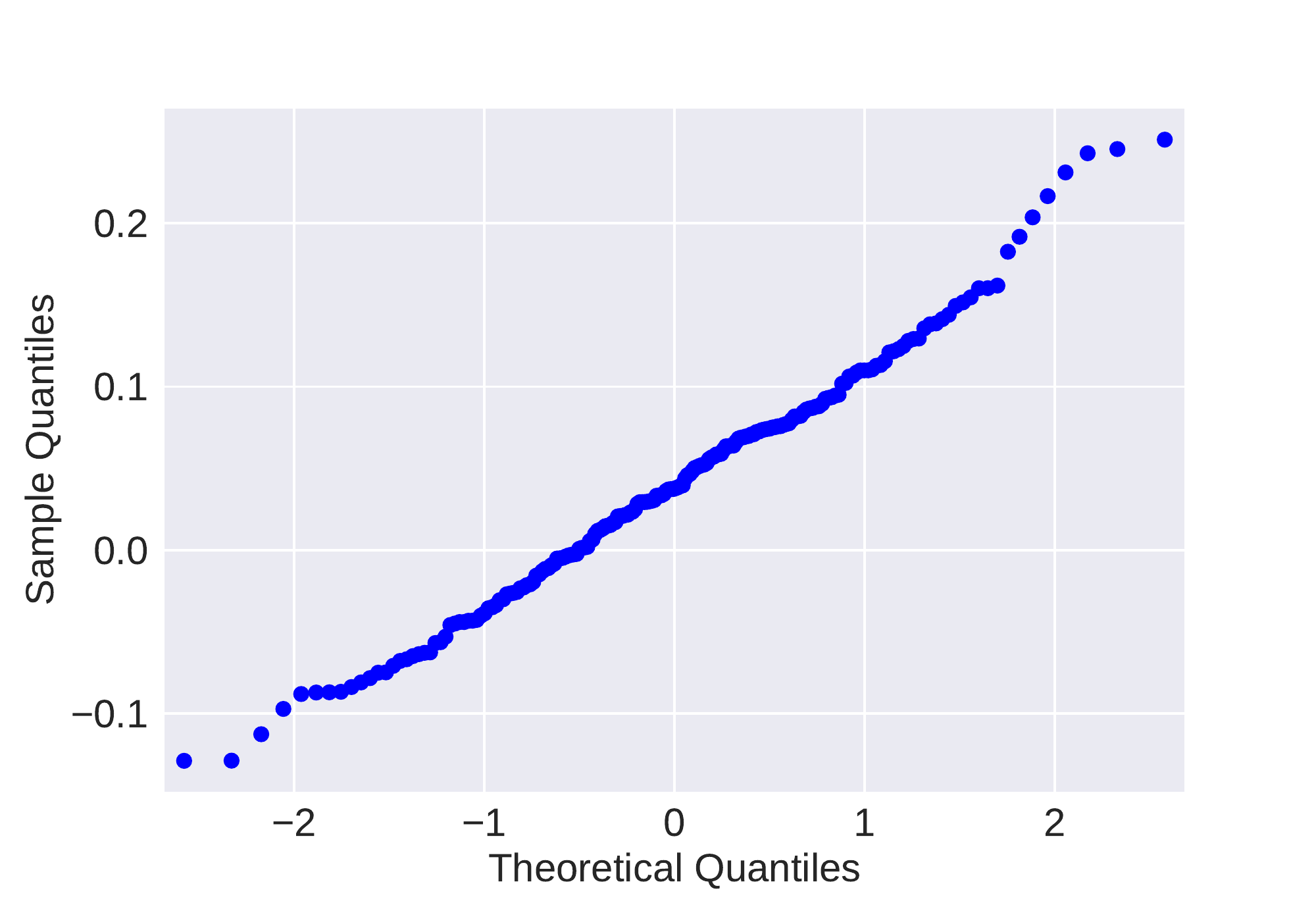}
\includegraphics[width=0.19\textwidth]{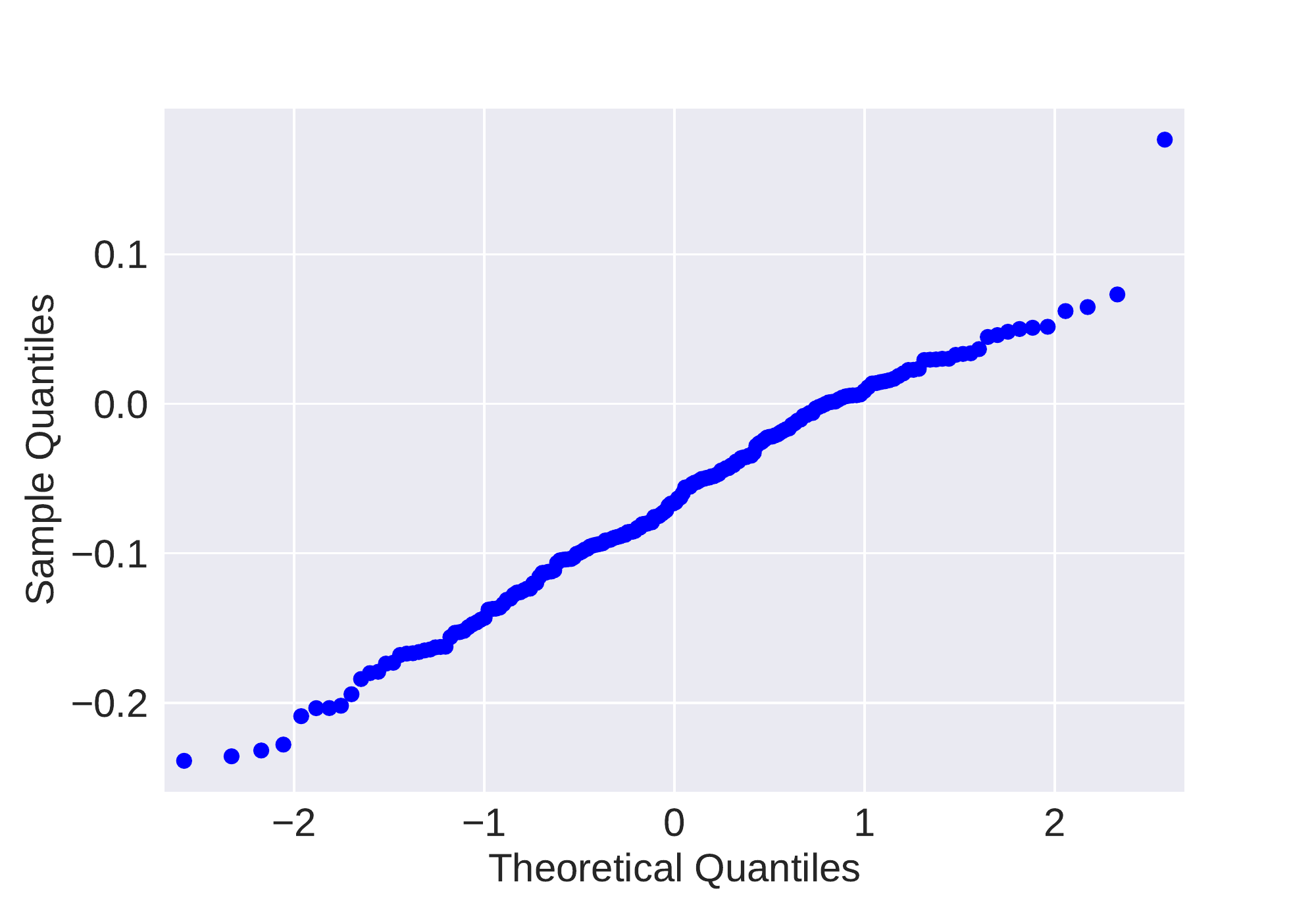}
\includegraphics[width=0.19\textwidth]{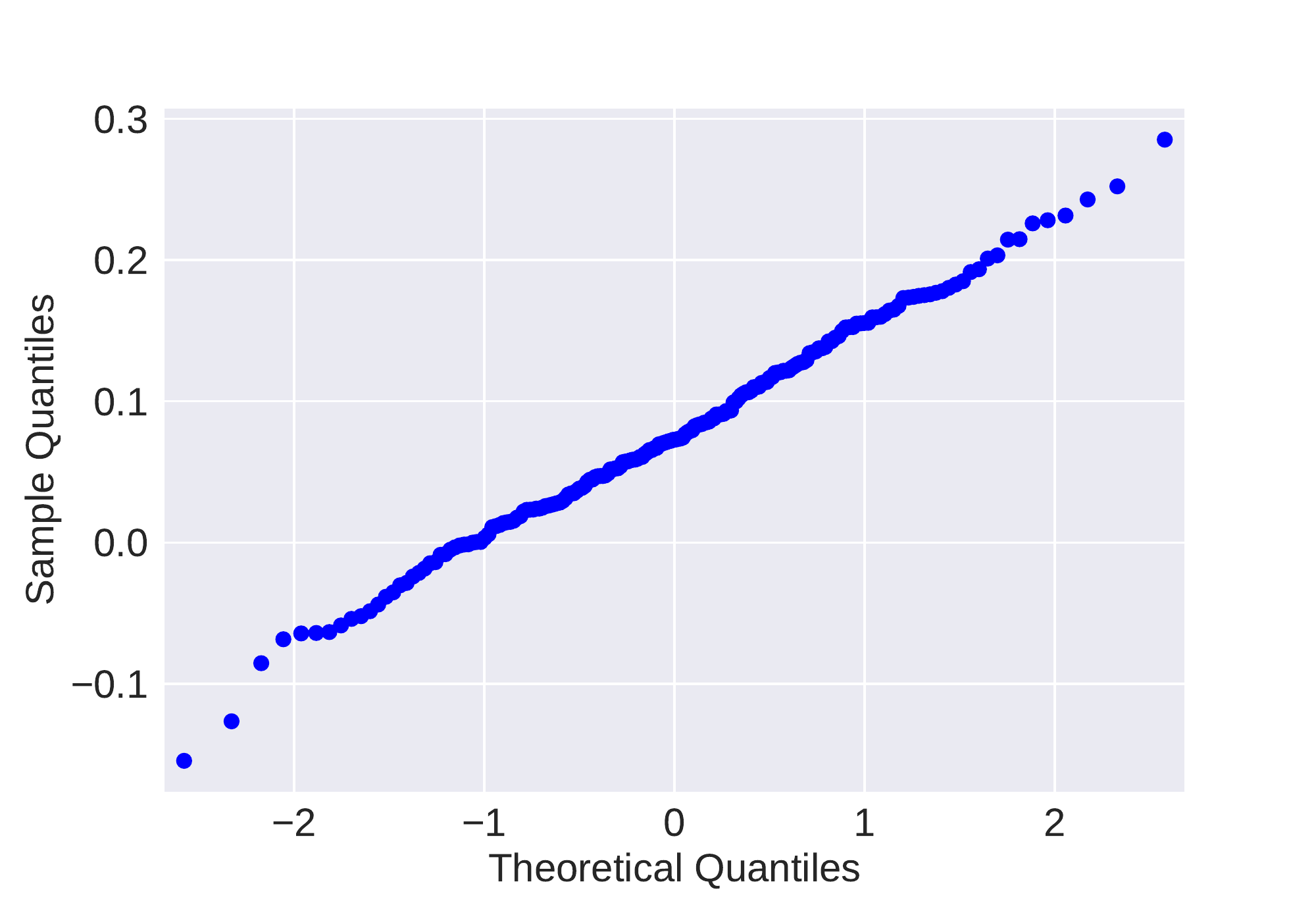}
\includegraphics[width=0.19\textwidth]{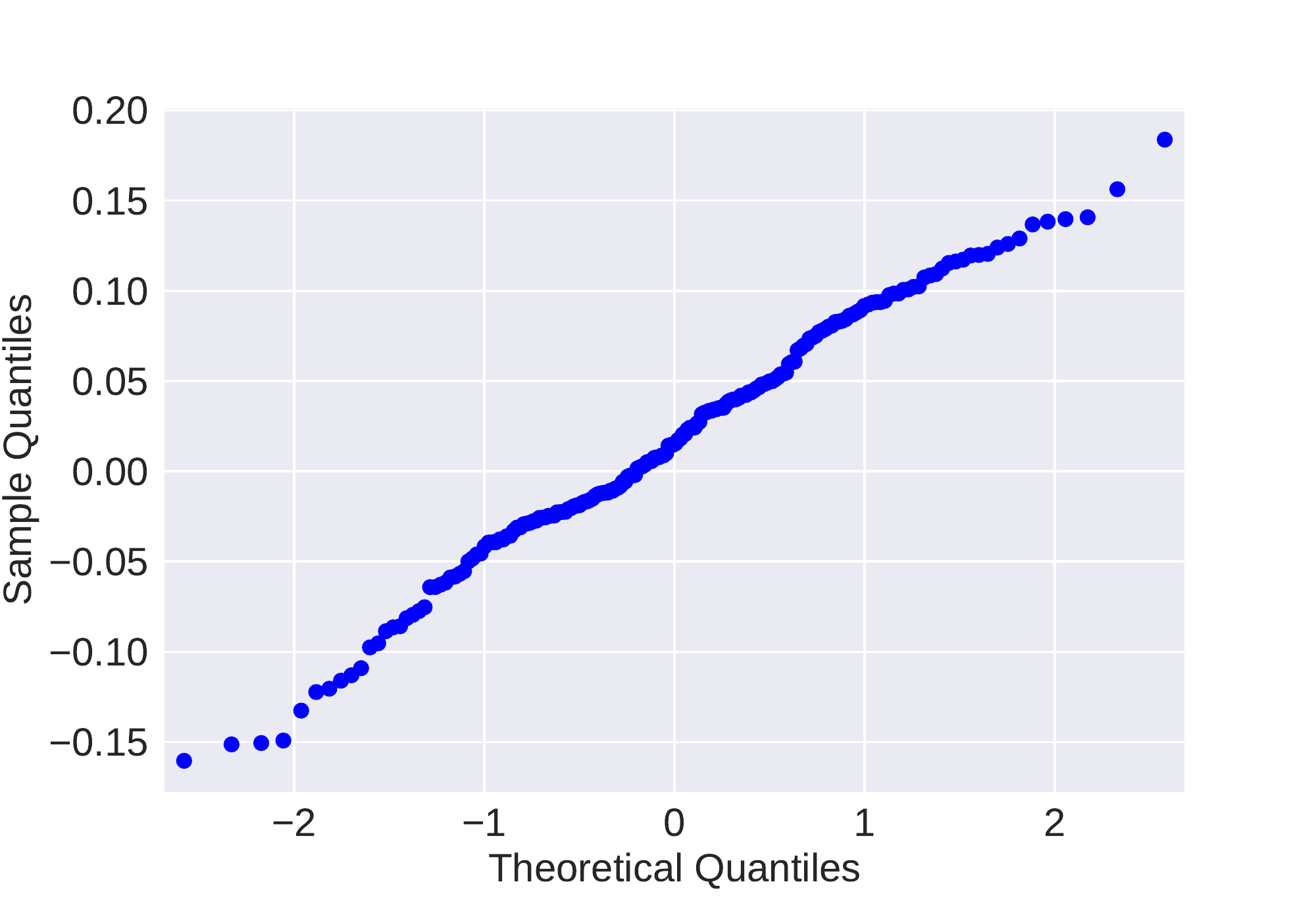}
\caption{Logistic regression experiment 2: Q-Q plots per coordinate}
\label{fig:nips2017:appendix:logistic2:qq-coord}
\end{figure}

\paragraph{Additional experiments}\

{\em 2-Dimensional Linear Regression.}
Consider:
\begin{align*}
y = x_1 + x_2 + \epsilon, \quad \quad \text{where $\begin{bmatrix}
x_1 \\ x_2
\end{bmatrix} \sim \Norm\left(0, \begin{bmatrix}
1 & 0.8 \\
0.8 & 1
\end{bmatrix}\right)$ and $\epsilon \sim \Norm(0, \sigma^2 = 10^2)$.}
\end{align*}

Each sample consists of $Y = y$ and
$X = [x_1, ~~x_2]^\top$.
We use linear regression to estimate $w_1, w_2$ in $y=w_1 x_1 + w_2 x_2$.
In this case, the minimizer of the population least square
risk is $w_1^\star = 1, w_2^\star = 1$.

For 300 i.i.d. samples, we plotted 100 samples from SGD inference in Figure \ref{fig:exp:sim:2d-linear}. We compare our SGD inference procedure against bootstrap in Figure \ref{fig:exp:sim:2d-linear:BS}. Figure \ref{fig:exp:sim:2d-linear:t} and Figure \ref{fig:exp:sim:2d-linear:eta} show samples from our SGD inference procedure with different parameters.

\begin{figure}[!b]
\centering
\subfloat[SGD inference vs. bootstrap\label{fig:exp:sim:2d-linear:BS}]{
\includegraphics[width=0.33\textwidth]{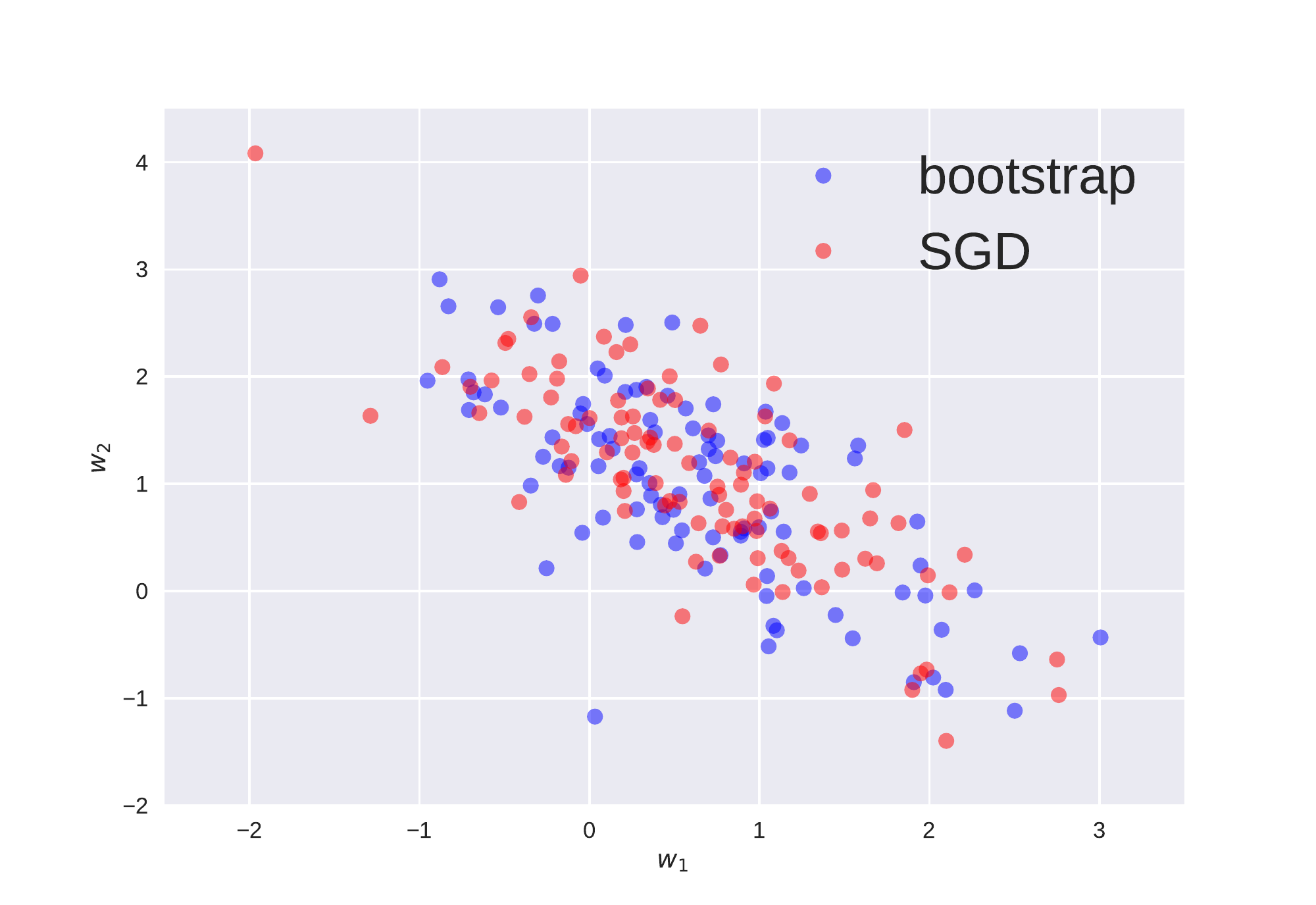}
}
\subfloat[$t=800$\label{fig:exp:sim:2d-linear:t}]{
\includegraphics[width=0.33\textwidth]{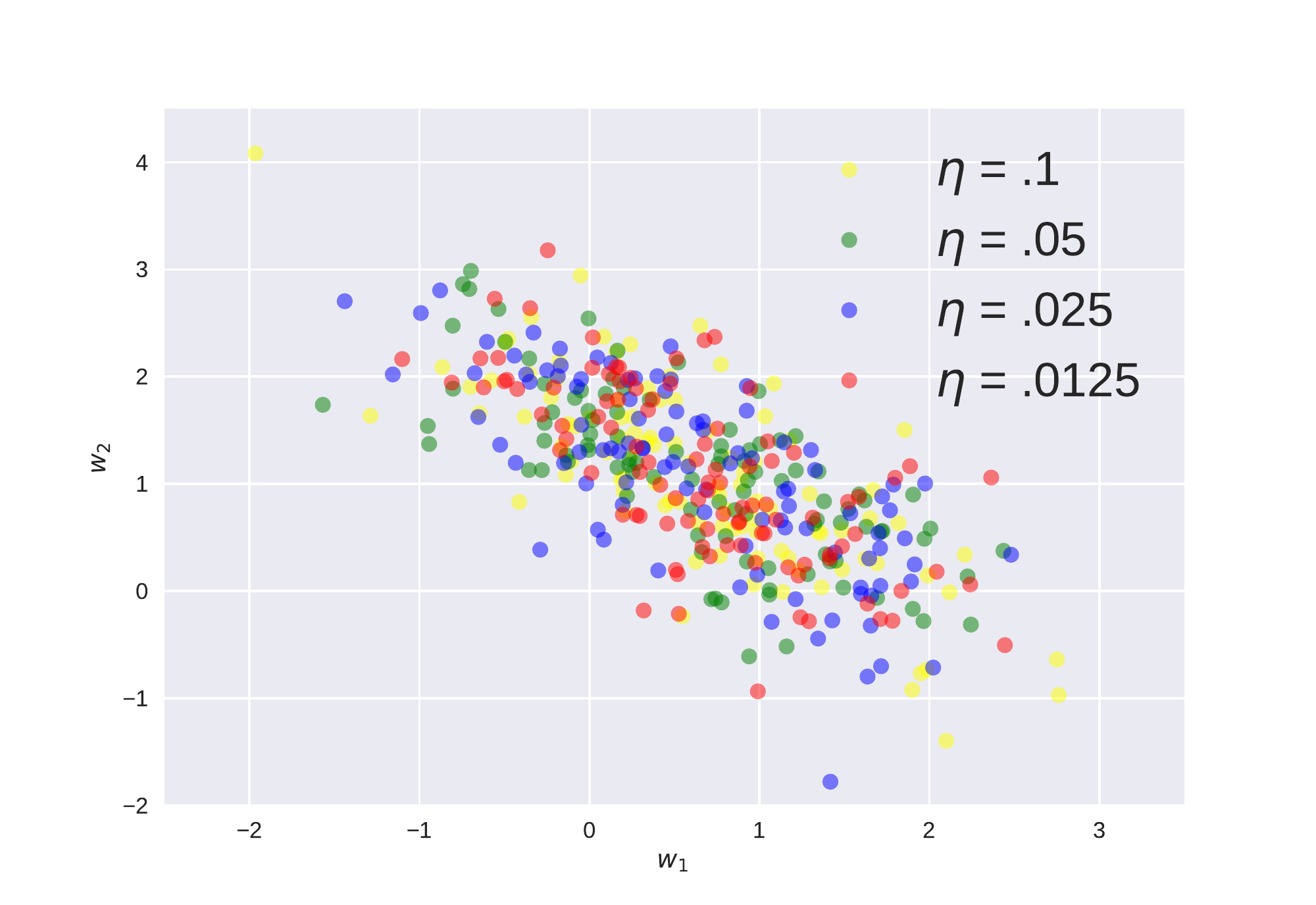}
}
\subfloat[$\eta=0.1$\label{fig:exp:sim:2d-linear:eta}]{
\includegraphics[width=0.33\textwidth]{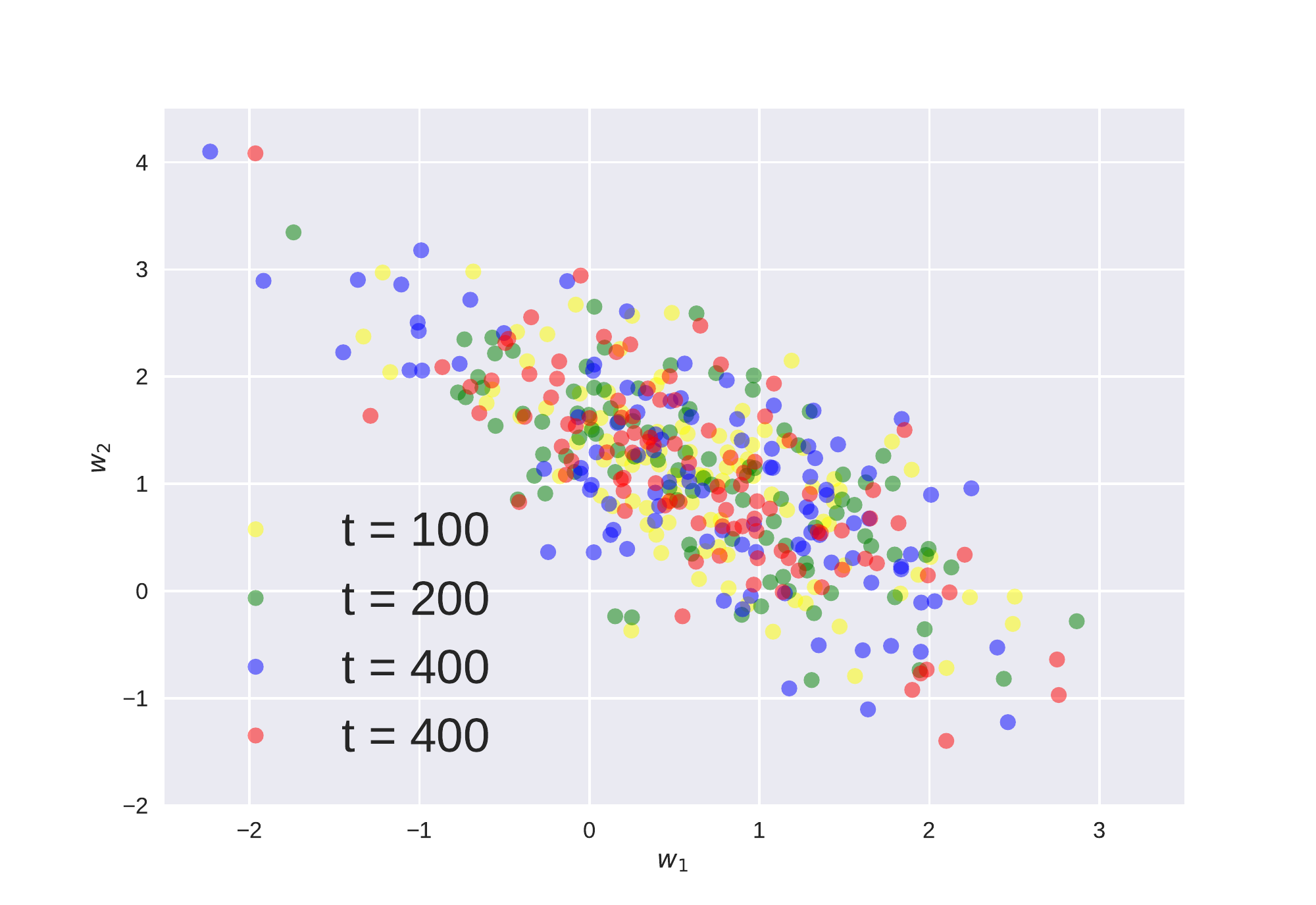}
}
\caption{2-dimensional linear regression}
\label{fig:exp:sim:2d-linear}
\end{figure}

{\em 10-Dimensional Linear Regression.}

Here we consider the following model
\begin{align*}
y = x^\top w^\star + \epsilon ,
\end{align*}
where ${w^\star} = \frac{1}{\sqrt{10}} [1,~~1,~~\cdots,~~1]^\top \in \Real^{10} $, $x \sim \Norm(0, \Sigma)$ with $\Sigma_{ij} = 0.8^{|i-j|}$,
and $\epsilon \sim \Norm(0, \sigma^2 = 20^2)$,
and use $n=1000$ samples.
We  estimate the parameter using
\begin{align*}
\widehat{w} = \argmin_w \frac{1}{n} \sum_{i=1}^n \tfrac{1}{2} (x_i^\top w - y_i)^2 .
\end{align*}

Figure \ref{fig:sim:10-dim-linear-sandwich-diagonal} shows the diagonal terms of of the covariance matrix computed using the sandwich estimator and our SGD inference procedure with different parameters.
100000 samples from our SGD inference procedure are used to reduce the effect of randomness.

\begin{figure}
\centering
\subfloat[$t=500$]{
\includegraphics[width=0.33\textwidth]{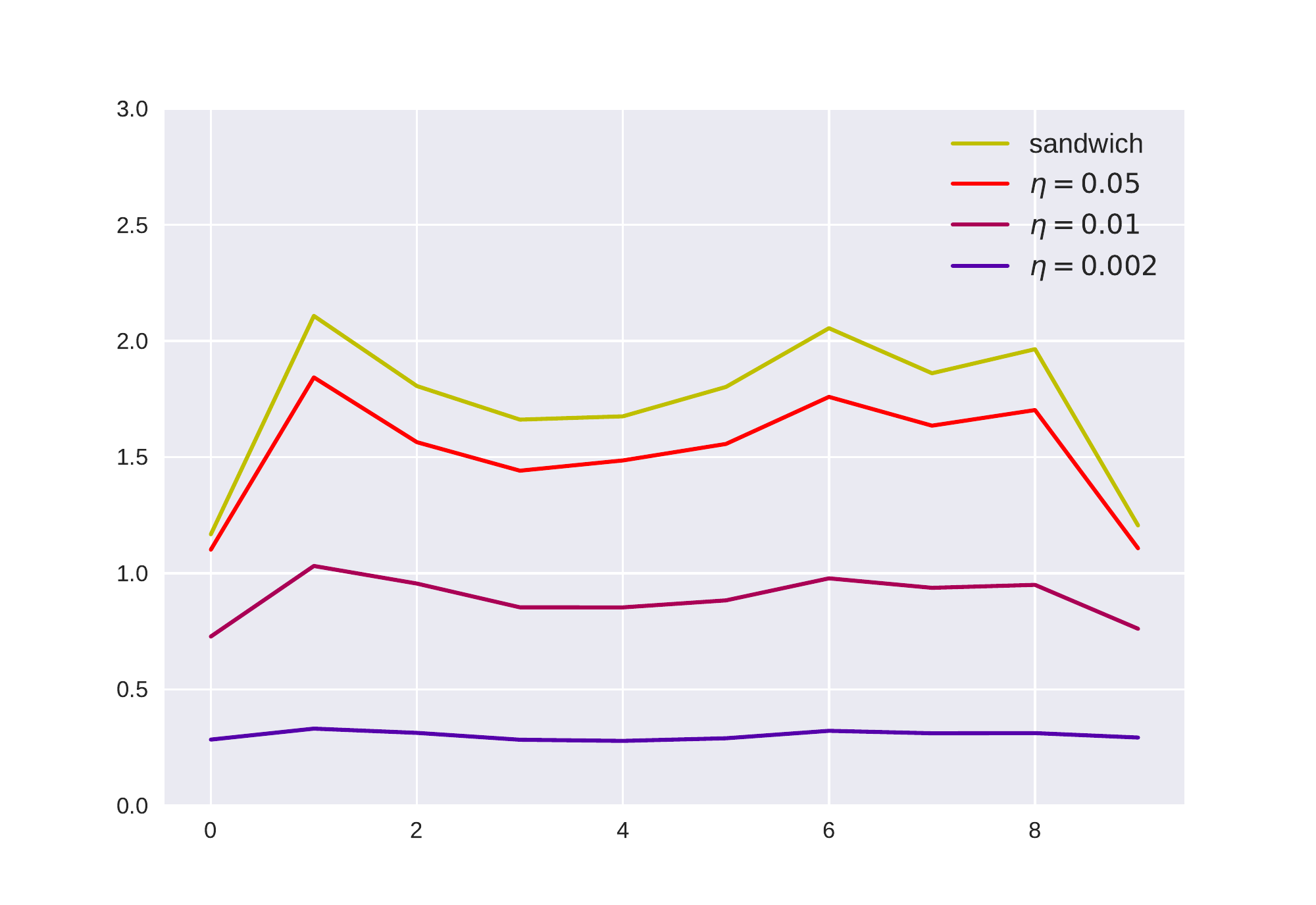}
}
~
\subfloat[$t=2500$]{
\includegraphics[width=0.33\textwidth]{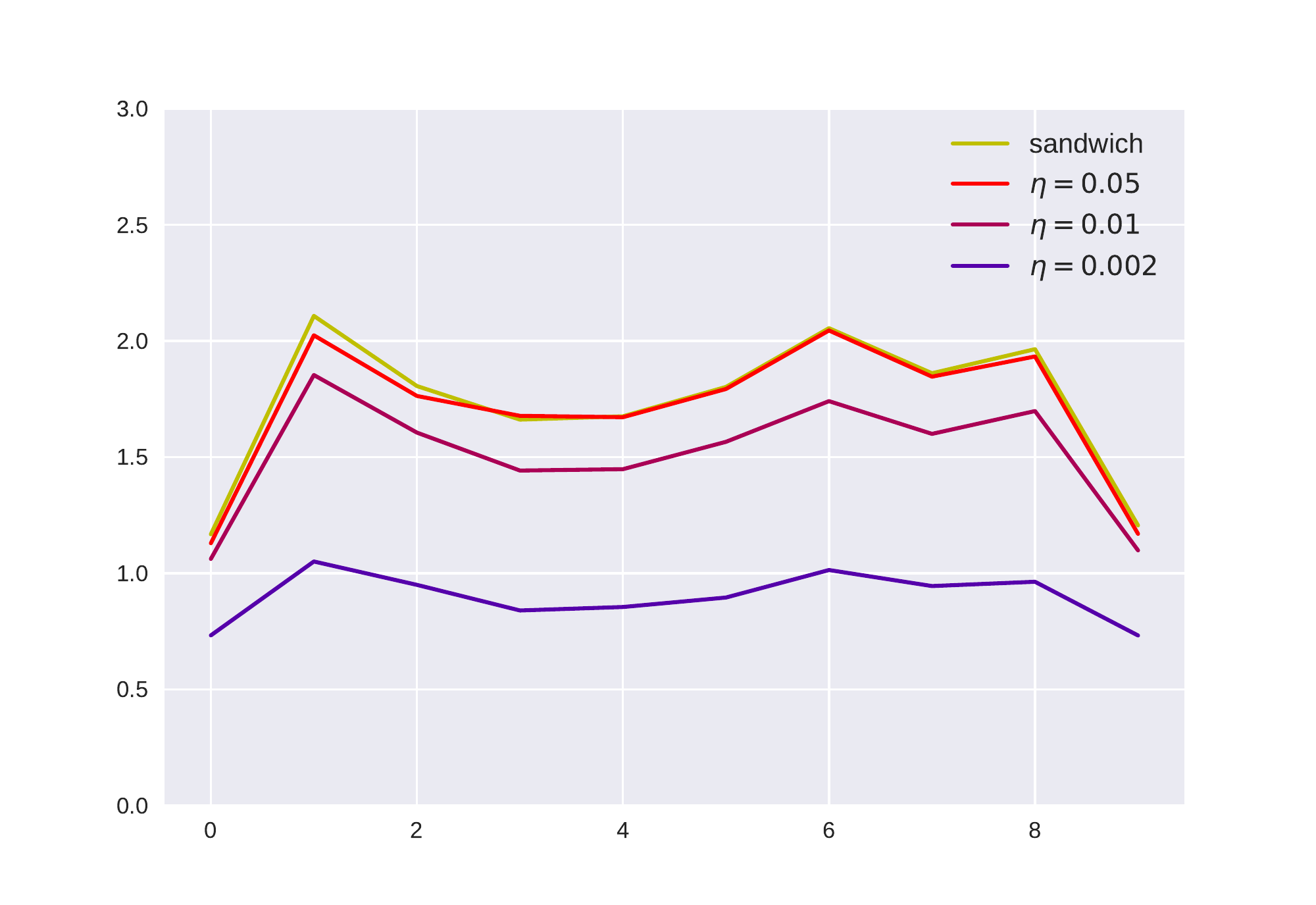}
}
~
\subfloat[$t=12500$]{
\includegraphics[width=0.33\textwidth]{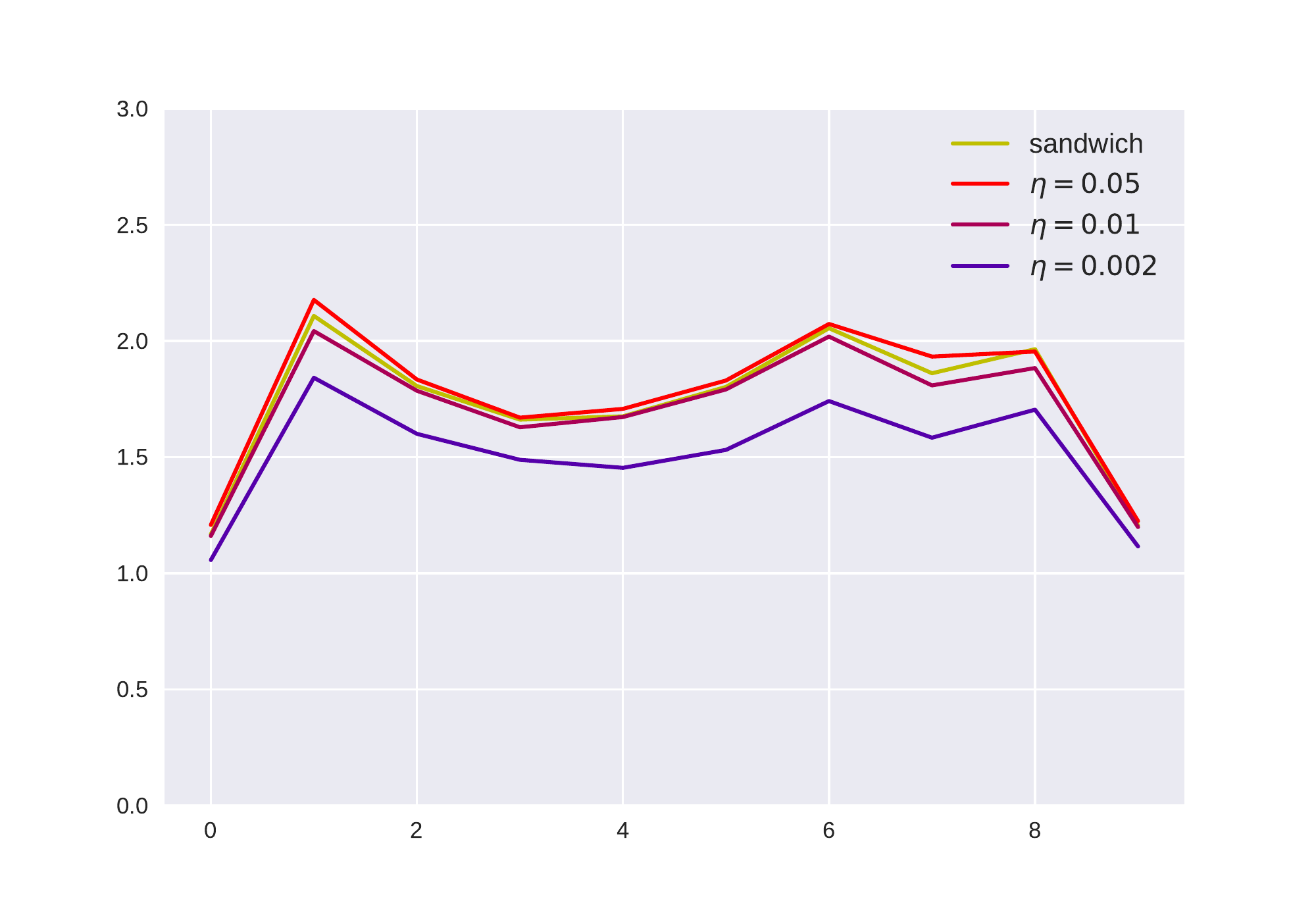}
}

\subfloat[$\eta=0.05$]{
\includegraphics[width=0.33\textwidth]{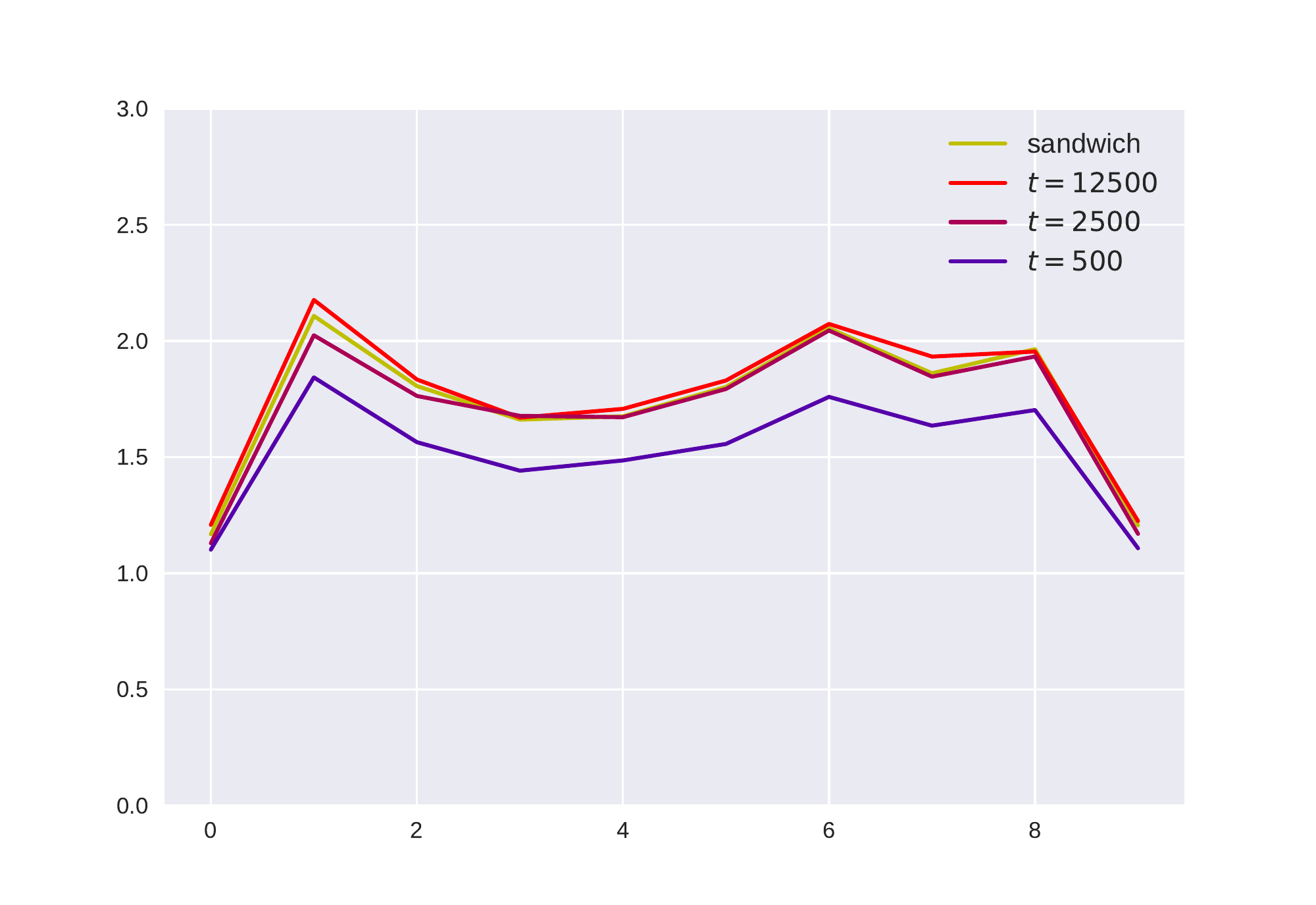}
}
~
\subfloat[$\eta=0.01$]{
\includegraphics[width=0.33\textwidth]{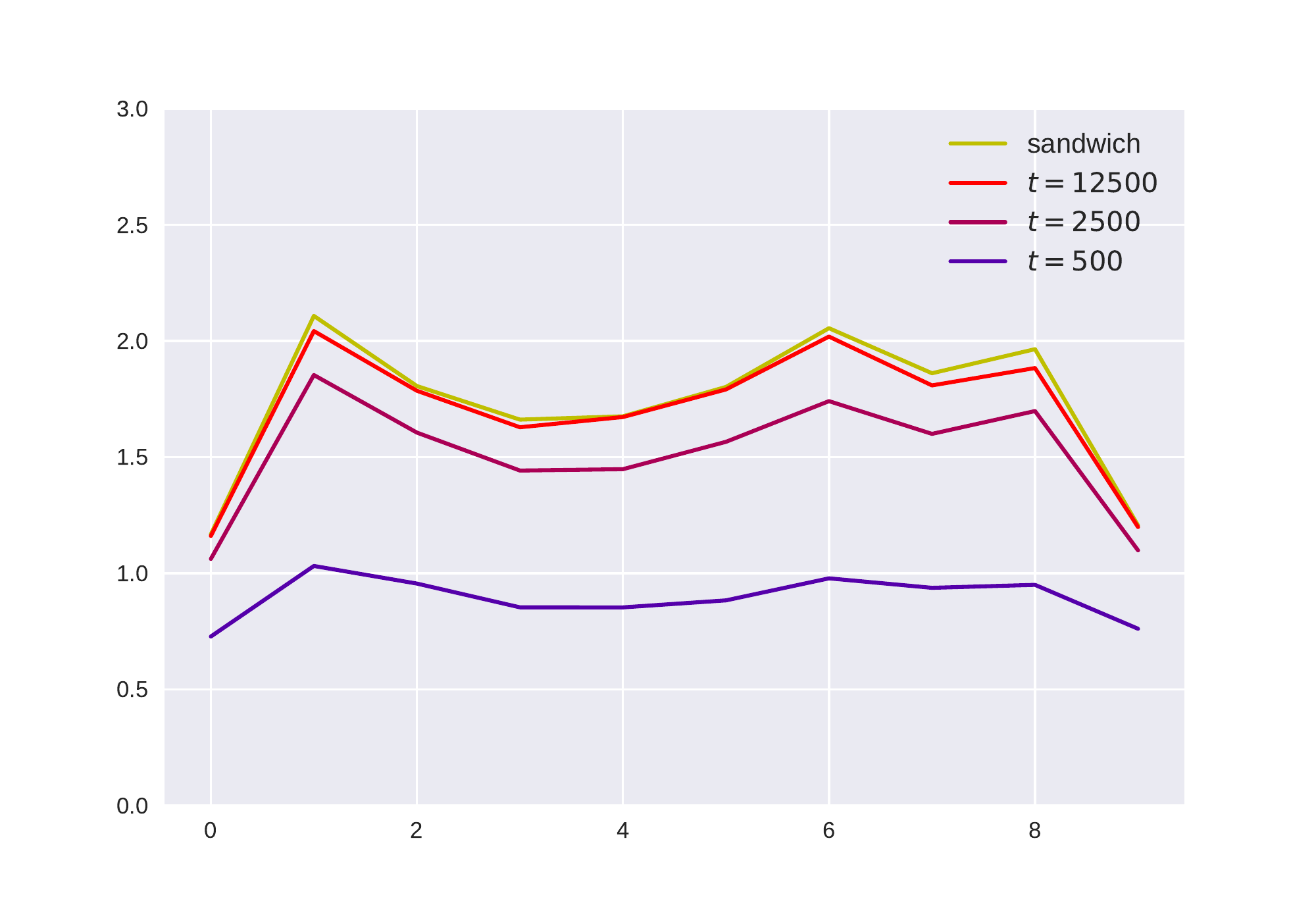}
}
~
\subfloat[$\eta=0.002$]{
\includegraphics[width=0.33\textwidth]{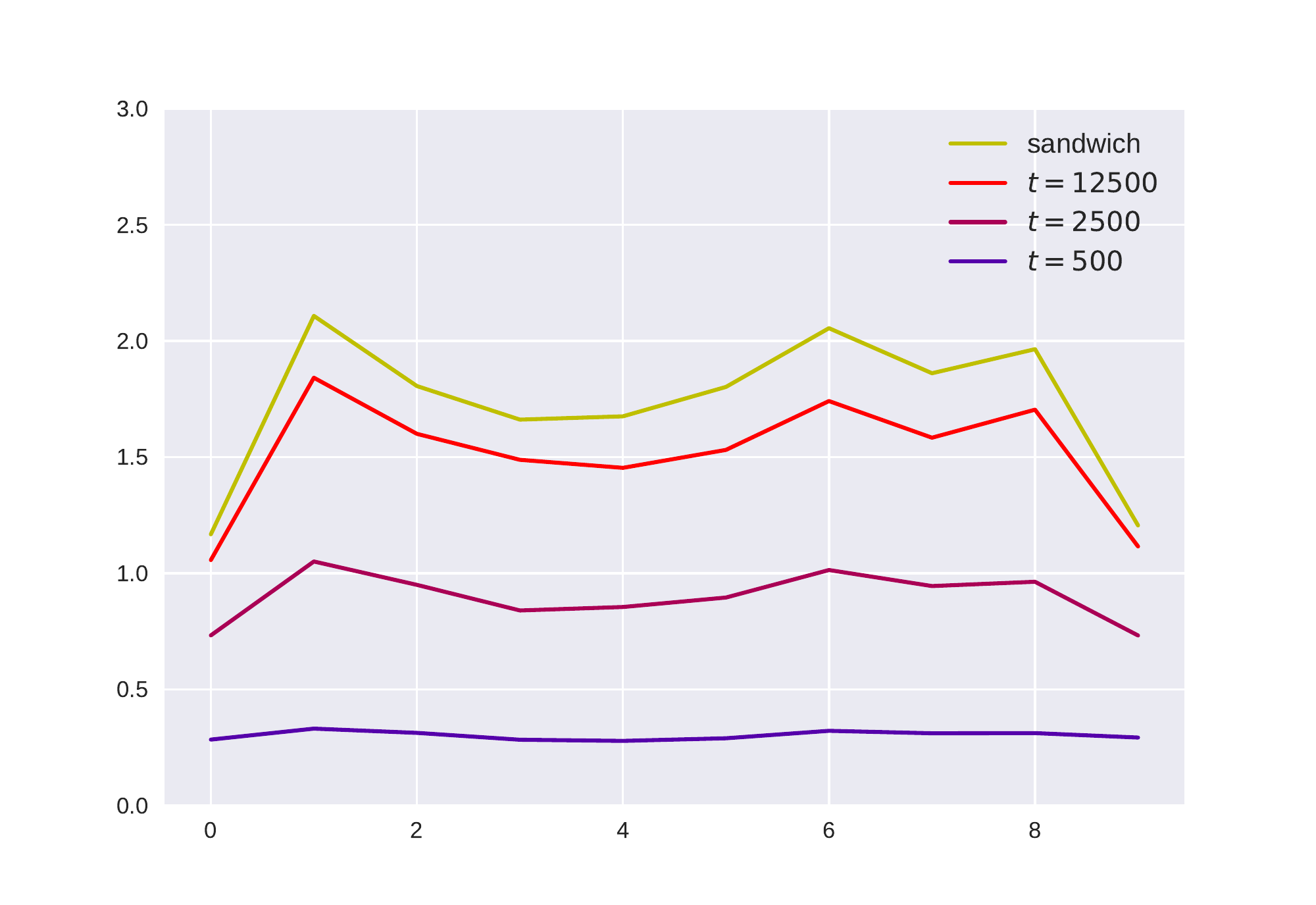}
}
\caption{11-dimensional linear regression: covariance matrix diagonal terms of SGD inference and sandwich estimator}
\label{fig:sim:10-dim-linear-sandwich-diagonal}
\end{figure}

{\em 2-Dimensional Logistic Regression.}

Here we consider the following model
\begin{align}
\Pr[Y =  +1] = \Pr[Y=-1]=\frac{1}{2}, \quad \quad X \mid Y \sim \Norm\left(\mu = 1.1 +0.1 Y,~\sigma^2 =1\right).
\end{align}
We use logistic regression to estimate $w, b$ in the classifier
$\sign(w x + b)$ where the minimizer of the population logistic risk is
$w^\star = 0.2, b^\star = -0.22$.

For 100 i.i.d. samples, we plot 1000 samples from SGD in
Figure \ref{fig:exp:sim:2d-logistic}.
In our simulations, we notice that our modified SGD for logistic regression behaves similar to vanilla logistic regression. T
his suggests that an assumption weaker than
$ (\theta - \htheta)^\top \nabla f(\theta) \geq \alpha \| \theta - \htheta \|_2^2 $ (assumption ($F_1$) in Theorem \ref{thm:sgd-strongly-convex-lipschitz-stat-inf}) is sufficient for SGD analysis.
Figure \ref{fig:logistic:modified:compare-t} and Figure \ref{fig:logistic:vanilla:compare-t} suggest that the $t \eta^2$ term in Corollary \ref{cor:sgd-stat-inf-logistic-regression} is an artifact of our analysis, and can be improved.

\begin{figure}[t]
\centering
\subfloat[Modified SGD with $t=1000$ and $\eta=0.1$]{
\includegraphics[width=0.33\textwidth]{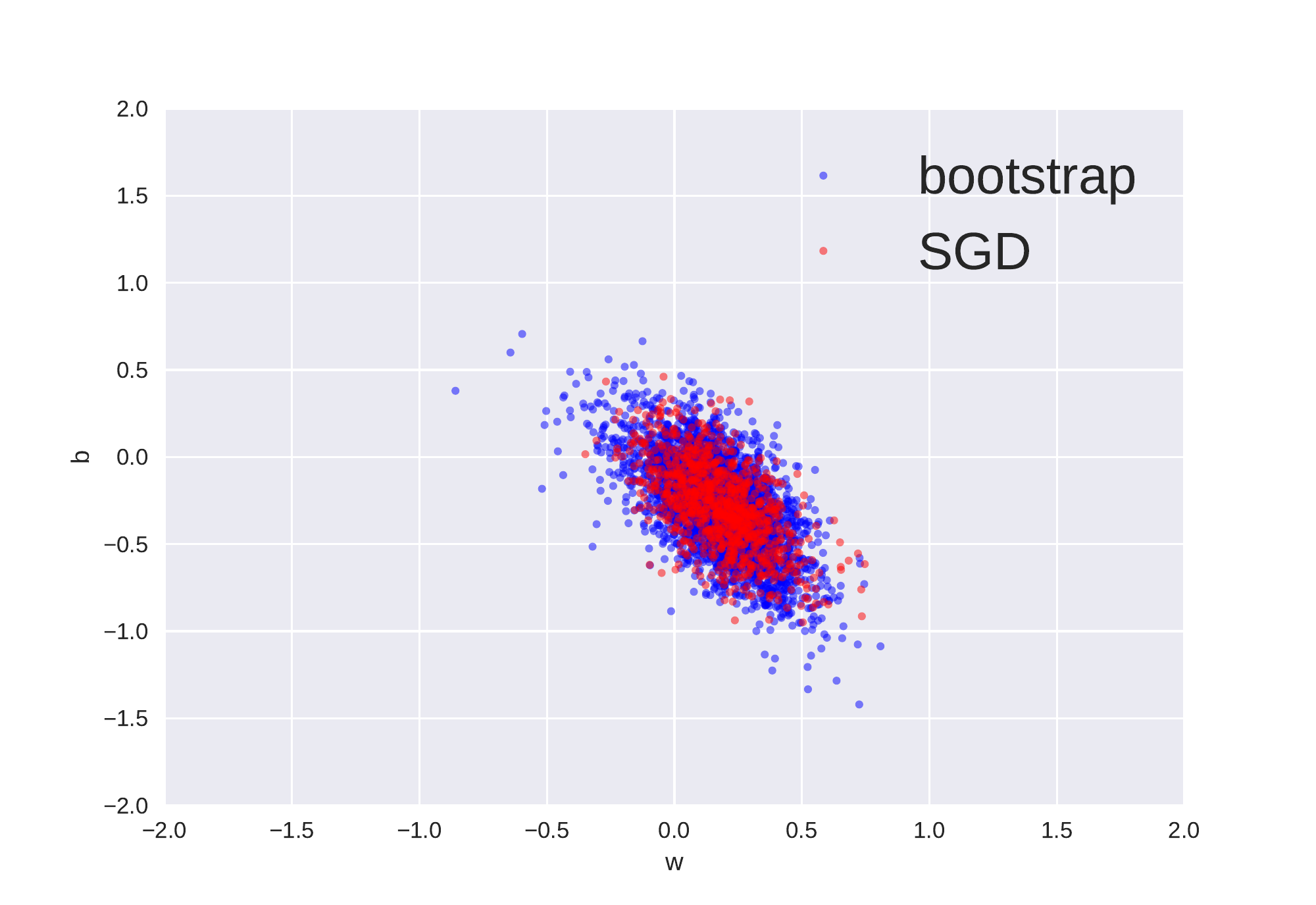}
}
\subfloat[Modified SGD with $\eta=0.1$\label{fig:logistic:modified:compare-t}]{
\includegraphics[width=0.33\textwidth]{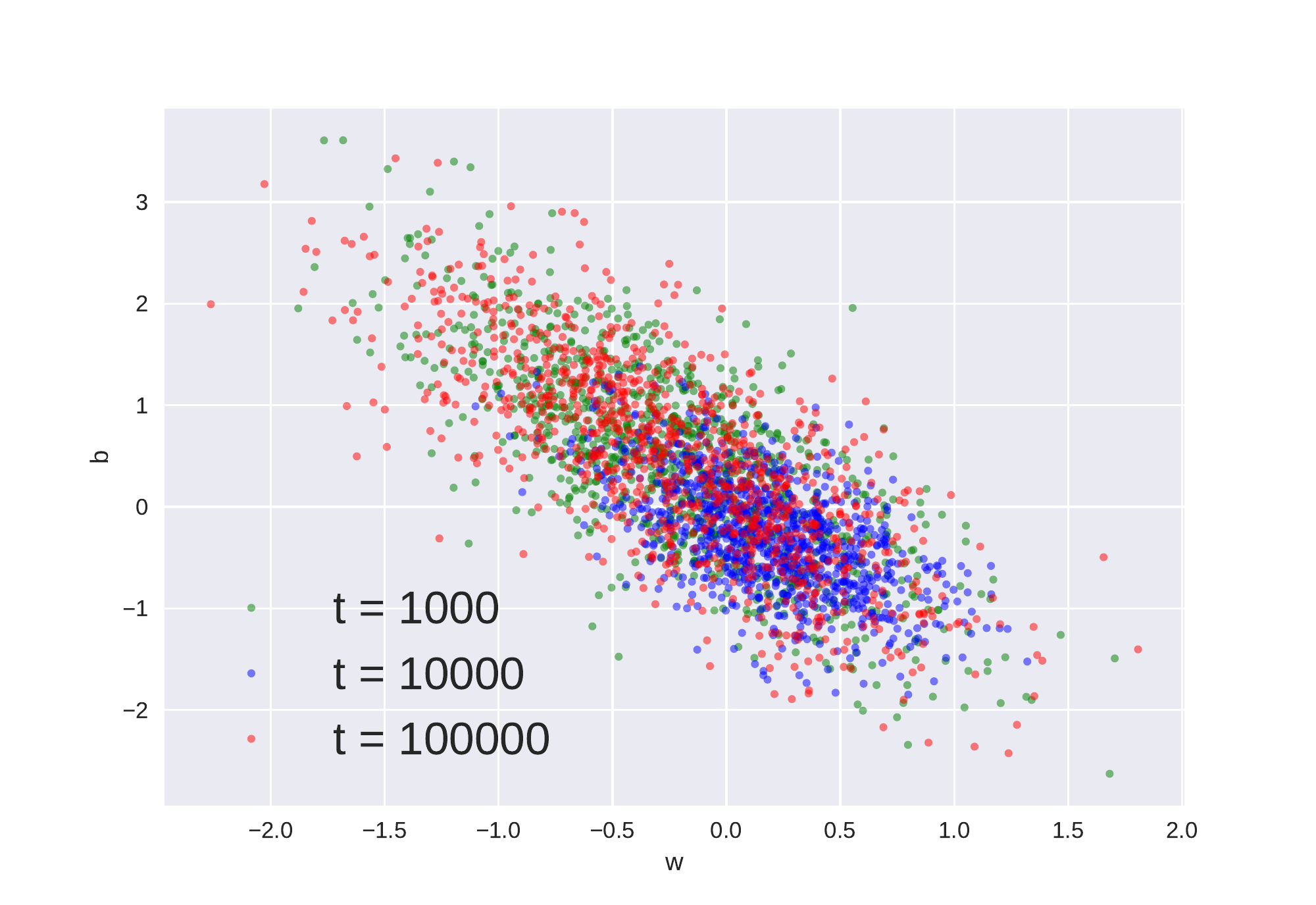}
}
\subfloat[Vanilla SGD with $t=1000$ and $\eta=0.1$]{
\includegraphics[width=0.33\textwidth]{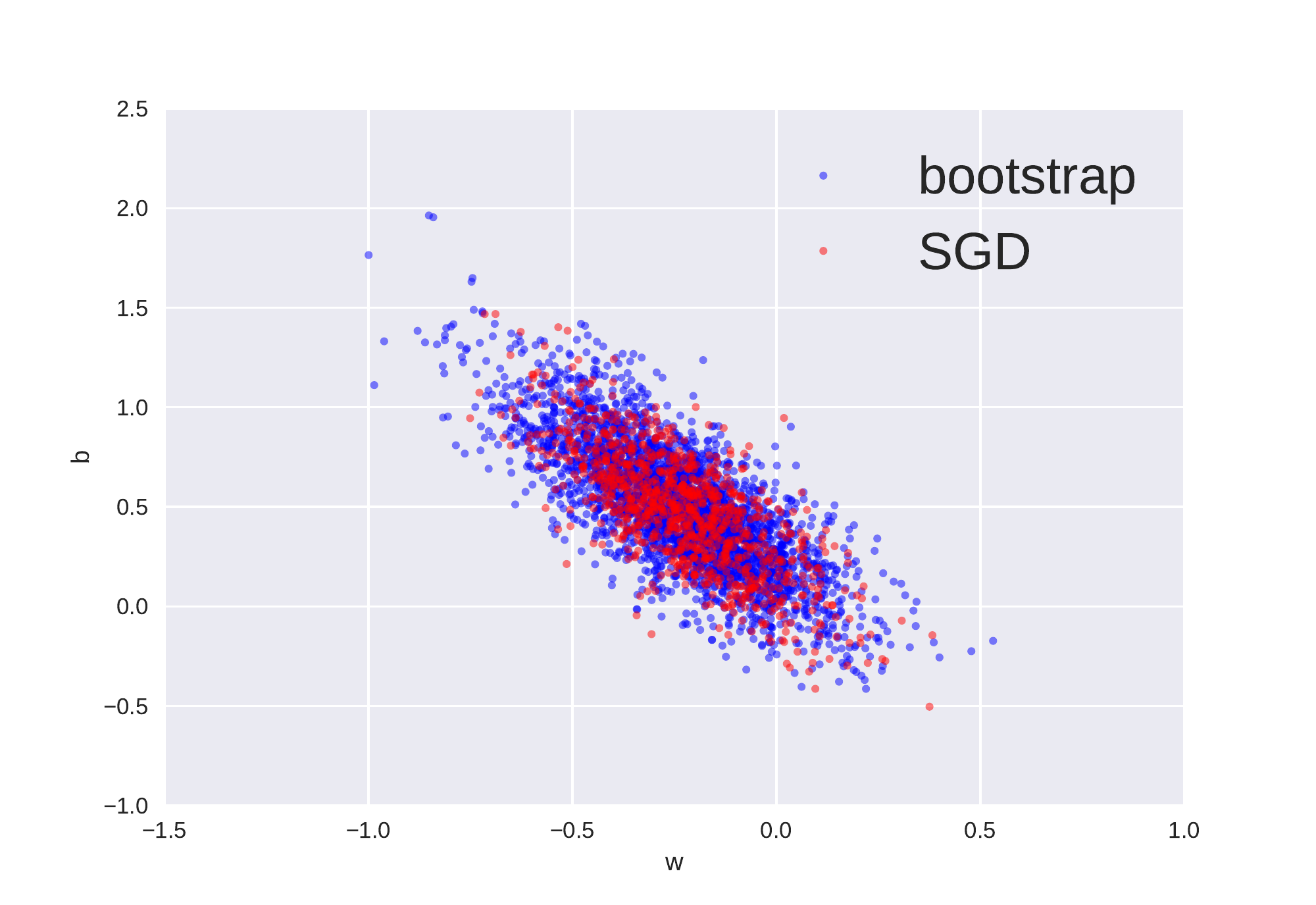}
} \\
\subfloat[Vanilla SGD with $\eta=0.1$\label{fig:logistic:vanilla:compare-t}]{
\includegraphics[width=0.3\textwidth]{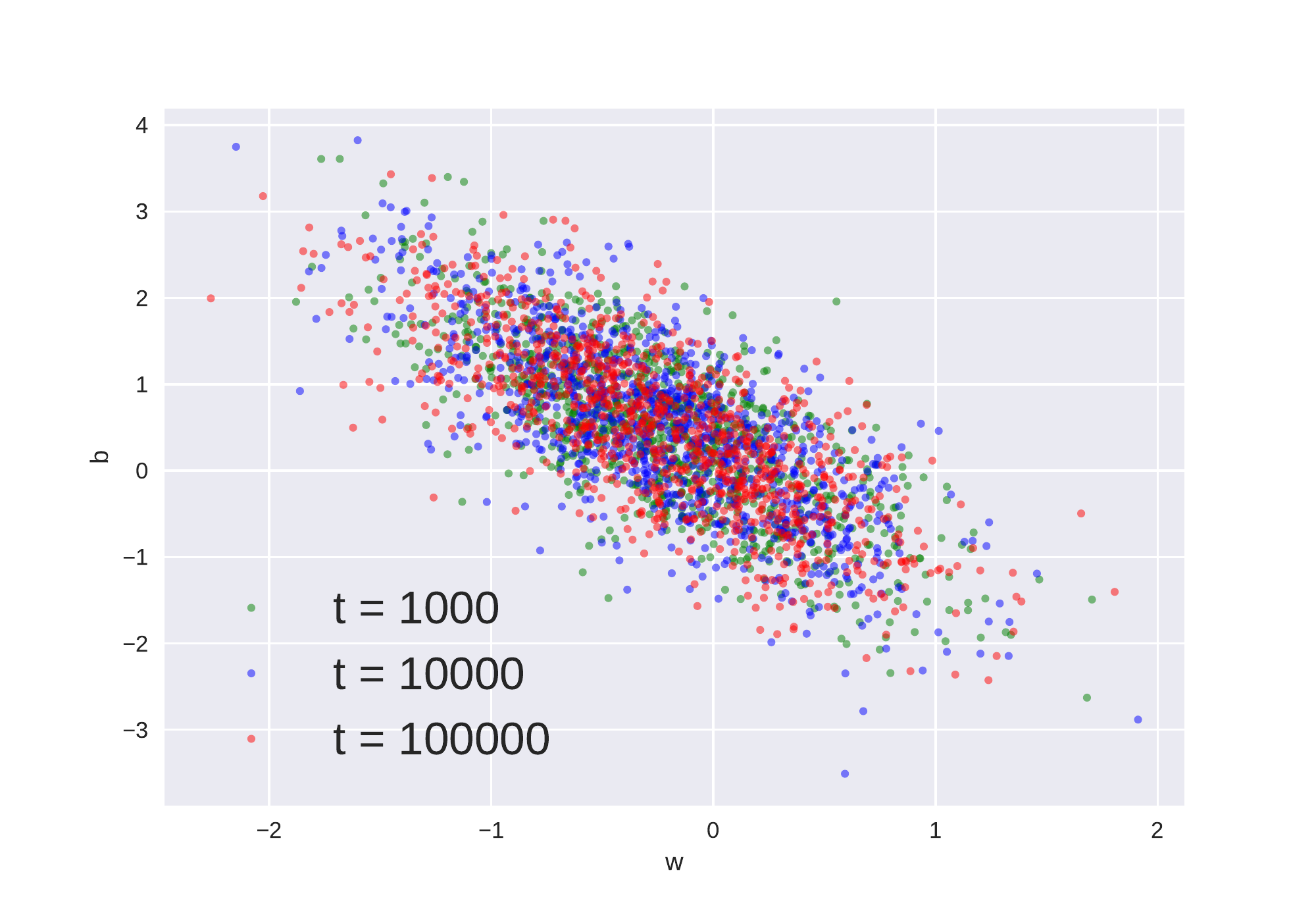}
}
\subfloat[$t=1000$ and $\eta=0.1$]{
\includegraphics[width=0.3\textwidth]{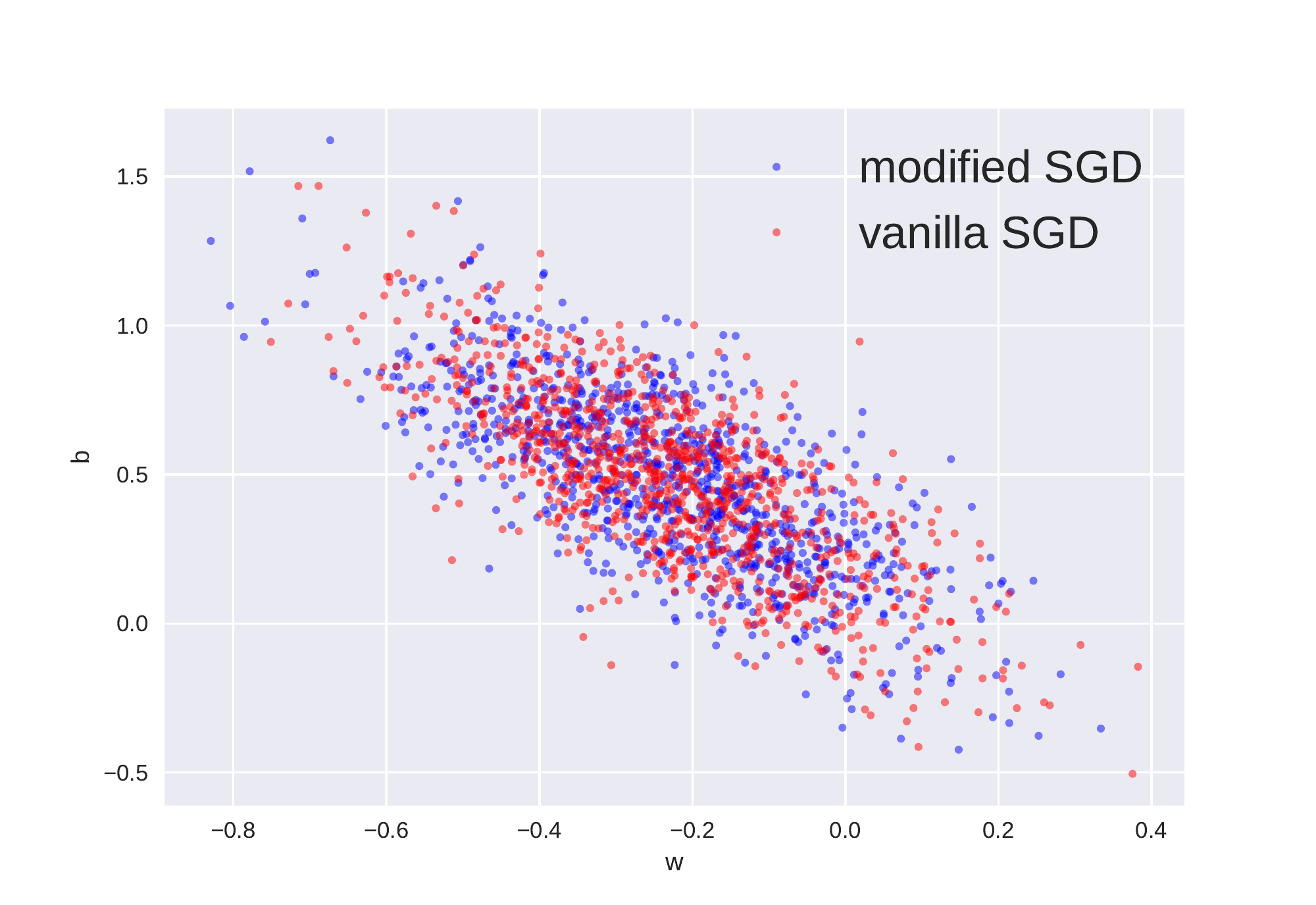}
}
\caption{2-dimensional logistic regression}
\label{fig:exp:sim:2d-logistic}
\end{figure}

{\em 11-Dimensional Logistic Regression.}

Here we consider the following model
\begin{align*}
\Pr[Y=+1] = \Pr[Y=-1]=\frac{1}{2}, \quad \quad X \mid Y \sim \Norm\left(0.01 Y  \mu, \Sigma\right) ,
\end{align*}
where $\Sigma_{ii}=1$ and when $i \neq j$ $\Sigma_{ij} = \rho^{|i-j|}$ for some $\rho \in [0, 1)$,
and $ \mu = \frac{1}{\sqrt{10}} [1,~~1,~~ \cdots, ~~ 1]^\top \in \Real^{10} $.
We estimate a classifier $\sign(w^\top x + b)$ using
\begin{align}
\widehat{w}, \widehat{b} = \argmin_{w,b} \frac{1}{n} \sum_{i=1}^n \log\left(1 + \exp(-Y_i ( w^\top X_i + b ) )\right) .
\end{align}

Figure \ref{fig:exp:sim:logistic:11d-rho-0_0} shows results for $\rho = 0$
with $n=80$ samples.
We use $t = 100$, $d=70$, $\eta=0.8$, and mini batch of size 4 in vanilla SGD.
Bootstrap and our SGD inference procedure each generated 2000 samples.
In bootstrap, we used Newton method to perform optimization over each replicate,
and 6-7 iterations were used.
In figure \ref{fig:exp:sim:logistic:11d-rho-0_6}, we follow the same procedure for $\rho = 0.6$
with $n=80$ samples.
Here, we use $t = 200$, $d=70$, $\eta=0.85$; the rest of the setting is the same.

\begin{figure}
\centering
\subfloat[SGD covariance]{
\includegraphics[width=0.33\textwidth]{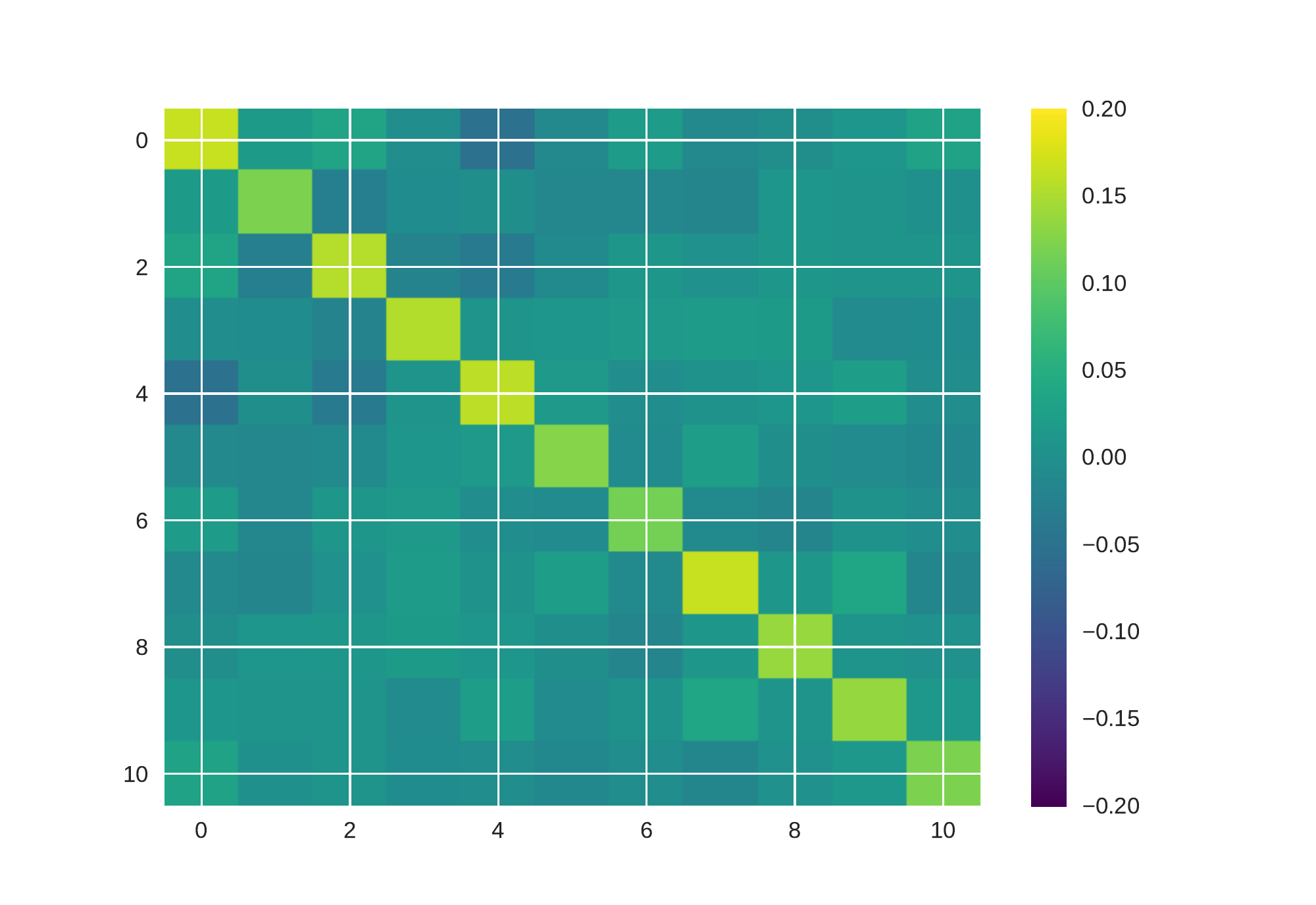}
}
~
\subfloat[Bootstrap estimated covariance]{
\includegraphics[width=0.33\textwidth]{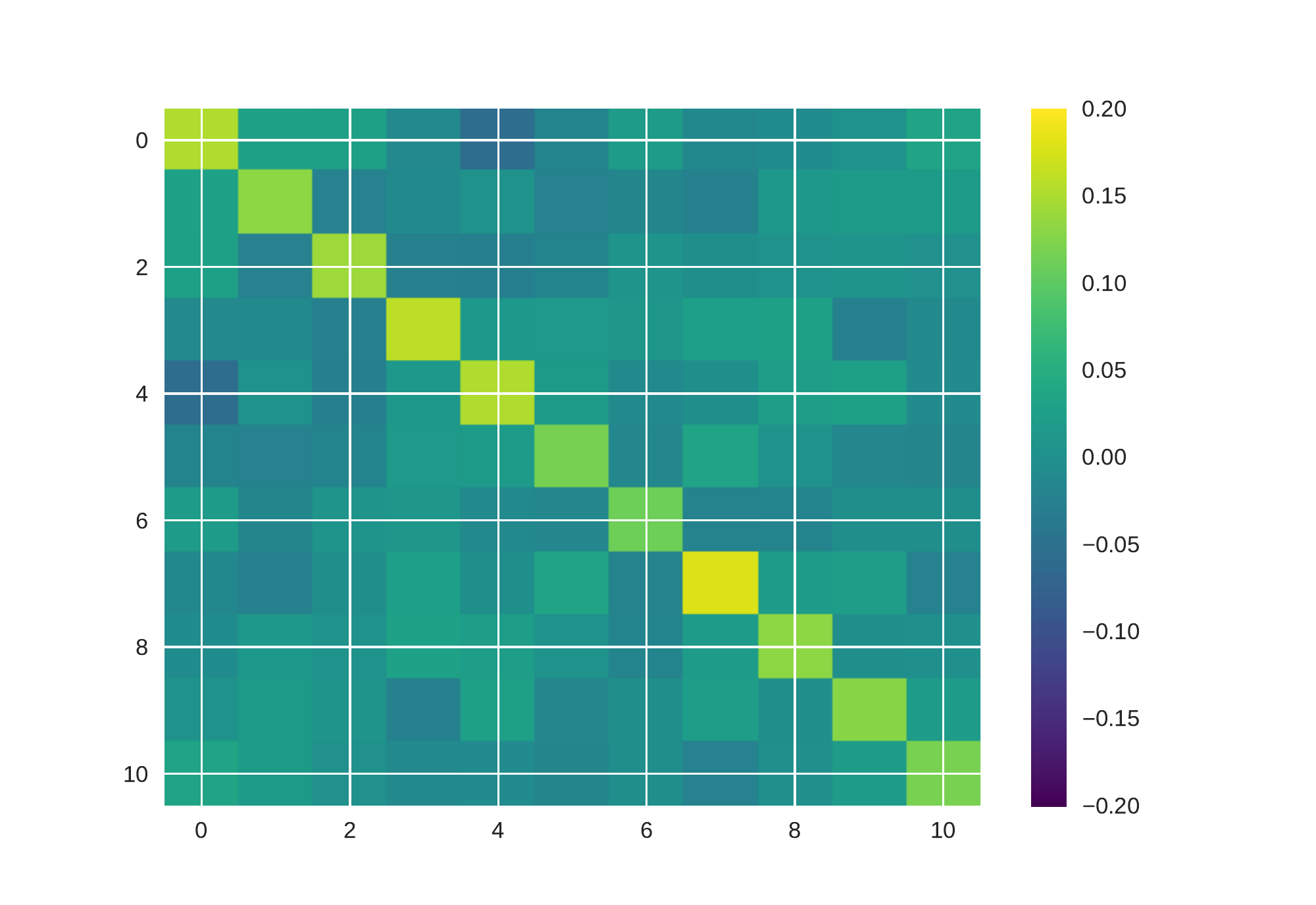}
}
\subfloat[Diagonal terms]{
\includegraphics[width=0.33\textwidth]{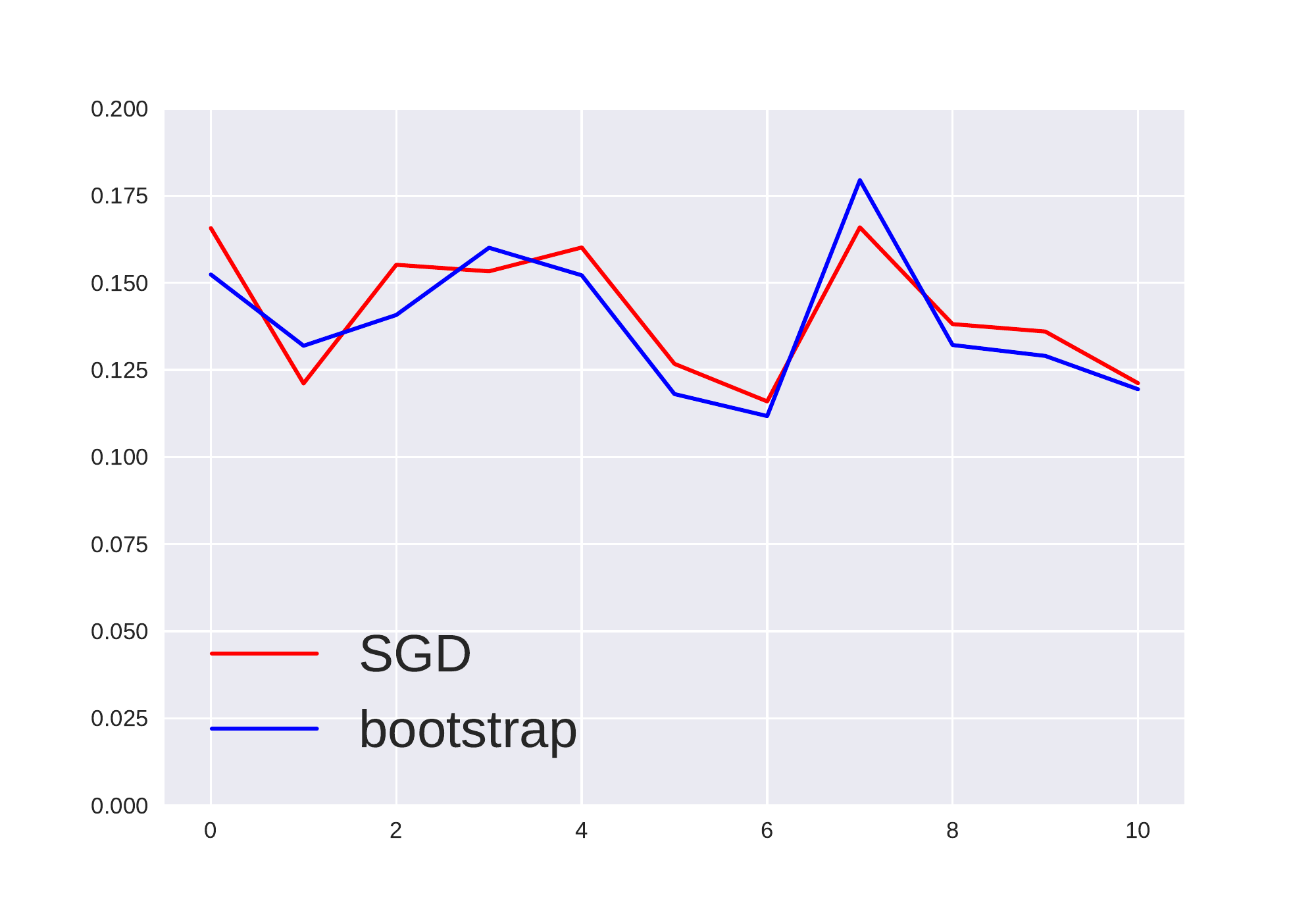}
}
\caption{11-dimensional logistic regression: $\rho =0$}
\label{fig:exp:sim:logistic:11d-rho-0_0}
\end{figure}

\begin{figure}
\centering
\subfloat[SGD covariance]{
\includegraphics[width=0.33\textwidth]{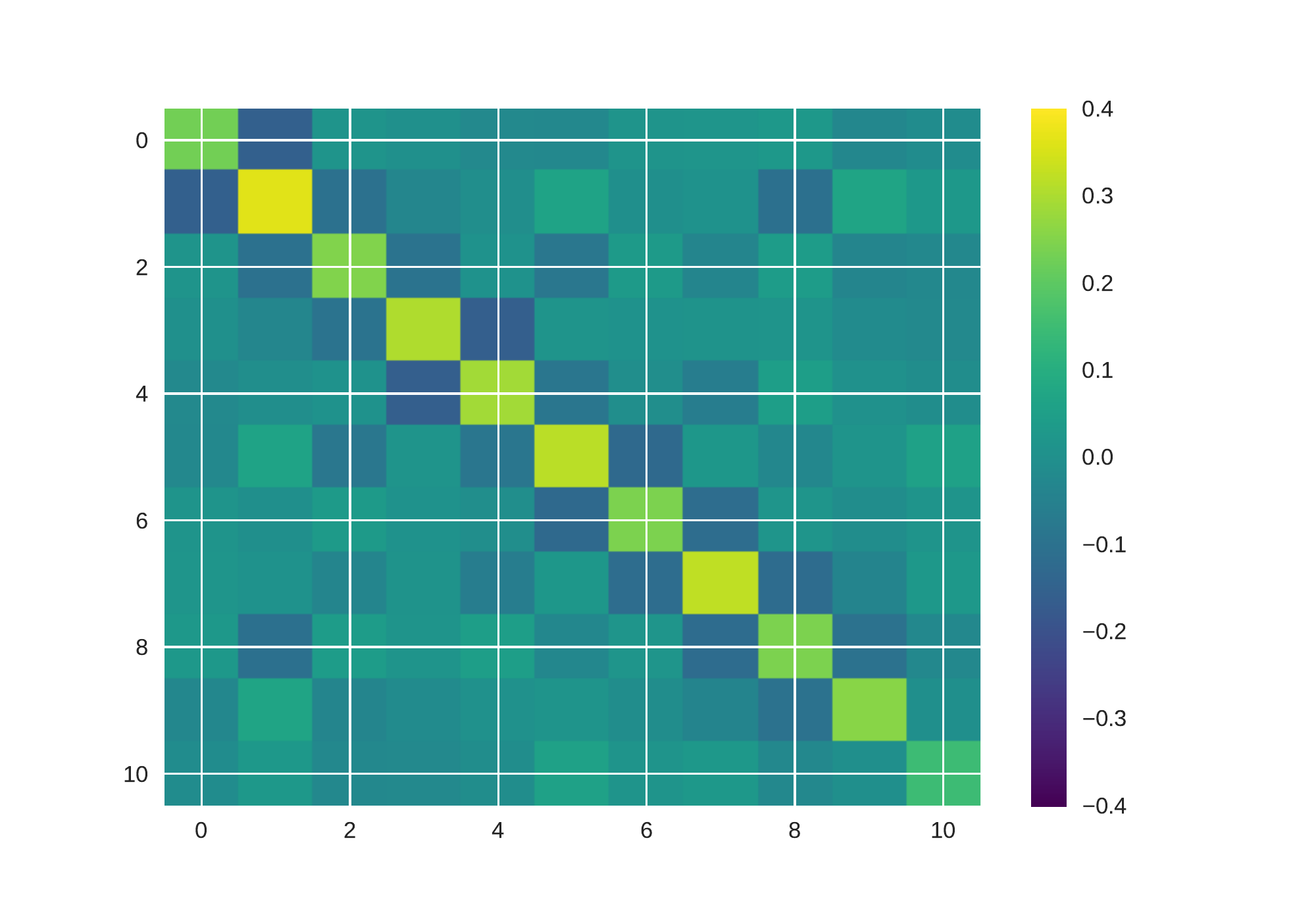}
}
\subfloat[Bootstrap estimated covariance]{
\includegraphics[width=0.33\textwidth]{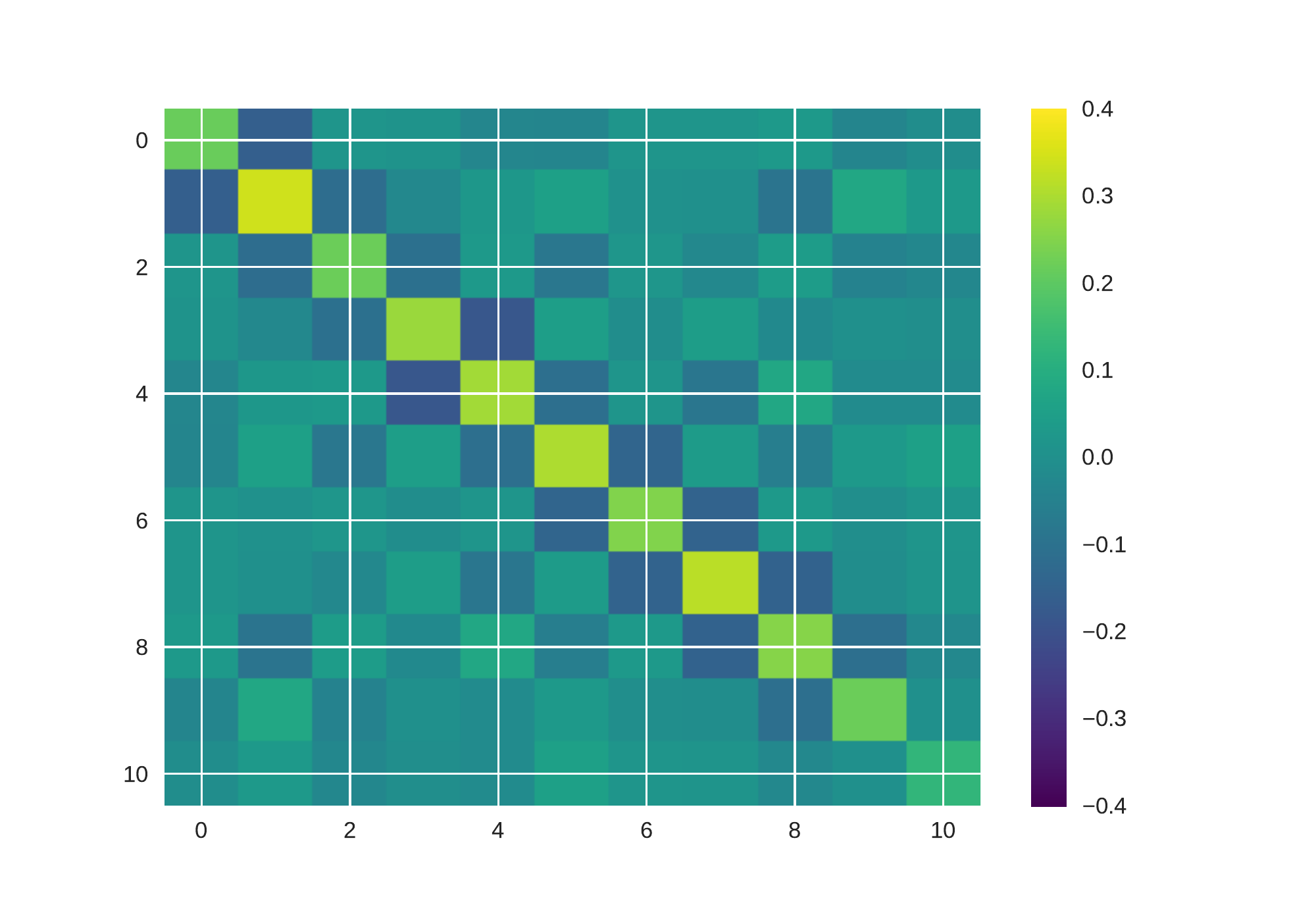}
}
\subfloat[Diagonal terms]{
\includegraphics[width=0.33\textwidth]{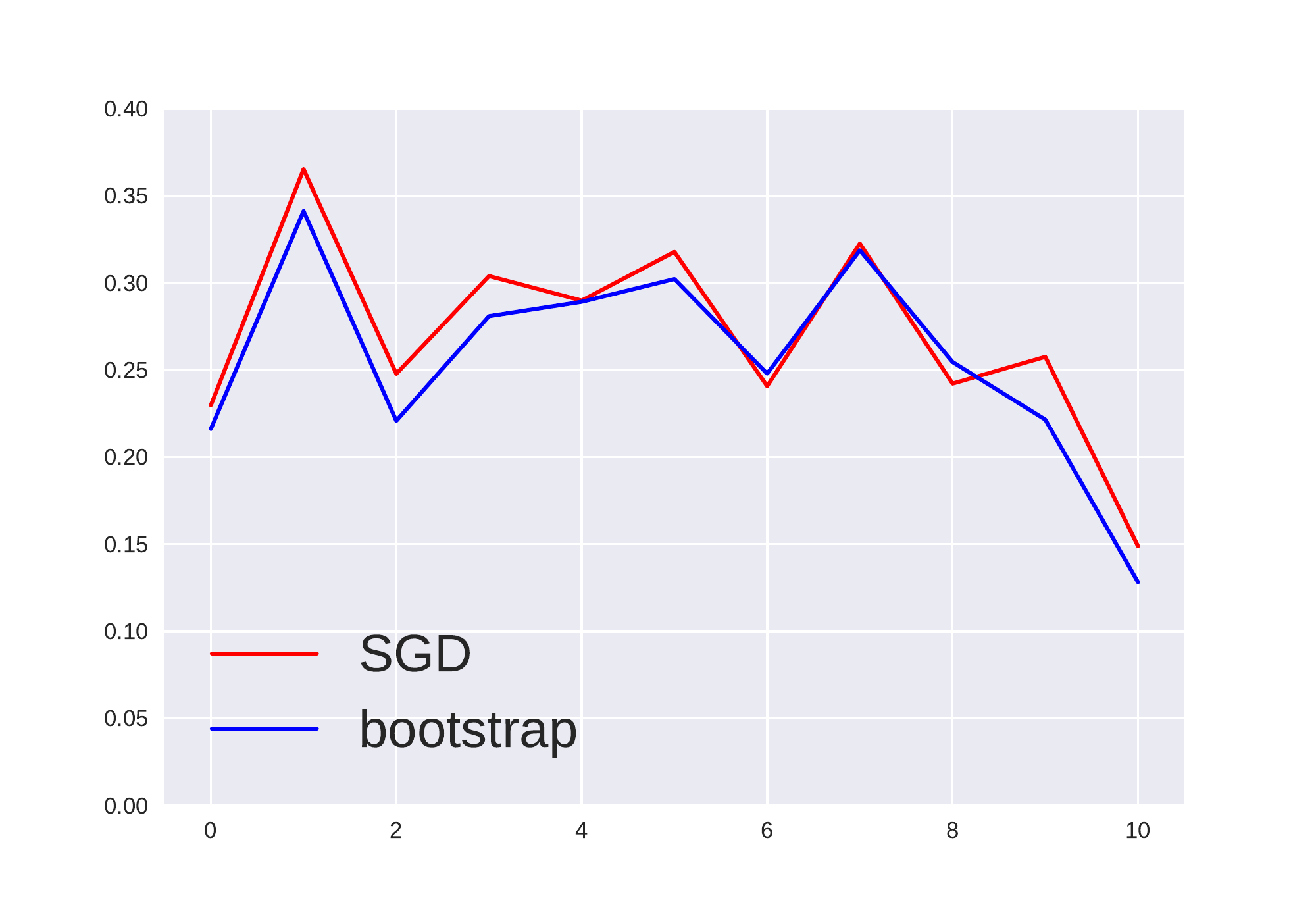}
}
\caption{11-dimensional logistic regression: $\rho =0.6$}
\label{fig:exp:sim:logistic:11d-rho-0_6}
\end{figure}

\subsection{Real data}
\label{subsec:nips2017:appendix:experiemnts:real}

Here, we compare covariance matrix computed using our SGD inference procedure,
bootstrap, and inverse Fisher information matrix
on the Higgs data set \cite{baldi2014searching} and the LIBSVM Splice data set,
and we observe that they have similar statistical properties.

\subsubsection{Higgs data set}

The Higgs data set \footnote{\url{https://archive.ics.uci.edu/ml/datasets/HIGGS}} \cite{baldi2014searching}
contains 28 distinct features with 11,000,000 data samples.
This is a classification problem between two types of physical
processes: one produces Higgs bosons and the other is a background process that does not.
We use a logistic regression model, trained using vanilla SGD,
instead of the modified SGD described in Section \ref{subsec:logistic-stat-inf}.

To understand different settings of sample size, we subsampled the
data set with different sample size levels: $n = 200$ and $n = 50000$.
We investigate the empirical performance of SGD inference on this subsampled data set.
In all experiments below, the batch size of the mini batch SGD is 10.

In the case $n = 200$, the asymptotic normality for the MLE is not a good enough approximation.
Hence, in this small-sample inference, we compare the SGD inference covariance matrix with the one obtained by inverse Fisher information matrix and bootstrap in Figure \ref{fig:exp:real:higgs_Small}.

For our SGD inference procedure, we use $t = 100$ samples to average, and discard $d = 50$ samples. We use $R = 20$ averages from 20 segments (as in Figure \ref{fig:sgd-inf}).
For bootstrap, we use 2000 replicates, which is much larger than the sample size $n=200$.

Figure \ref{fig:exp:real:higgs_Small}
shows that the covariance matrix obtained by SGD inference is comparable
to the estimation given by bootstrap and inverse Fisher information.

\begin{figure}[t]
\centering
\subfloat[Inverse Fisher information]{
\includegraphics[width=0.33\textwidth]{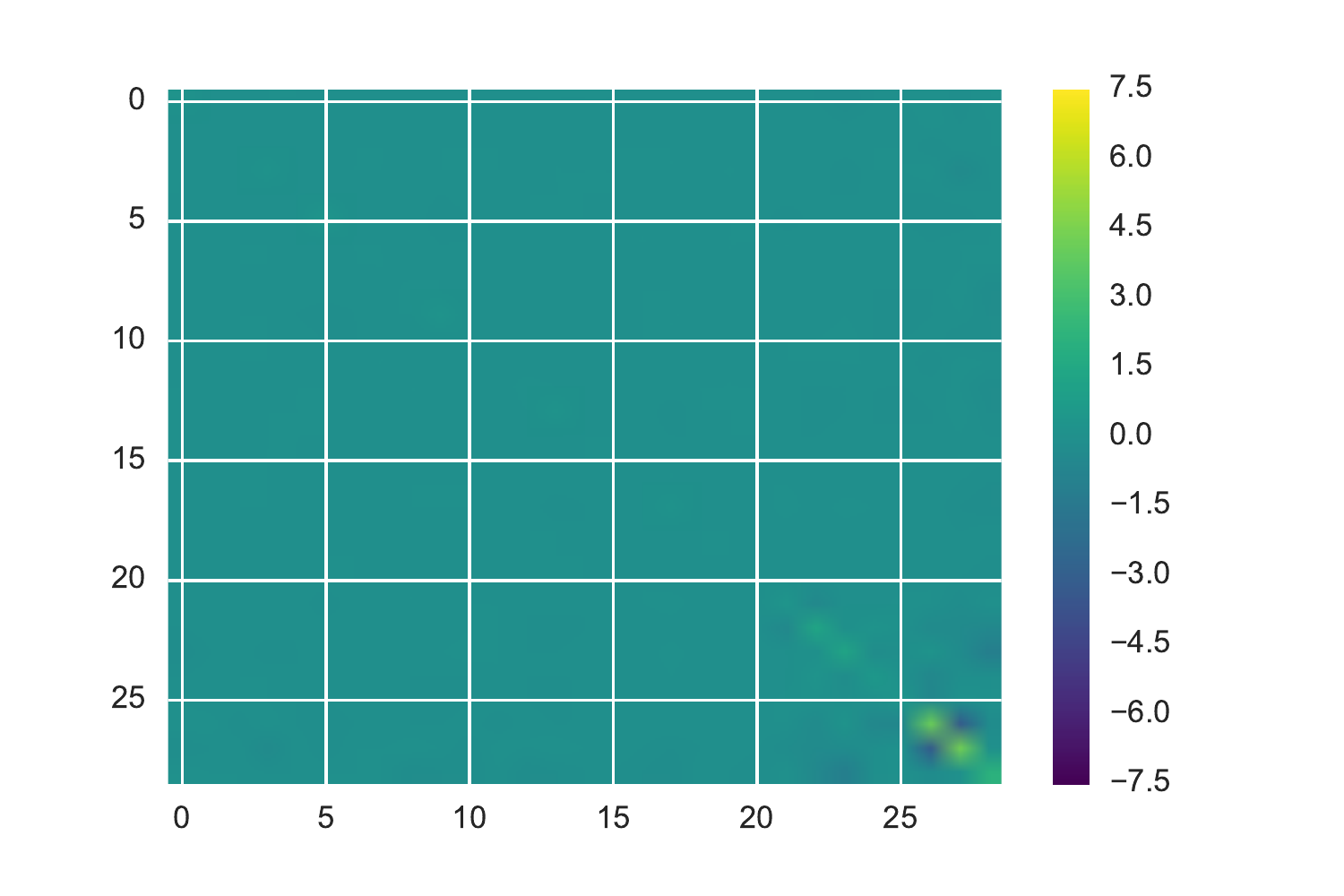}
}
\subfloat[SGD inference covariance]{
\includegraphics[width=0.33\textwidth]{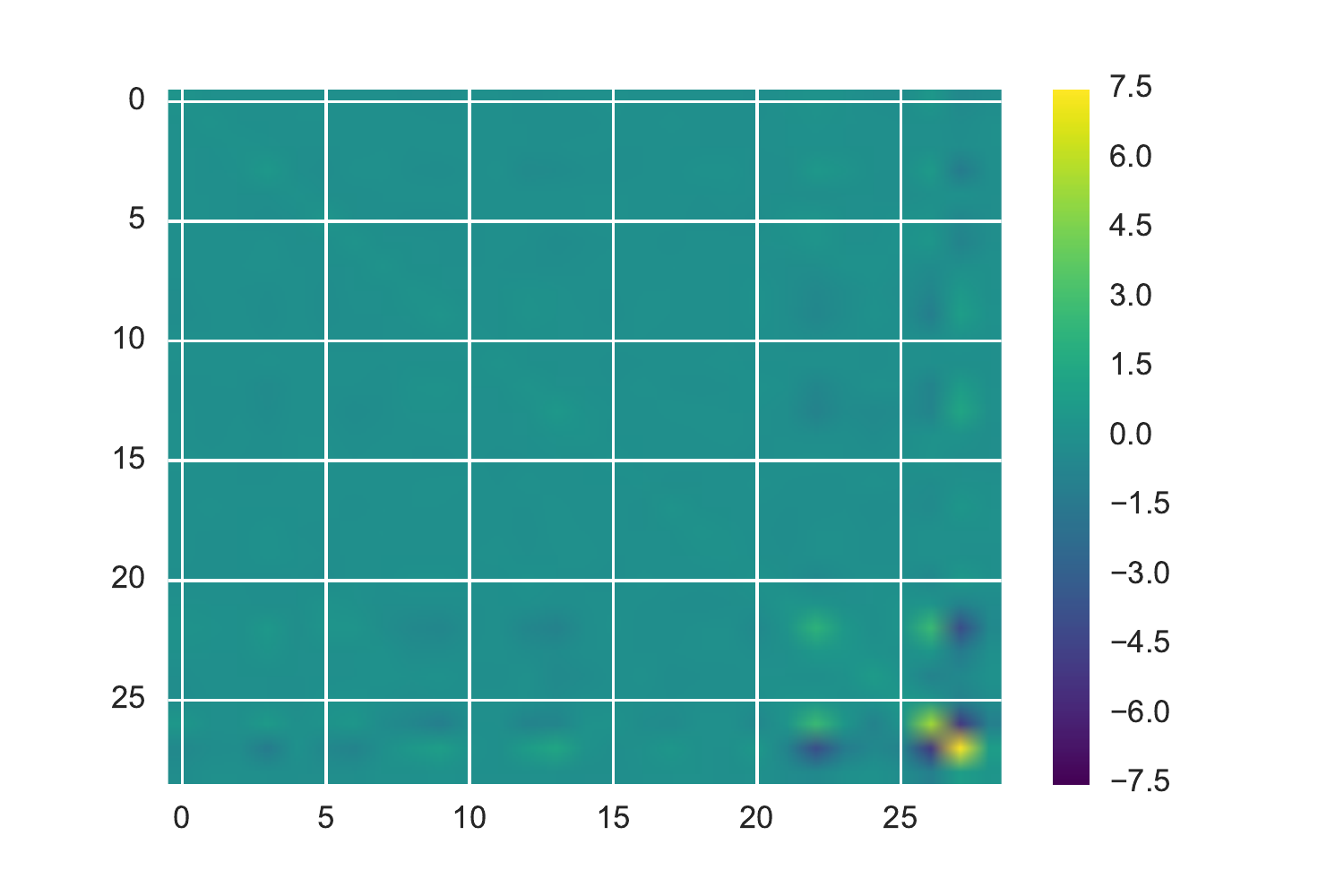}
}
\subfloat[Bootstrap estimated covariance]{
\includegraphics[width=0.33\textwidth]{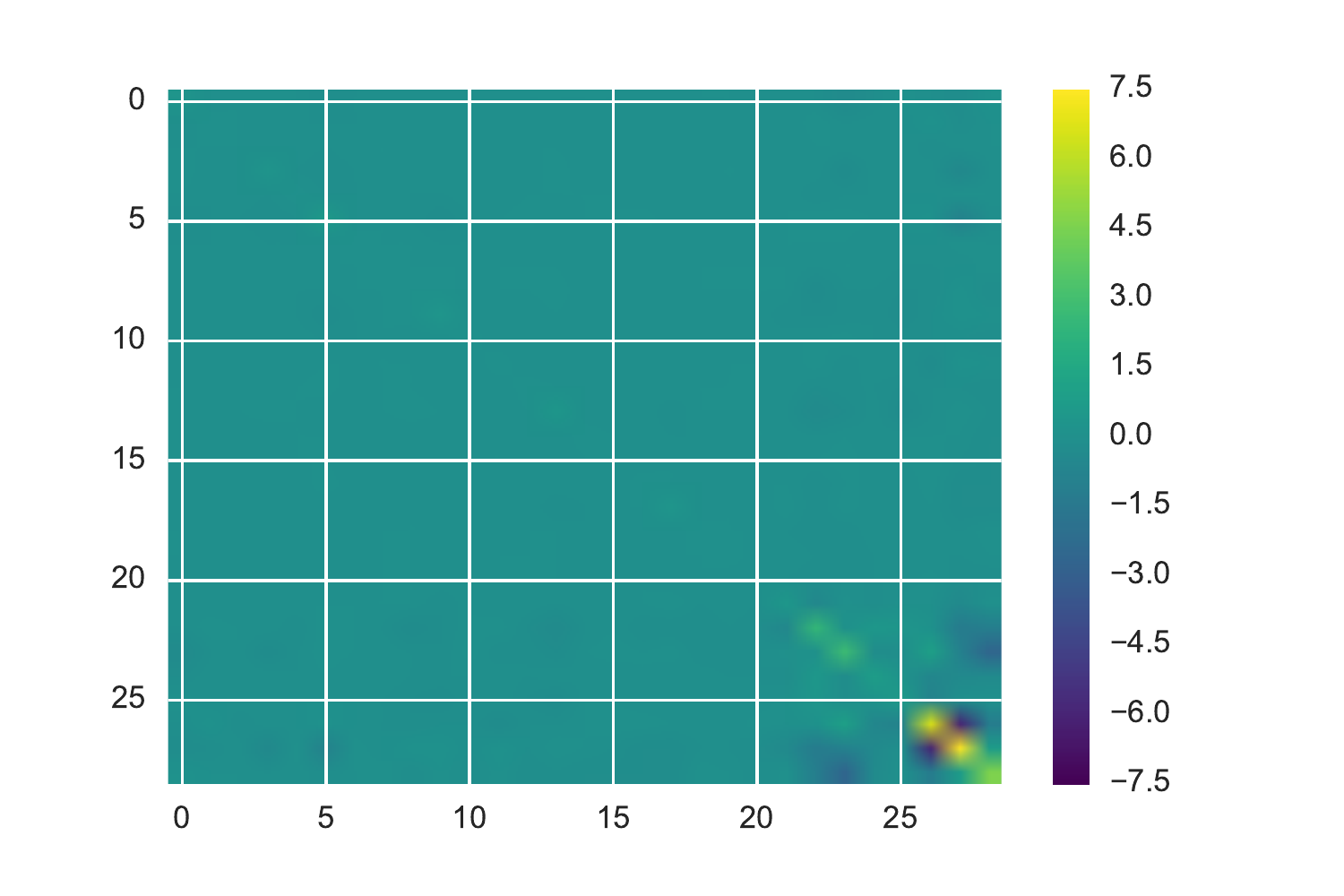}
}
\caption{Higgs data set with $n=200$}
\label{fig:exp:real:higgs_Small}
\end{figure}

In the case $n = 50000$, we use $t = 5000$ samples to average, and discard $d = 500$ samples.
We use $R = 20$ averages from 20 segments (as in Figure \ref{fig:sgd-inf}).
For this large data set, we present the estimated covariance of SGD inference procedure and inverse Fisher information (the asymptotic covariance) in Figure \ref{fig:exp:real:higgs_Large} because bootstrap is computationally prohibitive.
Similar to the small sample result in Figure \ref{fig:exp:real:higgs_Small}, the covariance of our SGD inference procedure is comparable to the inverse Fisher information.

\begin{figure}[t]
\centering
\subfloat[Inverse Fisher information]{
\includegraphics[width=0.33\textwidth]{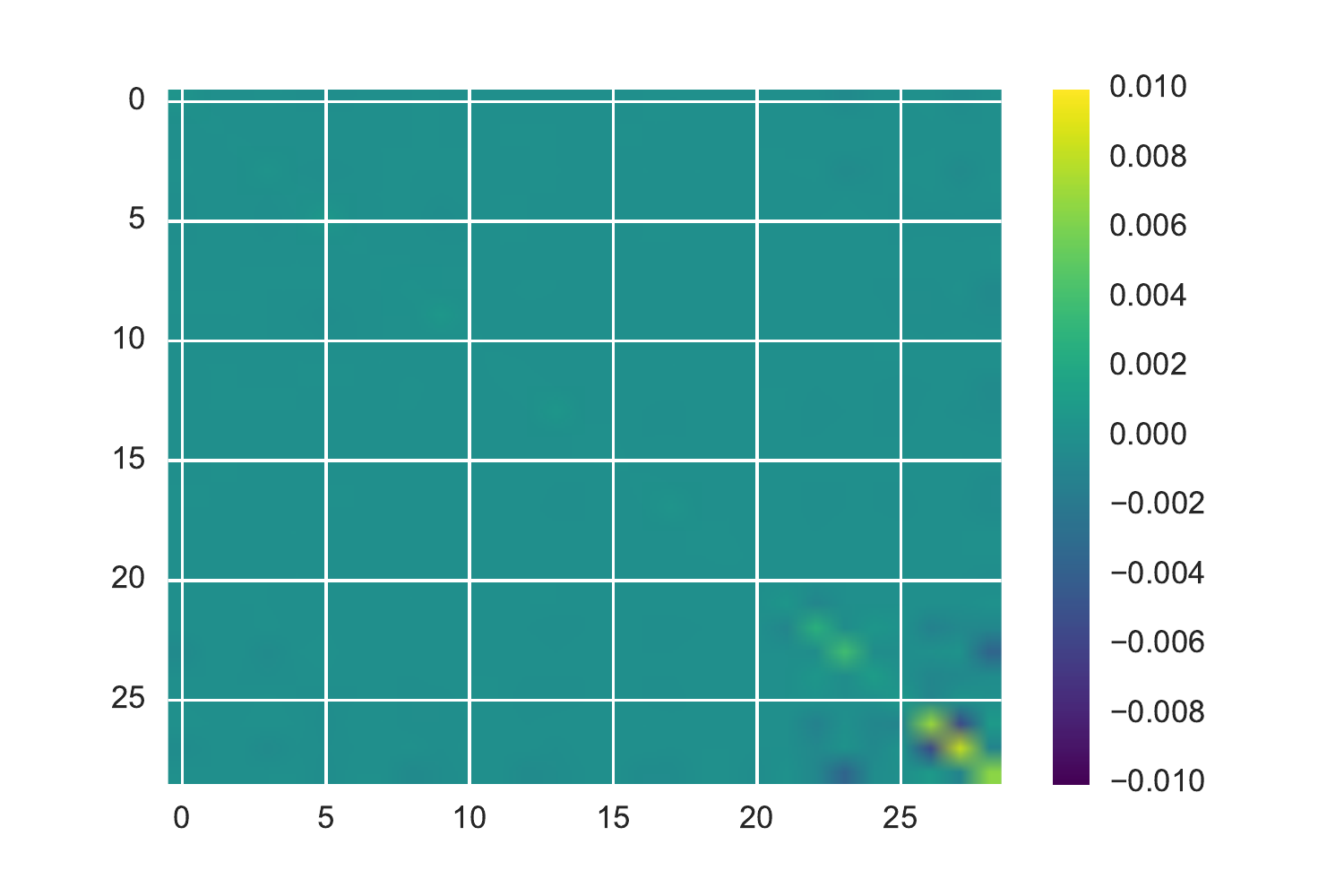}
}
\subfloat[SGD inference covariance]{
\includegraphics[width=0.33\textwidth]{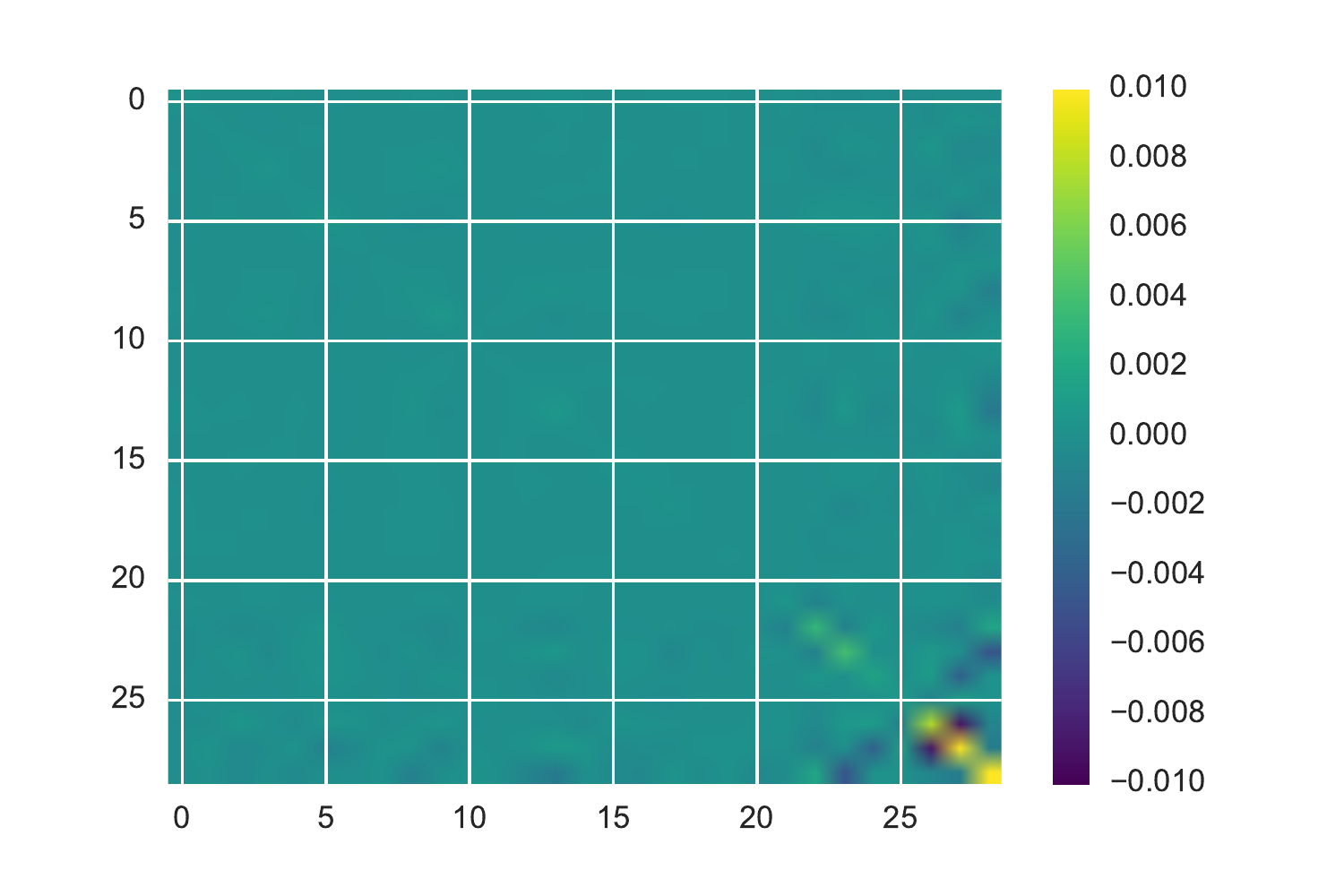}
}
\caption{Higgs data set with $n=50000$}
\label{fig:exp:real:higgs_Large}
\end{figure}

In Figure \ref{fig:exp:real:higgs_90k}, we compare the covariance matrix computed using our SGD inference procedure and inverse Fisher information with
$n=90000$ samples .
We used 25 samples from our SGD inference procedure with
$t=5000$, $d=1000$, $\eta=0.2$, and mini batch size of $10$.

\begin{figure}
\centering
\subfloat[Inverse Fisher information]{
\includegraphics[width=0.33\textwidth]{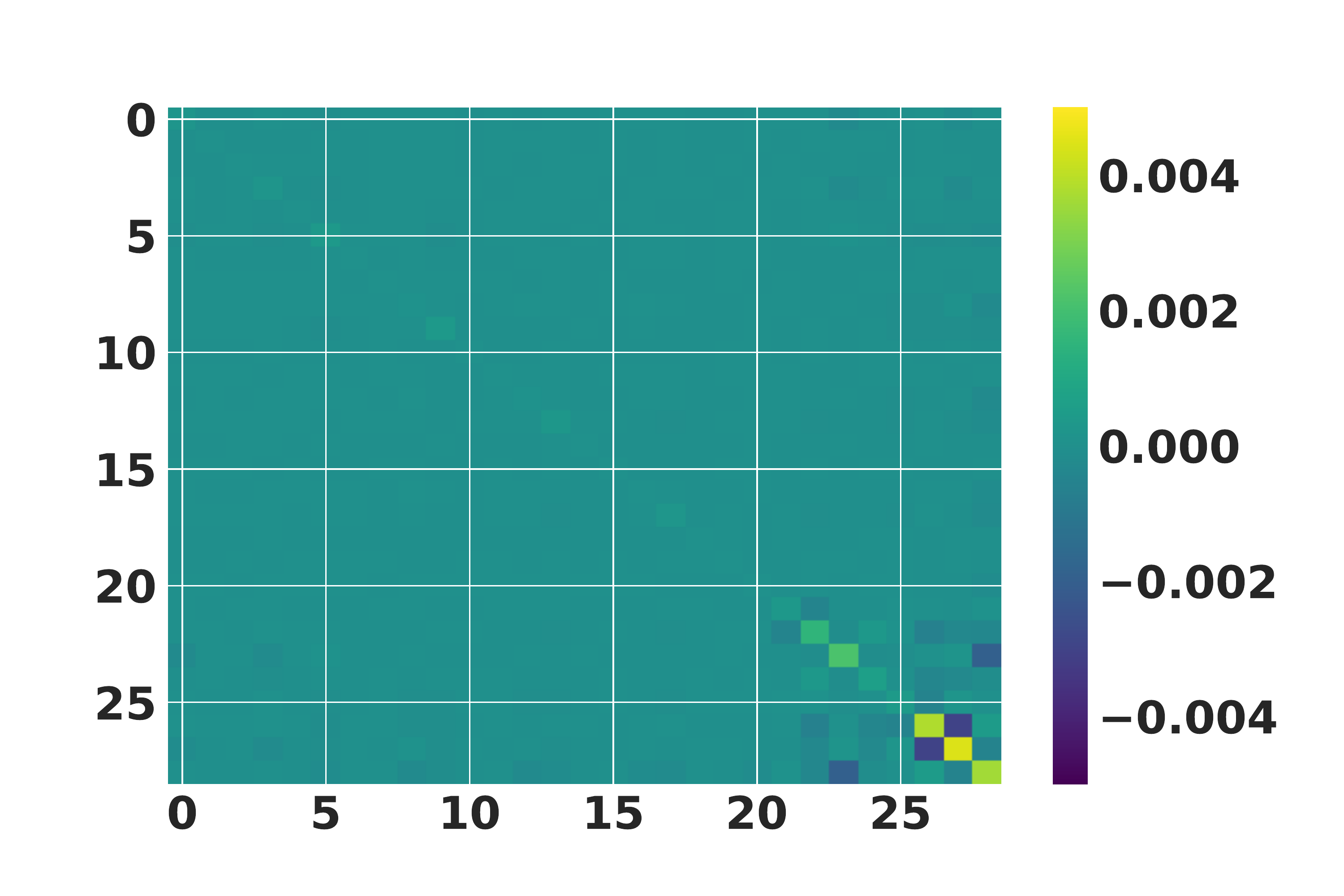}
}
\subfloat[SGD inference covariance]{
\includegraphics[width=0.33\textwidth]{figs/HiggsArxiv_Fisher_PDF.pdf}
}
\subfloat[Diagonal terms]{
\includegraphics[width=0.33\textwidth]{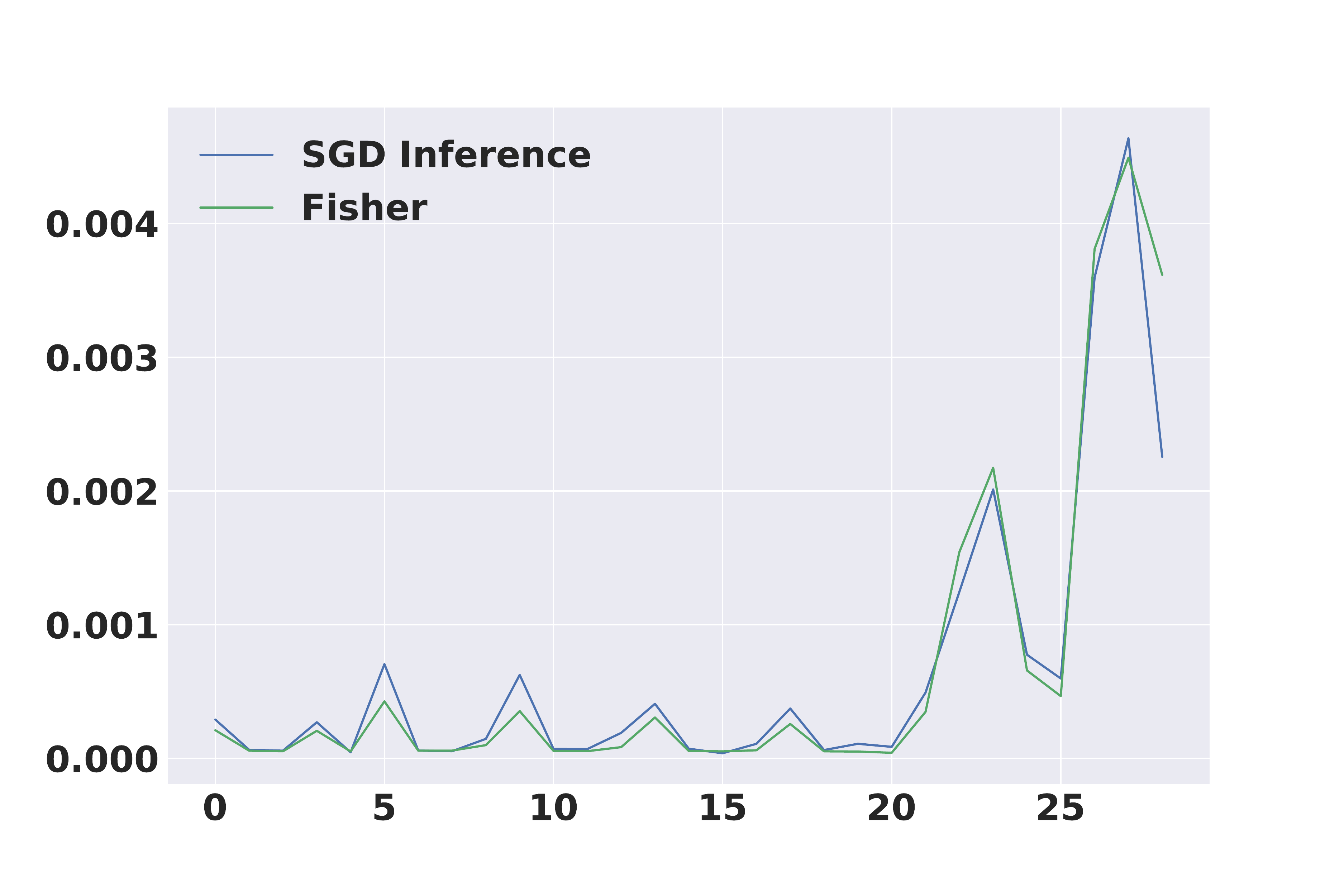}
}
\caption{Higgs data set with $n=90000$}
\label{fig:exp:real:higgs_90k}
\end{figure}

\subsubsection{Splice data set}

The Splice data set~\footnote{\url{https://www.csie.ntu.edu.tw/~cjlin/libsvmtools/datasets/binary.html}}
 contains 60 distinct features with 1000 data samples.
This is a classification problem between two classes of splice junctions in a DNA sequence.
Similar to the Higgs data set, we use a logistic regression model, trained using vanilla SGD.

In Figure \ref{fig:exp:real:libsvm:splice},
we  compare  the covariance matrix computed using our SGD inference procedure and bootstrap
$n=1000$ samples.
We used 10000 samples from both bootstrap and our SGD inference procedure with
$t=500$, $d=100$, $\eta=0.2$, and mini batch size of $6$.

\begin{figure}
\centering
\subfloat[Bootstrap]{
\includegraphics[width=0.33\textwidth]{figs/libsvm-splice---COV_BS.pdf}
}
\subfloat[SGD inference covariance]{
\includegraphics[width=0.33\textwidth]{figs/libsvm-splice---COV_SGD.pdf}
}
\subfloat[Diagonal terms]{
\includegraphics[width=0.33\textwidth]{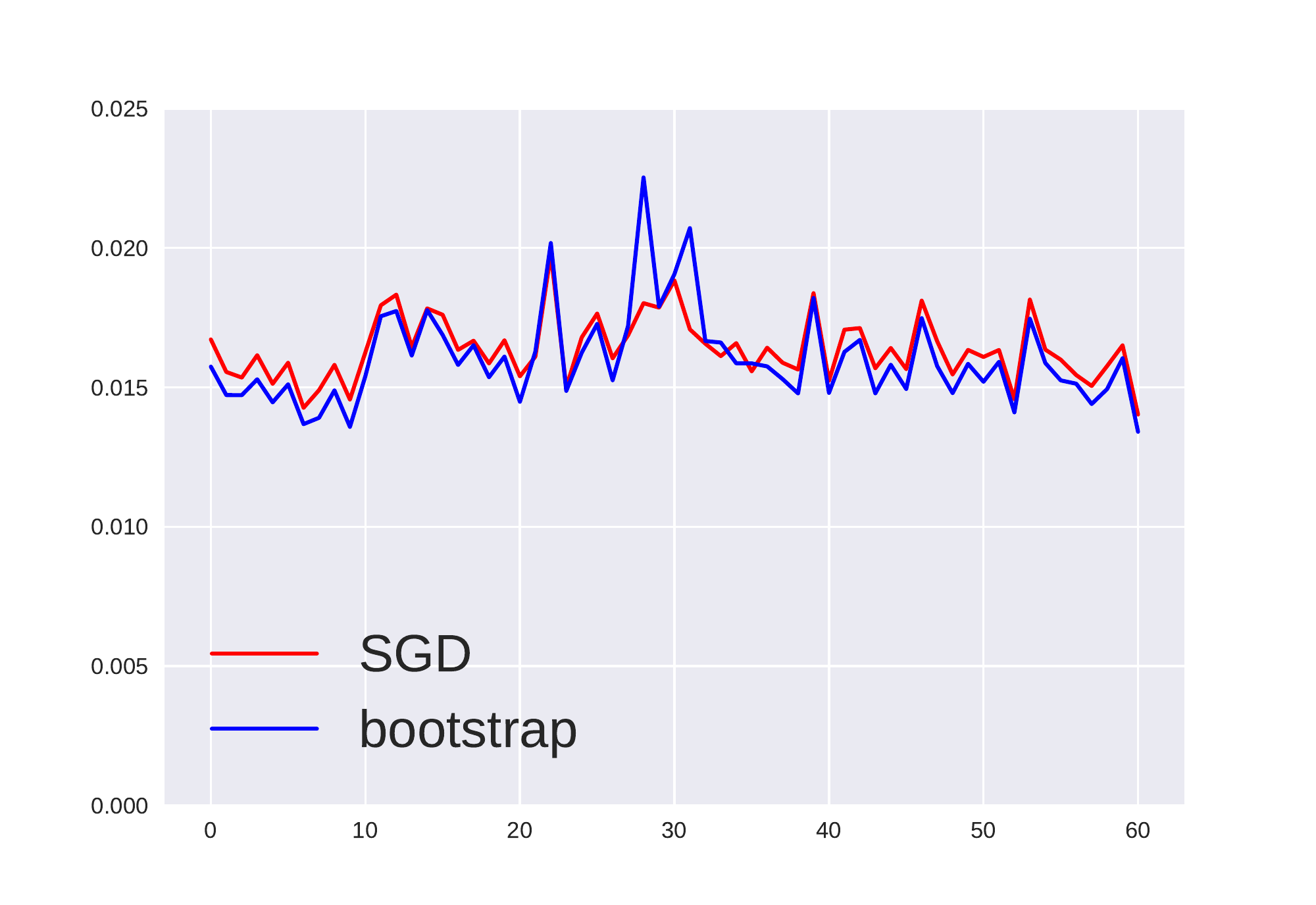}
}
\caption{Splice data set}
\label{fig:exp:real:libsvm:splice}
\end{figure}

\subsubsection{MNIST}

Here, we train a binary logistic regression classifier to classify 0/1 using perturbed MNIST data set, 
and demonstrate that certain adversarial examples 
(e.g. \cite{goodfellow2014explaining}) can be detected using prediction confidence intervals. 
For each image, where each original pixel is either 0 or 1, 
we randomly changed 70\% pixels to random numbers uniformly on $[0, 0.9]$. 
\Cref{fig:exp:real:mnist-appendix} shows each image's logit value 
($\log \frac{\Pr[1 \mid \text{image}]}{\Pr[0 \mid \text{image}]}$) 
and its 95\% confidence interval computed using our SGD inference procedure. 
The adversarial perturbation used here is shown in \Cref{fig:exp:real:mnist-appendix:adv-perturb} (scaled for display).

\begin{figure}[h!]
\centering
\subfloat[
Original ``0'': 
logit -72.6, \newline
CI (-106.4, -30.0)
]{
\includegraphics[width=0.24\textwidth]{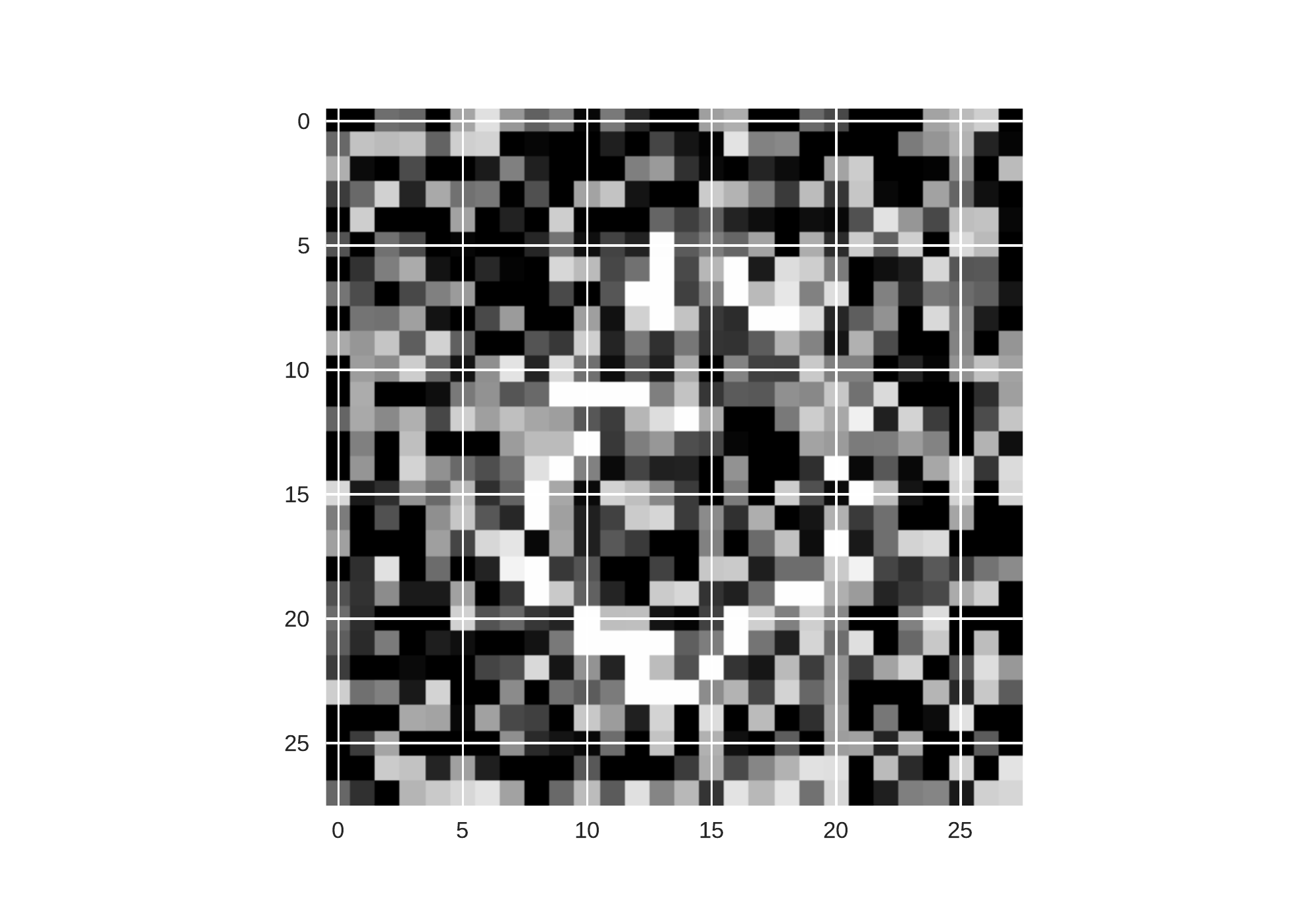}
}
\subfloat[
Adversarial ``0'': 
logit 15.3, \newline
CI (-6.5, 26.2)
]{
\includegraphics[width=0.24\textwidth]{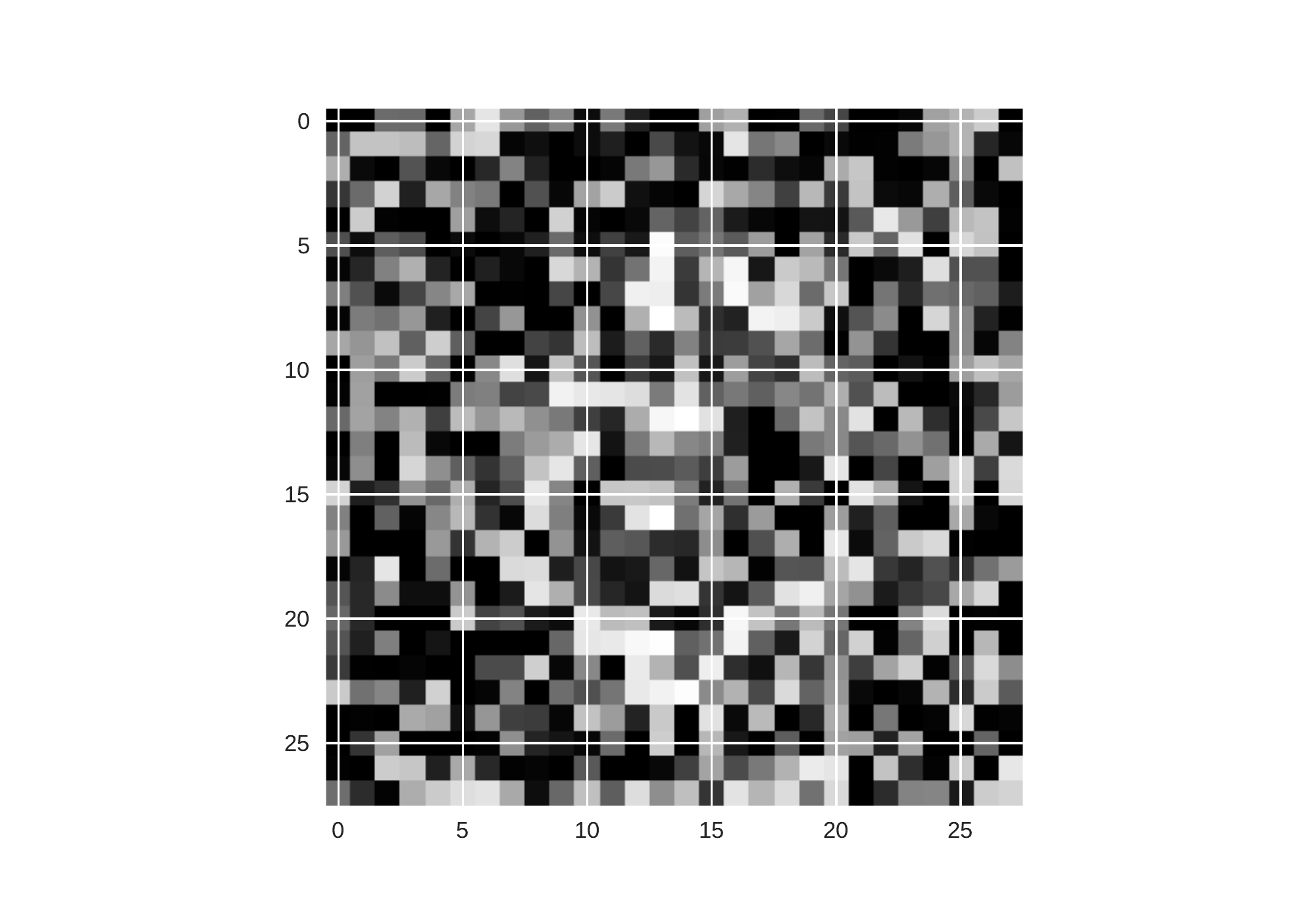}
} \\
\subfloat[
Original ``0'': 
logit -62.1, \newline
CI (-101.6, -5.5)
]{
\includegraphics[width=0.24\textwidth]{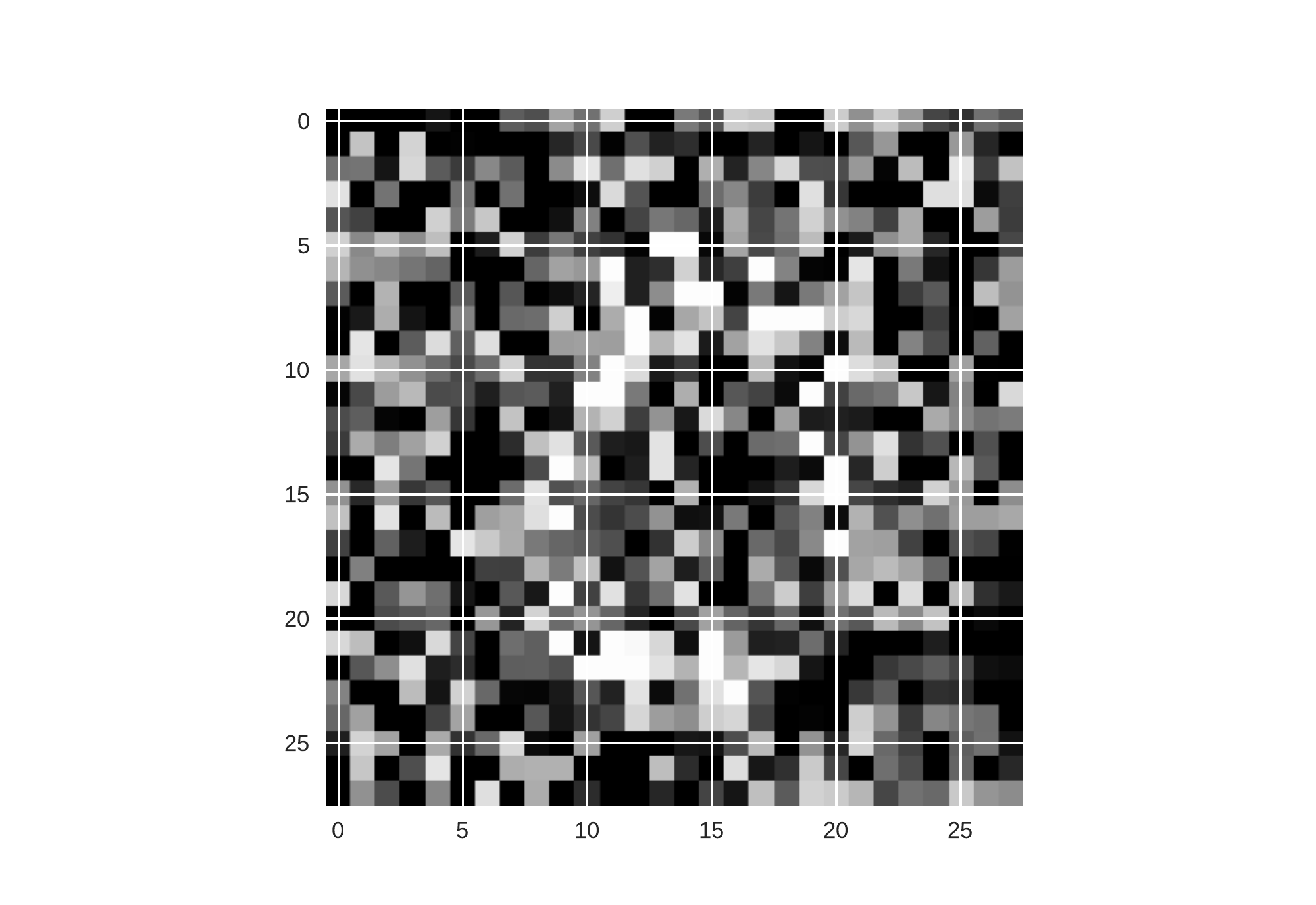}
}
\subfloat[
Adversarial ``0'': 
logit 1.9, \newline
CI (-4.9, 11.6)
]{
\includegraphics[width=0.24\textwidth]{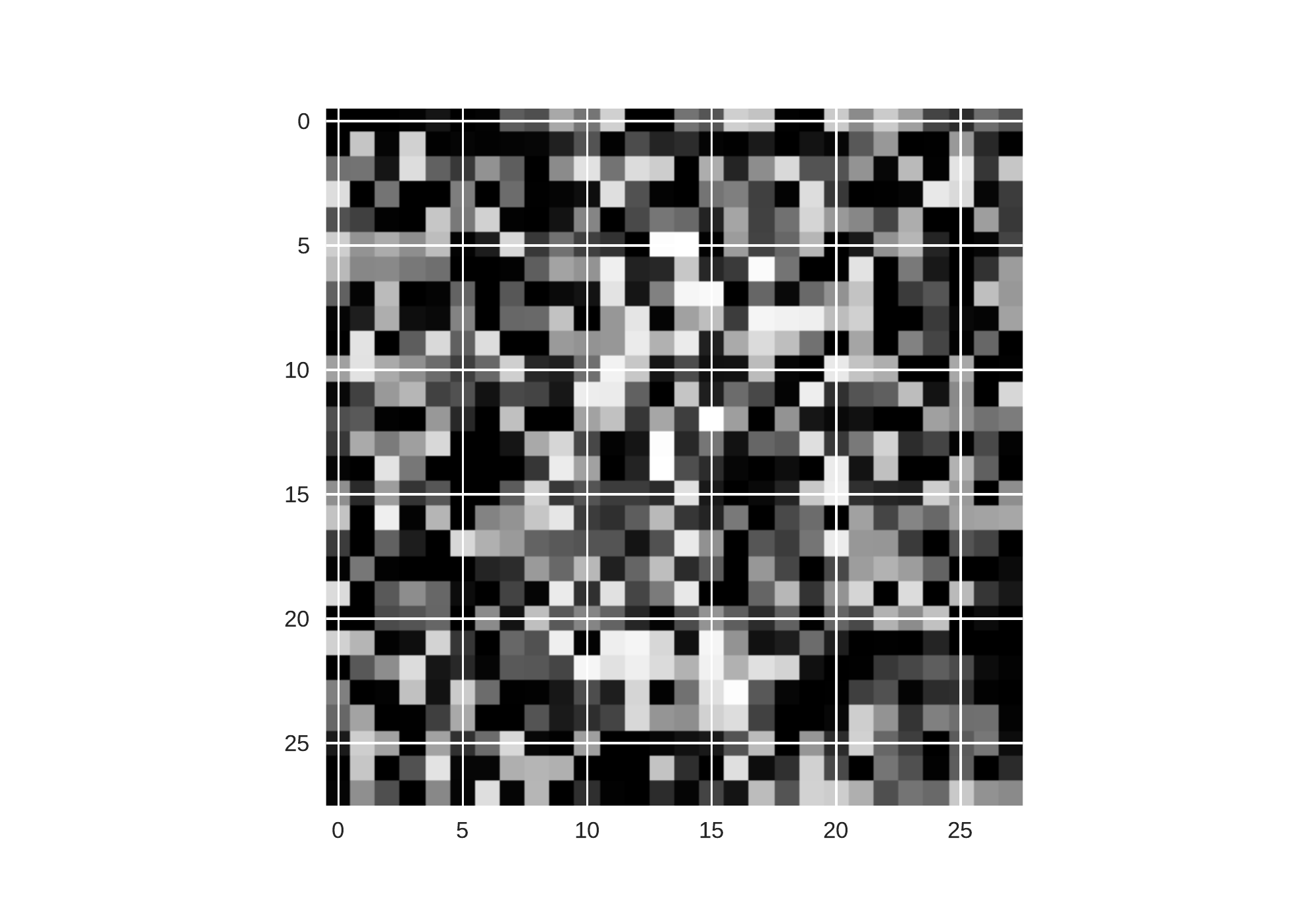}
} \\ 
\subfloat[
Original ``0'': 
logit -42.9, \newline
CI (-75.4, 5.1)
]{
\includegraphics[width=0.24\textwidth]{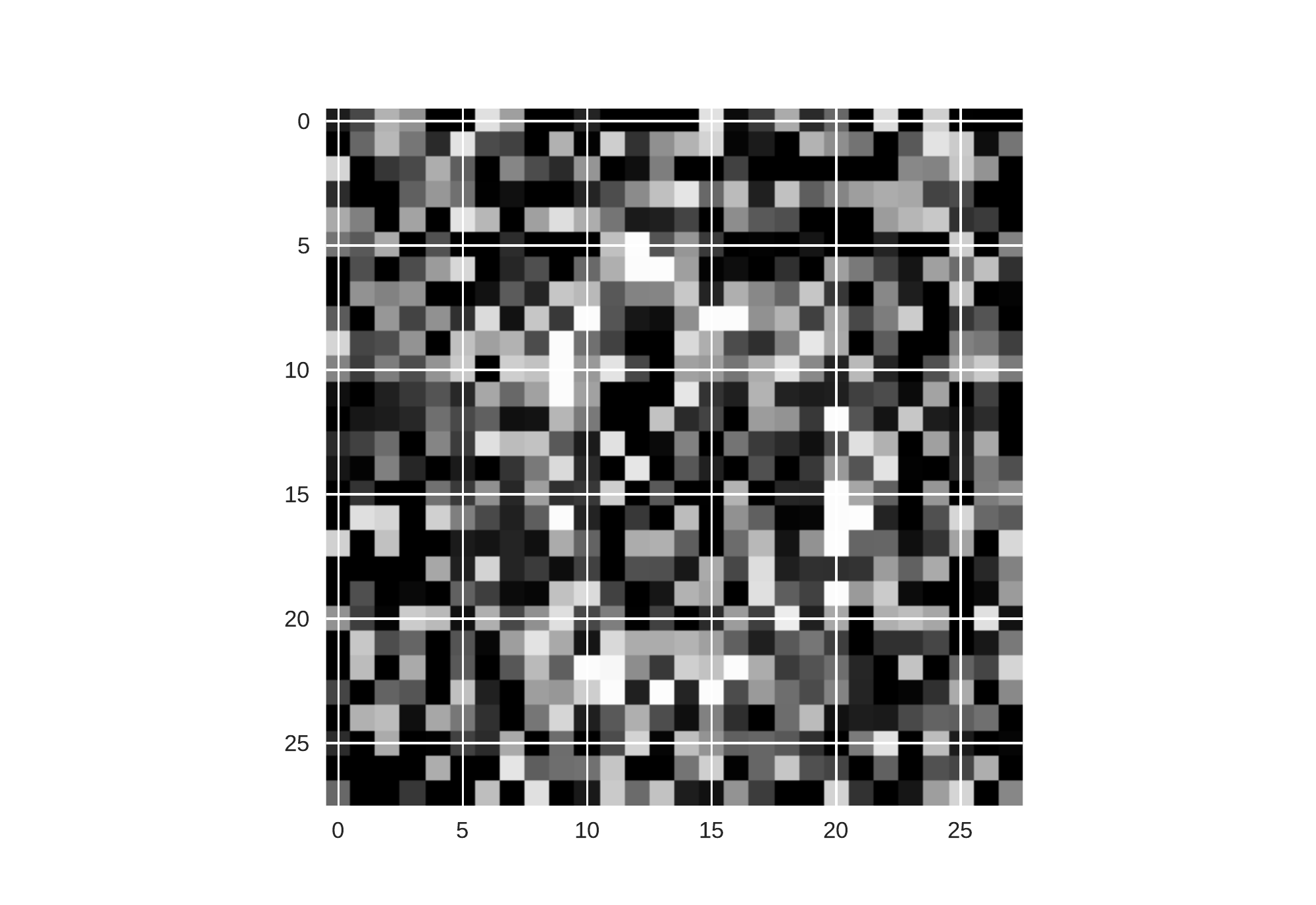}
}
\subfloat[
Adversarial ``0'': 
logit 4.8, \newline
CI (-3.4, 17.9)
]{
\includegraphics[width=0.24\textwidth]{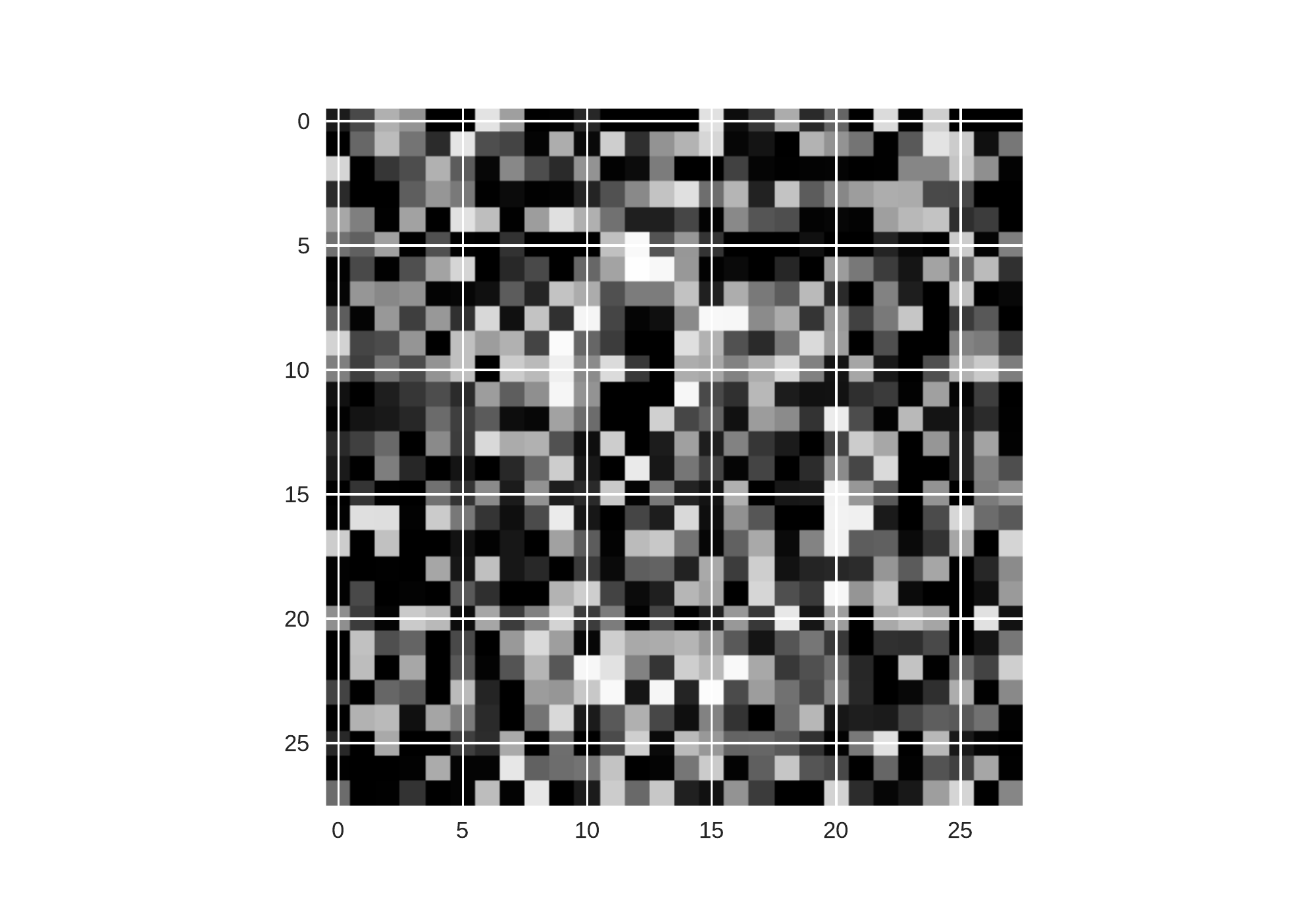}
} \\
\subfloat[
Original ``0'': 
logit -77.0, \newline
CI (-110.7, -32.2)
]{
\includegraphics[width=0.24\textwidth]{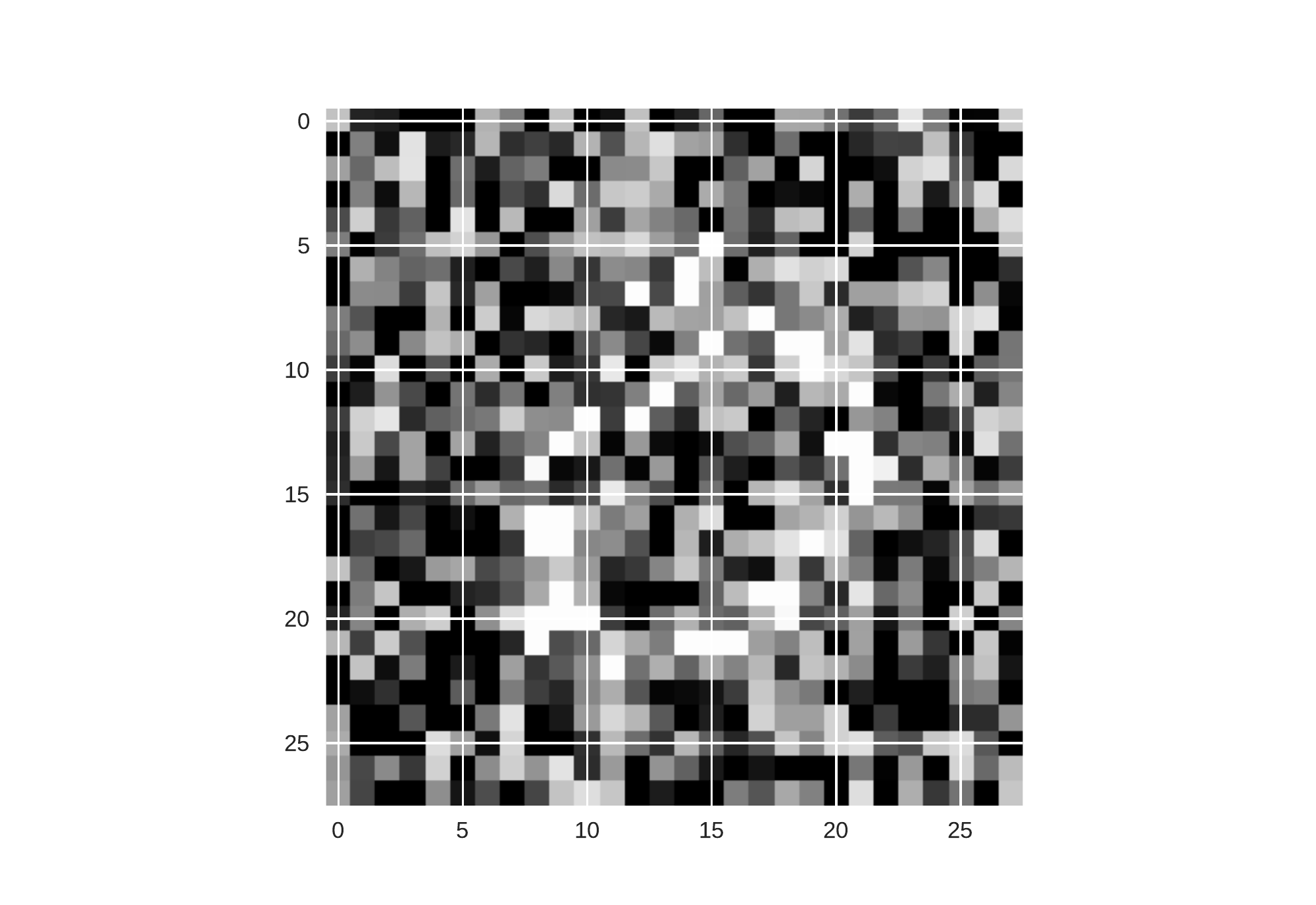}
}
\subfloat[
Adversarial ``0'': 
logit 13.3, \newline
CI (-8.0, 25.7)
]{
\includegraphics[width=0.24\textwidth]{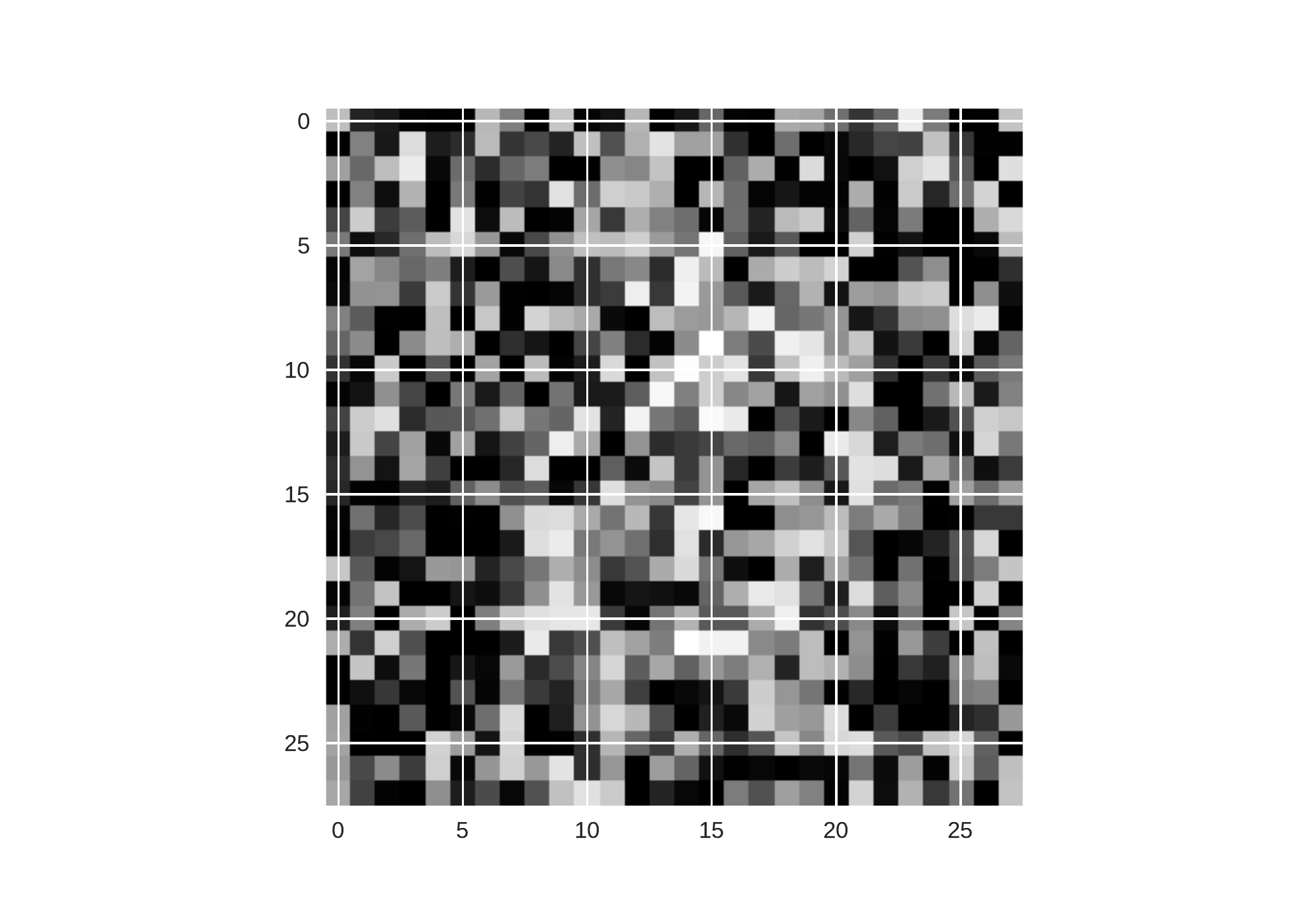}
} 
\caption{MNIST}
\label{fig:exp:real:mnist-appendix}
\end{figure}

\begin{figure}[h!]
\centering
\includegraphics[width=0.24\textwidth]{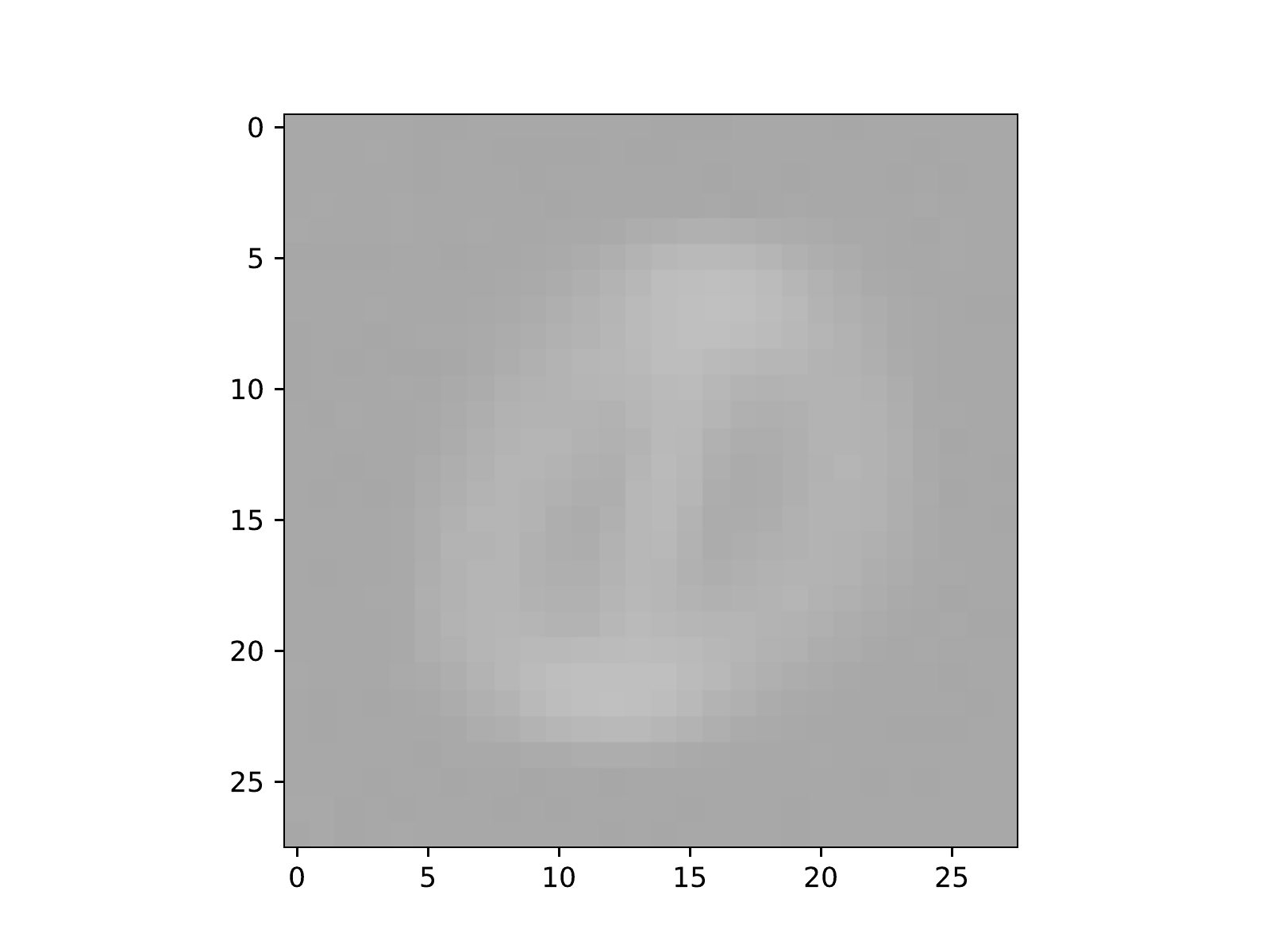}
\caption{MNIST adversarial perturbation (scaled for display)}
\label{fig:exp:real:mnist-appendix:adv-perturb}
\end{figure}


\end{document}